\documentclass[phd,tocprelim]{cornell}
\let\ifpdf\relax
\usepackage{amsmath,amsfonts,bm}

\def\eqref#1{equation~\ref{#1}}

\def\1{\bm{1}}

\DeclareMathAlphabet{\mathsfit}{\encodingdefault}{\sfdefault}{m}{sl}
\SetMathAlphabet{\mathsfit}{bold}{\encodingdefault}{\sfdefault}{bx}{n}

\DeclareMathOperator*{\argmax}{arg\,max}

\usepackage{xcolor}
\usepackage{xspace}
\usepackage[utf8]{inputenc} 
\usepackage[T1]{fontenc}    
\usepackage{hyperref}       
\usepackage{url}            
\usepackage{booktabs}       
\usepackage{amsfonts}       
\usepackage{nicefrac}       
\usepackage{microtype}      
\usepackage{amsmath}
\usepackage{amssymb}
\usepackage{graphicx,pstricks}
\usepackage{graphics}
\usepackage{caption}
\usepackage{subcaption}
\usepackage{wrapfig}
\usepackage{mathtools}
\usepackage{amsthm}
\usepackage{moreverb}
\usepackage{epsfig}
\usepackage{hangcaption}
\usepackage{txfonts}
\usepackage{palatino}
\usepackage[numbers, sort, comma, square]{natbib}

\usepackage{times}
\usepackage{soul}
\usepackage{algorithm}
\usepackage{algorithmic}
\usepackage{xcolor}
\usepackage{enumitem}

\usepackage[capitalize,noabbrev]{cleveref}

\crefname{equation}{Equation}{Equations}
\newtheorem{theorem}{Theorem}[section]

\newtheorem{definition}[theorem]{Definition}

\crefname{supp}{Supplementary Material}{Supplementary Materials}

\renewcommand{\caption}[1]{\singlespacing\hangcaption{#1}\normalspacing}

\title {Responsible Emergent Multi-Agent Behavior}
\author {Niko A. Grupen}
\conferraldate {August}{2023}
\degreefield {Ph.D.}
\copyrightholder{Niko A. Grupen}
\copyrightyear{2023}

\begin{document}

\maketitle
\makecopyright

\begin{abstract}
Responsible AI has risen to the forefront of the AI research community. As neural network-based learning algorithms continue to permeate real-world applications, the field of Responsible AI has played a large role in ensuring that such systems maintain a high-level of human-compatibility. Despite this progress, the state of the art in Responsible AI has ignored one crucial point: \textbf{human problems are multi-agent problems}. Predominant approaches largely consider the performance of a single AI system in isolation, but human problems are, by their very nature, multi-agent. From driving in traffic to negotiating economic policy, human problem-solving involves interaction and the interplay of the actions and motives of multiple individuals. 

This dissertation develops the study of \textbf{responsible emergent multi-agent behavior}, illustrating how researchers and practitioners can better understand and shape multi-agent learning with respect to three pillars of Responsible AI: interpretability, fairness, and robustness. First, I investigate multi-agent interpretability, presenting novel techniques for understanding emergent multi-agent behavior at multiple levels of granularity. With respect to low-level interpretability, I examine the extent to which implicit communication emerges as an aid to coordination in multi-agent populations. I introduce a novel curriculum-driven method for learning high-performing policies in difficult, sparse reward environments and show through a measure of position-based social influence that multi-agent teams that learn sophisticated coordination strategies exchange significantly more information through implicit signals than lesser-coordinated agents. Then, at a high-level, I study concept-based interpretability in the context of multi-agent learning. I propose a novel method for learning intrinsically interpretable, concept-based policies and show that it enables novel behavioral analysis tools via concept intervention that can reliably detect emergent coordination, coordination failures (lazy agents), the emergence of strategy and role assignment, and other dependencies between agent behavior. In the second part of the thesis, I study fairness through the lens of cooperative multi-agent learning. There I show that, despite being necessary for learning sophisticated coordination, mutual reward alone does not incentivize fair multi-agent behavior. I introduce novel group-based measures of fairness for multi-agent learning and develop two novel algorithms that achieve provably fair outcomes via equivariant policy learning. The third part of this thesis addresses robustness. I present a systematic analysis of search-based multi-agent learning systems such as AlphaZero and identify concrete failure modes that are present in its policy and value networks, but are disguised by search. I use these empirical findings to derive a novel extension of AlphaZero that combines uncertainty-informed value estimation and improved exploration to align AlphaZero’s policy and value predictions; thereby improving its robustness. 

Altogether, this body of work develops a framework within which researchers and practitioners can begin to understand and shape multi-agent learning systems; representing an initial step towards connecting Responsible AI and multi-agent learning.
\end{abstract}

\begin{biosketch}
Niko A. Grupen is a Ph.D. candidate in the Department of Computer Science at Cornell University, advised by Bart Selman and Daniel Lee. Niko's research centers on bridging advances in multi-agent learning, and reinforcement learning more broadly, with Responsible AI. His work has appeared in top conferences in AI and multi-agent systems such as NeurIPS, AAAI, ICML, AAMAS, and ICRA; and he has spent time at top research labs including Google Brain and Google's People + AI Research organization. Niko is also an avid writer and has contributed to the public discourse on AI and society through articles in popular venues such as Future by Andreessen Horowitz and a chapter in the 2020 book Future of Text by Frode Hegland (with contributions by Vint Cerf and Alan Kay, amongst others). Prior to his Ph.D., Niko was a machine learning engineer at Apple, working on emergent technologies. Niko received his B. Sc. of Computer Science \textit{summa cum laude} from Villanova University, with minors in Mathematics and Economics. He was elected to Phi Beta Kappa and was the recipient of the Gregor Mendel medallion for excellence in Computer Science (awarded to one student annually).  
\end{biosketch}

\begin{dedication}
To my parents, Rod \& Mary, and my brother, Yianni.
\end{dedication}

\begin{acknowledgements}
I am extremely fortunate to have been surrounded by mentors, professors, friends, and family who have generously shared their wisdom, expertise, and support with me along this journey. Their passion for knowledge and their willingness to invest time and effort in my growth have made this thesis possible.

I am forever thankful to my advisor Bart Selman for his wisdom and guidance over these many years, for his patience and tact in shaping my messy ideas into the work it has become today, and for our many Gimme chats. It has been an honor to learn from my co-advisor Daniel Lee, whose careful attention and feedback has sharpened my wit and taught me how to pursue ambitious, open-ended ideas in a principled manner. I am grateful to David Field for his kindness and encouragement, for cultivating my interest in psychology and cognitive science, and for the many stimulating conversations at his weekly round-table. I am honored to have had the opportunity to work with the three of you as a committee---I will carry your many lessons with me.

I am extremely grateful to Shayegan Omidshafiei, Natasha Jaques, and Been Kim for their mentorship and support during (and since) my time at Brain/PAIR; and for believing in Concept Bottleneck Policies when it was just a bullet point under "Possible Research Ideas:" on a literature review. I would also like to thank my other collaborators and colleagues from Brain, PAIR, and DeepMind for many thoughtful discussions: Yannick Assogba, Lucas Dixon, Joel Leibo, Edgar Duéñez Guzmán, Asma Ghandeharioun, Adam Pearce, Ryan Mullins, Peter Hase, Srivatsan Krishnan, and Blair Bilodeau.

I would like to thank my brilliant friends and colleagues at the Stanford Digital Economy lab---Stephan Zheng, Seth Benzell, Victor Ye, Michael Curry, James Chapman---who graciously let me participate as they navigated the frontier of multi-agent RL + macroeconomics. Thanks also to Aviv Tamar and Carmel Rabinowitz for your joint love of intuitive physics and robot learning, and for our collaborations in those areas. Thanks also goes to Benjamin Laufer for our collaboration on collective obfuscation and to Michael Hanlon and Alexis Hao for our work together on AlphaZero. To many lab-mates and friends at Cornell and Cornell Tech---namely, Travers Rhodes, Chanwoo Chun, Ge Gao, Ilan Mandel, David Goedicke, Rei Lee, Frank Bu, Matt Franchi, Maria Teresa Parreira, Natalie Friedman, Mark Colley, Alexandra Bremers, Elif Çelikörs, Alena Hutchinson, Christoforos Mavrogiannis, Wil Thomason, Claire Liang, Andy Ricci, and Valts Blukis---thank you for making this such an incredible experience. And special thanks goes to George Karagiannis and Florian Suri-Pyaer---two funny guys.

I am indebted to Gabriel Pereyra and Winston Weinberg for taking a chance on me. I look forward to the journey ahead, and to learning from the the brilliant minds of Daniel Hunter, Julio Pereyra, Spencer Poff, BK, Beth Lebens, Karl de la Roche, Lisa Patel, John LaBarre, and Gordon Moodie.

To my close and long-term friends---especially Connor and Alexandra, Tyler and Alison, Curtis, Danny, Ryan and Liz, Sean, and Kirsten---thank you for always being there, supporting me, and mostly for bearing with me.

Most of all, I am grateful for the love and support of my parents, Roderic Grupen and Mary Andrianopoulos. Thank you for instilling in me a love for science, an insatiable curiosity, and for always being a sound board to my craziest ideas. To my brother, Yianni, thank you for constantly challenging me to be my best, and teaching me to be determined in what I do. I would like to thank my grandparents, William and Doris Grupen, for supporting the pursuit of education and knowledge, and for always inviting me over for a coffee table chat. And to my grandparents, John and Eugenia Andrianopoulos, who gave up their lives in Greece, immigrated through Ellis Island, and worked in a pie factory so that future generations of their family could have an opportunity at a better life---there is no way that I can thank you enough.
\end{acknowledgements}

\contentspage
\tablelistpage
\figurelistpage

\normalspacing \setcounter{page}{1} \pagenumbering{arabic}
\pagestyle{cornell} \addtolength{\parskip}{0.5\baselineskip}

\chapter{Introduction}

\section{Move 37, Emergence, and Multi-Agent Problems}
For thirty-six moves, there was nothing unusual about Lee Sedol's second match against DeepMind's AlphaGo. Sedol and AlphaGo had each spent their first eighteen stones establishing positions in the upper and lower corners of the board, as is standard in high-level Go play. With its 19th stone, however, AlphaGo opted to play a stone in a largely unoccupied middle section on the right-hand side of the board, rather than continuing to fortify its hold on the corners. To capture just how unprecedented this move was, Sedol---then the world's highest-ranked Go player---was so flummoxed that he actually left the room to collect himself and, upon returning, studied the board for a full 15 minutes before playing another stone. What is more, though many viewers took AlphaGo's move to be a mistake at first, Move 37 initiated a highly complex and influential sequence of moves that ultimately led to AlphaGo's victory in the match; and was heralded afterwards as "beautiful" by fellow Go champion Fan Hui \cite{move37alphago}.

Understanding how a distinctly non-human move like Move 37 could occur requires an understanding of how AlphaGo learns in the first place. AlphaGo is a self-play reinforcement learning (RL) algorithm---starting from scratch, it plays \textit{against itself}, improving its strategy slowly over time as guided exclusively by a straightforward objective: win the game. This simple strategy, repeated across millions of games, enables AlphaGo to grow from novice to super-human level expert. Moreover, aside from setting up the algorithm and defining the rules of the game, self-play involves zero human input. This means that a system trained via self-play is not constrained to make the same decisions that a human would make---every decision is an \textbf{emergent behavior} picked up during its training. In the case of Go, self-play allowed AlphaGo to explore move sequences that were effective within the context of winning the game, yet highly non-traditional from the perspective of an extensive history of human gameplay.

AlphaGo's success is representative of a larger paradigm shift sweeping across AI research---one defined by simple yet highly scalable algorithms that are used to train an appropriately large neural network over vast quantities of data with little to no human intervention---and there are few domains where algorithms leveraging these ingredients have yet to achieve super-human levels of performance. For example, successor algorithms to AlphaGo have generalized self-play learning to other complex games such as Chess and Shogi \cite{silver2017alphazero}, and even to general planning problems \cite{schrittwieser2020mastering}. Elsewhere, these same principles have led to diffusion-based model architectures, which have redefined the fields of image \cite{ramesh2022hierarchical, saharia2022photorealistic, rombach2022high} and video generation \cite{singer2022make, ho2022video, ho2022imagen}; and transformer-based architectures that have subsumed the majority of natural language processing tasks \cite{brown2020language, bubeck2023sparks}, many programming-related tasks \cite{li2022competition}, and even previously infeasible tasks such as protein structure prediction \cite{jumper2021highly} and mathematical problem solving \cite{lightman2023let}.

Despite the success of this design pattern that promotes scale-at-all-costs and minimal human input, it is important to recognize that AI systems do not exist in a vacuum. In their current form, AI systems are primarily used with or alongside human counterparts. Even in scenarios where AI automation removes humans from the problem-solving loop, the impact of AI systems materially changes the human experience and the intent is to do so positively.

Recently, the fields of Responsible AI and AI Safety have grown to fill the role of mediator for AI progress. As AI continues to permeate real-world use cases and be deployed in human-facing applications, it is the role of these communities to ensure that it is done safely and with human-compatibility as an utmost priority. Thus far, they have been successful in this goal. For example, issues of bias and fairness have risen to the forefront of supervised learning research \cite{mitchell2018prediction, barocas2019fairmlbook}, leading to formal measures that quantify the presence of unfair or biased outcomes and novel algorithms that improve model performance with respect to those measures \cite{dwork2012fairness, kleinberg2016inherent, hardt2016equality, feldman2015certifying, zafar2017fairness, johndrow2019algorithm, kamiran2009classifying}. Discussions of human-centered AI have also percolated more broadly \cite{russell2019human} to include studies of how learning objectives can be better aligned with human preferences \cite{hadfield2019incomplete, zhuang2020consequences}, or avoid causing societal harm by way of under-specification or uncertainty in learning processes \cite{amodei2016concrete, hendrycks2021unsolved}. Most recently, these ideas have been channeled towards generative models by studying the extent to which generative outputs can be steered towards human-aligned outcomes by fine-tuning from human feedback \cite{christiano2017deep, ziegler2019fine, ouyang2022training} or constraining outputs to adhere to human-specified principles \cite{bai2022constitutional}.

However, the aforementioned research vectors within Responsible AI have focused primarily on supervised, self-supervised/generative, or single-agent contexts, ignoring one crucial point: \textbf{human problems are multi-agent problems.} Predominant approaches largely consider the performance of a single AI system in isolation, but human problem-solving is, by its very nature, multi-agent. From driving in traffic \cite{vinitsky2023optimizing} to negotiating economic policy \cite{zheng2020ai}, human problem-solving involves interaction and the interplay of the actions and motives of multiple individuals.

Multi-agent learning has a deep-rooted history and it is one that has also been rejuvenated by deep learning. In reinforcement learning settings specifically, the incorporation of neural network-based policy and value functions has led to novel algorithms and network architectures that stabilize and scale multi-agent learning \cite{lowe2017multi, foerster2018counterfactual, yang2018cm3, vinyals2019grandmaster, perolat2022mastering, yu2021surprising, he2016opponent, wang2020rode, berner2019dota, mahajan2019maven, rashid2020monotonic}; and has further enabled the exploration of solution spaces for complex social interactions involving diverse incentive structures \cite{team2021open, dennis2020emergent, jaques2019social, johanson2022emergent, leibo2021scalable, leibo2017multi, baker2019emergent} and \textit{ab initio} strategies for learning inter-agent communication that simulate the development of language-like communication protocols \cite{foerster2016learning, lazaridou2016towards, lazaridou2020emergent, das2019tarmac, chaabouni2020compositionality, cogswell2019emergence, resnick2019capacity, chaabouni2019anti, mordatch2017emergence, sukhbaatar2016learning, havrylov2017emergence, eccles2019biases, lowe2019pitfalls}. 

However, the majority of multi-agent learning research is targeted at maximizing traditional measures of performance (e.g. reward, Elo rating, competition against human experts, etc). To realize the full potential of AI systems, the research community will need to bridge multi-agent learning with Responsible AI. This thesis introduces \textbf{responsible emergent multi-agent behavior} as a new class of techniques developed at the intersection of these two points; with the goal of reliably understanding and shaping emergent behavior as a first-step towards unlocking the full potential of AI systems. Focusing on responsible emergent multi-agent behavior will help multi-agent learning in reaching technical parity with other learning disciplines, developing fundamentally new algorithms and techniques for understanding and shaping emergent multi-agent behavior, and formalizing measures of performance that meet the unique demands of multi-agent problems.

\section{Towards Responsible Emergent Multi-Agent Behavior}
The goal of this thesis is to initiate the study of responsible emergent multi-agent behavior, illustrating how researchers and practitioners can better understand and shape emergent multi-agent behavior with respect to three pillars of Responsible AI: interpretability, fairness, and robustness. Each of these sub-disciplines and the motivation behind the contributions of this thesis are outlined below.

\subsection{Interpretability}
\Cref{part:interpretabiltiy} of this thesis focuses on \textbf{interpreting} emergent multi-agent behavior. Unlike supervised learning, where there have been significant efforts to understand a model's decisions, multi-agent interpretability remains under-investigated. This is in part due to the increased complexity of the multi-agent setting. For one, interpreting the decisions of multiple agents over time is combinatorially more complex than understanding individual, static decisions. Moreover, due to non-stationary learning dynamics, the scope and variability of emergent behavior is vast---the behavior of each individual agent, as well as the joint behavior of the multi-agent system, can vary drastically with environmental conditions, or even random seeds (even more so than in single-agent RL). Finally, behavioral changes are difficult to identify through traditional measures of performance (e.g. reward, Elo, evaluations against human experts, etc) alone; and so the trailing pace of multi-agent interpretability relative other learning settings is also a reflection of the limited availability of interpretability tools. Overall, the nature of interactions between agents remains difficult to gauge in MARL.

This thesis posits that a rich set of interpretability techniques are needed, specific to multi-agent learning and that match the expressivity and level of insight outlined by these in other disciplines (e.g. supervised learning). Such tools should enable us to answer \textbf{high-level} strategic questions, such as:
\begin{itemize}
    \item \textit{In cooperative environments, do agents learn to coordinate, opt for less structured independent action, or some combination both?}
    \item \textit{In mixed-motive environments, do agents act selfishly, in a manner that improves social welfare, or both?}
    \item \textit{In competitive environments, when during the training process do strategic behaviors such as role assignment emerge, if at all?}
\end{itemize}
\noindent as well as \textbf{low-level} behavioral questions, such as:
\begin{itemize}
    \item \textit{To what extent do agents exchange information implicitly through their respective action spaces?}
    \item \textit{Can we measure the influence of one agent's behavioral cues on its teammates?}
\end{itemize}

\subsection{Fairness}
Next, this thesis examines the implications of emergent behavior with respect to \textbf{algorithmic fairness}. Increasingly, multi-agent algorithms are being employed alongside human counterparts \cite{claure2022fairness}, as a proxy for human engagement \cite{park2023generative}, or to solve problems that may materially impact a human's socio-economic well-being (e.g. tax policy, social planning \cite{zheng2020ai}). However, as we have seen, simple reward maximization yields a wide range of multi-agent behavior. The emergence of sub-optimal behavior in these human-centered applications is problematic, as it may impact human outcomes.

The goal of algorithmic fairness is to ensure that outcomes or decisions derived from machine learning systems do not come to rely on sensitive information, or disproportionately favor particular individuals or groups based on that information (in terms of accuracy or output quality). As in the interpretability domain, however, algorithmic fairness has not received the depth of formal treatment in MARL settings as it has in supervised learning\footnote{There is a large body of literature studying fairness in the game-theoretic contexts, but our focus here is on formalizing multi-agent fairness more generally (irrespective of reward structure) and, in particular, highlighting the role of sensitive attributes and their interactions with multi-agent outcomes}; and there are a number of key differences in the problem setup across these domains. In supervised learning, decisions (i.e. predictions or classifications) are made statically for one individual at a time \cite{barocas2019fairmlbook, mitchell2018prediction, dwork2012fairness, hardt2016equality}, whereas in multi-agent settings, decisions (i.e. actions) influence multiple individuals (or agents) simultaneously and occur repeatedly over many time-steps. Reward as a measure of outcomes further complicates the picture---agent incentives may be aligned (cooperative), conflicting (competitive), or mixed between local vs. global interests (social dilemmas).

Existing fairness frameworks do not yet fully capture these multi-agent nuances. This part of the thesis takes first steps towards addressing fairness in multi-agent learning settings by asking: 
\begin{itemize}
    \item \textit{How do we formalize fairness in multi-agent settings? For example, what do sensitive variables mean for an agent within a multi-agent team? Can we reconcile reward as an outcome variable?}
    \item \textit{In cooperative environments, is shared reward enough to incentivize fair behavior to emerge, or does unfair behavior emerge naturally?}
    \item \textit{Can we shape multi-agent behavior during training to achieve fairer outcomes? If so, to what extent does a trade-off between fairness and utility/efficiency exist, and what properties of the MARL problem contribute to that trade-off?}
\end{itemize}

\subsection{Robustness}
Finally, we turn our attention to issues of robustness in search-based multi-agent learning systems, such as AlphaZero \cite{silver2017alphazero}. By combining tree search with neural network-based value estimation and action selection, AlphaZero and its variants (e.g. AlphaGo \cite{silver2017alphago}) have achieved impressive feats, including reaching super-human levels of play in Go \cite{silver2017alphago}, Chess, and Shogi \cite{silver2017mastering}. However, the combinatorial complexity of these games\footnote{Chess has up to ${\sim} 10^{44}$ possible board states, Shogi ${\sim} 10^{71}$, and Go ${\sim} 10^{170}$} and the complexity of the AlphaZero algorithm itself makes interpretation and analysis of its performance challenging. For this reason, despite the clear evidence that AlphaZero's best case performance surpasses human game players, researchers have a limited understanding of what exactly AlphaZero has learned about the games it plays and why it makes the decisions that it does.

This is especially problematic in situations where AlphaZero's decision-making is sub-optimal. For example, recent work has shown that many board states elicit surprising \textit{yet ineffective} actions from AlphaGo \cite{lan2022alphazero}; and that simple\footnote{Simple enough that each of the adversarial attacks was easily countered by amateur-level Go players} opponent strategies \cite{wang2022adversarial} create the equivalent of adversarial examples \cite{goodfellow2014explaining} for AlphaZero's policy network. Subsequent work in concept-based interpretability has gone a long way towards understanding where in AlphaZero's large network certain strategic concepts are encoded \cite{mcgrath2022acquisition}, yet the issue of \textbf{robustness} still remains.

The goal of this work is not only to examine \textit{how} AlphaZero fails, but also to uncover \textit{why} AlphaZero fails; and use that understanding as motivation for algorithmic improvements to AlphaZero's generalization and robustness. In particular, the following questions are relevant to this study of robustness:
\begin{itemize}
    \item \textit{Why does AlphaZero's performance degrade so severely in the presence of simple adversarial examples?}
    \item \textit{Can the robustness of AlphaZero's policy and value networks be improved algorithmically?}
\end{itemize}

\section{Contributions}
The primary contribution of this thesis is to initiate the study of responsible emergent multi-agent behavior, bridging Responsible AI and multi-agent learning to better understand and shape emergent phenomena in multi-agent systems. In doing so, I introduce several novel methods and improvements to multi-agent learning algorithms, propose novel ways of measuring and analyzing emergent phenomena, and investigate their benefits across a range of multi-agent problem domains.

With respect to the core research areas outlined in the previous section, my contributions can be enumerated as follows:

\paragraph{Interpretability}
In the pursuit of interpretability, I introduce methods to better understand emergent multi-agent behavior at both a low-level, by examining agent action selection and behavioral cues, and at a high-level, by deriving concept-based explanations that are more amenable to human understanding.

First, at the low-level, I examine the extent to which implicit communication aides coordination in multi-agent populations. To accurately model the conditions under which coordination and communication emerges naturally, I prioritize decentralized multi-agent learning and introduce novel algorithms that enable multi-agent teams to learn to coordinate from scratch in difficult sparse-reward environments (where decentralized learning typically fails). Specifically, I introduce Curriculum-Driven Deep Deterministic Policy Gradients (CD-DDPG) as a policy learning method that combines environment shaping and behavioral curricula. I then compare the highly-coordinated strategies learned by CD-DDPG agents to those of lesser-skilled agents; and introduce a method of position-based social influence to measure the exchange of implicit information between agents. Results of this study show that agents trained with CD-DDPG exchange almost \textbf{three times more information} on average, compared to other less-coordinated strategies  This shows that agents trained with our method learn to rely heavily on the exchange of implicit signals as a means to coordinating in a sophisticated manner (and outperforming other strategies).

\noindent In sum, my contributions in this area are as follows:
\begin{enumerate}
    \item I highlight the importance of studying implicit signals as a form of emergent communication.
    \item I introduce a curriculum-driven learning strategy, Curriculum-Driven DDPG, that enables cooperative agents to solve difficult coordination tasks with sparse reward.
    \item I show that, using this strategy, agents learn strong coordination strategies that significantly outperforming sophisticated analytical and learned methods.
    \item I examine the use of implicit signals in emergent multi-agent coordination through position-based social influence and find that agents trained with my strategy exchange up three times more information per time-step on average than competing methods.
\end{enumerate}

Next, with the goal of understanding emergent multi-agent behaviors at a higher-level, I examine the intersection of multi-agent learning and concept-based interpretability. I introduce an intrinsically-interpretable, concept-based policy architecture for multi-agent learning---the Concept Bottleneck
Policy (CBP)---that makes multi-agent decision-making immediately transparent
by distilling each agent's decisions into a set of human-understandable concepts. The CBP architecture forces an agent to make decisions by first predicting an intermediate set of human-understandable concepts, then using those concepts to select actions. In turn, agents reveal the environmental and inter-agent factors influencing each of their actions. 

This method unlocks a novel class of techniques for understanding emergent multi-agent behavior. By conditioning each agent’s actions on its own concepts from construction, CBPs support behavioral analysis via concept intervention. Specifically, it is possible to perform interventions over an individual agent’s CBP, examine its impact on performance (with respect to rewards, actions, environment features, etc), and derive insights into key aspects of the emergent behavior of the larger multi-agent system. I perform this analysis over multi-agent teams in a range of domains, including cooperative environments---where intervention can be used to identify coordination from independent action, expose inter-agent factors that drive coordination, and even diagnose common failure modes such as lazy agents---and mixed-motive social dilemmas---where intervention unlocks our ability to examine inter-agent social dynamics (e.g., exploitation, public resource sharing). Finally, I show that CBPs perform comparably to non-concept based policies, thus achieving interpretability without sacrificing performance

\noindent In sum, my contributions in this area are as follows:
\begin{enumerate}
    \item I introduce Concept Bottleneck Policies as an interpretable, concept-based architecture for multi-agent reinforcement learning.
    \item I introduce a suite of novel techniques for multi-agent behavioral understanding that leverage concept intervention and graph learning.
    \item I show that CBPs can match the performance of non-concept-based network architectures.
\end{enumerate}

\paragraph{Fairness}
In the pursuit of fairness, I examine algorithmic fairness in the context of multi-agent learning and develop the connective tissue needed to understand fairness in cooperative multi-agent settings. I begin with a study of emergent behavior amongst cooperative multi-agent teams that are trained under mutual reward. Initial experiments suggest that mutual reward, though crucial for learning coordination, is not enough to incentivize fair emergent coordination. Specifically, I find that naive, unconstrained maximization of mutual reward yields a symmetry-breaking form of role assignment that in turn creates unfair individual outcomes for cooperative teammates.

I introduce a novel measure of fairness for multi-agent learning settings---team fairness---that is able to capture this phenomenon quantitatively. Team fairness is a group-based fairness measure, inspired by demographic parity \cite{dwork2012fairness, feldman2015certifying}, that requires the distribution of a team's reward to be equitable across sensitive groups. Next, to incorporate team fairness into policy learning, I introduce a novel multi-agent learning strategy---\textit{Fairness through Equivariance} (Fair-E). Fair-E enforces team fairness during policy optimization by transforming the team's joint policy into an equivariant map and achieves provably fair reward distributions under assumptions of agent homogeneity. 

Despite achieving fair outcomes, Fair-E represents a binary switch---one can either choose fairness (at the expense of utility) or utility (at the expense of fairness). In many cases, however, it is advantageous to modulate between fairness and utility. To this end, I also introduce a soft-constraint version of Fair-E that incentivizes equivariance through regularization: \textit{Fairness through Equivariance Regularization} (Fair-ER. With Fair-ER, it is possible to tune fairness constraints over multi-agent policies by adjusting the weight of equivariance regularization. Empirically, Fair-ER reaches higher levels of utility than Fair-E while achieving fairer outcomes than non-equivariant policy learning.

Finally, as in both prediction-based settings \cite{corbett2017algorithmic, zhao2019inherent} and in traditional multi-agent variants of fairness \cite{okun2015equality, le1990equity}, it is important to understand the ``cost" of fairness. I present novel findings regarding the fairness-utility trade-off for cooperative multi-agent settings, showing that the magnitude of the trade-off depends on the skill level of the multi-agent team. When agent skill is high (making the task easier to solve), fairness comes with no trade-off in utility, but as skill decreases (making the task more difficult), gains in team fairness are increasingly offset by decreases in team utility.

\noindent In sum, my contributions in this area are as follows:
\begin{enumerate}
    \item I show that mutual reward is critical to multi-agent coordination. Experimentally, agents trained with mutual reward learn to coordinate effectively, whereas agents trained with individual reward do not. Despite this, the agents learn coordination strategies that are unfair with respect to individual agent outcomes.
    \item I introduce team fairness as a group-based fairness measure for multi-agent teams that requires equitable reward distributions across sensitive groups.
    \item I introduce Fairness through Equivariance (Fair-E), a novel multi-agent strategy leveraging equivariant policy learning. Fair-E achieves provably fair outcomes for individual members of a cooperative team.
    \item I introduce Fairness through Equivariance Regularization (Fair-ER) as a soft-constraint version of Fair-E. Fair-ER reaches higher levels of utility than Fair-E while achieving fairer outcomes than non-equivariant learning.
    \item I present novel findings regarding the fairness-utility trade-off for cooperative settings, showing that the magnitude of the trade-off depends on agent skill---when agent skill is high, fairness comes for free; whereas with lower skill levels, fairness is increasingly expensive.
\end{enumerate}

\paragraph{Robustness}
Finally, in the pursuit of robustness, I examine emergent phenomena within search-based multi-agent learning algorithms like AlphaZero. As mentioned previously, recent studies have shown that, in addition to positive emergent phenomena like Move 37, there are may negative phenomena hidden within AlphaZero's vast neural network backbone. For example, studies leveraging adversarial game-play have shown that AlphaZero is surprisingly susceptible to minor perturbations in opponent behavior \cite{lan2022alphazero} and can be exploited to lose games that could be easily won by amateur human players.

To better understand AlphaZero's hidden failure modes, I conduct a thorough analysis of AlphaZero's gameplay in solved board games. This examination reveals multiple phenomena that underlie AlphaZero's documented failure modes. First, despite learning to play a game optimally \textit{within its search process}, AlphaZero's policy network and value function are \textbf{far from optimal themselves}. This reveals an interesting tension between learning and search: search is necessary to fight the combinatorial exploration problem of complex games, but also offers neural networks a crutch that leads to inadequate generalization. Next, I identify a novel emergent phenomenon, \textbf{policy-value misalignment}, in which AlphaZero's policy and value predictions contradict each other. This misalignment is a direct consequence of AlphaZero's objective, which includes separate terms for policy and value error, but no constraint on their consistency. I introduce an information-theoretic definition of policy-value misalignment and use it to quantify AlphaZero's misalignment. Lastly, I find evidence of \textbf{inconsistency within AlphaZero's value function} with respect to symmetries in the environment. As a result, AlphaZero's value predictions generalize poorly to infrequently-visited and unseen states. Crucially, these phenomena are \textbf{masked by search} and exist even when AlphaZero plays optimally within its search process. This work therefore demonstrates that, though search allows AlphaZero to achieve unprecedented scale and performance, it also hides deficiencies in its learned components.

To address these issues, I propose two modifications to AlphaZero aimed directly at policy-value alignment and value function consistency. First, I introduce Value-Informed Selection (VIS): a modified action-selection rule in which action probabilities are computed using a one-step lookahead with AlphaZero's value network. VIS then alternates stochastically between this value action selection rule and AlphaZero's standard policy (which is derived from search probabilities). During training, this forces alignment between AlphaZero's policy and value predictions. Next, I introduce Value-Informed Symmetric Augmentation (VISA): a data augmentation technique that takes into account uncertainty in AlphaZero's value network. Unlike data augmentation techniques that naively add states to a network's replay buffer, VISA considers symmetric transformations of a given state and stores only the transformed state whose value differs most from the network's value estimate for the original, un-transformed state.

VISA and VIS can improve both policy-value alignment and value robustness in AlphaZero simultaneously. I refer to this joint method as VISA-VIS. In experiments across a variety of symmetric, two-player zero-sum games, VISA-VIS reduces policy-value misalignment by \textbf{up to 76\%}, reduces value generalization error by \textbf{up to 50\%}, and reduces average value prediction error \textbf{by up to 55\%}; as compared to AlphaZero.

\noindent In sum, my contributions in this area are as follows:
\begin{enumerate}
    \item I systematically study AlphaZero and identify multiple emergent phenomena that are masked by search.
    \item I define an information-theoretic measure of policy-value misalignment.
    \item I propose Value-Informed Selection (VIS): a modified action-selection rule for AlphaZero that forces alignment between policy and value predictions.
    \item I propose Value-Informed Symmetric Augmentation (VISA): a data augmentation technique that leverages symmetry and value function uncertainty to learn a more consistent value network.
    \item I introduce a new method, VISA-VIS, that combines VIS and VISA into a single algorithm. Experimentally, I show that VISA-VIS improves the quality of AlphaZero's policy and value networks, and significantly reduces policy-value misalignment.
\end{enumerate}

\section{Overall Structure}
This dissertation is organized in three parts, each addressing one of the outstanding challenges outlined in the preceding section. The core contributions are preceded by a review of background material. \Cref{sec:background} reviews material in each of the key areas presented throughout the thesis. There exists a setup of the reinforcement learning (RL) problem and discussion of the RL and multi-agent RL algorithms that provide a foundation for this work. Then, a short introduction to notation and definitions for related areas of machine learning is provided---namely, interpretability and fairness---that will be incorporated into the MARL framework. Finally, background on the multi-agent environments used in our experiments is provided.

\Cref{part:interpretabiltiy} initiates a study of interpretability in the context of emergent multi-agent behavior. \Cref{chapter:implicit_comm}, addresses low-level interpretability through emergent implicit communication in multi-agent RL. Specifically, this work highlights the importance of recognizing communication as a spectrum---from fine-grained behavioral cues that are exchanged implicitly to rich, compositional protocols that are exchanged explicitly---and show that agents learn to rely upon implicit signals from their teammates when solving difficult coordination tasks. In turn, the volume of implicit information exchange is indicative of the strength of their coordination. \Cref{chapter:interpretability}, targets high-level, concept-based explanations of emergent multi-agent behavior and introduces an intrinsically interpretable policy architecture for MARL agents that conditions each agent's policy on a set of human-understandable concepts. By conditioning on understandable concepts, this method enables post-hoc behavioral analysis via concept intervention. Concept intervention is then shown to reliably detect emergent coordination, expose sub-optimal behavior (e.g., lazy agents), identify complex inter-agent social dynamics, and strategic behavior as they emerge throughout training.

\Cref{part:fairness} connects algorithmic fairness to cooperative multi-agent settings. There we provide formal definitions of fairness for multi-agent teams, introducing a group-based fairness measure that requires equitable reward distributions across sensitive groups. Then, we introduce a novel actor-critic algorithm for learning provably fair multi-agent policies and conduct an empirical analysis of fairness in cooperative games. Finally, we revisit the canonical fairness-utility trade-off for multi-agent learning and show that, in cooperative settings, the magnitude of this trade-off is largely dependent on the relative skills of each agent in the team.

\Cref{part:robustness} addresses robustness in large-scale multi-agent systems. In particular, conduct a thorough analysis of AlphaZero's failure modes is conducted, using board games with enumerable state spaces to gain a holistic picture of AlphaZero's behavior. Two issues are found at the root of AlphaZero's brittleness. First, AlphaZero's policy and value functions are \textit{misaligned} in many states---i.e. AlphaZero's policy network selects a winning action, yet its value network estimates that AlphaZero is losing (or vice versa). This is possible because AlphaZero employs a joint objective (weighted sum) that includes separate terms promoting action and value accuracy, but does not enforce that they are consistent. Further inconsistency within AlphaZero's value function itself, which causes it to generalize poorly to infrequently visited or unseen states. Using these empirical findings as motivation, we derive a novel extension of AlphaZero that combines uncertainty-informed value estimation and improved exploration. Together, this algorithm aligns AlphaZero's policy and value predictions. Experimentally, our algorithm greatly improves AlphaZero's robustness.

\chapter{Background}
\label{sec:background}

\section{Reinforcement Learning}
Reinforcement Learning (RL) refers to a class of sequential decision-making problems in which an \textbf{agent} resides in an \textbf{environment} and interacts with the environment by performing \textbf{actions} over the course of some time horizon. The environment responds to agent actions by transitioning from its current \textbf{state}  to a new state---as governed by a (possibly unknown) transition function---and providing a \textbf{reward} that quantifies the strength or weakness of the agent's action. The goal of the agent is to maximize the cumulative reward it obtains over time.

\subsection{Problem Setup}
Under full-observability, RL problems are formalized with a \textbf{Markov Decision Process} (MDP). An MDP $\mathcal{M}$ is a tuple $\mathcal{M} = \{S, A, T, R \}$ consisting of a state-space $\mathcal{S}$, action-space $A$, transition function $T$, and reward function $R$, respectively. The state-space $S$ defines the set of all environmental states that an agent may encounter. The action-space $A$ defines the set of possible actions an agent can take. Both states and actions may be discrete or continuous. The transition function $T$ defines the probability $p(s' \mid s, a)$ of transitioning from a state $s \in S$ to a new state $s' \in S$ after taking an action $a \in A$. Finally, the reward function $R$ defines the single-step reward $r(s, a)$ that an agent receives for taking the action $a$ in state $s$.

An agent's decision-making is dictated by a \textbf{policy}, which is a function $\pi$ mapping states to actions. Policies can be either deterministic, where $\pi(s)=a$ (deterministic policies may also be denoted $\mu$), or stochastic, wherein $\pi$ defines a conditional probability distribution $\pi(a \mid s) = P(A=a \mid S=s)$. Using these ingredients, the agent-environment interaction more formally as follows: for each time-step $t$, the agent observes a state $s_t$, selects an action $a_t \sim \pi(a \mid s_t)$, transitions to a new state $s_{t+1}$ with probability $p(s_{t+1} \mid s_t, a_t)$, and receives a reward $r_t = r(s_t, a_t)$. For a time horizon of $T$ steps, this produces a $T$-step trajectory:
\begin{equation}
    \tau = \{s_1, a_1, ..., s_T, a_T\}
\end{equation}
with probability:
\begin{align}
    p(\tau) &= p(s_1, a_1, ..., s_T, a_T)\\
    &= p_\emptyset(s_1) \prod_{t=1}^T \pi_\theta(a_t \mid s_t) p(s_{t+1} \mid s_t, a_t)
\end{align}
where the initial state $s_1$ is drawn from a special start-state distribution $p_\emptyset$.

\subsubsection{Value Functions}
Recall that the goal of an RL agent is to maximize the cumulative reward it obtains over time. It is therefore important for an agent to compute or estimate the reward that it can expect to receive in the future. This quantity is known as the \textbf{return} (denoted $G$) and it can be computed for any time-step $t$, as the total sum of discounted rewards going forward:
\begin{equation}
    \label{eqn:background_return}
    G_t = r_t + \gamma r_{t+1} + \gamma^2 r_{t+2} + \cdots = \sum_{k=0}^\infty \gamma^k r_{t+k}
\end{equation}
where $\gamma \in [0, 1]$ is a discount factor for future rewards.

From $G_t$, we define two functions that quantify the value of a state (or state-action pair) with respect to expected reward. The \textbf{state-value function} $V^\pi(s)$ for a state $s_t$ defines the expected return of an agent, starting from $s_t$ and following a specific policy $\pi$ thereafter:
\begin{equation}
    \label{eqn:background_v_func}
    V^\pi(s_t) = \mathbb{E} [G_t \mid S_t = s_t]
\end{equation}
and the \textbf{action-value function} (or Q-function) $Q^\pi(s_t, a_t)$ defines the expected return of an agent, after taking the action $a_t$ from $s_t$, and following the policy $\pi$ thereafter:
\begin{equation}
    \label{eqn:background_q_func}
    Q^\pi(s_t, a_t) = \mathbb{E} [G_t \mid S_t = s_t, A_t = a_t]
\end{equation}
Since $V^\pi$ and $Q^\pi$ are defined with respect to a specific policy $\pi$ (hence the superscript), $\pi$ can recover $V^\pi$ from $Q^\pi$ as follows:
\begin{equation}
    \label{eqn:background_v_from_q}
    V^\pi(s) = \sum_{a\in A} Q^\pi(s,a) \pi(a \mid s)
\end{equation}
Another useful quantity is the difference between $Q^\pi(s,a)$ and $V^\pi(s)$:
\begin{equation}
    \label{eqn:background_adv}
    A^\pi(s, a) = Q^\pi(s, a) - V^\pi(s)
\end{equation}
\noindent which is known as the \textbf{advantage function}. Intuitively, $A^\pi(s, a)$ tells us how strong of an action $a$ is relative all actions available from $s$.

Importantly, using \cref{eqn:background_return}, both the state and state-action value functions can be decomposed into a sum of the immediate reward for the current time-step and the discounted future values. For $V^\pi(s_t)$, this is done as follows:
\begin{align}
    V^\pi(s_t) &= \mathbb{E} [G_t \mid S_t = s_t] \\
    &= \mathbb{E} [r_t + \gamma r_{t+1} + \gamma^2 r_{t+2} + \cdots \mid S_t = s_t] \\
    &= \mathbb{E} [r_t + \gamma(r_{t+1} + \gamma r_{t+2} + \cdots) \mid S_t = s_t] \\
    &= \mathbb{E} [r_t + \gamma G_{t+1} \mid S_t = s_t] \\
    &= \mathbb{E} [r_t + \gamma V^\pi(S_{t+1}) \mid S_t = s_t]
    \label{eqn:background_bellman_v}
\end{align}
\noindent and similarly for $Q^\pi(s_t, a_t)$ we have:
\begin{align}
    Q^\pi(s_t, a_t) &= \mathbb{E} [G_t \mid S_t = s_t, A_t = a_t] \\
    &= \mathbb{E} [r_t + \gamma V^\pi(S_{t+1}) \mid S_t = s_t, A_t = a_t] \\
    &= \mathbb{E} [r_t + \gamma \mathbb{E}_{a \sim \pi} Q^\pi(S_{t+1}, a_t) \mid S_t = s_t, A_t = a_t]
    \label{eqn:background_bellman_q}
\end{align}
\noindent \Cref{eqn:background_bellman_v} and \cref{eqn:background_bellman_q} are referred to as the \textbf{Bellman equations}.

The recursive nature of the Bellman equations are key to solving the RL problem, as it means that it is possible to either (i) compute $V^\pi$ and $Q^\pi$ directly (if a model of the environment is available) or (ii) estimate them from experience (if such a model is not available). In both cases, once the value function is known (or learned), a policy can be derived by acting greedily with respect to the value function (i.e. taking the highest-value action from every state).

\subsection{Value-based Learning}
In complex environments, it is often the case that a model---i.e. the transition probabilities $T$ and rewards $R$ specified by the environment---is unavailable. In such cases, $V^\pi$ and/or $Q^\pi$ must be computed in a \textbf{model-free} manner, by collecting an agent's interactions with the environment (exploration) and learning from that experience. Value-based learning refers to a class of algorithms that learn a value function from experience. Here we review a few model-free methods that our key to our work.

\subsubsection{Temporal-Difference Learning}
Temporal-difference (TD) learning is the canonical algorithm for model-free, value-based learning. TD learning is derived from the following key insight: for a given value $V(s_t)$, it is possible to create a slightly better estimate of $s_t$'s true value using a one-step look-ahead. That look-ahead is computed as the discounted sum of the immediate reward $r_{t+1}$ received when exiting $s_t$ and the value $V(s_{t+1})$ of the subsequent state $s_{t+1}$ as follows:
\begin{equation}
    r_{t+1} + \gamma V(s_{t+1})
\end{equation}
\noindent where $\gamma$ is a discount factor (as in \cref{eqn:background_return}). This look-ahead is known as the \textbf{TD target}. Importantly, the TD target can be computed via bootstrapping. Namely, the value estimate $V(s_t)$ can be computed using an existing \textit{estimate} of $V(s_{t+1})$, rather than computing $V(s_{t+1})$ recursively down to some terminal state. Updating the TD target with existing estimates means that TD learning can operate over \textit{partial episodes} of experience (as opposed to requiring entire trajectories).

Together, these insights motivate the following value function update rule:
\begin{equation}
    \label{eqn:background_td_learning_v}
    V(s_t) \gets V(s_t) + \alpha(r_{t+1} + \gamma V(s_{t+1}) - V(s_t))
\end{equation}
\noindent where $\alpha$ is akin to a learning rate hyperparameter dictating how rapidly existing estimates are updated. The difference between the TD target $r_{t+1} + \gamma V(s_{t+1})$ and current value estimate $V(s_t)$ is known as the \textbf{TD error}. Similarly for the action-value function, we have:
\begin{equation}
    \label{eqn:background_td_learning_q}
    Q(s_t, a_t) \gets Q(s_t, a_t) + \alpha(r_{t+1} + \gamma Q(s_{t-1}, a_{t-1}) - Q(s_t, a_t))
\end{equation}
\noindent This simple update rule is the basis for nearly every value-based RL algorithm.

\subsubsection{Q-Learning}
Q-learning \cite{watkins1992q} is an off-policy algorithm for finding an optimal policy via TD learning. Q-learning first learns the optimal action-value function $Q^*(s,a)$ using the update rule outlined in \cref{eqn:background_td_learning_q}, then recovers the optimal policy as $\pi^*(a|s) = \argmax_a Q^*(s,a)$. More concretely, starting from an initial $Q$ (random or all zeros), Q-learning iteratively performs the following steps until convergence:
\begin{enumerate}
    \item At time-step $t=0$, initialize a trajectory starting from a state $s_0$.
    \item Select an action $a_t = \argmax_{a \in A}Q(s_t, a)$.
    \item After applying $a_t$, collect the reward $r_{t+1}$ and next state $s_{t+1}$ returned by the environment.
    \item Update $Q$ using \cref{eqn:background_td_learning_q}.
    \item Increment $t$ and repeat steps 2-4 until the trajectory ends.
\end{enumerate}
\noindent An exploration strategy such as $\epsilon$-greedy action selection is commonly used in place of greedy action selection in step (2) to ensure a diverse set of actions are used during training.

\subsubsection{Deep Q-Learning}
Though Q-learning is effective for simple environments, computing $Q^*(s,a)$ becomes computationally infeasible as the size of the environment's state and action spaces increases (and is impossible for continuous states and actions). The Deep Q-Network (DQN) \cite{mnih2013playing, mnih2015human} was introduced to circumvent these issues by using neural networks to approximate $Q$. Specifically, DQN defines a neural network function $Q_{\omega}$ with parameters $\omega$ and optimizes it with the following TD error-inspired loss:
\begin{equation}
    \label{eqn:background_dqn_loss}
    \mathcal{L}(\theta) = \mathop{\mathbb{E}}_{s, a, r, s'} \big [ (Q_\omega (s_t, a_t) - y_t)^2  \big]
\end{equation}
\noindent where:
\begin{equation*}
    y_t = r(s_t, a_t) + Q_\omega(s_{t+1}, \pi(s_{t+1}))
\end{equation*}
\noindent Tuples of experience ($s_t$, $a_t$, $r_t$, $s_{t+1}$) are collected throughout training by a separate exploration policy, stored in a replay buffer $\mathcal{D} = \{(s_t, a_t, r_t, s_{t+1})\}_{i=1}^N$, and sampled uniformly at random when computing \cref{eqn:background_dqn_loss}. In practice, the use of a replay buffer is vital to the performance of DQN, as it improves data efficiency and reduces correlations between sequential data points.

Also important to the stability of DQN training (and most value-based learning methods) is the use of a target network. A target network is copy $Q'_\omega(s,a)$ of the original Q-network $Q_\omega(s,a)$ that is in a separate process such that the copied parameters $\omega'$ trail those of the original network:
\begin{equation*}
    \omega' \gets \tau\omega + (1-\tau)\omega'
\end{equation*}
\noindent where $\tau << 1$. The target values provided by $Q'_\omega(s,a)$ are constrained to change slowly, which prevents $Q_{\omega}$ from computing Q-values as it is updating. This in turn improves the stability.

\subsection{Policy Gradient Methods}
\label{sec:background_rl_pg}
Policy gradient methods are a class of policy-based RL solutions \cite{sutton2000policy}. Unlike value-based methods, which derive and improve a policy implicitly through an approximate value function, policy gradient methods assume the policy itself is a function $\pi_\phi$ with parameters $\phi$ and seek to optimize this policy directly. Policy-based methods have an advantage over value-based methods in continuous action spaces, as they avoid computing a max over actions; though they pay for this convenience with efficiency, suffering from higher variance and, in turn, slower convergence. Below is an introduction top policy gradients and a review of prior methods that are key to this work.

\subsubsection{The Gradient of a Policy}
Let $r(\tau)$ represent the cumulative reward the agent receives over a trajectory $\tau = \{s_1, a_1, ..., s_T, a_T \}$:
\begin{equation*}
    r(\tau) = \sum_{t=1}^T r(s_t, a_t)
\end{equation*}
\noindent We can use this to write the RL objective (reward maximization) as:
\begin{equation}
    \label{eqn:background_rl_objective_tau}
    J(\phi) = \mathop{\mathbb{E}}_{\tau \sim p_\phi(\tau)} \bigg [\sum_t r(\tau) \bigg] = \int \pi_\phi(\tau)r(\tau)d\tau
\end{equation} 
To improve $\pi_\phi$, the gradient of \cref{eqn:background_rl_objective_tau} is computed with respect to parameters $\phi$:
\begin{equation}
    \label{eqn:background_grad_objective}
    \nabla_\phi J(\phi) = \int \nabla_\phi \pi_\phi(\tau)r(\tau)d\tau
\end{equation}

\noindent and ascend accordingly:
\begin{equation}
    \label{eqn:background_grad_ascent}
    \phi \leftarrow \phi + \alpha \nabla_\phi J(\phi)
\end{equation}

\noindent Though gradient ascent is simple in principle, computing $\nabla_\phi J(\phi)$ is difficult in practice, as $\pi_\phi(\tau)$ can contain a product of many terms (depending on trajectory length). Moreover, $\pi_\phi(\tau)$ includes terms that may not be known \textit{a priori}, such as the transition distribution $p(s_{t+1} | s_t, a_t)$. We must therefore construct an objective that does not depend on these terms explicitly. Fortunately, the policy gradient theorem allows us to do just that.

\subsubsection{Policy Gradient Theorem}
For any policy $\pi_\phi(\tau)$ and any objective $J(\phi)$, we can write the policy gradient as:
\begin{align}
    \label{eqn:background_pg_theorem}
    \nabla_\phi J(\phi) &= \mathop{\mathbb{E}}_{\tau \sim \pi_\phi(\tau)} \bigg[ \bigg( \sum_t \nabla_\phi \textrm{log} \pi_\phi(a_t|s_t) \bigg) \bigg(\sum_t r(s_t, a_t) \bigg) \bigg ]
    \\[8pt]
    &\approx \frac{1}{N} \sum_{i=1}^N \bigg( \sum_{t=1}^T \nabla_\phi \textrm{log} \pi_\phi(a_{i,t}|s_{i,t}) \bigg) \bigg(\sum_{t=1}^T r(s_{i,t}, a_{i,t}) \bigg)
\end{align}

\noindent This theorem alone enables the construction of a theoretically sound policy gradient method. In fact, the REINFORCE algorithm \cite{willianms1988toward, williams1992simple} does just that---the method generates samples trajectories from the current policy $\pi_\phi(a_t|s_t)$, computes the gradient using \cref{eqn:background_pg_theorem}, and updates the parameters $\phi$ of the policy according to \cref{eqn:background_grad_ascent}.

\noindent \textbf{Note:} The summation over reward in \cref{eqn:background_pg_theorem} starts at time-step $t=1$, even when the summation over policy $\pi_\phi$ is at a time-step $t>1$. In reality, the policy at time $t'$ cannot effect the reward at time $t$ when $t < t'$. To account for this, the policy gradient $\nabla_\phi J(\phi)$ is often rewritten as:
\begin{align}
    \nabla_\phi J(\phi) &\approx \frac{1}{N} \sum_{i=1}^N \bigg( \sum_{t=1}^T \nabla_\phi \textrm{log} \pi_\phi(a_{i,t}|s_{i,t}) \bigg) \bigg(\sum_{t'=t}^T r(s_{i,t'}, a_{i,t'}) \bigg)
    \\
    \label{eqn:background_pg_with_q}
    & \approx \frac{1}{N} \sum_{i=1}^N \bigg( \sum_{t=1}^T \nabla_\phi \textrm{log} \pi_\phi(a_{i,t}|s_{i,t}) \bigg) Q_{i,t}^{\pi_\phi}(s_{i,t}, a_{i,t})
\end{align}
\noindent where the summation of ``reward-to-go", conditioned on actions, is equivalent to the action-value function in \cref{eqn:background_q_func}.

\subsubsection{Baselines}
Though simple, one drawback of REINFORCE is that it results in a high-variance gradient estimator, which leads to inefficient learning in practice. Due to randomness, the return of each trajectory may vary drastically episode-to-episode. By estimating the gradient with finite samples, it is possible to receive a different gradient estimate for every set of samples. This reduces the smoothness of the gradient path, in turn increasing the amount of time needed to converge. The problem of variance is inherent to all policy gradient methods and there exist many extensions of this basic idea that specifically target the amount of variance in the learning process.

The most prominent variance reduction technique is to subtract some quantity from cumulative reward. For a single trajectory $\tau$, we have:
\begin{equation}
    \nabla_\phi J(\phi) \approx \frac{1}{N}\sum_{i=1}^N \nabla_\phi \textrm{log} \pi_\phi(\tau)[r(\tau) - b]
\end{equation}
\noindent The quantity $b$ is known as a \textbf{baseline} and allows us to change the amount of variance present in the system while keeping the same gradient in expectation. The baseline can be chosen strategically to maximize variance reduction while maintaining the direction of the gradient. For example, a popular choice is $b(s_t) = V^{\pi_\phi}(s_t)$, which results in the modified objective:
\begin{align}
    \nabla_\phi J(\phi) &\approx \frac{1}{N} \sum_{i=1}^N \bigg( \sum_{t=1}^T \nabla_\phi \textrm{log} \pi_\phi(a_{i,t}|s_{i,t}) \bigg) \bigg( Q_{i,t}^{\pi_\phi}(s_{i,t}, a_{i,t}) - V_{i,t}^{\pi_\phi}(s_{i,t})\bigg)
    \\
    \label{eqn:background_pg_baseline}
    &\approx \frac{1}{N} \sum_{i=1}^N \bigg( \sum_{t=1}^T \nabla_\phi \textrm{log} \pi_\phi(a_{i,t}|s_{i,t}) \bigg) A_{i,t}^{\pi_\phi}(s_{i,t}, a_{i,t})
\end{align}
\noindent where the difference $Q^{\pi_\phi}(s_t, a_t) - V^{\pi_\phi}(s_t)$ is equivalent to the advantage function $A^{\pi_\phi}(s_t, a_t)$ from \cref{eqn:background_adv}.

\subsubsection{Proximal Policy Optimization}
Proximal Policy Optimization (PPO)~\citep{schulman2017proximal} is an on-policy, policy gradient learning method that streamlines recent advances in trust-region based policy optimization~\citep{schulman2015trust}. PPO iteratively updates the parameters $\phi$ of an agent's policy $\pi_\phi$ with the respect to the following clipped surrogate objective:
\begin{equation*}
    L_{\textrm{CLIP}}(\phi) = \underset{a, s \sim \pi_{\phi_{\textrm{old}}}}{\mathbb{E}} \bigg [\min \bigg( \frac{\pi_\phi(a|s)}{\pi_{\phi_{\textrm{old}}}} \hat{A}^{\pi_{\textrm{old}}}, \textrm{clip} \bigg( \frac{\pi_\phi(a|s)}{\pi_{\phi_{\textrm{old}}}} \hat{A}^{\pi_{\textrm{old}}}, 1 - \epsilon, 1 + \epsilon  \bigg) \hat{A}^{\pi_{\textrm{old}}} \bigg) \bigg]
\end{equation*}
\noindent where $\epsilon$ regulates the step size of policy updates (from $\phi_{\textrm{old}}$ to $\phi$) and $\hat{A}^{\pi_{\textrm{old}}}$ is an estimate of the agent's advantage function (via generalized advantage estimation~\citep{schulman2015high}). In practice, PPO benefits from two additional objectives:
\begin{enumerate}
    \item An error term on value estimation, which we take to be the mean-squared error, $(V_\omega(s) - V^{\textrm{targ}}_{\omega'}(s))^2$, between value estimates from the agent's value function $V_\omega$ (with parameters $\omega$) and a target value function $V^{\textrm{targ}}_{\omega'}$. Typically, parameter sharing between the policy and value functions is allowed (i.e. $\phi = \omega$).
    \item An entropy bonus $H(\pi_\phi(\cdot|s))$ over the agent's policy that encourages exploration.
\end{enumerate}
Altogether this yields the following three-pronged PPO objective:
\begin{equation}
    \label{eqn:background_ppo_loss}
    L_{\textrm{PPO}}(\phi, \omega) = L_{\textrm{CLIP}}(\phi) - \alpha_1(V_\omega - V^{\textrm{targ}}_{\omega'})^2 + \alpha_2H(\pi_\phi(\cdot|s))
\end{equation}
\noindent where $\alpha_1$ and $\alpha_2$ are coefficients weighing the relative importance of the value error term and entropy bonus term, respectively. Recently, PPO has shown strong performance in cooperative multi-agent settings~\citep{de2020independent, yu2021surprising}.

\subsection{Actor-Critic Algorithms}
Both \cref{eqn:background_pg_with_q} and \cref{eqn:background_pg_baseline} offer powerful ways to estimate the gradient while also reducing variance. In reality, however, we often do not have access to the true $Q^\pi(s,a)$ or $V^\pi(s)$ and must approximate them as well. Actor-critic methods, combining the benefits of value-based and policy-based RL methods, provide a mechanism to do so. In actor-critic methods, an \textit{actor} function controls the policy $\pi_\phi$, taking actions in the environment. A \textit{critic} function performs policy evaluation, fitting an approximate value function (either $Q_\omega(s,a)$ or $V_\omega(s)$) and using it to evaluate the actor's actions. The actor then updates $\phi$ in the direction suggested by the critic. This back-and-forth procedure yields an approximate policy gradient---adjust $\pi_\phi$ in the direction that will get more reward, according to the critic.

\subsubsection{Deep Deterministic Policy Gradients}
\label{sec:background_ddpg}
Deep Deterministic Policy Gradients (DDPG) is an off-policy actor-critic algorithm for policy gradient learning in continuous action spaces \cite{lillicrap2015continuous}. DDPG learns an optimal deterministic policy $\mu_\phi$ with respect to the RL objective:
\begin{equation}
    \label{eqn:background_ddpg_obj}
    J(\phi) = \mathbb{E}_s[Q_{\omega}(s,a) \mid_{s=s_t, a=\mu_{\phi}(s_t)} ]
\end{equation}
\noindent by performing gradient ascent over the following gradient:
\begin{equation}
    \label{eqn:background_ddpg_grad}
    \nabla_\phi J(\phi) = \mathbb{E}_s [\nabla_a Q_{\omega}(s,a) \mid_{s=s_t, a=\mu(s_t)}\nabla_\phi \mu(s)\mid_{s=s_t}]
\end{equation}
\noindent where $\phi$ and $\omega$ are parameters associated with $\mu$ and $Q$, respectively. \Cref{eqn:background_ddpg_grad} is a consequence of the deterministic policy gradient theorem \cite{silver2014deterministic}. For critic updates, DDPG minimizes the loss function:
\begin{equation}
    L(\omega) = \underset{s, a, r, s'}{\mathbb{E}} \big[\big(Q_\omega(s, a) - (r(s,a) + \gamma Q_\omega(s', \mu_\phi(s')))\big)^2\big]
\end{equation}
\noindent where ($s, a, r, s'$) are transition tuples sampled from a replay buffer and $\gamma {\in} [0,1]$ is a scalar discount factor. In this work, agents learn in a decentralized manner, each performing DDPG updates individually.

\subsection{Multi-Agent Settings}
In this section, we review a few notational and algorithmic considerations for RL in multi-agent settings.

\subsubsection{Markov Games}
\label{sec:background_markov_games}
A Markov game is a multi-agent extension of the Markov decision process (MDP) formalism \cite{littman1994markov}. For $n$ agents, it is represented by a state space $S$, joint action space $\boldsymbol{A} = \{A_1, ... , A_n\}$, joint observation space $\boldsymbol{O} = \{O_1, ... , O_n\}$, transition function $T:S \times \boldsymbol{A} \rightarrow S$, and joint reward function $\boldsymbol{r}$. Following multi-objective RL \cite{zimmer2020learning}, we define a vectorial reward $\boldsymbol{r}:S \times \boldsymbol{A} \rightarrow \mathbb{R}^n$ with each component $r_i$ representing agent $i$'s contribution to $\boldsymbol{r}$. Each agent $i$ is initialized with a policy $\pi_i:O_i \rightarrow A$ (or deterministic policy $\mu_i$) from which it selects actions and an action-value function $Q_i: S \times A_i \rightarrow \mathbb{R}$ with which it judges the value of state-action pairs. Following action selection, the environment transitions from its current state $s_t$ to a new state $s_{t+1}$, as governed by $T$, and produces a reward vector $\boldsymbol{r}_t$ indicating the strength or weakness of the group’s decision-making. In the episodic case, this process continues for a finite time horizon $T$, producing a trajectory $\tau = (s_1, \boldsymbol{a_1}, ..., s_{T-1}, \boldsymbol{a_{T-1}}, s_T)$ with probability:
\begin{equation}
    \label{eqn:background_traj_prob}
    P(\tau) = P_\emptyset(s_1)\prod_{t=1}^T  P(s_{t+1} \mid s_t, a_t) \pi(a_t \mid s_t)
\end{equation}

\noindent where $P_\emptyset$ is a special distribution specifying the likelihood of each ``start" state.

\subsubsection{Decentralized vs. Centralized Learning}
As discussed in \cref{sec:background_rl_pg}, policy gradient methods are inherently susceptible to high variance gradient updates. This variance problem is exacerbated in multi-agent settings, as the environment (and thus the learning problem) becomes increasingly non-stationary for each agent. For example, consider an agent $i$ that, at some time $t$, selects an action $a_{i,t}$ from some state $s_t$. In single-agent RL, the reward $r_{i,t}$ received by agent $i$ is primarily a function of the quality of agent $i$'s decision-making from $s_t$. In multi-agent settings, however, this is not always the case. In separate trajectories, agent $i$ could select the same action $a_{i,t}$ from $s_t$ and receive \textit{different rewards} because one or many agents in the environment changed their action selection. Since each agent runs its own learning process and updates its policy quickly, this scenario can occur frequently throughout training, to the detriment of learning stability.

To solve this problem, recent work has popularized centralized training, decentralized execution (CTDE) training paradigms in which each agent receives some information about the parameter updates of the other agents in the environment during training \cite{lowe2017multi, foerster2018counterfactual}. Here we present one such technique---Multi-Agent DDPG (MADDPG) \cite{lowe2017multi}---as a representative example.

\paragraph{Multi-Agent Deep Deterministic Policy Gradients}
MADDPG a multi-agent extension of the DDPG algorithm from \cref{sec:background_ddpg}. Given $N$ agents with policies $\boldsymbol{\pi_\phi} = \{\pi_{\phi_1}, ..., \pi_{\phi_N} \}$ parameterized by $\boldsymbol{\phi} = \{\phi_1, ...,  \phi_N \}$, we can write a multi-agent variant of the policy gradient as:
\begin{equation*}
    \nabla_{\phi_i} J(\phi_i) = \mathop{\mathbb{E}}_{\substack{s_{t+1} \sim p(s_{t+1}|s_t, a_t) \\ a_i \sim \pi_{\phi_i}}} \bigg[\nabla_{\phi_i} \textrm{log} \pi_{\phi_i}(a_i|o_i)Q^{\pi_i}(\boldsymbol{x}, a_1, ..., a_N) \bigg]
\end{equation*}

\noindent where $Q^{\pi_i}$ is a centralized action-value function that receives state information $\boldsymbol{x}$ and the actions of each agent as input. The state information $\boldsymbol{x} = (o_1, ..., o_N)$ consists of the observations of each agent plus, in some cases, additional state from the environment. As in DDPG, we can write the multi-agent objective in terms of deterministic policies $\boldsymbol{\mu_\phi} = \{\mu_{\phi_1}, ..., \mu_{\phi_N} \}$:
\begin{equation*}
    \nabla_{\phi_i} J(\phi_i) = \mathop{\mathbb{E}}_{\substack{s_{t+1} \sim p(s_{t+1}|s_t, a_t) \\ a_i \sim \mathcal{D}}} \bigg[\nabla_{\phi_i} \textrm{log} \mu_{\phi_i}(a_i|o_i)\nabla_{a_i}Q^{\mu_i}(\boldsymbol{x}, \mu_{\phi_1}(o_1), ..., \mu_{\phi_N}(o_N)) \bigg]
\end{equation*}
\noindent and make use of a replay buffer $\mathcal{D} = (\boldsymbol{x}, \boldsymbol{x'}, a_1, ..., a_N, r_1, ..., r_N)$ and target networks for both the actor ($\mu_{\phi_i}'$) and the critic ($Q_\omega^{\mu_i'}$). The critic is trained via Q-learning with the objective:
\begin{equation}
    \label{eqn:background_maddpg_q_obj}
    \mathcal{L}(\phi_i) = \mathop{\mathbb{E}}_{\boldsymbol{x}, a, r, \boldsymbol{x}'} \big [ (Q_\omega^{\mu_i} (\boldsymbol{x}, \mu_{\phi_1}(o_1), ..., \mu_{\phi_N}(o_N)) - y)^2  \big]
\end{equation}
\noindent where:
\begin{equation}
    \label{eqn:background_maddpg_q_target}
    y = r_i + (Q_\omega^{\mu_i'} (\boldsymbol{x}', \mu_{\phi_1}'(o_1), ..., \mu_{\phi_N}'(o_N)))
\end{equation}

In this work, however, we prioritize decentralized learning because it more accurately that type of learning problem that agents will face when learning alongside human counterparts.

\subsection{AlphaZero}
\label{sec:background_alphazero}
AlphaZero is a hybrid algorithm that combines self-play RL and Monte-Carlo Tree Search (MCTS). AlphaZero's search and learning processes are reviewed below, as both are crucial to understanding our method.

\paragraph{Setup}
AlphaZero maintains a search tree in which each node is a state $s$ and each edge $(s, a)$ represents taking an action $a$ from $s$. Each edge is described by a tuple of statistics $\{N(s,a), W(s,a), Q(s,a), P(s,a)\}$, holding the edge's visit count, total action value, mean action value, and selection probability, respectively. AlphaZero uses a two-headed neural network $f_\theta$ with parameters $\theta$ to compute policy probabilities $\boldsymbol{p}$ and a value estimate $v$ for a state $s$ as follows:
\begin{equation}
    \label{eqn:background_az_network}
    (\boldsymbol{p}, v) = f_\theta(s)
\end{equation}
MCTS uses $f_\theta$ to evaluate states during the search process. 

\paragraph{Search}
In each state $s$, AlphaZero executes $n$ simulations of MCTS tree search. In each simulation, MCTS traverses the game tree by selecting edges that maximizes the upper confidence bound $Q(s,a)+U(s,a)$, where:
\begin{equation}
    \label{eqn:background_az_puct}
    U(s,a) = cP(s,a)\frac{\sqrt{\sum_b N(s,b)}}{1 + N(s,a)}
\end{equation}
encourages the exploration of lesser-visited states (within the search tree); and $c$ is a coefficient determining the level of exploration. When a leaf node $s_L$ is encountered, it is expanded, evaluated by AlphaZero's network to produce $(\boldsymbol{p}_L, v_L) \sim f_\theta(s_L)$, and initialized as follows:
\begin{equation}
    \label{eqn:background_az_node_init}
    N(s,a) = 0, W(s,a) = 0, Q(s,a) = 0, P(s,a) = p_{L, a}
\end{equation}
\noindent Then, in a backward pass up the search tree, the statistics of each traversed edge are updated as follows: $N(s,a) = N(s,a) + 1$, $W(s,a) = W(s,a) + v_L$, $Q(s,a) = W(s,a) / N(s,a)$.

\paragraph{Action Selection}
After MCTS, AlphaZero constructs a policy $\boldsymbol{\pi}$, where the probability $\pi(a \mid s)$ of selecting action $a$ from state $s$ is computed directly from visitation counts:
\begin{equation}
    \label{eqn:background_mcts_policy}
    \pi(a \mid s) = \frac{N(s,a)^{1/\tau}}{\sum_b N(s,b)^{1/\tau}}
\end{equation}
and $\tau$ is a temperature parameter that controls exploration. During both training and test-time rollouts, AlphaZero samples actions $a \sim \boldsymbol{\pi}$ according to this policy.

\paragraph{Learning}
Starting from an initial state $s_0$ (e.g. an empty board), AlphaZero proceeds with self-play, using the aforementioned search process to select actions until a terminal state $s_T$ is reached. A score $z$ is produced for $s_T$ according to the rules of the game and tuples $(s_t, a_t, z)$ are stored in a replay buffer $D$ for each time-step $t \in [0, T]$ in the trajectory. AlphaZero's neural network is updated to minimize the loss function:
\begin{equation}
    \label{eqn:background_az_loss}
    l = (z - v)^2 - \boldsymbol{\pi}^T \log \boldsymbol{p} + \lambda||\theta||^2
\end{equation}
where $(\boldsymbol{p}, v) \sim f_\theta(s)$ as in \cref{eqn:background_az_network}, $\boldsymbol{\pi}$ follows from \cref{eqn:background_mcts_policy}, and $\lambda$ is a coefficient weighting $L2$-regularization. Intuitively, \cref{eqn:background_az_loss} encourages the network to (i) minimize the error between predicted values $v$ and game outcomes $z$; and (ii) maximize similarity between predicted policy probabilities $\boldsymbol{p}$ and search probabilities $\boldsymbol{\pi}$.

\section{Interpretability}
\label{sec:background_interpretability}
Neural networks, and any algorithms that leverage them, are inherently black-box decision-makers. The field of interpretable machine learning seeks to unlock this black-box such that neural networks have the "\textit{the ability to explain [their decisions] or to present [them] in understandable terms to a human}" \cite{doshi2017towards}. In general, interpretability techniques can be classified into one of two categories: (i) \textbf{intrinsic interpretability}, which describes algorithms and models that are architected to have an interpretable decision-making structure prior to training (e.g. decision trees); and (ii)\textbf{post-hoc interpretability}, where interpretation methods (e.g. linear probing) are applied to a pre-trained, often uninterpretable, model. The descriptions below prioritize concept-based interpretability---a subset of interpretability methods that aim to express network decisions and explanations in terms of a set of manually-curated, human-understandable concepts---as it is a primary motivation for the interpretability methods developed in this work.

\subsection{Intrinsic Interpretability}
\label{sec:background_interpretability_intrinsic}
Consider a standard supervised learning setting, such as regression, where the goal is to learn a function $f:\mathbb{R}^d \rightarrow \mathbb{R}$ that maps inputs $x \in \mathbb{R}^d$ to scalar outputs $y \in \mathbb{R}$, using a dataset $\{(x^{(j)}, y^{(j)})\}_{j=1}^n$ of input-output examples as supervision. The goal of intrinsic interpretability is to ensure that $f$ itself is interpretable---i.e. when an output $\hat{y} = f(x)$ is produced, the factors driving the prediction $\hat{y}$ are explicitly exposed by $f$ (e.g. feature importance, decision tree elements).

In most cases, this means that representing $f_\theta$ as a neural network (here with parameters $\theta$) is not feasible, because neural network activations do not provide this level of decision transparency. However, recent work has shown that it is possible to construct customized network architectures that enable intrinsic interpretability while still leveraging the expressivity of neural network function approximation \cite{koh2020concept}. One such method is a focal point here, as it is key to the motivation of the interpretability methods developed in this work, though many other intrinsic interpretability methods have been proposed in the literature \cite{zhang2021survey, doshi2017towards, ghorbani2019towards}.

\subsubsection{Concept Bottleneck Models}
Most relevant to our work is the Concept Bottleneck Model architecture introduced by \citet{koh2020concept}. Concept bottleneck models refer to a class of intrinsically-interpretable neural network architectures for supervised learning. The CBM network makes predictions by first estimating a set of human-understandable concepts, then producing an output based on those concept estimates---i.e. the network is ``bottlenecked" by concepts. Formally, given a dataset $\{(x^{(j)}, y^{(j)}, c^{(j)})\}_{j=1}^n$ consisting of inputs $x \in \mathbb{R}^d$, outputs  $y \in \mathbb{R}$, and human-understandable concepts $c \in \mathbb{R}^k$ (where $k$ is the number of unique concepts), a concept bottleneck model learns two mappings---$f: \mathbb{R}^d \rightarrow \mathbb{R}^k$ from input-space to concept-space, and $g: \mathbb{R}^k \rightarrow \mathbb{R}$ from concept-space to output-space. The model can then make predictions $\hat{y} = f(g(x))$ as a composition of those mappings. 

\citet{koh2020concept} demonstrate the increased interpretability achieved by concept bottlenecks in the context of an arthritis classification task in which arthritis severity is predicted from MRI images of knee joints. Unlike standard networks that learn an uninterpretable mapping from MRI images to arthritis severity labels, the concept bottleneck's output is conditioned on an intermediate set of human-understandable concepts. Since these concepts correspond to real-world features---e.g., the presence of bone spurs in the arthritis task---they can be used by human observers to understand model failures or misclassifications. Moreover, because the model's predictions are conditioning on these concepts, we can manually change the model's output (without retraining or fine-tuning) by intervening on its concept estimates.

\subsection{Post-hoc Interpretability}
\label{sec:background_interpretability_posthoc}
Despite the explicit decision-making transparency enabled by intrinsically interpretable models, there is often a trade-off in asymptotic performance that comes with the architectural choices needed to support that interpretability. For example, concept bottleneck models require a sufficiently rich set of concepts \cite{yeh2020completeness} to perform well, which can be difficult to obtain manually in complex settings. For this reason, post-hoc interpretability methods seek to generate or otherwise extract explanations from a pre-trained model.

\subsubsection{Linear Probing}
We review the predominant post-hoc interpretability method, linear probing, as a representative example. As above, we focus on a \textit{concept-based} variant of the problem domain \cite{kim2018interpretability}. Let $f_\theta$ be a neural network with parameters $\theta$ that has been trained via supervised learning (as outlined in \cref{sec:background_interpretability_intrinsic}). Given an input $\boldsymbol{x}$, and assuming that $f_\theta$ is composed of $L$ layers, we can rewrite the network's prediction $\hat{y} = f_\theta(\boldsymbol{x})$ as:
\begin{equation*}
    \hat{y} = f^L_\theta \circ ... \circ f^2_\theta \circ f^1_\theta(\boldsymbol{x}) = f^{1:L}_\theta(\boldsymbol{x})
\end{equation*}
\noindent where $f^l_\theta$ is an intermediate layer $f_\theta$ accessed via an index $l \in \{1, ..., L\}$. During a forward pass through the network, it is also possible to extract the activation $\boldsymbol{z}^l$ from layer $l$ by passing the output of the previous layer $\boldsymbol{z}^{l-1}$ through $f^l_\theta$:
\begin{equation*}
    \boldsymbol{z}^l = f^l_\theta(\boldsymbol{z}^{l-1})
\end{equation*}
\noindent where $\boldsymbol{z}^0 = \boldsymbol{x}$ and $\boldsymbol{z}^L = \hat{y}$. Now, given a human-understandable concept $c$, the goal of linear probing is to train a separate model $g^l$ over the activation $\boldsymbol{z}^l$ from some layer $l$ to approximate $c$. Importantly, $g^l$ is defined as a sparse linear regression function:
\begin{equation*}
    g^l(\boldsymbol{z}^l) = \boldsymbol{w}^T \boldsymbol{z}^l + \boldsymbol{b}^l
\end{equation*}
\noindent weighted by $\boldsymbol{w}$ (and with biases $\boldsymbol{b}$). Concept-based interpretability can then be measured as a function of the accuracy of the linear predictor. Intuitively, if $g^l(\boldsymbol{z}^l)$ accurately predicts $c$ on a held-out test set of examples, we can say that layer $l$ has learned to encode and therefore carries information about the concept $c$.

The recent work of \citet{mcgrath2022acquisition} underscores the power of linear probing as a post-hoc interpretability tool. Specifically, they conduct an interpretability analysis over AlphaZero's policy network using probing techniques and show that it is possible to identify both when (and where) in it has learned to encode strategic concepts that are pertinent to strong performance in board games. In the board game setting, a single concept $c:\mathbb{R} \rightarrow R$ is a mapping of the current game state $c(\boldsymbol{x})$ to a scalar value---e.g. in chess, $c$ could be defined in terms of piece location ($c(\boldsymbol{x}) = 1$ if white has both bishops, $c(\boldsymbol{x}) = 0$ otherwise) or some higher-level evaluation (material imbalance).

\section{Fairness}
\label{sec:background_fairness}
As machine learning algorithms proliferate real-world applications (e.g. lending \cite{hardt2016equality, liu2018delayed}, hiring \cite{hu2018short}), making decisions on behalf of real people and operating over the private and often personal information of individuals, it is of the utmost importance that the researchers and practioners behind the development and commercialization of these algorithms ensure that they do not disproportionately harm already disadvantaged groups or reinforce societal biases. This has driven significant research interest into definitions and quantifications of \textbf{fairness}, and algorithmic improvements that enable machine learning systems to uphold these fairness definitions.

\subsection{Prediction-based Fairness}
The majority of fairness work to date has examined fairness in the context of supervised learning; also known as \textbf{prediction-based fairness}. Prediction-based fairness considers a population of $n$ individuals (indexed $i = 1, ..., n$), each described by variables $v_i$ (i.e. features or attributes), which are separated into sensitive variables $z_i$ and other variables $x_i$. Variables $v_i$ are used to predict (typically binary) outcomes $y_i \in Y$ by estimating the conditional probability $P[Y=1 \mid V=v_i]$ through a scoring function $\psi:\mathcal{V} \rightarrow \{0,1\}$. Outcomes in turn yield decisions by applying a decision rule $\delta(v_i) = f(\psi(v_i))$. For example, in a lending scenario, a classifier may use $v_i$ to predict whether an individual $i$ will default on ($y_i = 0$) or repay ($y_i = 1$) his/her loan, which informs the decision to deny ($d_i = 0$) or approve ($d_i = 1$) the individual's loan application \cite{mitchell2018prediction}.

This notation supports a wide range of fairness criteria\footnote{We point to \citet{mitchell2018prediction} for a comprehensive overview of individual and group fairness definitions that are not used in this work.}. For one, it enables the comparison of outcome or decision fairness at the level of individuals. Such \textbf{individual fairness} measures \cite{dwork2012fairness} posit that two individuals with similar features should receive similar decision outputs from a classifier---i.e. similarity in feature-space implies similarity in decision-space---and many task-specific measures of similarity have been proposed in the literature \cite{barocas2019fairmlbook, chouldechova2018frontiers}.

It is also possible to quantify fairness across cohorts of individuals with a particular assignment of sensitive variables. These are known as \textbf{group fairness} measures. Group-based fairness examines how well outcome ($Y$) and decision ($D$) consistency is preserved across sensitive groups ($Z$) \cite{feldman2015certifying, zafar2017fairness}. We highlight the group-based measure of demographic parity, which requires that $D \perp Z$ or, equivalently, that $P[D=1 \mid Z=z] = P[D=1 \mid Z=z']$ for all $z, z'$ where $z \neq z'$. Group-based fairness definitions are of particular relevance to this work and will be examined in multi-agent settings in \cref{sec:fairness_intro}.

\subsection{Multi-Agent Fairness}
In multi-agent settings, fairness has traditionally been examined through a socio-economic lens, where outcomes are defined in terms of utility, and fairness is defined in terms of social welfare (the distribution of utility). For example, in the seminal work of \citet{zheng2020ai}, socio-economic agents are participants in a gather-and-build game where utility is derived from raw resources collected in the environment (and goods built from those resources) and fairness is measured wealth equality that is attained through taxation and transfer (as dictated by a governing body, such as a tax authority). Many social welfare functions have been proposed to measure the equality outcomes, such as the well-known Gini index \cite{gini1921measurement, dorfman1979formula}, and can be incorporated directly into the objectives of learning agents \cite{zimmer2020learning, siddique2020learning}.

These definitions are well-suited to multi-agent problems in which measuring outcome equality is straightforward, such as resource allocation \cite{elzayn2019fair, zhang2014fairness}; or in game-theoretic settings such as social dilemmas---where an agent's local incentives conflict with some global notion of utility that is defined for all agents (e.g. tragedy of the commons) \cite{leibo2021scalable, leibo2017multi, de2005priority}.

However, drawing an analogy to studies of fairness in supervised learning settings, measuring fairness through outcomes alone may not be a complete picture, as there may exist influential, though less visible, factors (i.w. sensitive variables) upon which the fairness of outcomes or decisions relies. We posit that the influence of sensitive variables---and even how they should be defined in the first place---is far more nuanced in multi-agent settings than what has been captured in the literature thus far. One of the primary goals of this work is to highlight and address these considerations.

\section{Environments}
\label{sec:background_environments}
Multi-agent learning environments come in many forms and with a variety of incentive structures. In cooperative settings, agents' incentives are aligned---in terms of reward (or other utility-based objectives) what is good for one agent is good for all agents in the multi-agent system and this alignment provides a rich ground for studying coordination dynamics among agents. Competitive multi-agent environments are characterized by conflicting incentives, where an increase in one agent's reward corresponds to a decrease in another's, and these settings are particularly interesting for studying strategic behavior and competition among agents. Mixed-motive social dilemmas present a more complex scenario where agents' incentives are partially aligned and partially in conflict, often leading to sub-optimal outcomes if agents act purely out of self-interest without considering the collective good. These environments are suitable for studying learned social dynamics, such as: exploitation, public resource sharing, and the emergence of free-rider problems. In this work, we investigate each type of environment with the goal of understanding and shaping emergent behavior in each.

\subsection{Cooperative Environments}
\label{sec:background_environments_cooperative}
Here we describe the cooperative multi-agent environments used in this work.

\subsubsection{Pursuit-Evasion}
Pursuit-evasion is a classic setting for studying multi-agent coordination \cite{isaacs1999differential}. Though often played on a graph \cite{parsons1978pursuit}, work on continuous-space pursuit-evasion has enabled real-world applications such as unmanned aerial vehicles \cite{vidal2002probabilistic} and mobile robots \cite{chung2011search}. Further work has shown that optimal control strategies can be derived from value functions \cite{jang2005control} or even learned from scratch in MARL setting \cite{lowe2017multi}. A relevant class of pursuit-evasion games define the evader to be of equal or greater speed than the pursuers. This setting highlights the need for coordinated motion (e.g. encircling) \cite{vicsek2010closing} and communication \cite{wang2020cooperative} by the pursuers and is subject to theoretical performance bounds under these conditions \cite{ramana2017pursuit}. We use this setting to study the emergence of implicit communication as a low-level form of behavioral understanding; and as a means for identifying issues of fairness that emerge amongst cooperative agents. 

Formally, the pursuit-evasion game is played between $n$ pursuers $\{p_1, ..., p_n \}$ and a single evader $e$. The goal of the pursuers is to catch the evader as quickly as possible and, conversely, the goal of the evader is to remain uncaught. Each agent $i$ is described by its current position and heading $q_i$ and is subject to planar motion $\dot{q}_i$:
\begin{align*}
    q_i =
    \begin{bmatrix}
        x_i\\
        y_i\\
        \theta_i
    \end{bmatrix}
    &&
    \dot{q}_i =
    \begin{bmatrix}
        \dot{x}_i\\
        \dot{y}_i\\
        \dot{\theta}_i
    \end{bmatrix}
    =
    \begin{bmatrix}
        \lvert \vec{v}_i \rvert \, \textrm{cos}(\theta_i) \\
        \lvert \vec{v}_i \rvert \, \textrm{sin}(\theta_i) \\
        \textrm{atan2}(\dot{y}_i, \dot{x}_i)
    \end{bmatrix}
\end{align*}
\noindent where $\vec{v}_i$ is the agent $i$'s velocity. The environment state $s_t$ is described by the position and heading of all agents $s_t = \{q_{p_1}, ..., q_{p_n}, q_e \}$. Upon observing $s_t$, each agent selects its next heading $\theta_i$ as an action. The chosen heading is pursued at the maximum allowed speed for each agent ($\lvert \vec{v}_p\rvert$ for the pursuers, $\lvert \vec{v}_e\rvert$ for the evader); with orientation changes being instantaneous. To encourage teamwork, we set $\lvert \vec{v}_p \rvert \leq \lvert \vec{v}_e\rvert$ in our work. 

This work utilizes the pursuit-evasion environment in \cref{fig:implicit_comm_torus_env}, which is a toroidal extension of the planar pursuit-evasion game proposed by \citet{lowe2017multi}. 
\begin{figure}[t!]
    \centering
    \makebox[\linewidth][c]{\includegraphics[width=0.99\linewidth]{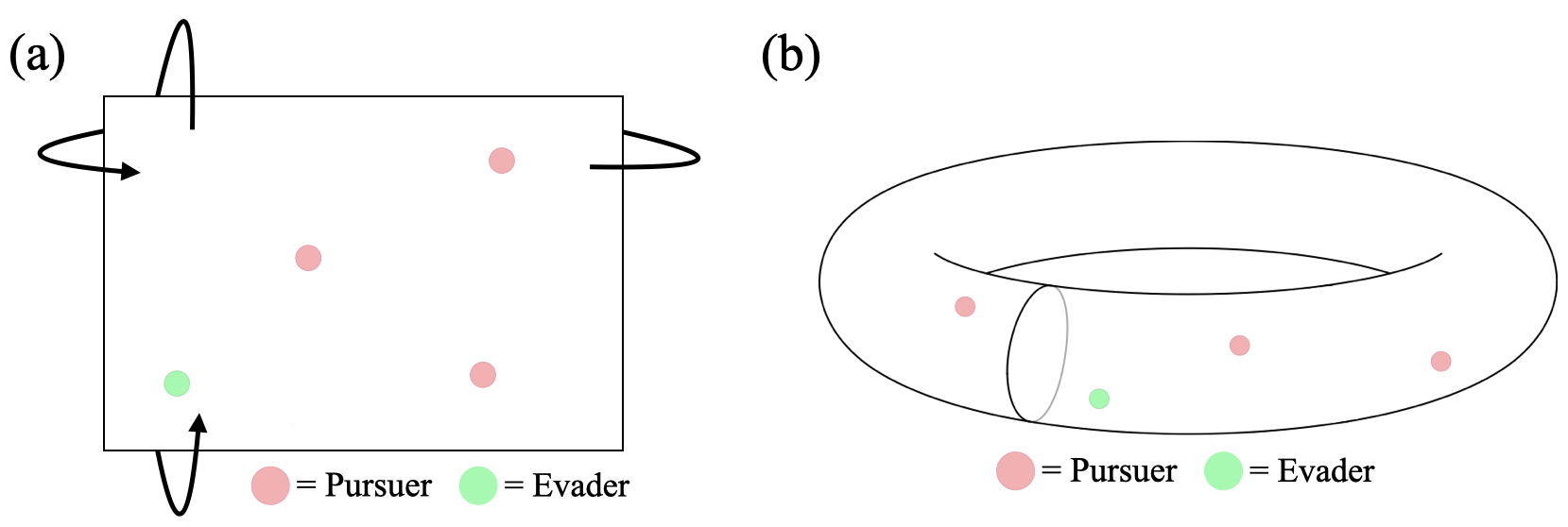}}
    \caption{A planar pursuit-evasion game with periodic boundary conditions interpreted as a toroidal pursuit-evasion environment.}
    \label{fig:implicit_comm_torus_env}
\end{figure}

In general, unbounded planar pursuit-evasion can be described by two cases:
\begin{itemize}
    \item Case 1: $\lvert \Vec{v}_p \rvert > \lvert \Vec{v}_e\rvert$. The game is solved by a straight-line chase towards the evader and is not interesting from the perspective of coordination.
    \item Case 2: $\lvert \Vec{v}_p \rvert \leq \lvert \Vec{v}_e\rvert$. The evader has a significant advantage. Pursuers have at most one opportunity to capture the evader and are usually only successful under strict initialization conditions \cite{ramana2017pursuit}.
\end{itemize}
\noindent \citet{lowe2017multi} addressed this by penalizing agents for leaving the immediate area defined by the camera with negative reward. The evader defined by \cref{eqn:implicit_comm_evader_objective}, however, will run away indefinitely in the $\lvert \Vec{v}_p \rvert \leq \lvert \Vec{v}_e\rvert$ case. To provoke consistent interaction between agents, we extend the planar environment with periodic boundary conditions. One can think of this as playing the pursuit-evasion game on a torus (see \cref{fig:implicit_comm_torus_env}).

Toroidal pursuit-evasion does not require strict initialization conditions or special rewards. Pursuers can be initialized randomly and allowed to construct ad-hoc formations. Second, pursuit is no longer a one-and-done proposition. This reflects the notion that, in nature, predators often do not give up after a single attempt at a prey---they regroup and pursue it again.

\subsubsection{Collaborative Cooking}
\label{sec:background_environments_cooperative_cooking}
\begin{figure}
    \centering
    \includegraphics[width=0.5\linewidth]{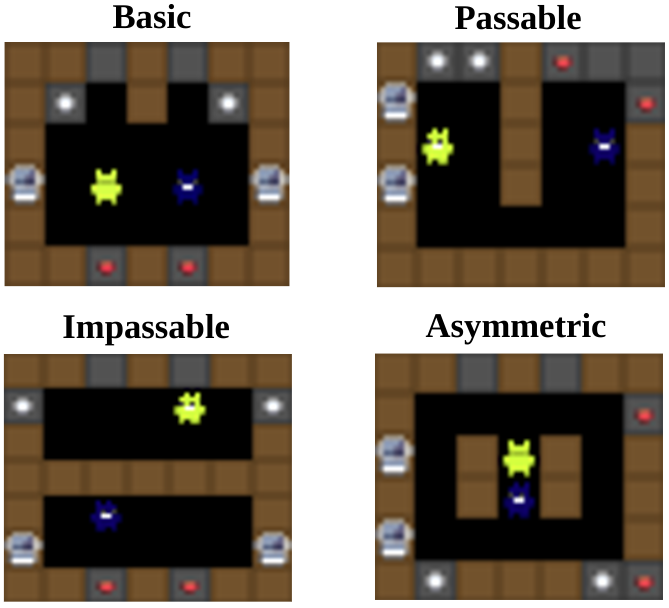}
    \caption{Collaborative Cooking Environments.}
    \label{fig:background_environments_cooking}
\end{figure}

Collaborative Cooking is a cooperative multi-agent environment from DeepMind's Melting Pot suite of environments \cite{leibo2021scalable}. Based on the game Overcooked~\citep{overcooked} and subsequent work that has developed Overcooked-like environments for studying multi-agent coordination~\citep{carroll2019utility, wu2021too}, Melting Pot's Collaborative Cooking is a game in which a group of agents inhabit a kitchen-like environment and must collaborate to find ingredients, complete recipes, and deliver finished dishes as quickly as possible. Solving the cooking task requires sophisticated coordination, involving both task partitioning---splitting a recipe into parts---and role assignment---distributing sub-tasks among agents. For these reasons, Collaborative Cooking is investigated in a number of prior works \citep{wu2021too, carroll2019utility, strouse2021collaborating} and is emerging as a strong benchmark for multi-agent learning. A visualization of the Collaborative Cooking environments that we use in our experiments is shown in \cref{fig:background_environments_cooking}.

More formally, the Collaborative Cooking game for $N$ agents is defined as follows:

\begin{itemize}
    \item \emph{States}: Each environment is a grid world. Grid cells can be filled by an agent or an environment-specific object.
    \item \emph{Observations}: Agents receive partial multi-modal observations consisting of their own position and orientation in the grid, as well as a partial RGB rendering of a $5\textrm{-cell} \times 5\textrm{-cell}$ window centered at the agent.
    \item \emph{Actions}: Agents can execute one of 8 actions: no-op, move $\{\texttt{up}, \texttt{down}, \texttt{left}, \texttt{right}\}$, turn $\{\texttt{left}, \texttt{right}\}$, and interact.
    \item \emph{Recipe}: (i) Bring a tomato to a cooking pot (3 times), (ii) Wait for soup to cook in the pot (20 time-steps), (iii) Bring a dish to the cooking pot, (iv) Pour soup from the pot into the dish, (v) Deliver soup to the delivery location. In practice, solving this task from scratch in its entirety is extremely difficult, as each of the aforementioned steps requires agents to execute a series of movement and interaction actions in sequence.
    \item \emph{States}: Grid cells can be filled by an agent or any of the following items: Floor, Counter, Cooking Pot, Dish, Tomato.
    \item \emph{Reward}: By default, agents share a positive reward for completing the entire recipe outlined above. In practice, however, solving the cooking task with this sparse reward alone is infeasible (completing each of the recipe steps through random exploration is prohibitively challenging) and successful approaches in prior works either pair learning agents with helpful bot agents or introduce ``densified" pseudorewards to augment the agents' learning signal~\citep{leibo2017multi}. We implement the latter, giving agents a small positive reward for completing steps (i) and (iii) of the recipe. More concretely, we define the following three-part reward:
    \begin{equation*}
        r_t =
        \begin{cases}
            20, & \textrm{if soup cooked and delivered}. \\
            1, & \textrm{if tomato placed in cooking pot.} \\
            1, & \textrm{if soup poured into dish.} \\
            0, & \textrm{otherwise.}
        \end{cases}
    \end{equation*}
    \item \emph{Concepts}: We assume that each environment supports the following concepts (and concept types): (i) agent position (scalar); (ii) agent orientation (scalar); (iii) whether or not an agent has a tomato, dish or soup (binary); (iv) cooking pot position (scalar) (v) the progress of the cooking pot (scalar); (vi) the number of tomatoes in the cooking pot (categorical); and (vii) the position of each tomato and dish (scalar).
\end{itemize}

\subsection{Social Dilemmas}
Here we describe the mixed-motive multi-agent environments used in this work.

\subsubsection{Clean Up}
\begin{figure}
    \centering
    \includegraphics[width=0.5\linewidth]{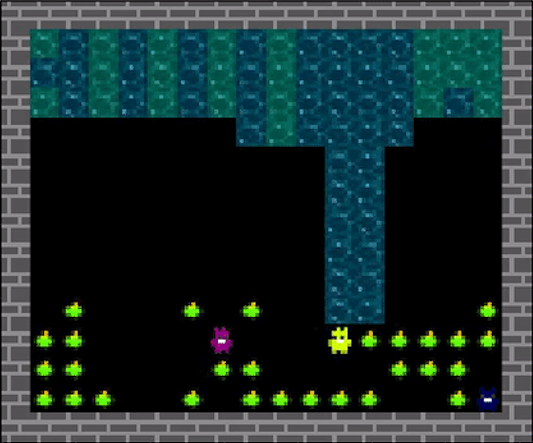}
    \caption{Clean Up Environment.}
    \label{fig:background_environments_cleaning}
\end{figure}
The Clean Up game, also from Melting Pot \cite{leibo2021scalable}, is shown in \cref{fig:background_environments_cleaning}. Clean Up is a classic social dilemma in which agents are forced to balance selfish behaviors with public goods. Each agent is rewarded proportionally to the number of apples it harvests. However, the rate at which apples respawn after they are eaten is directly related to the amount of pollution that exists in the river. Over time, pollution builds up in the river if it remains uncleaned until eventually, apples are prevented from growing altogether. Agents must learn to periodically clean the river to ensure that apples continuously grown, even though they are not directly incentivized to do so by default. Therefore, it is possible for agents to develop both emergent "free-loading" behaviors---harvesting apples without cleaning the river---and emergent public service behaviors---cleaning the river while trying to harvest as many of the remaining apples after cleaning. This makes Clean Up a fitting environment with which to study the extent to which inter-agent social dynamics can be revealed.

The state space, observation space, and action space for this environment are the same as outlined in the Collaborative Cooking environments in \cref{sec:background_environments_cooperative_cooking}, except for the following environment-specific details:
\begin{itemize}
    \item \emph{States}: Grid cells can be filled by an agent or any of the following items: Floor, Wall, Apple, Clean River, Polluted River.
    \item \emph{Reward}: Agents receive an individual positive reward for collecting apples. By default, agents are not rewarded for cleaning the river. We found that this greatly increased the amount of training time required for agents to learn to clean the river (because they have no incentive to do so). Because we are motivated by understanding social behaviors once they have emerged, we augmented each agent's reward with a small positive reward for cleaning the river. Specifically, we define the following reward for the cleaning task:
    \begin{equation*}
        r_t =
        \begin{cases}
            1, & \textrm{for eating an apple}. \\
            0.01, & \textrm{for cleaning pollution.} \\
            0, & \textrm{otherwise.}
        \end{cases}
    \end{equation*}
\end{itemize}

\subsection{Competitive Environments}
Here we describe the competitive multi-agent environments used in this work.

\subsubsection{Capture the Flag}
\label{sec:background_environments_competitive_ctf}
\begin{figure}
    \centering
    \includegraphics[width=.5\linewidth]{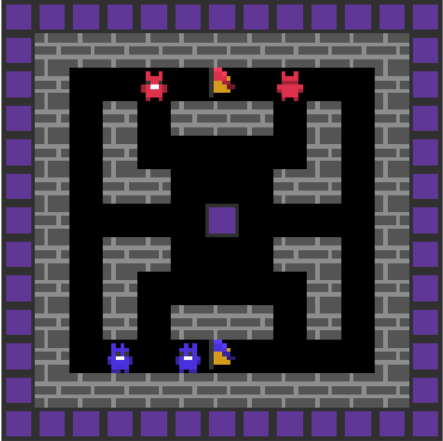}
    \caption{Capture the Flag Environment.}
    \label{fig:background_environments_ctf}
\end{figure}
The Capture the Flag game is shown in \cref{fig:background_environments_ctf}. Capture the Flag is a well-known competitive environment in the multi-agent literature~\citep{jaderberg2019human}. The game is played as follows: two multi-agent teams (red, blue) head-to-head in an arena-style environment. The red team is spawned at the top of the environment, and the blue team is spawned at the bottom. Both team have at their ``home base" a flag of their team's color. The goal of each team is to capture the opposing team's flag as many times as possible within the game's time horizon; while also preventing the opposing team from capturing the team's home flag. The primary defense mechanism for an agent is a ``zap" action that reduces an opponent agent's health if it is in the zap radius, and forces the agent to drop the flag (if it is holding one). In addition to zapping other agents, agents can choose a ``paint" action that colors the floor tiles with the painting agent's home color. Agents of the opposite color cannot move over painted tiles until they paint it their own color, so the painting mechanism can be used strategically to slow down the movement of opponents through certain corridors. There is also a set of purple tiles in the environment (one in the center, many surrounding the arena) that serves as an indicator of whether or not a flag has been taken from its home base. This tile serves as an observation for each agent, informing that agent of the current state of both team's flags (purple = no flags taken, blue = blue flag taken, red = red flag taken, gray = both flags taken).

Sophisticated Capture the Flag agents can learn a number of complex strategies, such as attacking an opponent's flag, battling opponent agents for territory, and base camping to protect the home flag. Agents are forced to learn these competitive behaviors from scratch and refine those behaviors during training. From the perspective of our work, therefore, Capture the Flag is a strong case study of the emergence of strategic behaviors over time. 

The state space, observation space, and action space are the same as outlined in the Collaborative Cooking environments in \cref{sec:background_environments_cooperative_cooking}. Some environment-specific details are outlined below:
\begin{itemize}
    \item \emph{States}: Grid cells can be filled by an agent or any of the following items: Wall, Red Flag, Blue Flag, Red Paint, Blue Paint, Flag Indicator Tile.
    \item \emph{Reward}: Agents receive a shared team reward for capturing the opposing team's flag (by running it back to base). Agents are also rewarded individually for picking up a flag from the opposing team's base, returning a previously stolen flag to base, and zapping an opponent flag carrier. Agents are therefore most strongly motivated to act offensively, but must learn to balance attacks with defensive strategy. In sum, we define the following reward for the task:
    \begin{equation*}
        r_t =
        \begin{cases}
            1, & \textrm{for capturing the opponent's flag.}. \\
            0.01, & \textrm{for picking up the opponent's flag.} \\
            0.01, & \textrm{for returning a flag to base.} \\
            0.01, & \textrm{for zapping an opponent flag carrier.} \\
            0, & \textrm{otherwise.}
        \end{cases}
    \end{equation*} 
\end{itemize}

\subsubsection{Board Games}
\label{sec:background_environments_board_games}
\begin{figure}
    \centering
    \includegraphics[width=.95\linewidth]{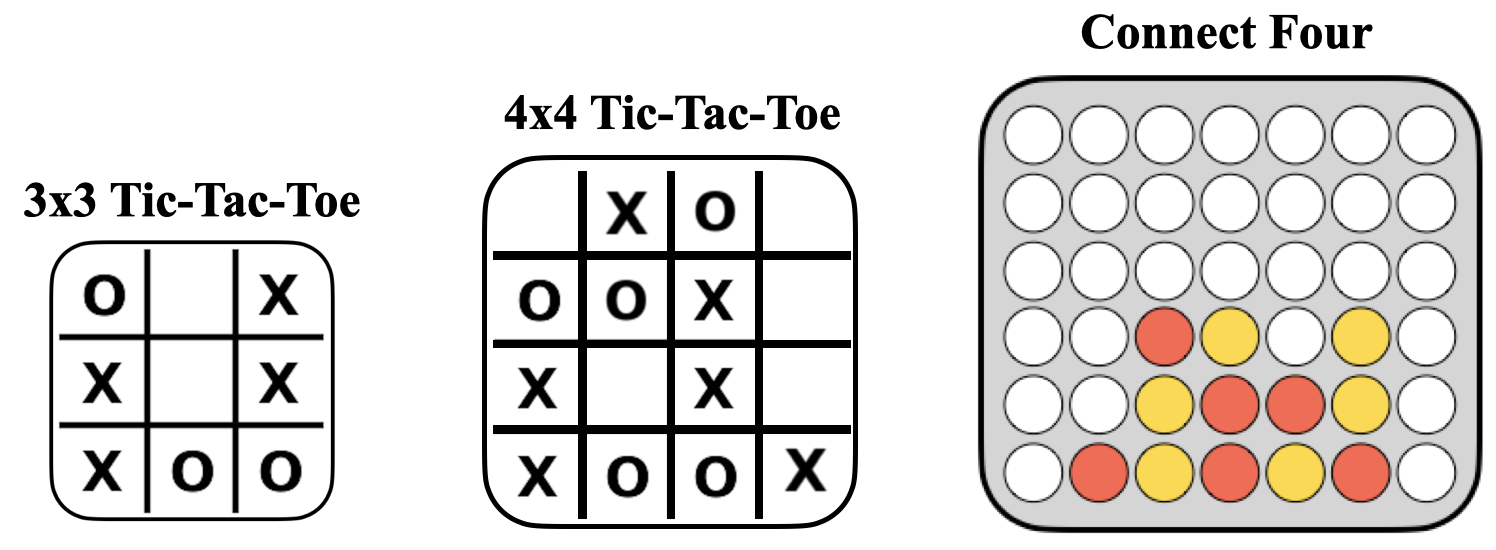}
    \caption{Board games studied in this thesis.}
    \label{fig:background_environments_board_games}
\end{figure}
Historically, board games have proven to be a fertile ground for testing and advancing artificial intelligence (AI) systems. The deterministic rules and clear win-loss conditions of these games offer a structured environment where AI can learn, experiment, and improve over time. The first notable success came when IBM's DeepBlue defeated the reigning world chess champion, Garry Kasparov, in 1997 \cite{campbell2002deep}. This was a landmark achievement, demonstrating the potential of AI in complex decision-making tasks. Decades later, Google's AlphaGo mastered the ancient Chinese board game of Go, a game with a complexity far exceeding that of chess \cite{silver2017alphago}. These victories underscored the power of AI and its capabilities in strategic thinking and decision-making.

Moreover, board games are fundamentally multi-agent environments---multiple players interact and make decisions that affect the game's outcome. This makes them an excellent testbed for studying multi-agent systems, which are critical in understanding how AI can function in environments where multiple, possibly competing, agents coexist.  In this work, three primary board games are studied experimentally (see \cref{fig:background_environments_board_games}): tic-tac-toe, a variant of tic-tac-toe on a 4x4 grid, and Connect Four.

Tic-tac-toe is a two-player game played on a 3x3 grid. Players take turns placing their mark (either an 'X' or an 'O') in an empty square. The objective is to place three of one's own marks in a horizontal, vertical, or diagonal row before the opponent can do the same. The game ends when either player achieves this goal or when all squares are filled, resulting in a draw. The 4x4 variant follows the same rules on a larger board.

Connect Four is a two-player game played on a 6x7 grid. Players take turns dropping a colored disc into one of the vertical columns, with the objective of forming a horizontal, vertical, or diagonal line of four of one's own discs. The game ends when either player achieves this goal or when all positions are filled, which results in a draw.

These games, with their simple rules but deep strategic possibilities, are an excellent platform for studying and understanding the dynamics of AI in multi-agent environments.

\part{Interpreting Emergent Multi-Agent Behavior}
\label{part:interpretabiltiy}
In this first part of the thesis, we focus on \textbf{interpreting} emergent multi-agent behavior. Drawing inspiration from the supervised learning literature, we will seek to answer both \textbf{high-level} strategic questions, such as:
\begin{itemize}
    \item \textit{In cooperative environments, do agents learn to coordinate or opt for less structured independent action?}
    \item \textit{In mixed-motive environments, do agents act selfishly or in a manner that improves social welfare?}
\end{itemize}
and low-level behavioral questions, such as:
\begin{itemize}
    \item \textit{To what extent do agents exchange information implicitly through their respective action spaces?}
    \item \textit{Can we measure the influence of one agent's behavioral cues on its teammates?}
\end{itemize}

In \cref{chapter:implicit_comm}, we address the low-level through a study of emergent implicit communication. We highlight the importance of recognizing communication as a spectrum and show that multi-agent teams, not unlike humans or animals, learn to rely upon implicit signals from their teammates when coordinating.

In \cref{chapter:interpretability}, we target high-level, concept-based explanations of emergent multi-agent behavior. We introduce an intrinsically interpretable policy architecture for MARL agents and show experimentally that it can reliably detect emergent coordination, expose sub-optimal behaviors (e.g., lazy agents), identify complex inter-agent social dynamics, and strategic behaviors as they emerge throughout training.

\chapter{Emergent Implicit Signaling}
\label{chapter:implicit_comm}

\section{Introduction}
Communication is a critical scaffolding for coordination. It enables humans and animals alike to coordinate on complex tasks, synchronize plans, allocate team resources, and share missing state. Understanding the process through which communication emerges has long been a goal of philosophy, linguistics, cognitive science, and AI. Recently, advances in multi-agent reinforcement learning (MARL) have propelled computational studies of emergent communication that examine the representations and social conditions necessary for communication to emerge in situated multi-agent populations \cite{lazaridou2020emergent}. Existing approaches, targeting language-like communication, have shown that it is possible for agents to learn protocols that exhibit language-like properties such as compositionality \cite{chaabouni2020compositionality,resnick2019capacity} and Zipf's Law \cite{chaabouni2019anti} when given additional learning biases \cite{eccles2019biases}. 

\begin{figure}
  \centering
  \includegraphics[width=.9\linewidth]{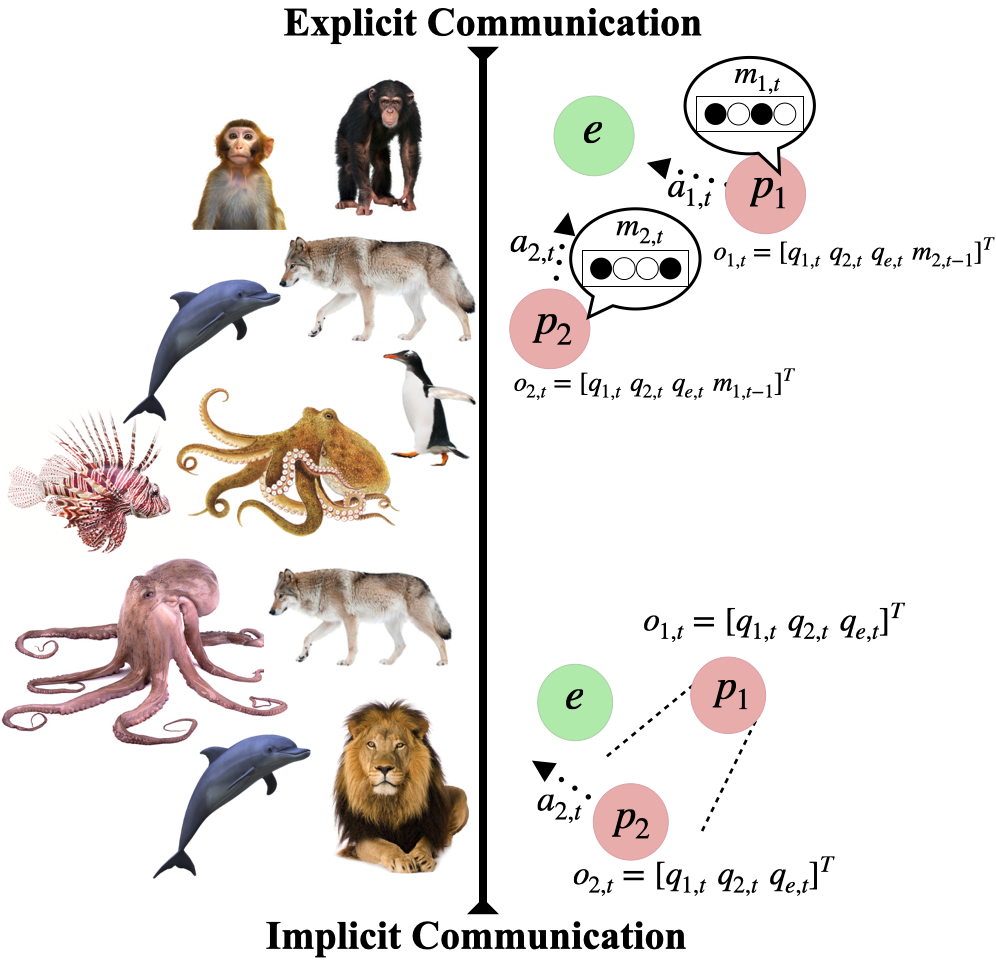}
  \caption{Left: Animal communication spans a complete spectrum from explicit to implicit communication. Right: A pursuit-evasion game with two pursuers $p_1$ and $p_2$ and an evader $e$ (with positions $q_{1,t}$, $q_{2,t}$, and $q_{e,t}$ at time $t$). The emergent communication literature has typically considered only language-like communication (top), where agents share information over specialized sender-receiver architectures---e.g. $p_1$ receives the extended observation $o_{1,t} = [q_{1,t} \ q_{2,t} \ q_{e,t} \ m_{2, t-1}]^T$ and selects actions using the extra information $m_{2, t-1}$ that $p_2$ chose to share at the previous time-step (and vice versa for $p_2$). With implicit signals (bottom), $p_1$ selects a movement action $a_{1,t}$ based only on the observation $o_{1,t} = [q_{1,t} \ q_{2,t} \ q_{e,t}]^T$. The only information $p_1$ gets from $p_2$ is generated by $p_2$'s physical position (and vice versa for $p_2$). We posit that implicit signals are an important step towards learning richer communication.}
  \label{fig:implicit_comm_spectrum}
\end{figure}

Though emergent communication is fundamentally an \textit{ab initio} approach compared to top-down language learning \cite{vaswani2017attention}, most recent methods have targeted protocols with complex structure and representational capacity, like that of human language \cite{lazaridou2020emergent, lowe2019pitfalls}. Such techniques equip agents with specialized sender-receiver architectures and allow them to exchange bits of information (or continuous signals) through these channels. One notable exception is the work of \citet{jaques2019social}, which considers how influential agents can pass information to other agents through their actions in the environment. We build on these ideas, recognizing the full spectrum of communication that exists in nature and proposing the study of lower-level forms of communication that do not require specialized architectures as a first step towards more sophisticated communication (see \cref{fig:implicit_comm_spectrum}).

Multi-agent cooperation in nature yields a wide range of communication protocols that vary in structure and complexity. In animal communication \cite{bradbury1998principles}, for example, reef-dwelling fish use body shakes \cite{vail2013referential, bshary2006interspecific} and octopuses punch collaborators \cite{sampaio2020octopuses}---forms of non-verbal communication---whereas chimps \cite{boesch1989hunting}, macaques \cite{mason1962communication}, and gentoo penguins \cite{choi2017group} each maintain a diverse vocal repertoire. Importantly, animals also exchange information through non-explicit channels. In the Serengeti, for example, lions use the posture of fellow pride members to stalk prey covertly \cite{schaller2009serengeti}. Moreover, wolves \cite{peterson2003wolf} and dolphins \cite{quick2012bottlenose}, both frequent vocal communicators, communicate implicitly during foraging---adjusting their group formation based on the position and orientation of other pack members \cite{herbert2016understanding}. \citet{breazeal2005effects} define such forms of communication, called \textit{implicit communication}, as "conveying information that is inherent in behavior but which is not deliberately communicated". Studies have shown that implicit cues are used frequently by humans as well. Human teams rely on gaze, facial expressions, and non-symbolic movement during cooperative tasks \cite{entin1999adaptive} including collaborative design \cite{wuertz2018design}, crowdsourcing \cite{zhou2017collaboration}, and multiplayer games (e.g. Hanabi \cite{liang2019implicit, bard2020hanabi}).

The role of implicit communication in teamwork has been studied extensively in the human-computer \cite{liang2019implicit, schmidt2000implicit} and human-robot interaction \cite{che2020efficient, breazeal2005effects, knepper2017implicit, butchibabu2016implicit} communities, and even in multi-agent systems \cite{pagello1999cooperative, gildert2018need}; but is less well-understood in the context of emergent agent-agent interactions \cite{de2010evolution}. This is due in large part to the anthropomorphized nature of standard definitions of implicit communication. Note that the key difference between explicit communication (a deliberate information exchange) and implicit communication (a non-deliberate exchange) in the definition above is \textit{intent} \cite{breazeal2005effects, gildert2018need}. Understanding implicit communication in the context of artificial agents requires (i) ascribing human-like mental states to artificial agents and (ii) giving agents the ability to reason about each other's mental states with respect to a common goal (\textit{\`a la} \citet{cohen1991teamwork}); which is a highly non-trivial undertaking.

In this work, we take a first step in this direction by examining the extent to which implicit communication aides coordination in multi-agent populations. Following the work of \citet{de2010evolution} in evolutionary settings, we study \textit{implicit signals}: signals generated by the physical position of an agent as it interacts with its environment. Implicit signals are observed passively by an agent's teammates, and so require neither special sender-receiver architectures nor reasoning about an agent's mental states. For RL agents, we distinguish implicit signals from explicit communication through each agent's actions and observations. As a simple example, consider the environment in \cref{fig:implicit_comm_spectrum} where, at time $t$, agents $p_1$, $p_2$, and $e$ are described by their positions $q_{1,t}$, $q_{2,t}$, and $q_{e,t}$ respectively. With implicit signals, $p_1$ selects a movement action $a_{1,t}$ based on the observation $o_{1,t} = [q_{1,t} \ q_{2,t} \ q_{e,t}]^T$---i.e. the only information $p_1$ gets from $p_2$ is the information contained in $p_2$'s physical position; and vice versa. With explicit communication, $p_1$ emits a communicative action $m_{1,t}$ in addition to $a_{1,t}$ (and same for $p_2$). In this case, $p_1$ receives the extended observation $o_{1,t} = [q_{1,t} \ q_{2,t} \ q_{e,t} \ m_{2, t-1}]^T$ and can select actions using the extra information $m_{2, t-1}$ that $p_2$ chose to share at the previous time-step.

Experimentally, we study coordination and the role of implicit communication with pursuit-evasion games \cite{isaacs1999differential}. Pursuit-evasion simulates an important coordination task, foraging, where communication is especially impactful. To accurately model the conditions under which coordination and communication emerges in the animal kingdom, we prioritize decentralized multi-agent learning. Despite the challenges posed by decentralized approaches \cite{foerster2016learning}, this feature is critical in the emergence of animal communications and so we focus on this important class of multi-agent learning frameworks to exploit inspiration from the natural world. Moreover, to encourage teamwork, we target pursuit-evasion games in which the pursuers are slower than the prey (coordination is not needed with the pursuers have a velocity advantage). Similar to \citet{lowe2017multi}, we study pursuit-evasion as a MARL problem. 

We show that, under these conditions, naive decentralized learning fails to solve the pursuit-evasion task. There are a number of challenges that lead to this outcome. First, decentralized learning is non-stationary \cite{lowe2017multi}---with multiple agents learning simultaneously, the learning problem is less stable for each agent individually. Next, since we prioritize tasks that require teamwork (e.g. pursuer speed less than evader speed), learning from scratch is difficult, as agents are forced to learn how to interact with the environment individually while simultaneously learning to coordinate with their teammates. Finally, pursuit-evasion is a sparse reward task---pursuers receive reward only when capturing the evader. During the early stages of training, when each agent is following a randomly initialized policy, it is virtually impossible for the agents to coordinate effectively enough to experience positive reward.

To address these challenges, we introduce a curriculum-driven approach for solving difficult, sparse reward coordination tasks. Our strategy, combining ideas from the ecological reinforcement learning \cite{co2020ecological} and automatic curriculum learning \cite{portelas2020automatic} literature, is motivated by the following insights: (i) we can create a \textit{curriculum over task difficulty by shaping the environment} to the current skill level of the cooperative agents; (ii) we can aide agents in learning multi-agent team behaviors by forcing them to \textit{learn single-agent behaviors first}. More specifically, we first adopt an environment shaping curriculum using velocity bounds, which allows pursuers to gather initial experience at a velocity greater than that of the evader, and then tune their strategies over time as velocity decreases. Then, we introduce a curriculum over agent behavior that warm-starts cooperative multi-agent learning by first seeding each agent with experience of successful single-agent behaviors.

Our first result confirms the importance of curriculum-driven multi-agent learning. We compare the performance of our strategy to ablations over the multi-agent curricula individually and show that, using our strategy, decentralized agents learn to solve difficult coordination tasks, such as the pursuit-evasion game.

Next, we stress-test both the strength of coordination and capacity for implicit signaling of policies learned with our method. We compare the performance of multi-agent coordination learned with our method against a set of analytical and learned strategies that represent ablations over both coordination and implicit signaling---i.e. outperforming one of these methods is equivalent to surpassing that method's level of sophistication in coordination, capacity for implicit signaling, or both. 

Empirical results show that our method significantly outperforms highly sophisticated multi-agent strategies. Coordination strategies learned with our method successfully complete the pursuit-evasion task at a \textbf{speed ratio of 0.5} (i.e. the pursuers are moving at half the evader's speed), whereas each of the baseline methods fails to complete the pursuit-evasion task when pursuer speed is only slightly below evader speed (\textbf{speed ratio of 0.8)}.

Further, we hone in on the nature of each strategy's coordination more directly by examining agent behavior during the most important subsets of the trajectory---those that immediately precede the accumulation of reward. In the context of pursuit-evasion, we examine the distribution of pursuer locations relative the evader at the time of capture. By comparing rotational symmetry and rotational invariance in the agents' capture distributions, we find that our strategy learns structured coordination while simultaneously allowing pursuers to make dynamic adjustments to their position relative the evader to successfully achieve capture.

Finally, we examine the role of implicit signals (as introduced by \citet{de2010evolution}) more directly using measures of social influence. Influence measures such as Instantaneous Coordination \cite{jaques2019social} quantify the extent to which one agent's behavior influences its teammates. We repurpose this method to measure the exchange of implicit signals as \textit{position-based social influence}. Results of this study show that pursuers trained with our strategy exchange \textbf{up to 0.375 bits of information per time-step} on average, compared to a maximum of \textbf{only 0.15 bits on average} from the baseline methods. This indicates that our method learns, and relies heavily on, the exchange of implicit signals.

\noindent \textbf{Preview of contributions}
Our work is summarized by the following contributions:
\begin{enumerate}
    \item We highlight the importance of studying implicit signals as a form of emergent communication.
    \item We introduce a curriculum-driven learning strategy that enables cooperative agents to solve difficult coordination tasks with sparse reward.
    \item We show that, using our strategy, pursuers learn to coordinate to capture a superior evader, significantly outperforming sophisticated analytical and learned methods. Coordination strategies learned with our method complete the pursuit-evasion task at a speed ratio of $0.5$ (i.e. half the evader's speed), whereas each of the competing methods fails to complete the pursuit-evasion task at any speed ratio below $0.8$.
    \item We examine the use of implicit signals in emergent multi-agent coordination through position-based social influence. We find that pursuers trained with our strategy exchange up to 0.375 bits of information per time-step on average, compared to a maximum of only 0.15 bits on average from competing methods, indicating that our method has learned, and relies heavily on, the exchange of implicit signals.
\end{enumerate}

\section{Related Work}
\paragraph{Emergent Communication}
Emergent communication examines the process by which cooperative agents learn communication protocols as a means to completing difficult multi-agent tasks \cite{lazaridou2020emergent}. Recent work has shown that MARL agents converge upon useful communication in referential games \cite{havrylov2017emergence} and can even develop language-like properties such as compositionality \cite{chaabouni2020compositionality,resnick2019capacity} and Zipf's Law \cite{chaabouni2019anti} when exposed to additional structural learning biases. More recently, this class of techniques has expanded to include complex situated environments \cite{das2019tarmac}, high-dimensional observations \cite{cowen2020emergent}, and the negotiation of belief states \cite{foerster2018bayesian}. Further work has shown that influential communication can be incentivized through additional learning objectives (i.e. inductive biases) \cite{eccles2019biases} and reward shaping \cite{jaques2019social}.

Implicit interactions are less well-studied in the emergent communication literature, despite their importance to teamwork, as outlined in the human-computer \cite{liang2019implicit, schmidt2000implicit} and human-robot interaction \cite{che2020efficient, breazeal2005effects, knepper2017implicit, butchibabu2016implicit} communities. Implicit communication has been studied for cooperative multi-agent tasks \cite{gildert2018need, grover2010implicit, pagello1999cooperative}, though not in the context of emergent behavior. Though some studies have shown that agents can learn to communicate non-verbally through actions and gestures \cite{mordatch2017emergence}, such forms of action-space communication \cite{baker2019emergent} are examples of non-verbal explicit communication, whereas our goal is to study completely implicit signals. Most similar to our work is that of \citet{de2010evolution}, which examines both implicit and explicit communication in evolutionary settings. Our work can be interpreted as bridging these ideas with the MARL literature.

\paragraph{Multi-Agent Reinforcement Learning}
Multi-agent reinforcement learning (MARL) encompasses a large body of literature extending RL techniques to multi-agent settings. In general, MARL algorithms subscribe to either decentralized or centralized learning. With decentralized learning (or independent learning), each agent is responsible for updating its own policy (or Q-network) individually \cite{tan1993multi}. Though recent work has shown that decentralized learning is feasible in complex environments \cite{de2020independent}, it is often challenging, due to the non-stationary (and therefore unstable) nature of the learning problem \cite{laurent2011world, matignon2012independent}. Centralized techniques stabilize multi-agent learning by allowing agents to share a joint Q-network during training, then act independently at test-time \cite{lowe2017multi, foerster2018counterfactual}. Centralization has also proven useful for emergent communication, as agents can share gradients directly through their communication channels \cite{foerster2016learning}. Despite these benefits, we prioritize decentralized learning because it more accurately represents the learning problem that humans and animals face in the real-world.

Recently, significant effort has been aimed at connecting ideas from the curriculum learning literature \cite{bengio2009curriculum} to RL \cite{portelas2020automatic}. Such methods have derived curricula from virtually every aspect of the RL problem; including reward shaping \cite{bellemare2016unifying, pathak2017curiosity}, modifying initial state distributions \cite{florensa2017reverse}, and procedurally-generating sub-tasks \cite{portelas2020teacher, risi2019procedural}. Some multi-agent curricula have been shown to lead to more generally-capable RL agents \cite{team2021open}. Further work has examined the environment's role in generating curricula, leading to new methods such as environment shaping \cite{co2020ecological} and unsupervised environment design \cite{dennis2020emergent}. Of particular relevance is the work of \citet{co2020ecological}, which considers curricula that manipulate the dynamics of the RL environment to the benefit of learning agents. Our work combines a velocity-based variant of environment shaping with an additional curriculum for bootstrapping multi-agent learning with single-agent experience, which itself is inspired by the strategy employed by \citet{yang2018cm3} for multi-goal learning.

\section{Preliminaries}
\label{sec:implicit_comm_prelims}

\subsection{Potential Field Navigation}
Given an agent $i$ with position $q_i$, we can define a potential function $U(q_i, q_{\textrm{goal}})$ between $q_i$ and a target point $q_{\textrm{goal}}$ such that the negative gradient $F(q_i, q_{\textrm{goal}}) = -\nabla U(q_i, q_{\textrm{goal}})$ specifies a control law for $i$'s motion. For example, let $U_{\textrm{att}}(q_i, q_{\textrm{goal}})$ be a quadratic function of distance between $q_i$ and a target point $q_{\textrm{goal}}$:
\begin{equation}
    \label{eqn:implict_comm_attractive_potential}
    U_{\textrm{att}}(q_i, q_{\textrm{goal}}) = \frac{1}{2}k_{\textrm{att}} \, d(q_i, q_{\textrm{goal}})^2
\end{equation}
where $k_{\textrm{att}}$ is an attraction coefficient and $d(,)$ is a measure of distance. The resulting force exerted on agent $i$ is:
\begin{equation}
    F_{\textrm{att}} = -\nabla U_{\textrm{att}}(q_i, q_{\textrm{goal}}) = -k_{\textrm{att}}(q_i - q_{\textrm{goal}})
\end{equation}
\noindent In this work, the agent's action-space is defined in terms of headings, so only the \textit{direction} of this force impacts our agents.

\subsection{Pursuit-evasion}
Experimentally, we use the pursuit-evasion environment outlined in \cref{sec:background_environments_cooperative}. Here we provide additional details that are specific to this work. We assume the evader to be part of the environment, as defined by the potential function:
\begin{equation}
    \label{eqn:implicit_comm_evader_objective}
    U_\textrm{evade}(\theta_e) = \sum_i \bigg(\frac{1}{r_i}\bigg) \cos(\theta_e - \tilde{\theta}_i)
\end{equation}
where $r_i$ and $\tilde{\theta}_i$ are the L2-distance and relative angle between the evader and the $i$-th pursuer, respectively, and $\theta_e$ is the heading of the evader. This objective is inspired by theoretical analysis of escape strategies in the pursuit-evasion literature \cite{ramana2017pursuit}. Intuitively, \cref{eqn:implicit_comm_evader_objective} pushes the evader away from pursuers, taking the largest bisector between any two when possible. The goal of the pursuers---to capture the evader as quickly as possible---is mirrored in the reward function, where $r(s_t, a_t) = 50.0$ if the evader is captured and $r(s_t, a_t) = -0.1$ otherwise. A derivation of \cref{eqn:implicit_comm_evader_objective} and additional environmental details are provided in \cref{apdx:implicit_comm_apdx_details}.

\subsection{Implicit Communication}
For clarity, we restate the definitions of implicit vs. explicit communication from \citet{breazeal2005effects} and implicit signals from \citet{de2010evolution}, which we adopt throughout this work.

\begin{definition}[Explicit Communication]
    \label{def:implicit_comm_explicit}
    Communication that is ``deliberate where the sender has the goal of sharing specific information with the collocutor" \cite{breazeal2005effects}.
\end{definition}
\begin{definition}[Implicit Communication]
    \label{def:implicit_comm_implicit_comm}
    Communication that conveys ``information that [is] inherent in behavior but which is not deliberately communicated" \cite{breazeal2005effects}.
\end{definition}
\begin{definition}[Implicit Signal]
    \label{def:implicit_comm_implicit_signal}
    The ``signal that is generated by the actual physical position of the [agents] and that is detected by the other [agents]" \cite{de2010evolution}.
\end{definition}

In the context of RL agents, an implicit signal refers to positional information observed by an agent's teammates as part of the environmental state $s_t$ or observation $o_t$ space, whereas explicit communication involves sending/receiving messages over a dedicated communication channel. We note that \cref{def:implicit_comm_explicit} and \cref{def:implicit_comm_implicit_comm} rely on deliberation (or intention), whereas \cref{def:implicit_comm_implicit_signal} does not. For the purposes of our study---exploring the first-step in a bottom-up approach to emergent communication---we therefore focus on implicit signals as a form of emergent implicit communication.

\subsection{Instantaneous Coordination}
Instantaneous Coordination (IC) is a measure of social influence between agents \cite{jaques2019social}. IC is defined for two agents $i$ and $j$ as the mutual information $\textrm{I}(a_i^t; a_j^{t+1})$ between $i$'s action at time $t$ and $j$'s action at the next time-step. Formally, assuming agent $i$'s actions are drawn from the random variable $A_i$ with marginal distribution $P_{A_i}$ (and similarly for agent $j$), we can rewrite IC using the standard definition of mutual information as the Kullback-Leibler divergence between the joint distribution and the product of the marginals:
\begin{align*}
    I(A_i; A_j) &= D_{\textrm{KL}}(P_{A_i A_j} || P_{A_i} \times P_{A_j}) \\
    &= \sum_{\substack{a_i \in \mathcal{A}_i,\\ a_j \in \mathcal{A}_j}} P_{A_i A_j}(a_i, a_j) \log \bigg( \frac{P_{A_i A_j}(a_i, a_j)}{P_{A_i}(a_i) \times P_{A_j}(a_j)} \bigg)
\end{align*}
where $\mathcal{A}_i$ and $\mathcal{A}_j$ are the spaces over $A_i$ and $A_j$, respectively. Intuitively, high IC is indicative of influential behavior, while low IC indicates that agents are acting independently. We highlight that, in the absence of explicit communicative actions, IC is a measure of implicit signals (as defined in the previous section). This is because agent $i$'s action at time $t$ inherently dictates $i$'s position $q_i$ at time $t+1$, which is observed by $i$'s teammates. IC in this context is therefore a measure of position-based social influence. Since this work considers deterministic policies and uses IC at test-time, IC is computed over Monte-Carlo estimates of the relevant distributions.

\section{Curriculum-Driven DDPG}
\label{sec:implicit_comm_cd_ddpg}
Our goal is to learn multi-agent coordination with decentralized training and, in doing so, explore the role of implicit signaling in teamwork. However, there are a number of issues with decentralized learning in difficult, sparse reward environments, such as our pursuit-evasion game. First, though setting $\lvert \Vec{v}_p \rvert \leq \lvert \Vec{v}_e\rvert$ is important for studying teamwork, it places the pursuers at a severe deficit. In the early stages of training---since each pursuer's action selection is determined by the randomly initialized weights of its policy network---the chance of slower pursuers capturing the evader defined in \cref{eqn:implicit_comm_evader_objective} is extremely low. The pursuers are unlikely to obtain a positive reward signal, which is vital to improving their policies. This issue is exacerbated by non-stationarity. In the case of decentralized DDPG, multiple agents learning in the same environment causes the value of state-action pairs for any one agent (as judged by its Q-function) to change as a result of policy updates of \textit{other} agents. This non-stationarity leads to higher-variance gradient estimates and unstable learning. Though recent advances in ``centralized training, decentralized execution" help in such cases \cite{lowe2017multi}, they violate our goal of decentralized learning.

To address these challenges, we introduce a \textit{curriculum-driven method} for decentralized multi-agent learning in sparse reward environments. Our strategy is motivated by the following principles: (i) we can create a \textit{curriculum over task difficulty by shaping the environment} to the current skill level of the cooperative agents; (ii) we can aide agents in learning multi-agent team behaviors by forcing them to \textit{learn single-agent behaviors first}. More specifically, we combine an environment shaping curriculum over agent velocities with a behavioral curriculum for bootstrapping cooperative multi-agent learning with successful single-agent experience. We refer to this curriculum-driven variant of DDPG as CD-DDPG throughout the rest of this paper.

\subsection{Velocity Ratio Curriculum}
Curriculum learning \cite{bengio2009curriculum} is a popular technique for solving complex learning problems by breaking them down into smaller, easier to accomplish tasks. Recent advances in ecological reinforcement learning \cite{co2020ecological} have introduced environment shaping---in which properties, initial states, or the dynamics of an environment are modified gradually---as a useful and more natural instantiation of curriculum learning for RL agents (than, say, traditional reward shaping). With this in mind, we construct a sequence of increasingly difficult pursuit-evasion environments by incrementally lowering the ratio of pursuer speed to evader speed ($\lvert \Vec{v}_p \rvert / \lvert \Vec{v}_e \rvert$).

More formally, we define a curriculum over velocity bounds. Let $\Vec{v}_0$ be an initial setting of the environment's velocity ratio $\lvert \Vec{v}_p \rvert / \lvert \Vec{v}_e \rvert$. We anneal $\lvert \Vec{v}_p \rvert / \lvert \Vec{v}_e \rvert$ to a target ratio $\Vec{v}_{\textrm{target}}$ over $v_{\textrm{decay}}$ epochs as:
\begin{equation}
    \label{eqn:implicit_comm_curriculum}
    \Vec{v}_{i} \gets \Vec{v}_{\textrm{target}} + (\Vec{v}_0 - \Vec{v}_{\textrm{target}})*\max\bigg(\frac{(v_{\textrm{decay}}-i)}{v_{\textrm{decay}}}, 0.0\bigg)
\end{equation}
\noindent where $i$ represents the current training epoch and, in turn, $\Vec{v}_{i}$ the current velocity ratio. In practice, we initialize the environment such that $\Vec{v}_0 > 1.0$ (i.e. pursuers are faster than the evader), then anneal this ratio slowly as training progresses. This gives the pursuers an opportunity enjoy a velocity advantage early on in training, then develop increasingly coordinated strategies to capture the evader as $\lvert \Vec{v}_p \rvert / \lvert \Vec{v}_e \rvert$ decays.

\subsection{Behavioral Curriculum}
We combine the aforementioned velocity curriculum with a behavioral curriculum that extends off-policy learning to allow for targeted single-agent exploration early in the training process. As a consequence of the deterministic policy gradient and importance sampling \cite{silver2014deterministic}, the policy gradient can be estimated in an off-policy manner---i.e. using trajectories sampled from a separate behavior policy $\beta(a \mid s)$ where $\beta(a \mid s) \neq \mu_\phi$. Formally, this means \cref{eqn:background_ddpg_grad} can be represented equivalently as:
\begin{align*}
    J_\beta(\phi) &= \int_\mathcal{S} p^\beta Q_{\omega}(s,a) \mid_{s=s_t, a=\mu_{\phi}(s_t)} \textrm{d}s
    \\[4pt]
    &= \mathbb{E}_{s \sim p^\beta}[Q_{\omega}(s,a) \mid_{s=s_t, a=\mu_{\phi}(s_t)} ]
\end{align*}
\noindent with the corresponding gradient:
\begin{equation*}
    \nabla_\phi J_\beta(\phi) = \mathbb{E}_{s \sim p^\beta} [\nabla_a Q_{\omega}(s,a)\mid_{s=s_t, a=\mu(s_t)}\nabla_\phi \mu(s)\mid_{s=s_t}]
\end{equation*}
\noindent where $p^\beta$ is the state distribution of the behavior policy $\beta$.

Our behavioral curriculum takes advantage of off-policy learning by splitting training into two exploration phases that use distinct behavior policies $\beta_0$ and $\beta_{\mu}$, respectively. 
The key is that we define $\beta_0$ strategically to be a supervisory policy that collects \textit{successful single-agent experience} and $\beta_{\mu}$ to be a standard exploration policy ($\epsilon-$greedy for discrete actions, random noise for continuous actions). In multi-agent settings, defining $\beta_0$ this way allows each agent to learn how to interact with the environment first, before learning to coordinate with teammates. Without this curriculum, agents are forced to learn both simultaneously.

In our pursuit-evasion environment, we define $\beta_0$ as follows:
\begin{equation}
    \label{eqn:implicit_comm_beta_0}
     \beta_0 = -\nabla U_{\textrm{att}}(q_{p_i}, q_{e})
\end{equation}
\noindent where $U_{\textrm{att}}$ is the attractive potential-field defined in \cref{eqn:implict_comm_attractive_potential}. Note that $\beta_0$ is a greedy policy that runs directly towards the evader. This strategy is obviously sub-optimal when $\lvert \Vec{v}_p \rvert / \lvert \Vec{v}_e \rvert \leq 1.0$, but helps pursuers learn to move in the direction of the evader when $\lvert \Vec{v}_p \rvert / \lvert \Vec{v}_e \rvert > 1.0$. In our experiments, we use $\beta_0$ only during the first phase of the velocity curriculum. After this phase, agents follow the standard DDPG behavior policy $\beta_\mu = \mu_\phi(s_t) + \mathcal{N}$, where $\mathcal{N}$ is the Ornstein-Uhlenbeck noise process. 

\section{Results}
\label{sec:implicit_comm_results}
Our evaluation addresses two primary questions: (i) Does our curriculum-driven strategy enable decentralized agents to learn to coordinate in difficult sparse reward environments? (ii) To what extent does implicit signaling emerge in the learned strategy? To answer these questions, we first perform an ablation study over each of the multi-agent curricula that comprise CD-DDPG. Then, we measure the performance of CD-DDPG against a set of analytical and learned strategies of increasing sophistication. We intentionally select these strategies to represent ablations over both coordination and implicit signaling.

\subsection{Multi-Agent Curricula}
\label{sec:implicit_comm_results_ablation}
\begin{figure}[]
    \centering
    \makebox[\linewidth][c]{\includegraphics[width=0.95\linewidth]{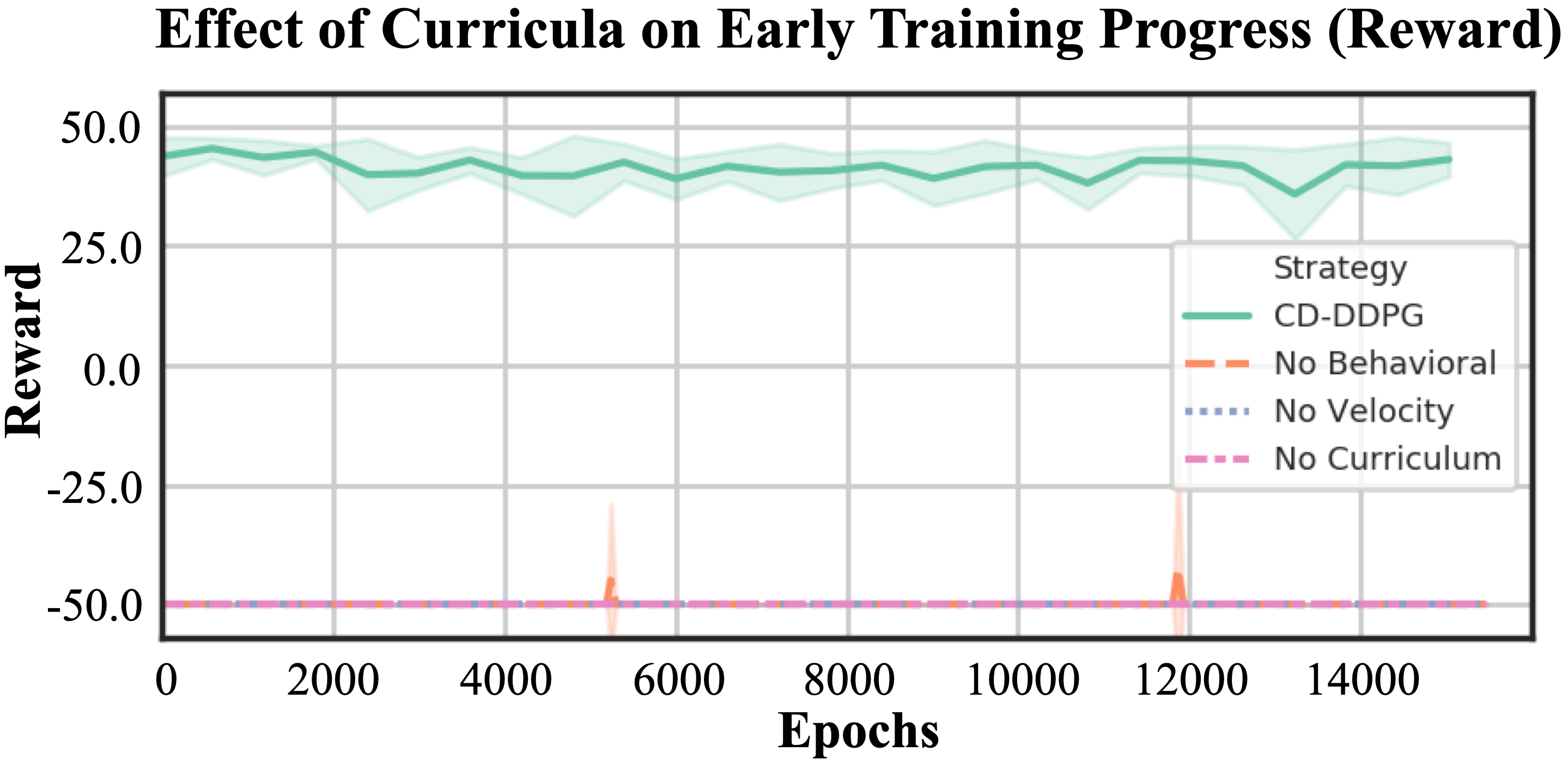}}
    \caption{An ablation study of curriculum-driven learning early in the training process. Note that $\lvert \Vec{v}_p \rvert / \lvert \Vec{v}_e \rvert$ is decreasing over time (increasing the task difficultly), so learning is characterized by \textit{sustained} reward rather than \textit{improved reward}. Only CD-DDPG successfully learns team coordination.}
    \label{fig:implicit_comm_ablation}
\end{figure}
To isolate our method's performance, we run ablations over each of CD-DDPG's constituent parts. We compare CD-DDPG to the following alternatives: \textit{No Behavioral}, which follows the velocity curriculum, but not the behavioral one; \textit{No Velocity}, which follows the behavioral curriculum, but not the velocity one; and \textit{No Curriculum}, which uses neither curriculum during training (i.e. it is vanilla DDPG trained at a constant ratio $\lvert \Vec{v}_p \rvert / \lvert \Vec{v}_e \rvert = 0.7$). We trained each method for $15,000$ epochs in our pursuit-evasion environment across 10 different random seeds. Results are shown in \cref{fig:implicit_comm_ablation}.

\textit{No Curriculum}, unsurprisingly, flat-lines throughout the training process. It never experiences a positive reward signal, reflecting the difficulty of pursuit-evasion when $\lvert \Vec{v}_p \rvert / \lvert \Vec{v}_e \rvert < 1.0$. \textit{No Behavioral}, trained from a velocity ratio of $\lvert \Vec{v}_p \rvert / \lvert \Vec{v}_e \rvert = 1.2$ downwards, experiences small spikes in reward, but not consistently enough to improve the pursuers' policies. Finally, \textit{No Velocity} also fails to gain any traction (it's curve is behind No Curriculum in the figure), highlighting the importance of both curricula working together. This validates our intuition from \cref{sec:implicit_comm_cd_ddpg}. Even with a velocity advantage, the agents struggle to learn individual interaction with the environment and coordination with their teammates simultaneously. \textit{CD-DDPG} is able to capture valuable experience even in the earliest stages of training. This bootstraps each pursuer's learning process, allowing them to maintain a high level of performance even after the warm-up is over. We therefore find that the combination of both curricula is crucial to warm-starting policy learning.

\subsection{Properties of Emergent Teamwork}
\begin{table}
  \caption{A summary of the strategies used to evaluate CD-DDPG and their capacity for coordination and implicit signaling. Outperforming one of these methods is equivalent to surpassing its level of sophistication in coordination, implicit signaling, or both.}
  \label{table:implicit_comm_eval_straties}
  \centering
  \begin{tabular}{lll}
    \toprule
    Name     & Coordination     & Implicit Signaling \\
    \midrule
    Greedy & \hspace{20pt}No & \hspace{40pt}No \\
    CD-DDPG (Partial) & \hspace{20pt}Yes & \hspace{40pt}No \\
    Pincer & \hspace{20pt}Yes & \hspace{40pt}Yes \\
    \bottomrule
  \end{tabular}
\end{table}
To study the strength of CD-DDPG's coordination further and identify the role of implicit signaling in the agents' emergent teamwork, we compare the performance of CD-DDPG to a set of analytical and learned strategies of increasing sophistication. These strategies represent ablations over both coordination and implicit signaling---i.e. outperforming one of these methods is equivalent to surpassing that method's level of sophistication in coordination, capacity for implicit signaling, or both. With this in mind, we evaluate CD-DDPG against the following policies, which are also summarized in \cref{table:implicit_comm_eval_straties} \footnote{We do not evaluate against centralized methods such as MADDPG \cite{lowe2017multi} because they violate our requirement of decentralized learning.}:

\paragraph{Greedy}
Each pursuer follows the greedy control strategy in \cref{eqn:implict_comm_attractive_potential}. In Greedy pursuit, each pursuer ignores the positions of its teammates. Greedy pursuit therefore represents independent action (i.e. no coordination, no communication).

\paragraph{CD-DDPG (Partial)}
We train CD-DDPG under partial observability. Instead of the complete environment state $s_t$, each pursuer $p_i$ receives a private observation $o_t = \{q_{p_i}, q_e \}$ consisting of its own location $q_{p_i}$ and the location of the evader $q_e$. Despite not observing each other, CD-DDPG (Partial) pursuers are capable of coordinating through static role assignment. This is equivalent to assigning roles before each trajectory---i.e. $p_1$ always approaches from the left, $p_2$ from the right, etc.---and coordinating through these roles during pursuit. CD-DDPG (Partial) pursuers are therefore coordinated, but with no ability to communicate implicitly to modify their behavior extemporaneously.

\paragraph{Pincer}
We define the Pincer strategy as an adversarial function that exploits knowledge of the evader's objective in \cref{eqn:implicit_comm_evader_objective}:
\begin{equation}
    \label{eqn:implicit_comm_pincer}
    F(\boldsymbol{\tilde{\theta}_i, r_i}) = \underset{\boldsymbol{\tilde{\theta}_i, r_i}}{\max} \big[ \underset{\theta_e}{\min} \big[ U_\textrm{evade}(\theta_e)\big] \big] = \underset{\boldsymbol{\tilde{\theta}_i, r_i}}{\max} \bigg[ \underset{\theta_e}{\min} \bigg[ \sum_i \bigg(\frac{1}{r_i}\bigg) \cos(\theta_e - \tilde{\theta}_i) \bigg] \bigg]
\end{equation}
\noindent where $\boldsymbol{\tilde{\theta}_i}$ and $\boldsymbol{r_i}$ are the polar coordinates of each pursuer relative the evader. Intuitively, Pincer encircles the evader and cuts off potential bisector escape paths while enclosing the circle. It therefore supports both coordination---pursuers uphold a circular formation---and implicit signaling---every time-step, pursuers use information from the locations of their teammates to adjust their own position on the circular formation. We provide additional details on the Pincer strategy in \cref{apdx:implicit_comm_details_pincer}.

All experiments involve $n=3$ pursers. For CD-DDPG and CD-DDPG (Partial), agents are trained for $50,000$ epochs. The velocity ratio is decreased from $\lvert \Vec{v}_p \rvert / \lvert \Vec{v}_e \rvert = 1.2$ to $\lvert \Vec{v}_p \rvert / \lvert \Vec{v}_e \rvert = 0.4$ in decrements of $0.1$ (i.e. eight separate training sessions each). Velocity ratio decay occurs over $v_{\textrm{decay}}=15000$ epochs, during which the behavioral curriculum $\beta_0$ is used (but only during the first training session). After training, we test the resulting policies at each velocity step. Test-time performance is measured over 100 independent trajectories (averaged over five random seeds each).

\subsubsection{Capture Success}
\begin{figure}
  \centering
  \includegraphics[width=.99\linewidth]{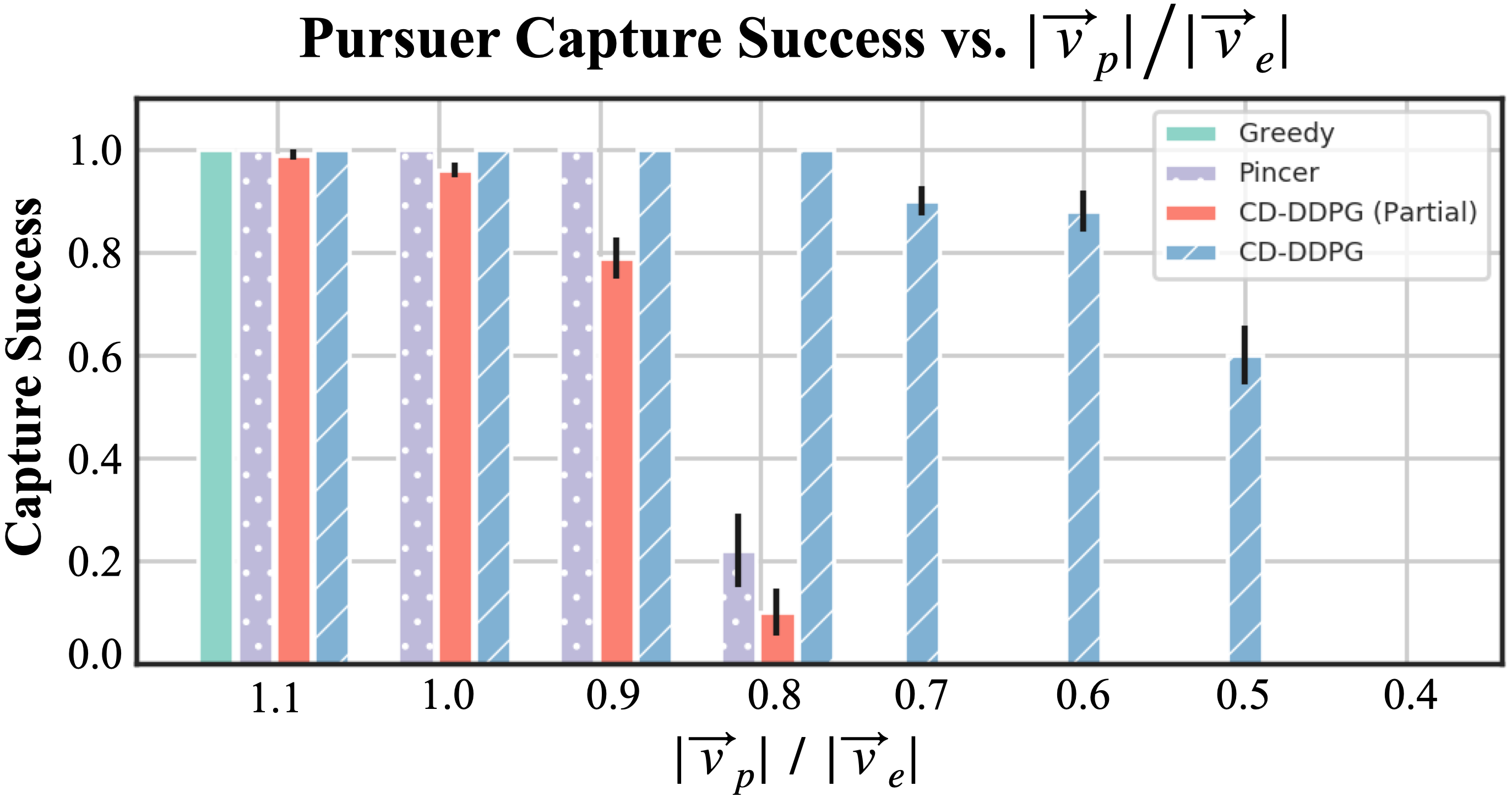}
  \caption{Capture success as a function of velocity. CD-DDPG succeeds at $\lvert \Vec{v}_p \rvert / \lvert \Vec{v}_e \rvert = 0.5$ whereas each competing method fails below $\lvert \Vec{v}_p \rvert / \lvert \Vec{v}_e \rvert = 0.8$.}
  \label{fig:implicit_comm_cap_success}
\end{figure}
First, we evaluate capture success as a function of the velocity advantage of the evader. The results are shown in \cref{fig:implicit_comm_cap_success}. Unsurprisingly, each method has a high capture success rate when $\lvert \Vec{v}_p \rvert / \lvert \Vec{v}_e \rvert > 1.0$. The Greedy strategy drops off at $\lvert \Vec{v}_p \rvert / \lvert \Vec{v}_e \rvert = 1.0$, which is also expected---a straight-line chase only works when $\lvert \Vec{v}_p \rvert / \lvert \Vec{v}_e \rvert > 1.0$. The Pincer and CD-DDPG (Partial) strategies are able to coordinate successfully at lower speeds, but eventually \textbf{fail to capture the evader at $\boldsymbol{\lvert \Vec{v}_p \rvert / \lvert \Vec{v}_e \rvert = 0.8}$} and below. As $\lvert \Vec{v}_p \rvert / \lvert \Vec{v}_e \rvert$ decreases further, CD-DDPG significantly outperforms the other strategies. Specifically, CD-DDPG successfully completes the pursuit-evasion task at a \textbf{speed ratio of 0.5} (i.e. pursuers moving at half the evader's speed). These results show that CD-DDPG learns to coordinate significantly more effectively than the other strategies.

\subsubsection{Relative position during capture}
\begin{figure}[t!]
    \centering
    \makebox[\linewidth][c]{\includegraphics[width=0.8\linewidth]{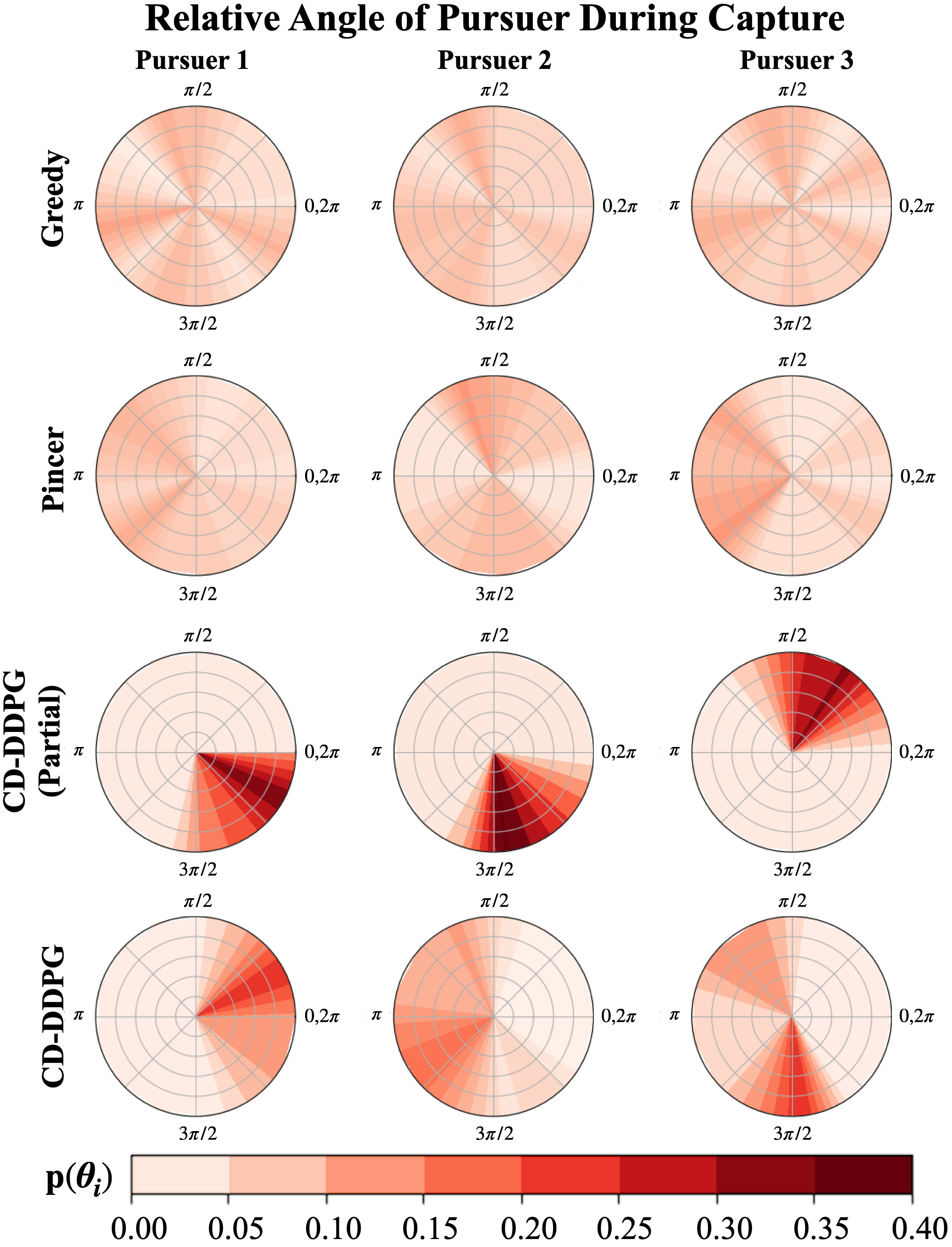}}
    \caption{Distribution of relative angle for each pursuer during capture, binned in the range $[0, 2\pi]$ and displayed as a heatmap. Each row represents a pursuit strategy.}
    \label{fig:implicit_comm_capture_pos}
\end{figure}
Next, we study the nature of each strategy's coordination with finer granularity. In sparse reward tasks, the most important time-steps are those that immediately precede reward. For this reason, we look at the distribution of pursuer locations at the time of capture. We collect 100 successful trajectories from each strategy and compute the distribution of pursuer positions relative the evader. In general, we expect coordinated pursuit to exhibit rotational symmetry during capture. Rotational symmetry suggests that pursuers have learned strategies which lead them to well-defined ``capture points" around the evader. Conversely, rotational invariance is indicative of independent pursuit---i.e. pursuers do not follow concrete patterns of attack. Results for this study are shown in \cref{fig:implicit_comm_capture_pos}.

We find that Greedy and Pincer both yield uniform capture distributions. This is unsurprising for Greedy pursuers, whose pursuit paths are not effected by their teammates. The Pincer strategy encircles the evader, but does not constrain pursuers to specific locations on the circle. This leads to less-structured capture(i.e. weaker role assignment) by the Pincer pursuers. In contrast, CD-DDPG (Partial) pursuers demonstrate very strong role assignment, with each pursuer capturing the evader from the same relative angle each time. Taken into context with the results from \cref{fig:implicit_comm_cap_success}, it is clear that this level of rotational symmetry impacts success. In fact, it is an example of over-commitment to role assignment. The pursuers adopt very constrained roles---e.g. ``$p_1$ always move left", ``$p_2$ always move right"---which works when $\lvert \Vec{v}_p \rvert / \lvert \Vec{v}_e \rvert \geq 1.0 $, but fails at lower velocities. CD-DDPG balances rotational symmetry and invariance. Each pursuer follows a unique angle towards the evader, but does not commit to that angle completely. CD-DDPG therefore learns structured coordination while allowing pursuers to make dynamic adjustments to their position relative the evader to achieve capture.

\subsubsection{Position-based Social influence}
\begin{figure}
  \centering
  \includegraphics[width=.85\linewidth]{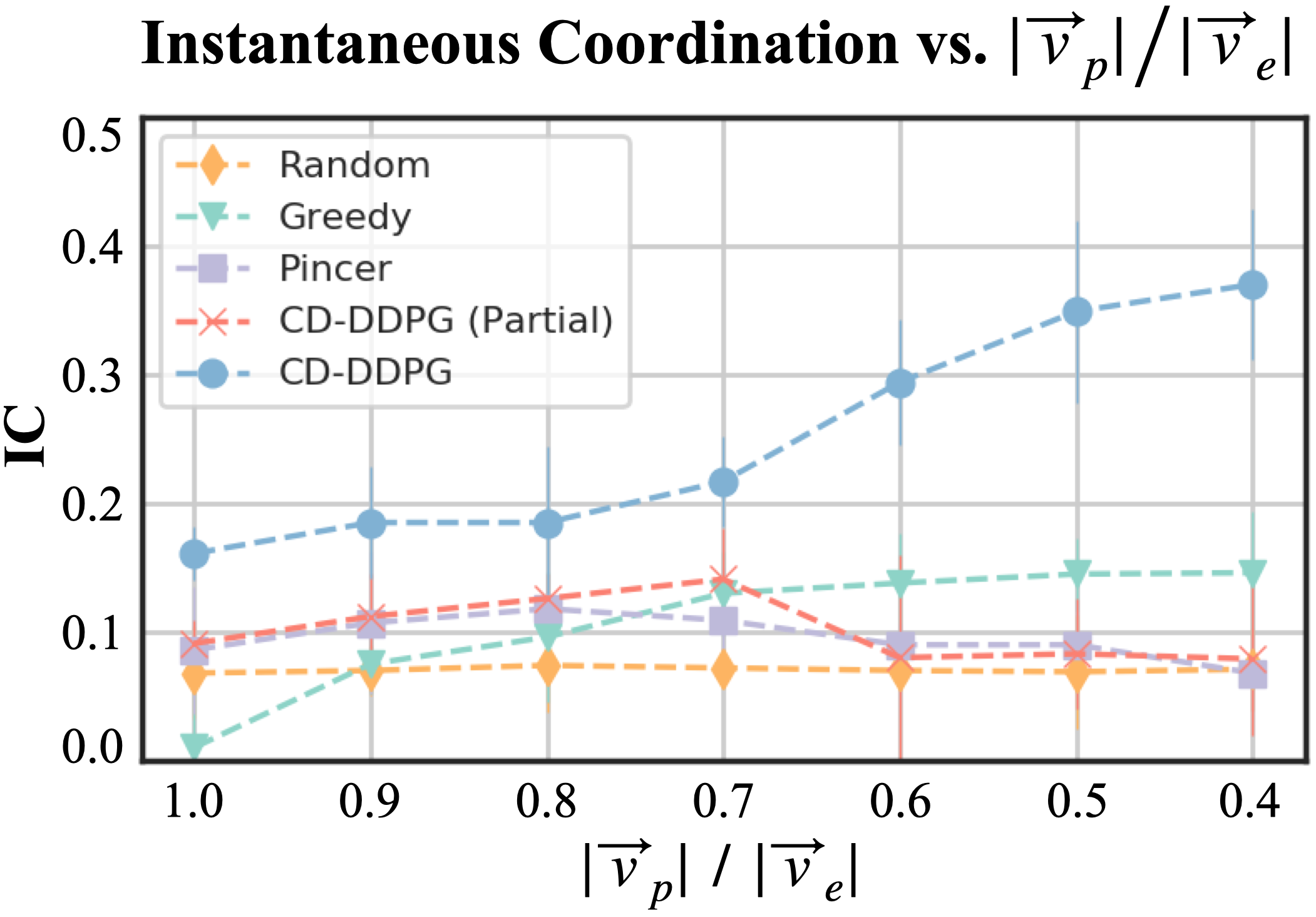}
  \caption{Instantaneous Coordination (IC) results as a function of velocity ratios. Agents trained with CD-DDPG exchange up to 0.375 bits of information per time-step on average, increasing as $\lvert \Vec{v}_p \rvert / \lvert \Vec{v}_e \rvert$ drops, whereas competing strategies peaks remains constant.}
  \label{fig:implicit_comm_ic_results}
\end{figure}
To further study the role of implicit signaling in pursuer performance, we compute the IC score for each strategy. As noted in \cref{sec:implicit_comm_prelims}, by measuring the amount that one agent's actions (and therefore its next position) influences the actions of its teammates, IC quantifies the exchange of implicit signals (see \cref{def:implicit_comm_implicit_signal}) amongst teammates. Following \citet{jaques2019social}, we compute IC empirically as a Monte-Carlo approximation over multi-agent trajectories. We average influence across all trajectory steps and for each agent-agent pair. We also evaluate pursuers that act randomly, which provides a baseline for independent action. The results are shown in \cref{fig:implicit_comm_ic_results}.

We find that, as $\lvert \Vec{v}_p \rvert / \lvert \Vec{v}_e \rvert$ decreases, the IC levels attained by CD-DDPG increases significantly, whereas it remains stagnant for other methods. In fact, across all $\lvert \Vec{v}_p \rvert / \lvert \Vec{v}_e \rvert$ levels, we find that pursuers trained with \textbf{CD-DDPG exchange up to 0.375 bits of information} per time-step on average, compared to a maximum of \textbf{only 0.15 bits on average} from the baseline methods. This indicates that CD-DDPG \textit{achieves increasingly complex coordination} as task difficulty increases and is a promising sign that the CD-DDPG team is exchanging implicit signals---i.e. each pursuer is responding to positional information from its teammates. Finally, we note a minor (though surprising) increase in coordination for the greedy pursuers at low velocities. This is an artifact we dub ``phantom coordination" and discuss further in the \cref{apdx:implicit_comm_apdx_qual_phantom}.

\subsubsection{High-influence moments}
\begin{figure}[t!]
    \centering
    \makebox[\linewidth][c]{\includegraphics[width=0.85\linewidth]{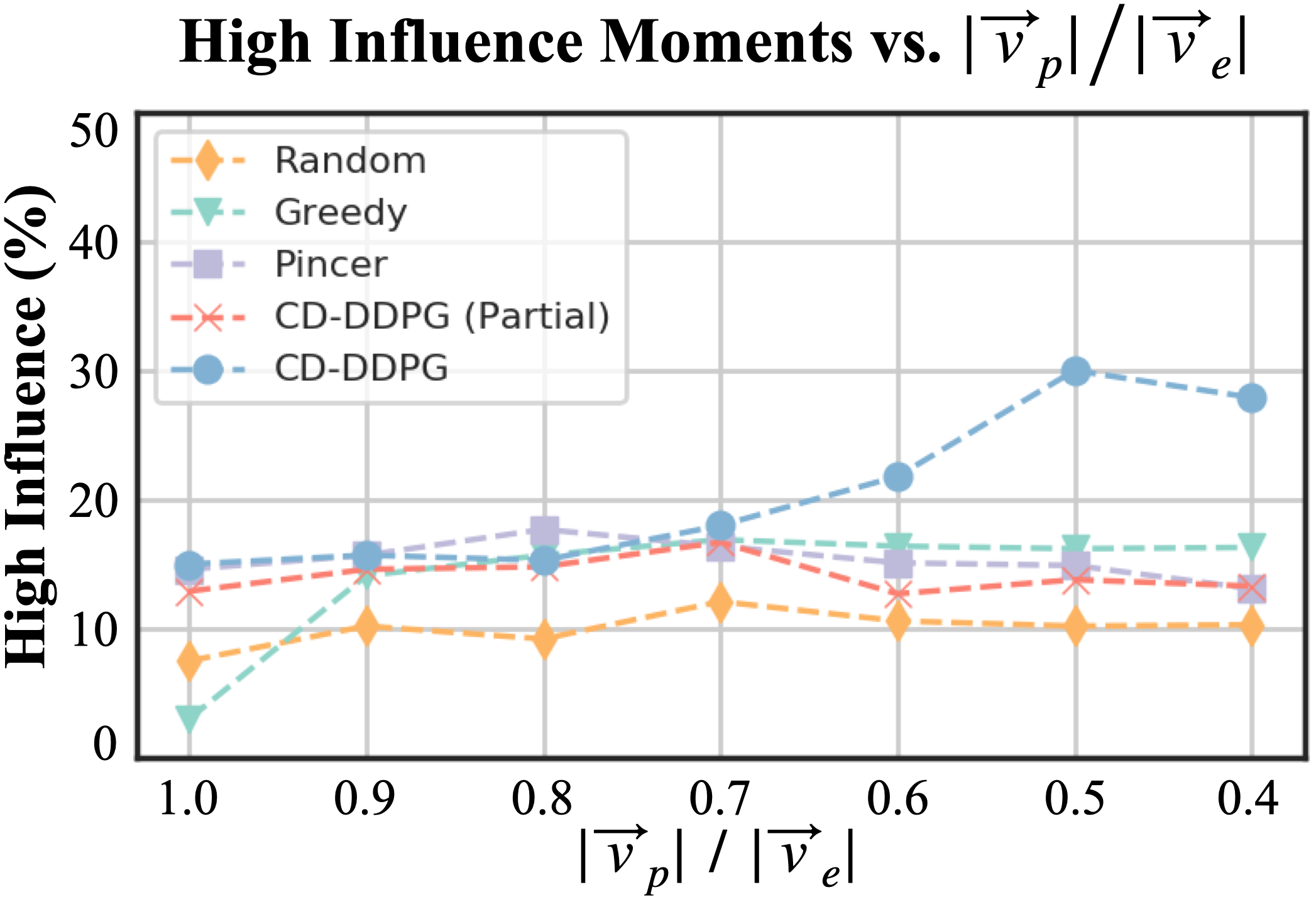}}
    \caption{Percentage of high-influence moments as a function of velocity ratios. Agents trained with CD-DDPG become more influential as $\lvert \Vec{v}_p \rvert / \lvert \Vec{v}_e \rvert$ decreases.}
    \label{fig:implicit_comm_high_ic}
\end{figure}
We also report on the percentage of high-influence moments that occur between pairs of agents (see \cref{fig:implicit_comm_high_ic}). A high-influence moment is a time-step in which IC is above the mean IC for the entire trajectory. Similar to previous work \cite{jaques2019social}, we find that influence is sparse in most cases---\textbf{only between 10-15\%} of trajectory steps exhibit high-influence across all $\lvert \Vec{v}_p \rvert / \lvert \Vec{v}_e \rvert$ levels. The exception, notably, occurs for CD-DDPG. At low speeds (i.e. $\lvert \Vec{v}_p \rvert / \lvert \Vec{v}_e \rvert \leq 0.7$), we see a significant increase in the percentage of high-influence moments between the CD-DDPG pursuers, reaching a \textbf{maximum near 30\%}. This is further evidence that, as $\lvert \Vec{v}_p \rvert / \lvert \Vec{v}_e \rvert$ decreases, CD-DDPG pursuers form increasingly highly-coordinated formations and make split-second decisions based on the movements of their teammates. This points more concretely to the use of implicit signals between CD-DDPG pursuers.

Interestingly, the behavior of CD-DDPG closely matches the documented behaviors of social predators such as dolphins and wolves---i.e. sudden changes of position/orientation as a response to the movements of other teammates \cite{herbert2016understanding}. We elaborate on these and other qualitative findings in \cref{apdx:implicit_comm_apdx_qual}.

\section{Conclusion}
\label{sec:implicit_comm_conclusion}
This work connects emergent communication to the spectrum of communication that exists in nature, highlighting the importance of interpreting communication as a spectrum from implicit to explicit communication. We proposed a curriculum-driven strategy for policy learning in difficult multi-agent environments. Experimentally, we showed that our curriculum-driven strategy enables pursuers to coordinate and capture a superior evader, outperforming other highly-sophisticated analytic pursuit strategies. We also provided evidence suggesting that the emergence of implicit signaling is a key contributor to the success of this strategy. There are a number of extensions of this work that study how common principles contribute to integrated, communicative behavior; including: imperfect state information, increased environmental complexity, and nuanced social dynamics between agents.

\chapter{Concept-based Understanding of Emergent Multi-Agent Behavior}
\label{chapter:interpretability}

\section{Introduction}
Multi-agent learning continues to play a crucial role in the development of scalable and generally-capable AI systems. In addition to well-known successes in board games \citep{perolat2022mastering, silver2016mastering, silver2017mastering} and online games \citep{berner2019dota, vinyals2019grandmaster}, multi-agent learning has enabled agents to develop a range of capabilities, such as navigating social dilemmas~\citep{leibo2017multi}, optimizing traffic flow for autonomous vehicles~\citep{vinitsky2023optimizing, wu2021flow}, and even learning to cooperate with humans~\citep{carroll2019utility}. 
For all its success, interpreting emergent multi-agent behavior remains an open challenge. In cooperative environments, for example, reward alone is not enough to understand the nature of a learned coordination strategy, or whether agents have learned to coordinate at all. In more complex settings, such as social dilemmas, agents often learn nuanced relationships that include both low-level spatial (e.g., navigation) and high-level social (e.g., exploitation, public resource allocation) interactions that can only be fully-understood through exhaustive visualization.

For this reason, recent work has shifted focus from traditional measures of performance (i.e. reward, human analysis) to better understanding emergent behaviors~\citep{omidshafiei2022beyond}. Related methods perform interpretability analysis in a \emph{post-hoc} manner, either by measuring basic behavioral statistics~\citep{liu2022embodied}, or regressing concepts from policy network activations~\citep{mcgrath2022acquisition}. An alternative to post-hoc analysis is \emph{intrinsic interpretability}, where a model's decisions incorporate human-understandable concepts directly. In supervised learning settings, for example, it has been shown that it is possible to train a neural network that is ``bottlenecked" by human-understandable concepts, enabling intrinsic interpretability while maintaining high performance in classification tasks~\citep{koh2020concept}.
\begin{figure}
    \centering
    \includegraphics[width=0.99\textwidth]{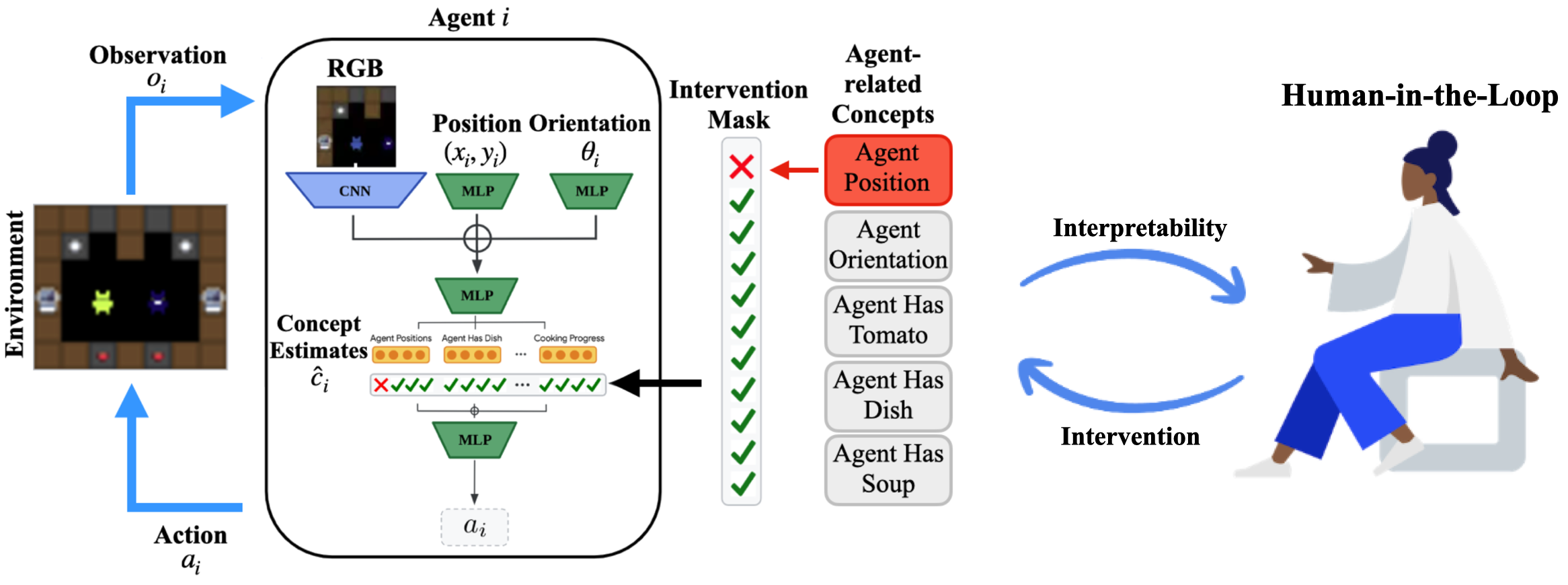}
    \caption{Concept Bottleneck Policies distill agent decisions into human-understandable concepts, yielding interpretable policies for MARL and enabling behavioral analysis via concept intervention.}
    \label{fig:interpretability_concept_bottleneck_policies}
\end{figure}

In this work, we introduce an intrinsically-interpretable, concept-based policy architecture for multi-agent learning. Our proposed architecture, the Concept Bottleneck Policy (CBP), forces an agent to make decisions by first predicting an intermediate set of human-understandable concepts, then using those concepts to select actions. This architecture makes decision-making immediately transparent, as agents reveal the environmental and inter-agent factors influencing each of their actions.

Beyond intrinsic interpretability on a per-agent basis, our method also \textbf{unlocks a novel class of techniques for understanding emergent multi-agent behavior}. By conditioning each agent's actions on its own concepts from construction, our method supports behavioral analysis via concept intervention. Specifically, we can perform interventions over each concept estimate in an agent's CBP and examine its impact on the larger multi-agent system (with respect to rewards, actions, environment features, etc). We use this technique to uncover many key aspects of multi-agent behavior.

In cooperative environments, our method can reliably detect agents who have learned to coordinate from those that act independently, determine which environments require coordination to solve vs. those that permit independent solutions, expose inter-agent factors that drive coordination, and even diagnose common failure modes such as lazy agents. In competitive environments, our method disentangles more sophisticated strategic behavior (e.g., role assignment) as it emerges during training. In social dilemmas, we show that it is possible to identify inter-agent social dynamics (e.g., exploitation, public resource sharing) by learning a graph over intervention outcomes. Specifically, we learn a sparse graph representing the pairwise relationships of inter-agent intervention effects and show that it captures both high-level behavioral patterns (free loader agents vs. public service agents) and low-level spatial information (agents who interact). Finally, we show that CBPs perform comparably to non-concept based policies, thus achieving interpretability without sacrificing performance.

In sum, our contributions are as follows: (i) We introduce Concept Bottleneck Policies as an interpretable, concept-based architecture for MARL; (ii) We introduce a suite of novel techniques for multi-agent behavioral understanding that leverage concept intervention and graph learning; (iii) We show that CBPs can match the performance of non-concept-based network architectures.

\section{Related Work}
Our work lies at the intersection of interpretability and MARL.
In interpretability, computer vision works have used saliency maps to help provide pixel-level explanations for model predictions ~\citep{simonyan2013deep,smilkov2017smoothgrad,sundararajan2017axiomatic,springenberg2014striving}.
Interpretability using grounded concepts has also been explored to provide more meaningful, human-understandable explanations of model decisions~\citep{kim2018interpretability,ghorbani2019towards,bau2020understanding,ghandeharioun2021dissect,chen2020concept,yeh2020completeness,stammer2021right}.
As described earlier, concept bottleneck models~\citep{koh2020concept} constrain the networks through such grounded concepts, enabling analysis of the model via intervention.

RL interpretability techniques include those that either increase the transparency of agent decision-making directly~\citep{bastani2017interpretability,madumal2020explainable}; or analyze agent behaviors from a post-hoc perspective~\citep{yang2018learn,amir2018highlights,sequeira2020interestingness,liu2022motor,jaderberg2019human,omidshafiei2022beyond,mcgrath2022acquisition,forde2022concepts}.
Approaches that directly increase transparency include those that combine model compression and decision trees to yield more interpretable agent policies~\citep{bastani2017interpretability}, or rely upon causal decision trees with limited depth to explore causality of agent decisions~\citep{madumal2020explainable}. 
However, these approaches have not yet been extended to RL settings with complex observations (e.g., images), where function approximators are typically needed. Post-hoc methods have used traditional saliency-mapping techniques to attribute decisions to image observation regions~\citep{yang2018learn}, highlighting states that lead to major differences in agent behaviors~\citep{amir2018highlights}, and summarized key agent behaviors using measures such as action uncertainty~\citep{sequeira2020interestingness}.
Increasingly, natural language has been used to aid human understanding of agents through instructions~\citep{ahn2022can} or explanations~\citep{hayes2017improving}.

Interpretability of MARL agents has received limited attention in prior works, with primary investigations gauging agents' internal representations by making predictions about future outcomes~\citep{liu2022motor}, visualizing latent clusterings of agents' neural activation vectors~\citep{jaderberg2019human, mahajan2019maven}, or using offline behavioral analysis to learn behavioral spaces over agents~\citep{omidshafiei2022beyond}.
Recent work has also analyzed decision-making in two-player games such as Chess~\citep{mcgrath2022acquisition}, and Hex~\citep{forde2022concepts}.
In contrast to our work, these approaches focus on post-hoc interpretability, whereas ours makes agent decisions directly transparent and enables novel forms of downstream analysis. To the best of our knowledge, ours is the first work to identify coordination, lazy agents, emergent skill learning, and social behaviors via intervention.

\section{Background}
\label{sec:interpretability_background}
\textbf{Markov Games}
A partially-observable Markov game \citep{littman1994markov} of $N$ agents is defined by the tuple $\mathcal{M} = (\mathcal{S}, \boldsymbol{\mathcal{A}}, \mathcal{T}, \boldsymbol{\mathcal{R}}, \boldsymbol{\Omega}, \mathcal{O})$ where $\mathcal{S}$, $\boldsymbol{\Omega}$, and $\boldsymbol{\mathcal{A}}$ are the game's global state-space, joint observation-space, and joint action-space, respectively. In each state, each agent $i$ selects an action $a_i \in \mathcal{A}_i$, yielding a joint action $\boldsymbol{a} = (a_1, ..., a_N)$ for all agents. Following action selection, the environment transitions to a new state according to the transition function $\mathcal{T}: \mathcal{S} \times \boldsymbol{\mathcal{A}} \rightarrow \mathcal{S}$, produces new observations for each agent following the observation function $\mathcal{O}: \mathcal{S} \times \boldsymbol{\mathcal{A}} \rightarrow \boldsymbol{\Omega}$, and emits a reward defined by the reward function $\boldsymbol{\mathcal{R}}: \mathcal{S} \times \boldsymbol{\mathcal{A}} \rightarrow \mathbb{R}^N$. The actions of each agent are dictated by a policy $\pi_i (a_i | o_i)$ and the collection of all individual policies $\boldsymbol{\pi} = (\pi_1, ..., \pi_N)$ is called the joint policy.

\textbf{Concept Bottleneck Models}
CBMs are a class of intrinsically-interpretable neural network architecture for supervised learning~\citep{koh2020concept}. CBMs make predictions by first estimating a set of human-understandable concepts, then producing an output based on those concepts--i.e., the network is ``bottlenecked" by concepts. Formally, given a dataset $\{(x^{(j)}, y^{(j)}, c^{(j)})\}_{j=1}^n$ consisting of inputs $x \in \mathbb{R}^d$, outputs  $y \in \mathbb{R}$, and human-understandable concepts $c \in \mathbb{R}^k$ (where $k$ is the number of unique concepts), a concept bottleneck model learns two mappings: $f: \mathbb{R}^d {\rightarrow} \mathbb{R}^k$ from input-space to concept-space, and $g: \mathbb{R}^k {\rightarrow} \mathbb{R}$ from concept-space to output-space. The model can then make predictions $\hat{y} = f(g(x))$ as a composition of those mappings. 
For example, in an application of arthritis detection from MRI images, the model first predicts related concepts such as knee joint spacing and the presence of bone spurs, then uses those concepts to classify arthritis severity.

\section{Concept Bottleneck Policies for MARL}
\label{sec:interpretability_method}
Here we introduce Concept Bottleneck Policies (CBPs) as an intrinsically interpretable, concept-based policy learning method for MARL (see \Cref{fig:interpretability_concept_bottleneck_policies}). First, we extend the Markov game formalism to support concept-based learning. Then we introduce the CBP architecture and show how it fits within any policy learning scheme. Finally, we demonstrate our method's potential for interpreting emergent multi-agent behavior, introducing multiple behavioral analysis techniques that use intervention to detect a variety of multi-agent social phenomena.

\subsection{Concept-based Markov Games}
We extend Markov games in two important ways. First, we assume that, in addition to a state space $\mathcal{S}$, there exists an interpretable concept-state space $\mathcal{C}$, where each concept-state $\boldsymbol{c} \in \mathcal{C}$ is a vector of human-understandable concepts that describe key features of the environment (e.g., agent positions, object states). This assumption holds in many RL domains like Atari~\citep{mnih2013playing} and Melting Pot~\citep{leibo2021scalable}. Second, we allow agents to generate estimates $\boldsymbol{\hat{c}}$ of the concept state as part of their internal representation. In sum, at each time-step $t$, the environment produces both a state $s_t$ and concept-state $\boldsymbol{c}_t$, and each agent $i$ selects both actions $a_{i,t}$ and concept estimates $\boldsymbol{\hat{c}}_{i,t}$.

\subsection{Concept Bottleneck Architecture}
There are many ways in which concepts can be modeled within the RL framework. Inspired by CBMs~\citep{koh2020concept}, we examine policy architectures in which an agent's action is conditioned entirely on its own concept estimates. To this end, we factorize an agent $i$'s policy $\pi_i(a_i|o_i)$ into a function $f_i: \mathcal{O}_i \rightarrow \mathcal{C}_i$ mapping observations $o_i$ to concept estimates $\boldsymbol{\hat{c}}_i$; and $\pi_i^{\text{act}}: \mathcal{C}_i \rightarrow \mathcal{A}_i$ mapping concept estimates $\boldsymbol{\hat{c}}_i$ to actions $a_i$. Composing $f_i$ and $\pi_i^{\text{act}}$ yields a standard policy mapping observations to actions: $\pi_i(a_i|o_i) = \pi_i^{\text{act}}(f(o_i)) = \pi^{\text{act}}(a_i|\hat{c}_i)$. Given an observation $o_{i,t}$ for some agent $i$ and time $t$, the bottleneck first produces concept estimates $\boldsymbol{\hat{c}}_{i,t} = f(o_{i,t})$, then uses those estimates alone to select an action $a_{i,t} \sim \pi_i^{\textrm{act}}$. Conditioning actions on concepts creates an intrinsically interpretable policy, as the policy is forced to provide a human-understandable rendering of the factors driving its decision making. Importantly, $\boldsymbol{\hat{c}_i}$ can be also used to interpret multi-agent behavior. For example, if two agents $i$ and $j$ collide while moving in the environment, examining $\boldsymbol{\hat{c}}_{i}$ and $\boldsymbol{\hat{c}}_{j}$ may identify which of the agents incorrectly modeled the location of its teammate.

\subsection{Concept Bottleneck Learning}
Under the CBP framing, learning a strong but interpretable policy requires both (i) learning to predict concepts accurately, and (ii) learning to select actions from those concepts effectively. To achieve the former, we introduce a  concept loss:
\begin{equation}
    \label{eqn:interpretability_concept_loss}
    L_C (\boldsymbol{c}, \boldsymbol{\hat{c}}) = \sum_{j=1}^{|C|} L_{C_j} (c_j, \hat{c}_j)
\end{equation}
\noindent where each component $L_{C_j}$ measures the error between the $j$'th predicted concept and its concept label. Each $L_{C_j}$ is a supervised loss, with a specific form dictated by the value of $c_j$---mean-squared error if $c_j$ is scalar, log loss if $c_j$ is binary, etc. In practice, each pair of concepts $\boldsymbol{c}$ and concept estimates $\boldsymbol{\hat{c}}$ are stored in an agent's replay buffer during training alongside standard $(s, a, r, s')$ tuples.

To learn a policy from concept estimates, we attach $L_C$ as an auxiliary loss to the reward-based loss defined by any base RL algorithm (e.g., PPO~\citep{schulman2017proximal}, etc). In general, if $L_{\textrm{RL}}$ is a generic reward-based loss, we construct a joint concept bottleneck policy loss:
\begin{equation}
     L_{CBP} = L_{\textrm{RL}} + \lambda L_C \, ,
\end{equation}
\noindent where the concept loss coefficient $\lambda$ weights the relative importance of concept prediction.

\subsection{Behavioral Analysis via Concept Intervention}
A key feature of our method is its support of intervention analysis. For an agent $i$, let $\textrm{Int}(\hat{c}_j, \bar{c}_j, i)$ be an intervention over the $j$'th concept estimate $\hat{c}_j$ from $i$'s bottleneck with the replacement value $\bar{c}_j$ (in the simplest case, $\bar{c}_j = 0$). We can then observe the effect, if any, that $\bar{c}_j$ has on the agent's behavior by comparing its impact on reward:
\begin{equation}
    \mathop{\mathbb{E}}\bigg [\sum_t r(s_t, \pi_i^{\text{act}}(\hat{c_j})) \bigg] - \mathop{\mathbb{E}}\bigg [\sum_t r(s_t, \pi_i^{\text{act}}(\bar{c_j})) \bigg]
\end{equation}
In mutual reward settings, this serves as a proxy for the intervention's impact on team behavior.\footnote{We carefully examine the possibility of OOD examples and propose ways to test this in \cref{apdx:interpretability_coordination_analysis}.}
Here we propose two examples of behavioral tests enabled by this approach:

\paragraph{Detecting Coordination:} Let $i$ and $j$ be two agents. For $i$ to coordinate with $j$, $i$ must condition its policy on information about $j$ ($j$'s position, orientation, etc). If the agents are coordinating and we remove $i$'s concept estimates pertaining to $j$, we should expect a decrease in reward. Conversely, if reward does not degrade, then $i$ and $j$ must not be explicitly coordinating (i.e., directly using signals from each other). By intervening over concepts pertaining to an agent's teammates, therefore, we can identify whether or not agents are coordinating.

\paragraph{Exposing Lazy Agents:} Here we define a lazy agent as one that does not contribute to increasing team reward through its own actions. For an agent to contribute, it must at least condition its policy on information about its own interactions with the environment. A lazy agent, therefore, is one that has learned a sub-optimal policy that does not encode such information, leading to unproductive behavior. It follows that we can test for the \emph{degree of laziness} of an agent by replacing its concept estimates \emph{about itself} and measuring the resulting degradation of team reward. If performance remains the same with the agent incapacitated, we conclude that it is a lazy agent.

\subsection{Modeling Inter-Agent Social Dynamics}
Here we show that it is possible to deepen our intervention-based interpretability by learning a sparse graph over CBP intervention outcomes. Specifically, we take each intervention $\textrm{Int}(\hat{c}_j, \bar{c}_j, i)$ to be a random variable that is described by some agent-related features, such as the average reward $\boldsymbol{r_{\hat{c}_j}} = \{r_{\hat{c}_j}^1, ..., r_{\hat{c}_j}^n \}$ each agent receives under the intervention\footnote{Outcomes can also be defined in terms of other features (e.g., resources collected, agent proximity, etc.)}. By performing $\textrm{Int}(\cdot)$ iteratively over each concept $c \in C$ and for each agent $i \in \{1, ..., N\}$, we can construct a matrix $\boldsymbol{X}$ wherein each row is the outcome of an intervention:
\begin{equation*}
    \boldsymbol{X}
    = 
    \begin{bmatrix}
        \textrm{Int}(\hat{c}_j, \bar{c}_j, 1) \\
        \textrm{Int}(\hat{c}_k, \bar{c}_k, 2) \\
        \vdots                      \\
        \textrm{Int}(\hat{c}_l, \bar{c}_l, N)
    \end{bmatrix}
    = 
    \begin{bmatrix}
        r_{\hat{c}_j}^1 & r_{\hat{c}_j}^2 & ... & r_{\hat{c}_j}^n \\
        r_{\hat{c}_k}^1 & r_{\hat{c}_k}^2 & ... & r_{\hat{c}_k}^n \\
        & \vdots                                \\
        r_{\hat{c}_l}^1 & r_{\hat{c}_l}^2 & ... & r_{\hat{c}_l}^n
    \end{bmatrix}
\end{equation*}
We can then learn a graph that encodes the pairwise relationships between interventions by performing Lasso neighborhood selection~\citep{meinshausen2006high} over $\boldsymbol{X}$. Specifically, for the intervention $\textrm{Int}(\hat{c}_j, i)$, we solve:
\begin{equation}
    \label{eqn:interpretability_concept_graph}
    \min_{\beta_{ij}} ||\boldsymbol{X}_{ij} - \boldsymbol{X}_{\textbackslash ij} \beta_{ij} ||^2_2 + \alpha ||\beta_{ij}||_1
\end{equation}
where $\boldsymbol{X}_{ij}$ is the outcome of $\textrm{Int}(\hat{c}_j, \bar{c}_j, i)$, $\boldsymbol{X}_{\textbackslash ij}$ represents the remaining $N \times |C|{-}1$ interventions, and $\alpha$ is an $L_1$-penalty, enforcing sparsity. Under this graphical interpretation, the coefficient vector $\boldsymbol{\beta}_{ij} \in \mathbb{R}^{n-1}$ is used to establish edges between intervention outcomes. Crucially, non-zero edge weights in $\boldsymbol{\beta}_{ij}$ mark the similarity of $\textrm{Int}(\hat{c}_j, i)$'s outcome to other interventions. Outcome similarities may hint at the nature of agent behavior. We provide a comprehensive overview of Lasso neighborhood selection in \cref{apdx:interpretability_lasso}.

\subsubsection{Illustrative Example:}
\begin{figure*}
    \centering
    \includegraphics[width=0.99\textwidth]{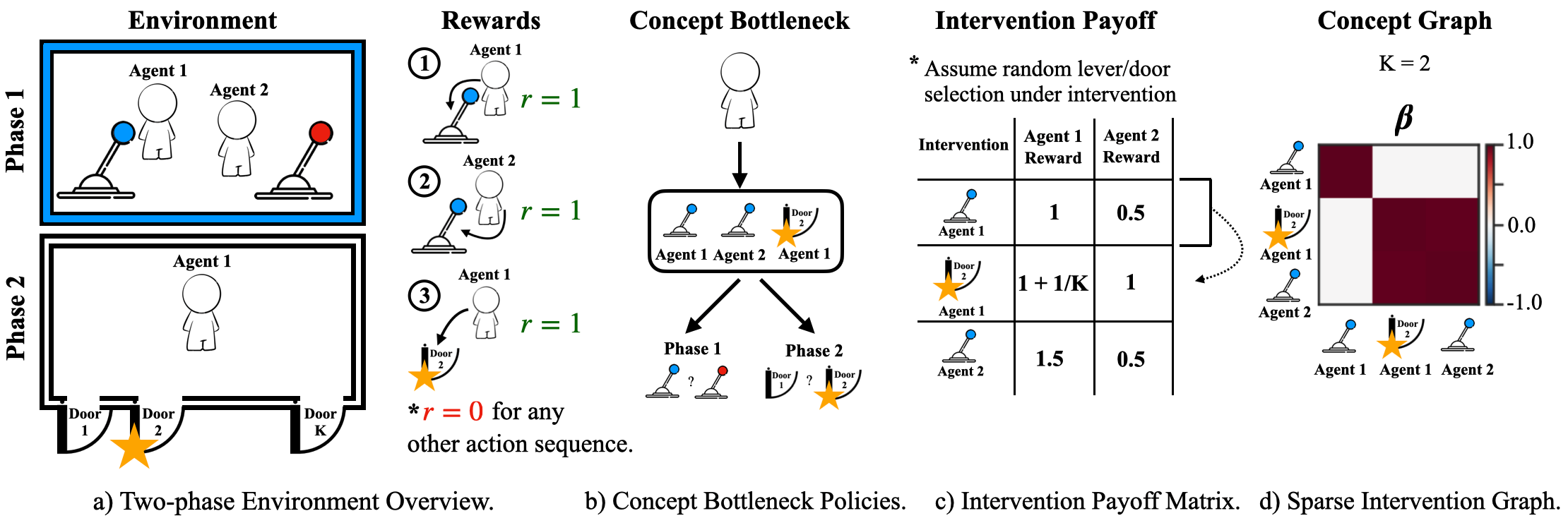}
    \caption{A toy example of graph learning over CBPs. \textbf{a)} A two-phase game. In phase one, Agent 1 and 2 choose from two levers (red, blue) using a special observation (wall color) and receive individual rewards for selecting the correct lever (blue) in the correct sequence (Agent 1 moves first, then Agent 2). If done correctly, the agents move on to the second phase. In phase two, Agent 1 selects from $K$ doors, also using a special observation (gold star), and receives an individual reward for selecting the correct door. \textbf{b)} Both agents use a CBP: Agent 1 has concepts for both lever color and the correct door, while Agent 2's estimates lever color alone. \textbf{c)} It is possible to compute each agent's expected payoff under concept intervention. \textbf{d)} Learning a graph over these payoffs reveals relationships between concept interventions.}
    \label{fig:interpretability_toy_example}
\end{figure*}
To build intuition, we present the following mathematical example. Consider the two-phase sequential game in \cref{fig:interpretability_toy_example}a. In the Phase 1 of this game, $N=2$ agents spawn in a room that contains two levers (red and blue). One lever is set as a reward-giving lever (blue here) and is identified as such by an observable feature (wall color). Agent 1 acts first and must choose a lever to pull, receiving an individual reward ($r=1.0$) for selecting the correct lever. Agent 2 then acts, also receiving $r=1.0$ for selecting the correct lever. If both agents select the correct lever, the game proceeds to Phase 2; otherwise it is terminated. In Phase 2, Agent 1 spawns in a new room containing $K$ unique doors (Agent 2 does not participate). Agent 1 must select a door to exit the room. One reward-giving door produces an individual reward ($r=1.0$), while the other $K{-}1$ doors produce zero reward. As before, an observable feature (gold star) indicates which door is the correct one.

Next, assume both agents select actions with a CBP (see \cref{fig:interpretability_toy_example}b)---Agent 1's CBP conditions actions on two concepts: lever color and the door indicator; and Agent 2's CBP estimates lever color alone. If we assume that intervening over any concept estimate reduces decision-making to a random policy. we can compute the expected rewards for each agent under intervention, as shown in \cref{fig:interpretability_toy_example}c. Computing \cref{eqn:interpretability_concept_graph} over these expected rewards with $K=2$ returns a graph with a strong bi-directional edge between Agent 1's door color concept and Agent 2's lever color concept (see \cref{fig:interpretability_toy_example}d). This edge reveals an important inter-agent relationship: Agent 1's door selection in Phase 2 relies heavily on Agent 2's lever selection in Phase 1. In contrast, Agent 1's lever selection in Phase 1 does not depend on Agent 2 (evident by the lone self-loop for Agent 1's lever color intervention).

\section{Experiments}

\label{sec:interpretability_experiments}
In this section, we evaluate the extent to which our method addresses the following questions:

\paragraph{Identifying Emergent Coordination}
In cooperative settings, can concept intervention identify emergent coordination from policies that act independently, or how much coordination the environment demands? Can it expose lazy agents and, in general, measure an agent's contribution to the multi-agent system? Moreover, if agents \emph{are} coordinating, what specific features underlie that coordination?

\paragraph{Emergent Strategic Behavior} In competitive environments, do CBPs reveal strategic behaviors as they emerge during multi-agent self-play training? For example, can intervention help us better understand both inter-team emergent roles and intra-team counter-strategies that develop over time?

\paragraph{Inter-Agent Social Dynamics}
In mixed-incentive environments, how well does concept intervention identify inter-agent social dynamics? Can it expose the \textit{functional connectivity} of the multi-agent system; or reveal chains of dependencies that might represent brittle inter-agent overfitting?

We use three environments from Melting Pot~\citep{leibo2021scalable} for our experiments. Melting Pot environments vary along a number of multi-agent axes---incentive alignment (cooperative, mixed motive, adversarial), coordination requirements, etc---and therefore provide a strong benchmark for studying emergent behavior. Specific environments are described in the subsections below, with further details (including concept states) in \cref{sec:background_environments} and \cref{apdx:interpretability_training}. For training, we use PPO~\citep{schulman2017proximal}, which has been shown to be state-of-the-art for multi-agent settings \cite{de2020independent,yu2021surprising}, and train in a fully-decentralized fashion. We augment the PPO objective with our concept loss (called ConceptPPO moving forward). Additional training details and hyperparameter sweeps are provided in \cref{apdx:interpretability_training}.

\subsection{Cooperative: Identifying Emergent Coordination}
\label{sec:interpretability_identifying_coord}

\begin{figure*}
     \centering
     \begin{subfigure}[b]{0.45\textwidth}
         \centering
         \includegraphics[width=\textwidth]{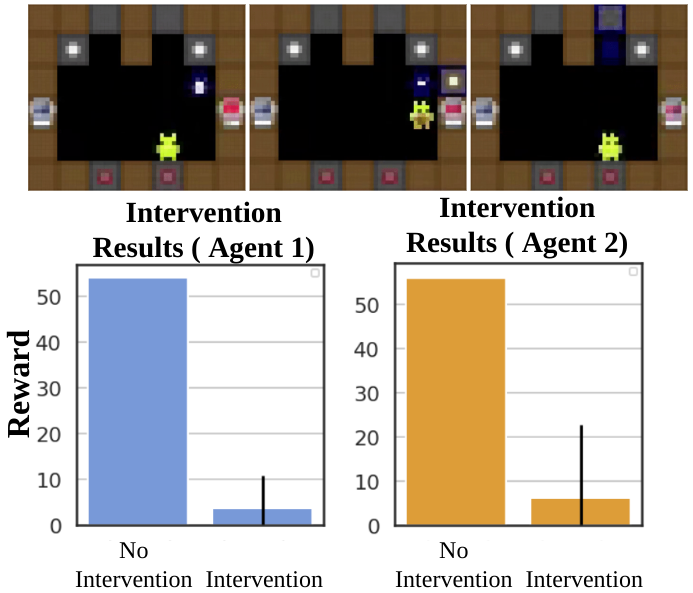}
         \caption{Coordinating Agents.}
         \label{fig:interpretability_results_coord}
     \end{subfigure}
     \hfill
     \begin{subfigure}[b]{0.45\textwidth}
         \centering
         \includegraphics[width=\textwidth]{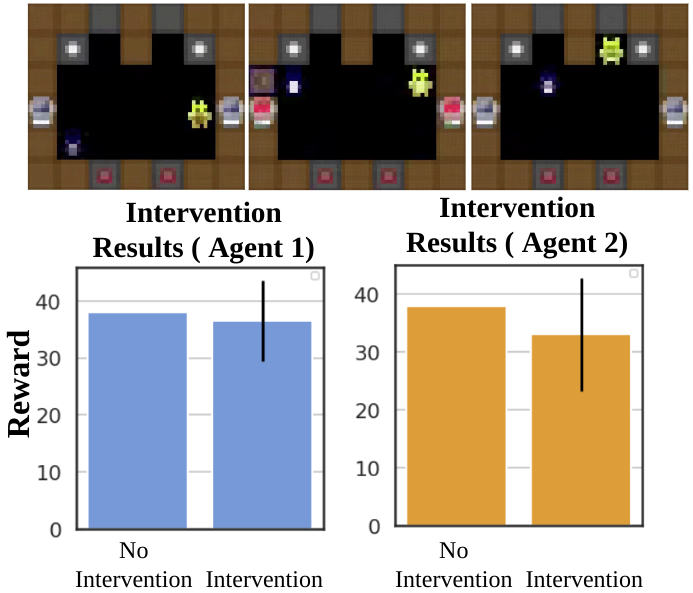}
         \caption{Independent Agents.}
         \label{fig:interpretability_results_ind}
     \end{subfigure}
        \caption{Performance degradation under concept intervention uncovers (a) policies that coordinate and (b) policies that act independently.}
        \label{fig:interpretability_results_coord_vs_ind}
\end{figure*}
We first evaluate CBPs as a tool for understanding emergent coordination in fully-cooperative settings. We use the game Collaborative Cooking, where a group of agents inhabit a kitchen-like environment and must collaborate to find ingredients (tomatoes), complete recipes (bringing tomatoes to and from cooking pots), and deliver food as quickly as possible. We train 10 ConceptPPO policies with a concept cost weight of $\lambda=0.1$ in this environment and use them in our evaluation below.

\paragraph{Do agents learn to coordinate?}
We aim to distinguish policies that learn coordination from those that solve the task independently. We evaluate each of the trained ConceptPPO policies over 100 test-time trajectories (and 5 seeds each) while intervening over all of the concepts related to each agent’s teammate. The average cumulative reward under intervention is shown in \cref{fig:interpretability_results_coord_vs_ind}.

\cref{fig:interpretability_results_coord_vs_ind} investigates a cooking environment with duplicate sets of ingredients (tomatoes) and tools (bowls, pots) on both sides. Therefore, both independent strategies and highly coordinated strategies can complete the task. In the trajectory in \cref{fig:interpretability_results_coord}, we see an emergent strategy in which two agents coordinate. The orange agent works the bottom of the environment picking up tomatoes and bringing them to the cooking pot, while the blue agent stays in the top running dishes to and from the pot to deliver soup. Intervention over teammate-related concepts leads to a catastrophic drop in reward, as the agent's coordination clearly hinges on an accurate modeling of its teammate. In \cref{fig:interpretability_results_ind}, a different strategy emerges, where agents complete the task independently on opposite sides of the kitchen. Here, intervention does not hurt reward, as neither agent needs to closely monitor its teammate to complete the task. This result therefore indicates that coordination detection through intervention is accurate and reliable with CBPs. In \cref{apdx:interpretability_environmental_demands}, we study several cooking environment layouts and average the reduction in reward observed when intervening on teammate concepts across all policies. We can thus measure the coordination demands of each environment.

\paragraph{How do they coordinate?}
\begin{figure}
    \centering
    \includegraphics[width=0.55\textwidth]{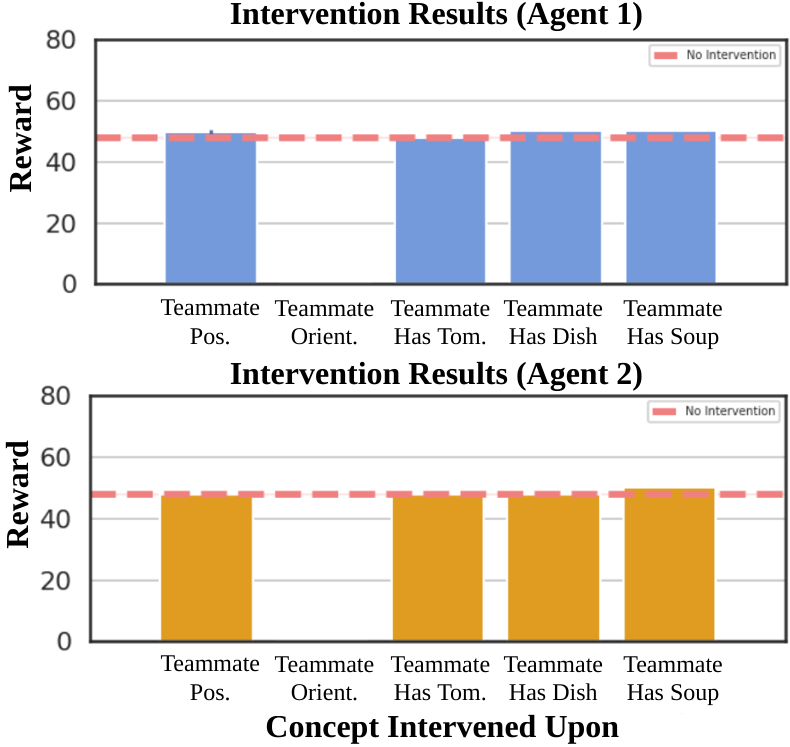}
    \caption{For coordinating agents, intervening over teammate-related concepts exposes the factors driving the team's coordination strategy. Interestingly, teammate orientation is relied upon heavily by both agents for coordination.}
    \label{fig:interpretability_results_teammate}
\end{figure}
Next, we examine the policies that \emph{have} learned to coordinate and pinpoint the specific inter-agent features that drive their coordination (\cref{fig:interpretability_results_teammate}). Rather than intervening over all teammate-related concepts at once, we intervene one at a time. If agents are coordinating using a particular signal, the corresponding concept intervention should result in a sharp drop in performance.

Interestingly, performance only drops when intervening over the orientation concept. This suggests that agents are primarily using the orientation of other agents as a coordination signal, which is curious, as we might expect coordination to involve multiple sources of inter-agent information. For completeness, we conduct two supporting analyses. First, we rule out the possibility that intervening with a fixed orientation creates an OOD (or adversarial) input. We do this using an empirical sample of agent orientations to manufacture interventions that are both in- and out-of-distribution; and show that our results are consistent in both cases (see \cref{apdx:interpretability_ood_analysis}). Second, we re-run this intervention over additional ConceptPPO policies trained without orientation as a concept. These results show that agent coordination latches onto yet another concept (teammate has\_soup) as its primary signal (see \cref{apdx:interpretability_int_wo_orn}). The fact that agents learn a policy in which coordination hinges only on one particular signal reveals the brittleness of deep multi-agent reinforcement learning policies and their tendency to rely on highly specialized conventions. This brittleness is one of the reasons that MARL agents often fail to generalize to novel partners, or have poor zero-shot coordination performance~\citep{hu2020other}.

\paragraph{Identifying Lazy Agents}
\begin{figure}
    \centering
    \includegraphics[width=0.99\textwidth]{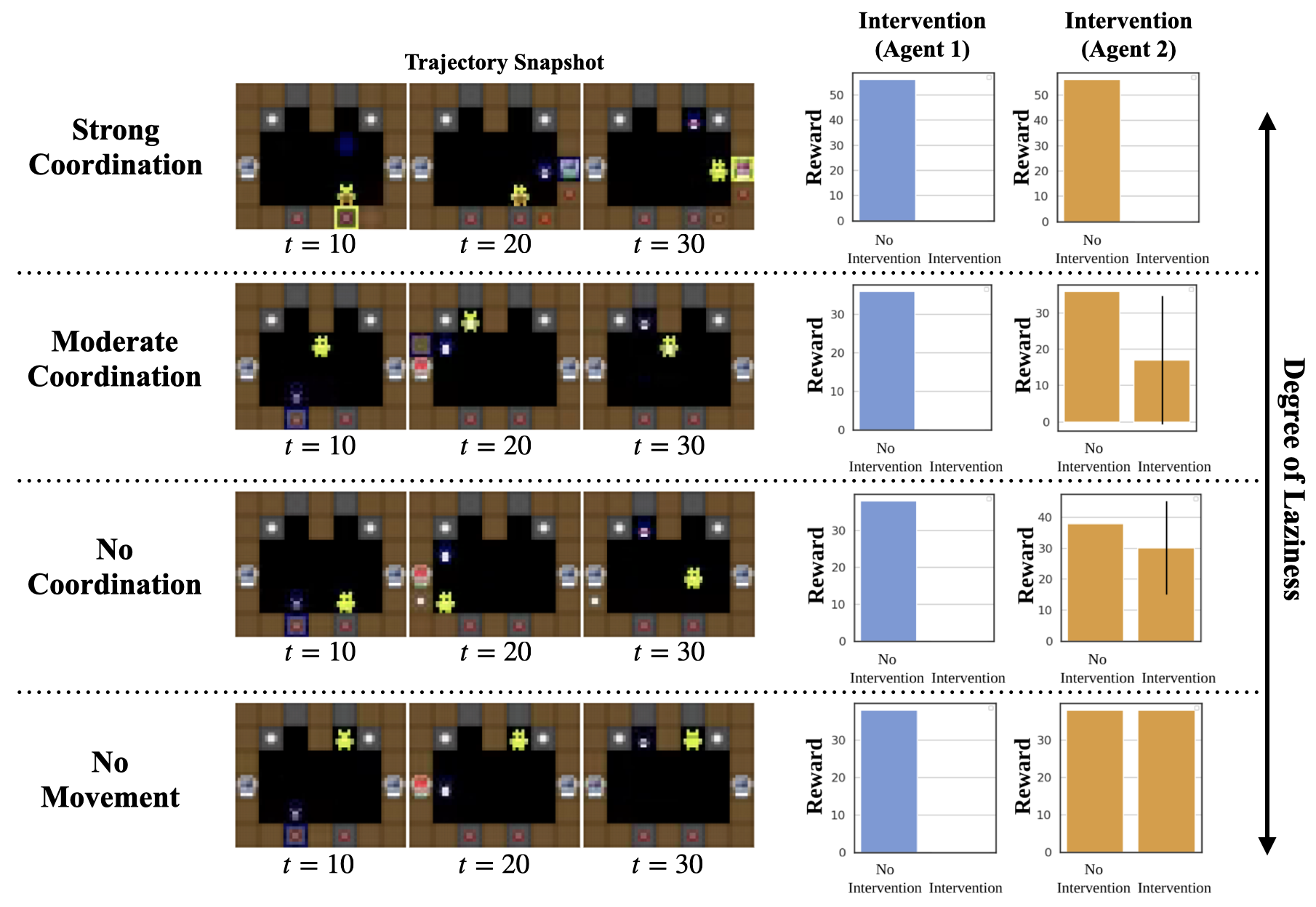}
    \caption{Intervening on an agent's concepts pertaining to itself exposes its \textit{degree of laziness}.}
    \label{fig:interpretability_results_lazy}
\end{figure}
Our method also allows us to test for \emph{coordination failures}. Here we test for lazy agents by replacing each agent's concept estimates about itself, including its own position, orientation, etc. If an agent is acting productively in the environment, removing this information will greatly hinder its performance; and team performance as a whole. If performance does not degrade, the agent likely is not contributing to the task. \Cref{fig:interpretability_results_lazy} shows the results of this test for three policies, each differing in the strength of their contribution to the team.

The magnitude of performance degradation is a direct function of the productivity of each agent. When agents are both contributing (left), concept intervention has a strong negative impact on the performance of the agents. 
In the extreme lazy agent case where one agent does not move over the course of a trajectory (right), reward is not impacted by intervention. Interestingly, this intervention also captures intermediate effects. For example, when only one agent is productive, but both agents share a workspace, there is a small but noticeable drop in performance (middle). Thus, these results not only demonstrate that CBPs give us a unique way of diagnosing lazy agents, but also shows that it can quantify the \emph{degree of laziness} of each agent as a function of performance degradation.

\subsection{Emergence of Strategic Behavior}
\label{sec:interpretability_ctf}
\begin{figure}
    \centering
    \includegraphics[width=0.95\textwidth]{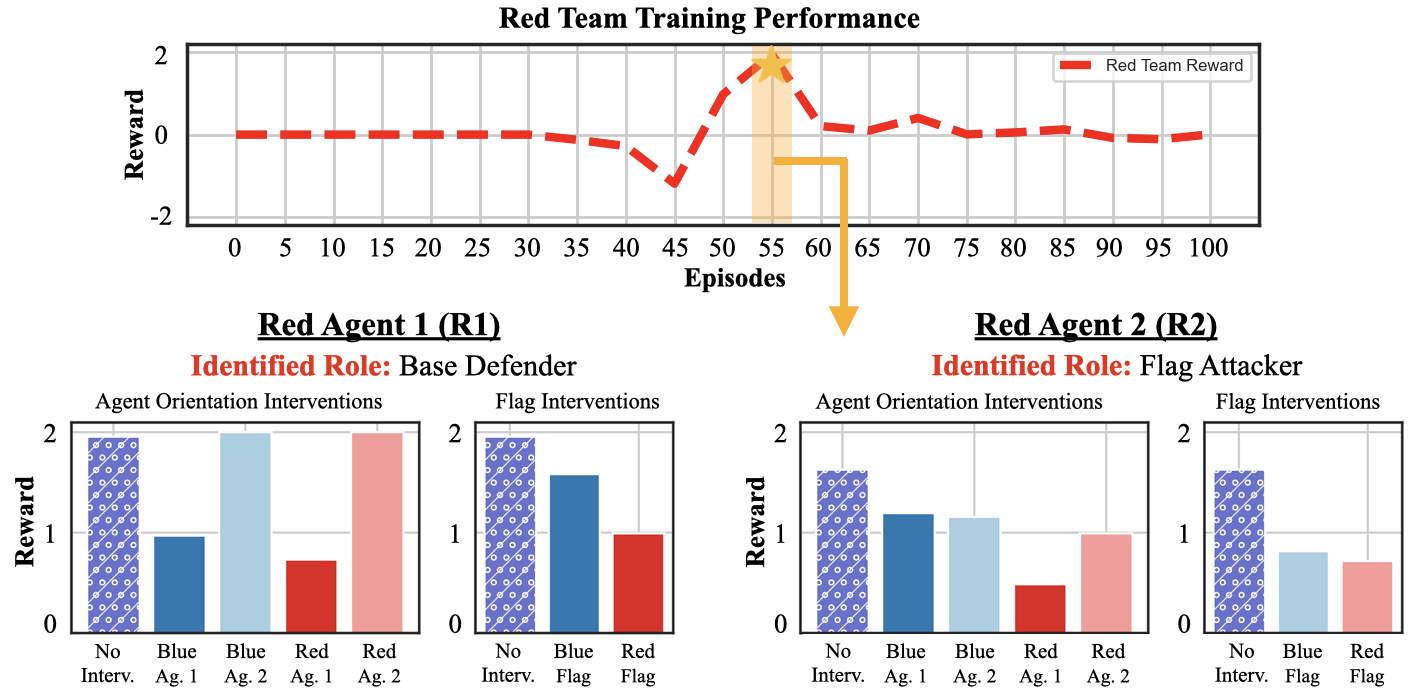}
    \caption{Understanding reward oscillations with CBPs. The red team discovers a dominant strategy at episode $55$. CBP interventions reveal two emergent roles: (i) Red agent R1 is identified as a flag defender agent---it models B1 (the blue attacking agent) and its home flag and is only negatively impacted by interventions over those concepts; (ii) R2 interacts with both blue agents, both flags, and is strongly impacted by intervention. R2 is thus identified as a flag attacking agent.}
    \label{fig:interpretability_ctf_int_results}
\end{figure}

Mixed cooperative-competitive environments induce an automatic skill learning curriculum for MARL agents similar to that of self-play~\citep{silver2017mastering}, which can lead to highly-sophisticated behaviors such as the emergent tool use observed in team hide-and-seek by~\citet{baker2019emergent}. Here we investigate the extent to which CBPs can expose these strategic behaviors during training. We train independent CBP agents in Melting Pot's Capture the Flag environment (outlined further in \cref{sec:background_environments} and \cref{apdx:interpretability_training}) and intermittently perform interventions over agent concept estimates during training. As before, we measure the magnitude of reward degradation.

Results from this analysis are shown in \cref{fig:interpretability_ctf_int_results}. The reward curve represents one oscillation of self-play strategy for the red team---the red team receives negative reward in episode 45, but eventually counters with a dominant strategy in episode 55. Oscillations such as this are indicative of a newly learned strategic behavior by one or both agents~\citep{baker2019emergent}. But what strategic behavior was learned?

Conducting intervention analysis at episode 55 reveals more precisely how the agents behave beyond red agents capturing the flag. Analysis of the first red agent (R1) shows that it is negatively impacted by interventions over itself, the blue team's attacking agent (B1), and its home flag. This pattern is consistent with the behavior of a \textit{base defender}. Interventions over the second red agent (R2) expose a different pattern---the agent is negatively impacted by interventions over both flags and each of the other agents. This pattern suggests that R2 is a \textit{flag attacker}, navigating between the opponent base to take flags and its home base to return them. A qualitative analysis of the agents' behavior is performed in \cref{apdx:interpretability_ctf_qualitative} and confirms these results. Thus, our intervention technique can be used \textit{during training} to augment reward-based analysis and investigate strategic behaviors \textit{as they emerge}.

\subsection{Social Dilemmas: Inter-agent Social Dynamics}
\label{sec:interpretability_social_dynamics}
Social dilemmas further extend the scope of emergent behavioral complexity to include exploitation, free-riding, and public resource sharing. We evaluate our method's ability to uncover such interactions in Melting Pot's Clean Up, an environment in which agents must balance selfish behavior (harvesting fruit) with public service (cleaning a river is necessary for fruit to grow).

As we've seen, a natural way to investigate inter-agent dynamics with CPBs is by looking at outcomes with and without intervention. In \cref{fig:interpretability_cleanup_raw_stats}, we do the same for each agent in Clean Up and find that the rewards of Agent 1 and 2 are correlated under intervention. We also see an increase in Agent 1 and 2's idleness and a decrease in their inter-agent distance. A reasonable conclusion, therefore, is that Agent 1 and Agent 2 are coordinating spatially and require accurate estimations of each other to complete the task. However, as we'll show, our method can uncover a more descriptive \textit{chain of social interactions} that lead to these intervention outcomes. The trajectory in \cref{fig:interpretability_cleanup_raw_stats} confirms this.
\begin{figure*}
    \centering
    \includegraphics[width=0.99\linewidth]{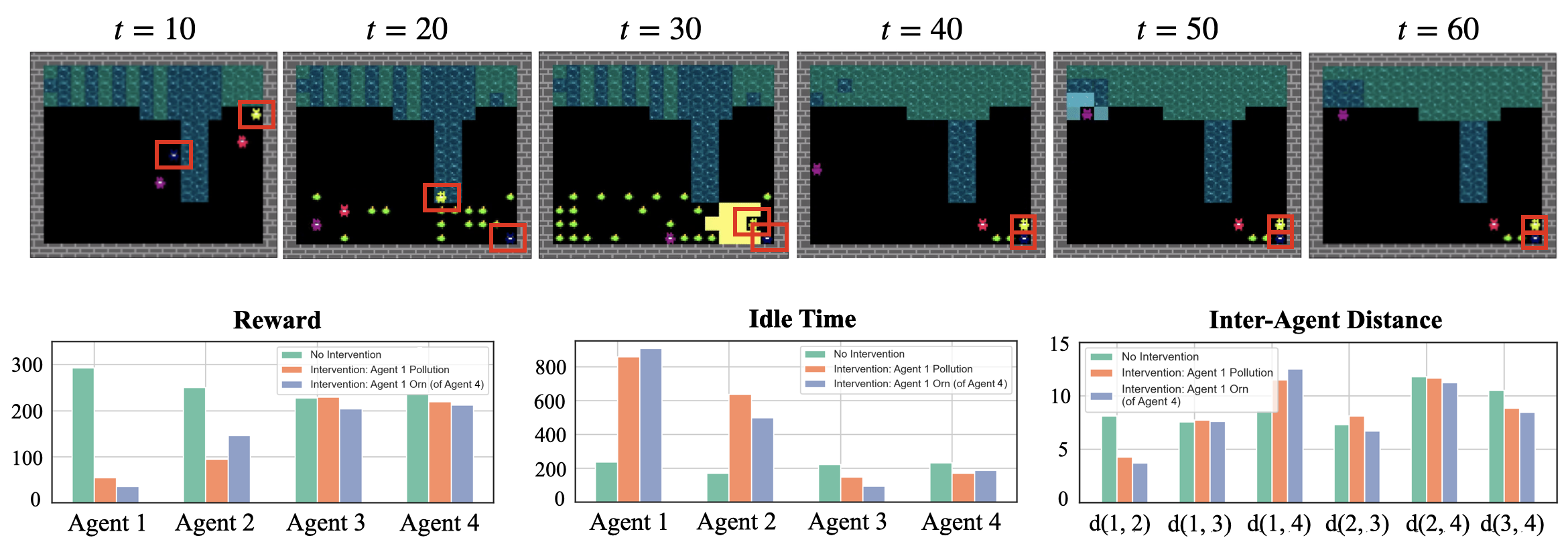}
    \caption{Visualization and raw statistics over interventions identify a collision between Agents 1 and 2 but not a complete chain of social dependencies of the agents.}
    \label{fig:interpretability_cleanup_raw_stats}
\end{figure*}

To investigate, we learn a series of graphs by computing \cref{eqn:interpretability_concept_graph} over all agent-concept interventions using average reward, pairwise inter-agent distance, and the number of idle (non-moving) steps as features, respectively. \cref{fig:interpretability_results_graph_reward} shows the graph computed over rewards (in matrix form), which exposes a number of intriguing concept relationships across agents. First, there are \textit{no} identifiable relationships between Agent 1 and 2 (no edges link interventions over Agent 1’s concepts pertaining to Agent 2 or vice versa). This means that Agent 1 and 2 are in fact not relying directly on each other's concepts. Next, we find a strong bi-directional relationship between Agent 1's closest pollution concept and Agent 1's estimate of Agent 4's orientation. Importantly, this same bi-directional edge also exists in the graphs computed over inter-agent distance and idleness (see \cref{fig:interpretability_graphs}).
\begin{figure*}
    \centering
    \includegraphics[width=0.9\linewidth]{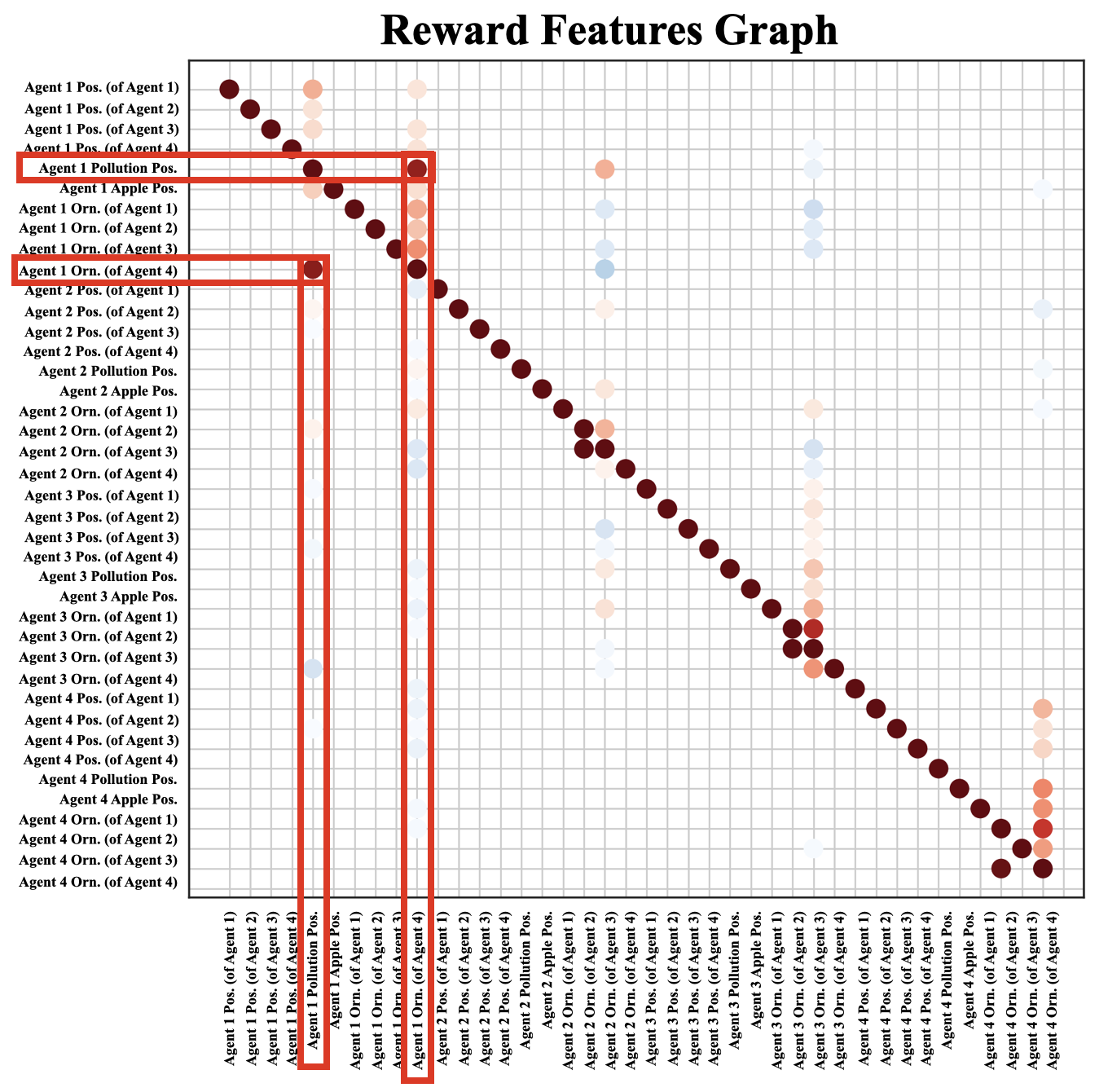}
    \caption{A sparse graph learned over concept interventions with respect to average reward features. A strong bi-directional relationship between two concepts: Agent 1's closest pollution, and Agent 1's estimate of Agent 4's orientation.}
    \label{fig:interpretability_results_graph_reward}
\end{figure*}

The graphs therefore reveal the following: First, the lack of a link between the Agent 1 and 2 in the graph suggests that their relationship is not one of coordination, but rather the result of an incidental interaction in the environment (they are not explicitly modeling each other). Next, the bi-directional edge reveals that Agent 1's performance is largely dependent on its exploitation of the public service of Agent 4. Specifically, Agent 4 runs to and from the river to clean, at which point apples appear in the area of the patch that Agent 4 left. Agent 1 uses both Agent 4's orientation and its estimate of pollution (which Agent 4 changes) to decide when it should run to the area of the patch left by Agent 4 to consume apples. Because Agent 1's policy is so reliant on these two concepts, intervening over either of them causes Agent 1's behavior to collapse; and this happens in such a way that it interrupts Agent 2. Concretely, intervening over both Agent 1's pollution concept and Agent 1's estimate of Agent 4 lead to the same outcome: a collision between Agent 1 and Agent 2.
\cref{fig:interpretability_graphs}).
\begin{figure}
     \centering
     \begin{subfigure}[b]{0.475\textwidth}
         \centering
         \includegraphics[width=\textwidth]{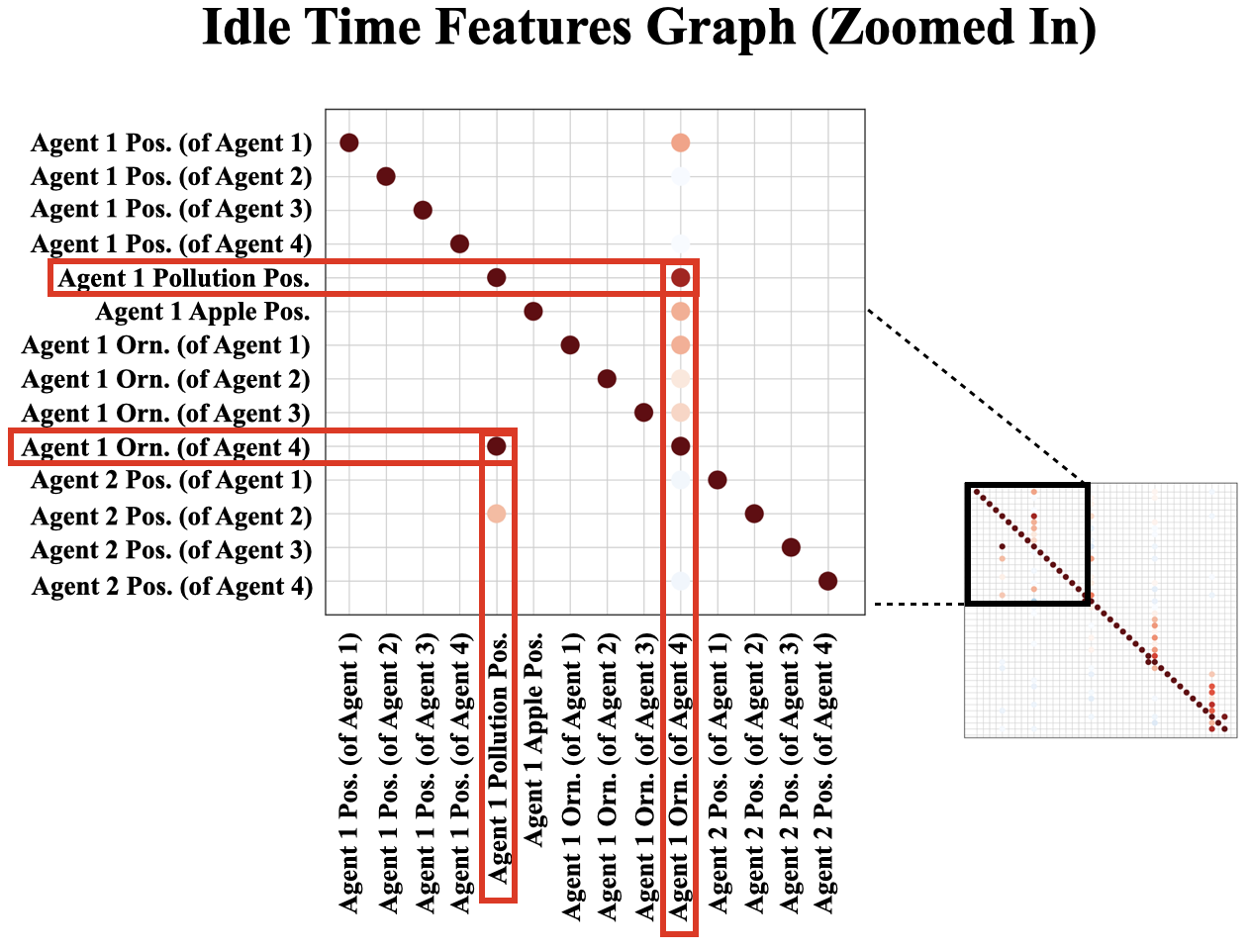}
         \caption{Graph over idle time features.}
         \label{fig:interpretability_results_graph_time}
     \end{subfigure}
     \hfill
     \begin{subfigure}[b]{0.475\textwidth}
         \centering
         \includegraphics[width=\textwidth]{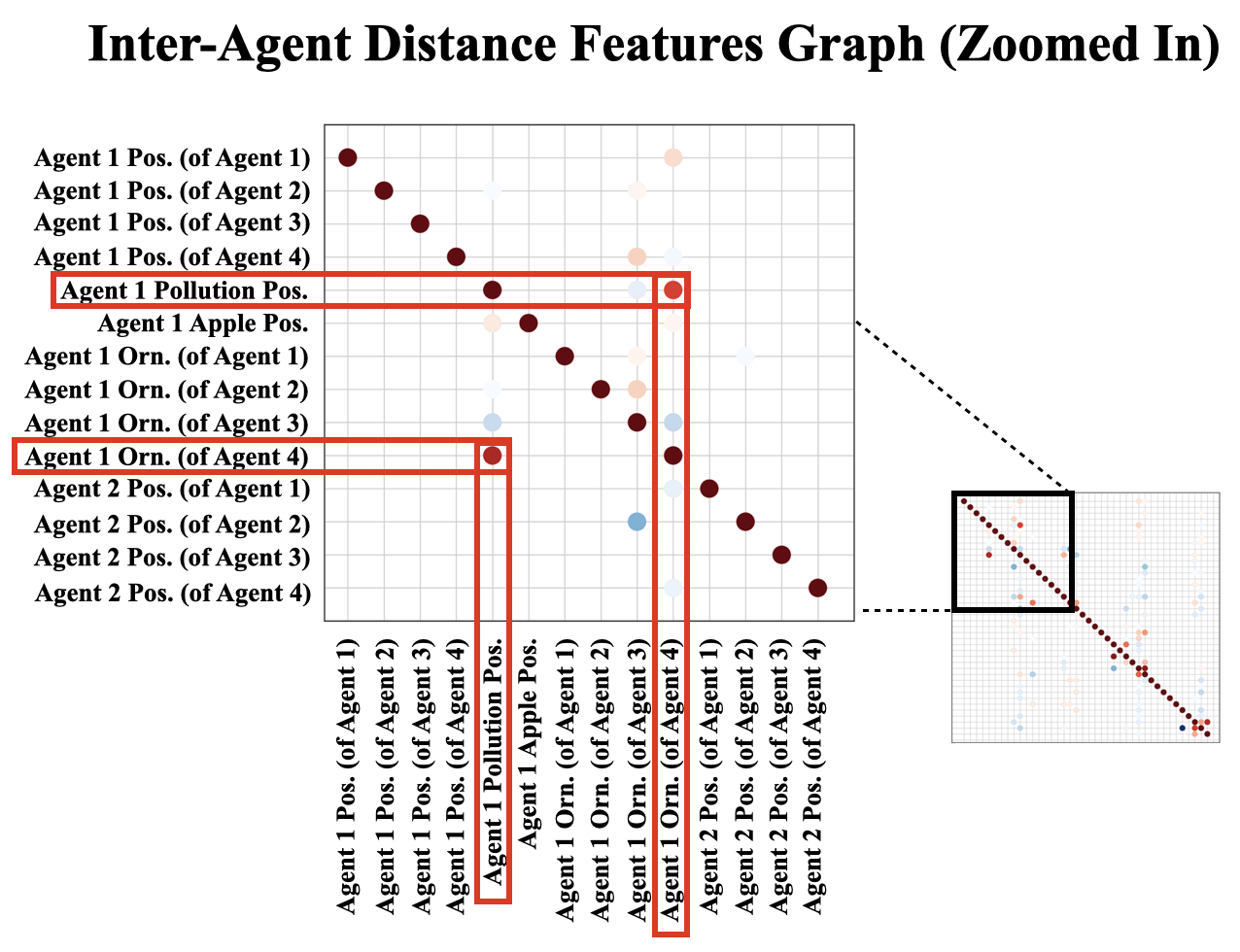}
         \caption{Graph over inter-agent distance features.}
         \label{fig:interpretability_results_graph_distance}
     \end{subfigure}
    \caption{Additional sparse graphs learned over concept interventions with respect to \textbf{a)} inter-agent distance and \textbf{b)} agent idleness features. The same strong bi-directional relationship exists between two concepts: Agent 1's closest pollution, and Agent 1's estimate of Agent 4's orientation.}
    \label{fig:interpretability_graphs}
\end{figure}
In sum, the chain of dependencies is not Agent 1 $\leftrightarrow$ Agent 2, but rather Agent 1 $\rightarrow$ Agent 4 / Pollution $\rightarrow$ Agent 2 and rather than coordinating, we have discovered that Agent 1 and Agent 2 have overfit to brittle behavioral patterns that are easily disrupted. This is in turn why there is a correlation between their rewards under intervention. Our graph learning technique over CBP interventions thus reveals complex inter-agent dynamics that could take humans many hours to detect.

\section{Conclusion and Future Work}
We introduced Concept Bottleneck Policies as an interpretable, concept-based policy learning method for MARL and demonstrated that they are effective for understanding emergent multi-agent behavior. In particular, CBPs support concept intervention, which can be used to identify when a multi-agent team has learned to coordinate, what inter-agent features drive that coordination, and to what extent coordination is required in an environment. Moreover, concept intervention helps expose coordination failures like lazy agents and complex inter-agent dynamics. We discuss broader impacts, limitations, and future work in \cref{apdx:interpretability_discussion}.

There are a number of interesting avenues for future work in the area of concept bottlenecks for MARL. First, taking inspiration from the literature on POMDPs and belief-space planning, we can explore extensions to our architecture where agents can leverage both histories of observations and histories of concepts in training and concept estimation. This may also be combined with architecture improvements, such as giving agents memory (e.g., through recurrent policies). Further, there is room to explore the use of concept prediction accuracy as an intrinsic reward that incentivizes agents to explore subsets of the state space in which they do not predict concepts accurately (thereby improving their concept estimates). An important future work is to extend our simple graph-based analysis technique to more complex graph learning paradigms~\citep{dong2019learning}.

\part{Fair Multi-Agent Behavior}
\label{part:fairness}
In this part of the thesis, we bridge multi-agent learning and algorithmic fairness, using the following questions as motivation: 
\begin{itemize}
    \item \textit{How do we formalize fairness in multi-agent settings? For example, what do sensitive variables mean for an agent within a multi-agent team? Can we reconcile reward as an outcome variable?}
    \item \textit{In cooperative environments, is shared reward enough to incentivize fair behavior to emerge, or does unfair behavior emerge naturally?}
    \item \textit{Can we shape multi-agent behavior during training to achieve fairer outcomes? If so, to what extent does a trade-off between fairness and utility/efficiency exist, and what properties of the MARL problem contribute to that trade-off?}
\end{itemize}
\noindent In chapter \cref{chapter:fairness}, we provide suitable definitions of fairness for multi-agent settings and introduce an actor-critic algorithm for learning provably fair multi-agent policies. We validate our method empirically in in cooperative games, and discuss the fairness-utility trade-off as it applies to multi-agent settings.

\chapter{Cooperative Multi-Agent Fairness and Equivariant Policies}
\label{chapter:fairness}

\section{Introduction}
\label{sec:fairness_intro}
Algorithmic fairness is an increasingly important sub-domain of AI. As statistical learning algorithms continue to automate decision-making in crucial areas such as lending \cite{fuster2020predictably}, healthcare \cite{potash2015predictive}, and education \cite{dorans2016fairness}, it is imperative that the performance of such algorithms does not rely upon sensitive information pertaining to the individuals for which decisions are made (e.g. race, gender). Despite its growing importance, fairness research has largely targeted prediction-based problems, where decisions are made for one individual at one time \cite{mitchell2018prediction}. Though recent studies have extended fairness to the multi-agent case \cite{jiang2019learning}, such work primarily considers social dilemmas in which team utility is in obvious conflict with the local interests of each team member \cite{leibo2017multi, rapoport1974prisoner, van2013psychology}.

Many real-world problems, however, must weigh the fairness implications of team behavior in the presence of a single overarching goal. In-line with recent work that has highlighted the importance of leveraging multi-agent learning to study socio-economic challenges such as taxation, social planning, and economic policy \cite{zheng2020ai}, we posit that understanding the range of team behavior that emerges from single-objective utility maximization is crucial for the development of fair multi-agent systems. For this reason, we study fairness in the context of \textit{cooperative multi-agent settings}. Cooperative multi-agent fairness differs from traditional game-theoretic interpretations of fairness (e.g. resource allocation \cite{elzayn2019fair, zhang2014fairness}, social dilemmas \cite{leibo2017multi}) in that it seeks to understand the fairness implications of emergent coordination learned by multi-agent teams that are bound by a shared reward. Cooperative multi-agent fairness therefore reframes the question---``Will agents cooperate or defect, given the choice between local and team interests?"---to a related but novel question---``Given the incentive to work together, do agents learn to coordinate effectively and fairly?"

Experimentally, we target pursuit-evasion (i.e. predator-prey) as a test-bed for cooperative multi-agent fairness. Pursuit-evasion allows us to simulate a number of important components of socio-economic systems, including: (i) Shared objectives: the overarching goal of pursuers is to capture an evader;
(ii) Agent skill: the speed of the pursuers relative the evader serves as a proxy for skill; (iii) Coordination: success requires sophisticated cooperation by the pursuers. Using pursuit-evasion, we study the fairness implications of behavior that emerges under variations of these ``socio-economic" parameters. Similar to prior work \cite{lowe2017multi, mordatch2018emergence, grupen2020low}, we cast pursuit-evasion as a multi-agent reinforcement learning (RL) problem. 

Our first result highlights the importance of shared objectives to cooperation. In particular, we compare policies learned when pursuers share in team success (mutual reward) to those learned when pursuers do not share reward (individual reward). We find that sophisticated coordination only emerges when pursuers are bound by mutual reward. Given individual reward, pursuers are not properly incentivized to work together. However, though mutual reward aides coordination, it does not specify how to coordinate fairly. In our experiments, we find that naive, unconstrained maximization of mutual reward yields unfair individual outcomes for cooperative teammates. In the context of pursuit-evasion, the optimal strategy is a form of \textit{role assignment}---the majority of pursuers act as supporting agents, shepherding the evader to one designated ``capturer" agent. Solving this issue is the subject of the rest of our analysis.

Addressing this form of unfair emergent coordination requires connecting fairness to multi-agent learning settings. To do this, we first introduce \textit{team fairness}, a group-based fairness measure inspired by demographic parity \cite{dwork2012fairness, feldman2015certifying}. Team fairness requires the distribution of a team's reward to be equitable across sensitive groups. We then show that it is possible to enforce team fairness during policy optimization by transforming the team's joint policy into an equivariant map. We prove that equivariant policies yield fair reward distributions under assumptions of agent homogeneity. We refer to our multi-agent learning strategy as \textit{Fairness through Equivariance} (Fair-E) and demonstrate its effectiveness empirically in pursuit-evasion experiments.

Despite achieving fair outcomes, Fair-E represents a binary switch---one can either choose fairness (at the expense of utility) or utility (at the expense of fairness). In many cases, however, it is advantageous to modulate between fairness and utility. To this end, we introduce a soft-constraint version of Fair-E that incentivizes equivariance through regularization. We refer to this method as \textit{Fairness through Equivariance Regularization} (Fair-ER) and show that it is possible to tune fairness constraints over multi-agent policies by adjusting the weight of equivariance regularization. Moreover, we show empirically that Fair-ER reaches higher levels of utility than Fair-E while achieving fairer outcomes than non-equivariant policy learning.

Finally, as in both prediction-based settings \cite{corbett2017algorithmic, zhao2019inherent} and in traditional multi-agent variants of fairness \cite{okun2015equality, le1990equity}, it is important to understand the ``cost" of fairness. We present novel findings regarding the fairness-utility trade-off for cooperative multi-agent settings. Specifically, we show that the magnitude of the trade-off depends on the skill level of the multi-agent team. When agent skill is high (making the task easier to solve), fairness comes with no trade-off in utility, but as skill decreases (making the task more difficult), gains in team fairness are increasingly offset by decreases in team utility.
\newline

\noindent \textbf{Preview of Contributions}
In sum, our work offers the following contributions:
\begin{enumerate}
    \item We show that mutual reward is critical to multi-agent coordination. In pursuit-evasion, agents trained with mutual reward learn to coordinate effectively, whereas agents trained with individual reward do not.
    \item We connect fairness to cooperative multi-agent settings. We introduce team fairness as a group-based fairness measure for multi-agent teams that requires equitable reward distributions across sensitive groups.
    \item We introduce Fairness through Equivariance (Fair-E), a novel multi-agent strategy leveraging equivariant policy learning. We prove that Fair-E achieves fair outcomes for individual members of a cooperative team.
    \item We introduce Fairness through Equivariance Regularization (Fair-ER) as a soft-constraint version of Fair-E. We show that Fair-ER reaches higher levels of utility than Fair-E while achieving fairer outcomes than non-equivariant learning.
    \item We present novel findings regarding the fairness-utility trade-off for cooperative settings. Specifically, we show that the magnitude of the trade-off depends on agent skill---when agent skill is high, fairness comes for free; whereas with lower skill levels, fairness is increasingly expensive.
\end{enumerate}

\section{Related Work}
\label{sec:fairness_related_work}
At a high-level, the prediction-based fairness literature can be split into two factions: individual fairness and group fairness. Introduced by \citet{dwork2012fairness}, individual fairness posits that two individuals with similar features should be classified similarly (i.e. similarity in feature-space implies similarity in decision-space). Such approaches rely on task-specific distance metrics with which similarity can be measured \cite{barocas2019fairmlbook, chouldechova2018frontiers}. Group fairness, on the other hand, attempts to achieve outcome consistency across sensitive groups. This idea has given rise to a number of methods such as statistical/demographic parity \cite{feldman2015certifying, johndrow2019algorithm, kamiran2009classifying, zafar2017fairness}, equality of opportunity \cite{hardt2016equality}, and calibration \cite{kleinberg2016inherent}. Recent work has extended fairness to the RL setting to consider the feedback effects of decision-making \cite{jabbari2017fairness, wen2021algorithms}.

In multi-agent systems, fairness is typically studied in game-theoretic settings in which individual payoffs and overall group utility are in obvious conflict \cite{de2005priority}---such as resource allocation \cite{elzayn2019fair, zhang2014fairness} and social dilemmas \cite{leibo2017multi, rapoport1974prisoner, van2013psychology}. In multi-agent RL settings, these tensions have been addressed through myriad techniques, including reward shaping \cite{peysakhovich2017prosocial}, intrinsic reward \cite{wang2018evolving}, parameterized inequity aversion \cite{hughes2018inequity}, and hierarchical learning \cite{jiang2019learning}. Also related is the Shapley value: a method for sharing surplus across a coalition based on one’s contributions to the coalition \cite{shapley201617}. Shapley value-based credit assignment techniques have recently been shown to stabilize learning and achieve fairer outcomes when incorporated into the multi-agent RL problem \cite{wang2020shapley, li2021shapley}.

Our work differs from this prior work in two key ways. First we target fully-cooperative multi-agent settings \cite{hao2016fairness} in which fairness implications emerge naturally in the presence of a single overarching goal (i.e. mutual reward). In this fully-cooperative setting, individual and team incentives are not in obvious conflict. Our motivation for studying fully-cooperative team objectives follows from recent work that highlights the role of multi-agent learning in real-world problems characterized by shared objectives, including taxation and economic policy \cite{zheng2020ai}. Moreover, we study modifications to the utility-maximization objective that yield fairer outcomes by incentivizing agents to change their behavior, rather than redistributing outcomes after-the-fact. Most relevant is \citet{siddique2020learning} and \citet{zimmer2020learning}, which introduce a class of algorithms that successfully achieve fair outcomes for multi-agent teams through pre-defined social welfare functions that encode specific fairness principles. Our work, conversely, introduces task-agnostic methods for incentivizing fairness through both hard-constraints on agent policies and soft-constraints (i.e. regularization) \cite{liu2018delayed} on the RL objective.

Finally, discussion of the fairness-utility (or fairness-efficiency) trade-off has a long history in game-theoretic multi-agent settings \cite{okun2015equality, le1990equity, bertsimas2012efficiency, joe2013multiresource, bertsimas2011price} and is also prevalent throughout the prediction-based fairness literature \cite{menon2018cost}. Existing work has shown both theoretically \cite{calders2009building, kleinberg2016inherent, zhao2019inherent} and empirically \cite{dwork2012fairness, feldman2015certifying, kamiran2009classifying, lahoti2019operationalizing, pannekoek2021investigating} that gains in fairness come at the cost of utility. Our discussion of the fairness-utility trade-off is most similar to \citet{corbett2017algorithmic} in this regard, as we study the trade-off through the lens of constrained vs. unconstrained optimization. However, we take this trade-off a step further, outlining a relationship between fairness, utility, and agent skill that is not present in prior work.

\section{Preliminaries}
\label{sec:fairness_preliminaries}

\paragraph{Mutual Information}
\label{sec:fairness_mi}
Given random variables $X_1 {\sim} P_{X_1}$ and $X_2 {\sim} P_{X_2}$ with joint distribution $P_{X_1 X_2}$, mutual information is defined as the Kullback-Leibler (KL-) divergence between the joint $P_{X_1X_2}$ and the product of the marginals $P_{X_1} \otimes P_{X_2}$:

\begin{equation}
    I(X_1;X_2) := D_{KL}(P_{X_1 X_2} || P_{X_1} \otimes P_{X_2})
    \label{eqn:fairness_kl_div}
\end{equation}

\noindent Mutual information quantifies the dependence between $X_1$ and $X_2$ where, in \cref{eqn:fairness_kl_div}, larger divergence represents stronger dependence. Importantly, mutual information can also be represented as the decrease in entropy of $X_1$ when introducing $X_2$:

\begin{equation}
    I(X_1;X_2) := H(X_1) - H(X_1 \mid X_2)
    \label{eqn:fairness_mi_entropy}
\end{equation}

\paragraph{Equivariance}
Let $g_1$ and $g_2$ be G-sets of a group $G$ and $\sigma$ be a symmetry transformation over $G$. Then a function $f:g_1 \rightarrow g_2$ is equivariant with respect to $\sigma$ if the commutative relationship $f(\sigma \cdot x) = \sigma \cdot f(x)$ holds. Equivariance in the context of RL implies that separate policies will take the same actions under permutations of state space.
\begin{figure*}[t!]
    \centering
    \includegraphics[width=0.99\textwidth]{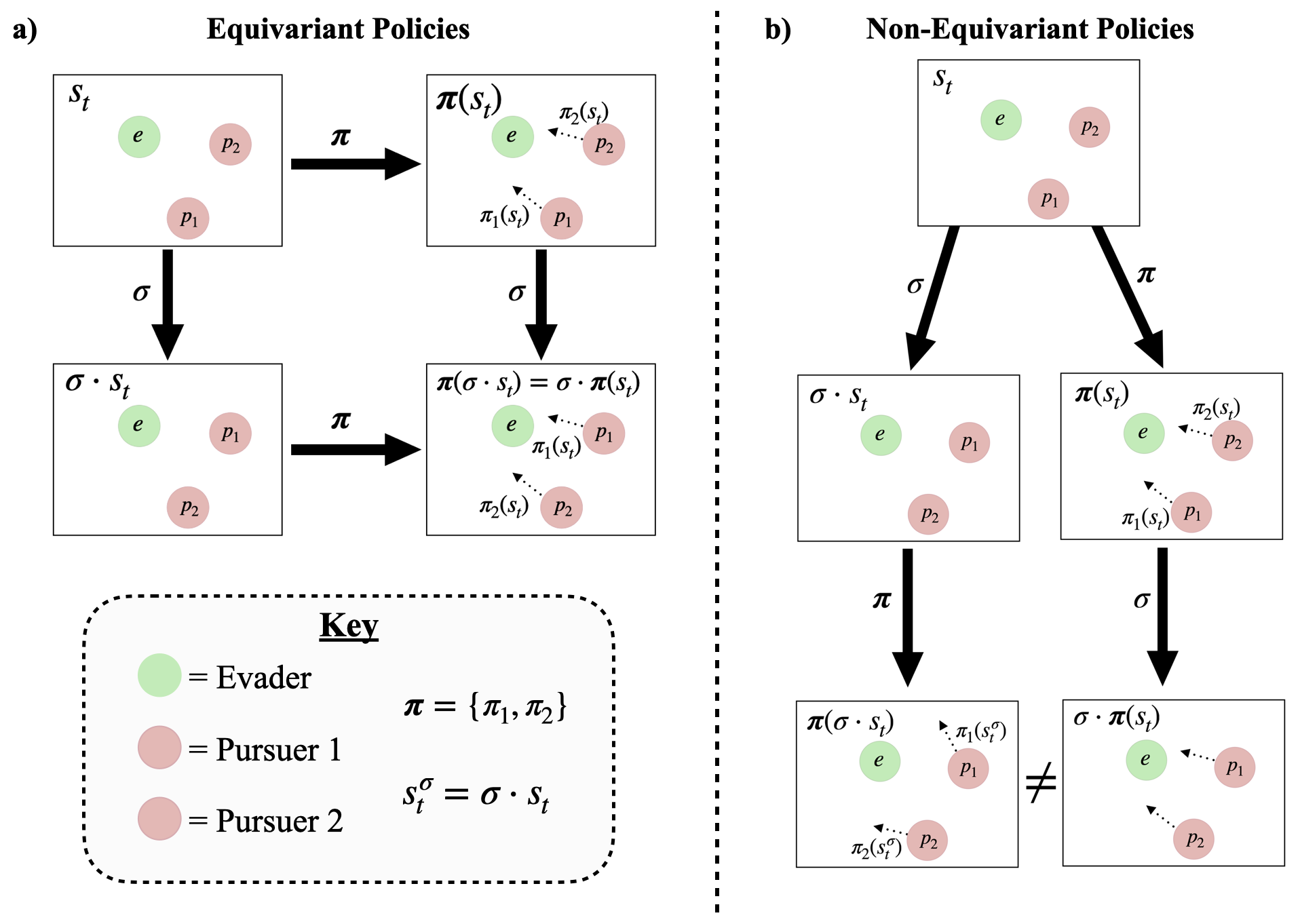}
    \caption{Snapshot of a pursuit-evasion game. Pursuers $p_1$ and $p_2$ (red) chase an evader $e$ (green) with the goal of capturing it. Given a state $s_t$, each pursuer selects its next heading (dotted arrow) from its policy, yielding the joint policy $\boldsymbol{\pi} {=} \{\pi_1(s_t), \pi_2(s_t)\}$. Evader action selection is omitted for clarity. \textbf{a)} For an equivariant joint policy, applying the transformation $\sigma$ to the state $s_t$ (producing $\sigma \cdot s_t$ which, in this example, swaps the positions of $p_1$ and $p_2$) and running the policy $\boldsymbol{\pi}(\sigma \cdot s_t)$ is equivalent to running the joint policy first and transforming the joint action afterwards (i.e. the commutative relationship $\boldsymbol{\pi}(\sigma \cdot s) = \sigma \cdot \boldsymbol{\pi}(s)$ holds). \textbf{b)} This commutative relationship does not necessarily hold for non-equivariant policies.}
    \label{fig:fairness_equivariance}
\end{figure*}

\paragraph{Pursuit-Evasion}
Here we again leverage the pursuit-evasion environment experimentally (see \cref{sec:background_environments_cooperative}). As in \cref{sec:implicit_comm_prelims}, we assume the evader to be part of the environment, defined by the potential-field policy:
\begin{equation}
    \label{eqn:fairness_evader_objective}
    U(\theta_e) = \sum_i \bigg(\frac{1}{r_i}\bigg) \cos(\theta_e - \tilde{\theta}_i)
\end{equation}

where $r_i$ and $\tilde{\theta}_i$ are the L2-distance and relative angle between the evader and the $i$-th pursuer, respectively, and $\theta_e$ is the heading of the evader. Intuitively, $U(\theta_e)$ pushes the evader away from pursuers, taking the largest bisector between any two when possible. The goal of the pursuers---to capture the evader as quickly as possible---is mirrored in the reward function, where $r(s_t, a_t) {=} 50.0$ if the evader is captured and $r(s_t, a_t) {=} -0.1$ otherwise. Note that $\lvert \vec{v}_p \rvert$ serves as a proxy for agent skill level. When $\lvert \vec{v}_p \rvert > \lvert \vec{v}_e \rvert$, pursuers are skilled enough to capture the evader on their own, whereas $\lvert \vec{v}_p \rvert \leq \lvert \vec{v}_e \rvert$ requires that pursuers work together.

\section{Method: Equivariant Policy Learning}
\label{sec:fairness_fair_marl}
In this section, we present a novel interpretation of fairness for cooperative multi-agent teams. We then introduce our proposed method---Fairness through Equivariance (Fair-E)---and prove that it yields fair outcomes. Finally, we present Fairness through Equivariance Regularization (Fair-ER) as a soft-constraint version of Fair-E. 

\paragraph{Notation}
Let $n$ be the number of agents in a cooperative team. We describe each agent $i$ by variables $v_i = (z_i, x_i)$, consisting of sensitive variables $z_i \in Z$ and non-sensitive variables $x_i \in X$. In team settings, we define non-sensitive variables $x_i$ to be any variables that affect agent $i$’s performance on the team; such as maximum speed. Other variables that should not impact team performance, such as an agent’s identity or belonging to a minority group, are defined as sensitive variables $z_i$. We define fairness in terms of reward distributions $R$---where each $\boldsymbol{r} \in R$ is a vectorial team reward and each component $r_i$ is agent $i$'s contribution to $\boldsymbol{r}$. In the following definitions, let $I(R;Z)$ be the mutual information between reward distributions $R$ and sensitive variables $Z$.

\subsection{Team Fairness}
\label{sec:fairness_team_fairness}
 We now define team fairness, a group-based fairness measure for multi-agent learning.
\begin{definition}[Exact Team Fairness]
    \label{def:fairness_exact_fairness}
    A set of cooperative agents achieves exact team fairness if $\textrm{I}(R;Z) = 0$.
\end{definition}

\begin{definition}[Approximate Team Fairness]
    \label{def:fairness_approx_fairness}
     A set of cooperative agents achieves approximate team fairness if $\textrm{I}(R;Z) \leq \epsilon$ for some $\epsilon > 0$.
\end{definition}

\noindent Team fairness connects cooperative multi-agent learning to group-based fairness, as $\textrm{I}(R;Z) = 0$ is equivalent to requiring $R \perp Z$ \cite{barocas2019fairmlbook}.

\subsection{Fairness Through Equivariance}
\label{sec:fairness_equivariance}
To enforce team fairness during policy optimization, we introduce a novel multi-agent learning strategy. The key to our approach is equivariance: by enforcing parameter symmetries \cite{ravanbakhsh2017equivariance} in each agent $i$'s policy network $\pi_{\phi_i}$, we show that equivariance propagates through the multi-agent RL problem. In particular, we show that the joint policy $\boldsymbol{\pi} = \{\pi_{\phi_1}, ..., \pi_{\phi_n} \}$ is an equivariant map with respect to permutations over state and action space. Further, we show that equivariance in policy-space begets equivariance in trajectory-space; namely, the terminal state $s_T$ following a multi-agent trajectory is equivariant to that trajectory's initial state $s_1$. Finally, we prove that equivariance in multi-agent policies and trajectories yields exact team fairness. A comparison of equivariant vs. non-equivariant joint policies is provided in \cref{fig:fairness_equivariance}.

In the proofs that follow, we assume: (i) homogeneity across agents on the team---i.e. agents are identical in their non-sensitive variables $x$; (ii) the distribution of agent positions satisfies exchangeability. Finally, though our derivations utilize general (stochastic) policies, we provide equivalent proofs for deterministic policies in \cref{apdx:fairness_apdx_deterministic_proofs}.

\begin{theorem}
    \label{thm:fairness_policy_eqv_map}
    If individual policies $\pi_{\phi_i}$ are symmetric, then the joint policy $\boldsymbol{\pi} = \{\pi_{\phi_1}, ..., \pi_{\phi_n} \}$ is an equivariant map.
\end{theorem}
\begin{proof}
    Let $\sigma$ be a permutation operator that, when applied to a vector (such as a state $s_t$ or action $\boldsymbol{a}_t$), produces a permuted vector ($\sigma \cdot s_t = s^\sigma_t$ or $\sigma \cdot \boldsymbol{a}_t = \boldsymbol{a}^\sigma_t$, respectively). Under parameter symmetry (i.e. $\phi_1$=$\phi_2{= \cdots =}\phi_n$), we have:
    \begin{equation}
        \label{eqn:fairness_equivariant_policy}
        \boldsymbol{\pi}(\sigma \cdot s) 
        = \boldsymbol{\pi}(s^{\sigma}) 
        = \boldsymbol{a}^\sigma 
        = \sigma \cdot \boldsymbol{a} 
        = \sigma \cdot \boldsymbol{\pi}(s)
    \end{equation}
    \noindent where the commutative relationship $\boldsymbol{\pi}(\sigma \cdot s) = \sigma \cdot \boldsymbol{\pi}(s)$ implies that $\boldsymbol{\pi}$ is an equivariant map. Commutativity here is crucial---\cref{eqn:fairness_equivariant_policy} and therefore \cref{thm:fairness_states_eqv} and \cref{thm:fairness_deterministic_sym_fair} do not hold for non-equivariant policies (see \cref{fig:fairness_equivariance}b).
\end{proof}

\begin{theorem}
    \label{thm:fairness_states_eqv}
    Let $p^\pi(s \rightarrow s', k)$ be the probability of transitioning from state $s$ to state $s'$ in $k$ steps \cite{sutton2018reinforcement}.
    Given that the joint policy $\boldsymbol{\pi}$ is an equivariant map, it follows that $p^{\boldsymbol{\pi}}(s_1 \rightarrow s_T, T) = p^{\boldsymbol{\pi}}(s_1^\sigma \rightarrow s_T^\sigma, T)$.
\end{theorem}
\begin{proof}
    It follows from our assumption of agent homogeneity that permuting a state $\sigma \cdot s_t$, which (from \cref{thm:fairness_policy_eqv_map}) permutes action selection $\sigma \cdot \boldsymbol{a}_t$, also permutes the environment's transition probabilities:
    \begin{equation*}
        P(s_{t+1} \mid s_t, \boldsymbol{a}_t) = P(s_{t+1}^\sigma \mid s_t^\sigma, \boldsymbol{a}_t^\sigma)
    \end{equation*}
    This is because, from the environment's perspective, a state-action pair is indistinguishable from the state-action pair generated by the same agents after swapping their positions and selected actions. Assuming a uniform distribution of start-states $P_\emptyset$, we also have $P_\emptyset(s_1) = P_\emptyset(s_1^\sigma)$. Recall the probability of a trajectory from \cref{eqn:background_traj_prob}. Given the equivariant function $\boldsymbol{\pi}$ and the two equalities above, it follows that:
    \begin{align*}
        P_\emptyset(s_1)\prod_{t=1}^T  & P(s_{t+1} \mid s_t, a_t)\boldsymbol{\pi}(a_t \mid s_t) = P_\emptyset(s_1^\sigma)\prod_{t=1}^T  P(s_{t+1}^\sigma \mid s_t^\sigma , a_t^\sigma )\boldsymbol{\pi}(a_t^\sigma \mid s_t^\sigma)
    \end{align*}
    We can represent the probability of a trajectory as a single transition from initial state $s_1$ to terminal state $s_T$ by marginalizing out the intermediate states, so it follows that:
    \begin{align*}
        p^{\boldsymbol{\pi}}(s_1 \rightarrow s_T, T) &= \int_{s_1} \cdots \int_{s_{T-1}} P_\emptyset(s_1)\prod_{t=1}^T  P(s_{t+1} \mid s_t, a_t)\boldsymbol{\pi}(a_t \mid s_t) \\
        &= \int_{s_1^\sigma} \cdots \int_{s_{T-1}^\sigma} P_\emptyset(s_1^\sigma)\prod_{t=1}^T  P(s_{t+1}^\sigma \mid s_t^\sigma, a_t^\sigma)\boldsymbol{\pi}(a_t^\sigma \mid s_t^\sigma) \\
        &= p^{\boldsymbol{\pi}}(s_1^\sigma \rightarrow s_T^\sigma, T)
    \end{align*}
    Thus, the probability of reaching terminal state $s_T$ from initial state $s_1$ is equivalent to the probability of reaching $s_T^\sigma$ from $s_1^\sigma$.
\end{proof}

\begin{theorem}
    \label{thm:fairness_sym_fair}
    Equivariant policies are exactly fair with respect to team fairness.
\end{theorem}
\begin{proof}
    The proof follows directly from \cref{thm:fairness_states_eqv}. Since $p^{\boldsymbol{\pi}}(s_1 \rightarrow s_T, T) = p^{\boldsymbol{\pi}}(s_1^\sigma \rightarrow s_T^\sigma, T)$, the probability of the agents obtaining reward $\boldsymbol{r}$ must be equal to obtaining reward $\boldsymbol{r}^\sigma$. Under the full distribution of initial states, the equality:
    \begin{equation*}
        P[R=\boldsymbol{r} \mid Z=\boldsymbol{z}] = P[R=\boldsymbol{r}^\sigma \mid Z=\boldsymbol{z}^\sigma]
    \end{equation*}
    \noindent holds for all $\boldsymbol{r}$ and assignments of sensitive variables $\boldsymbol{z}$. This is only possible if $R \perp Z$ and, therefore, $\textrm{I}(R;Z) = 0$, which meets exact team fairness.
\end{proof}

\subsection{Fairness Through Equivariance Regularization}
Though Fair-E achieves team fairness, it does so in a rigid manner---imposing hard constraints on policy parameters. Fair-E therefore has no choice but to pursue fairness to the fullest extent (and accept the maximum utility trade-off in return). In many cases, it is advantageous to tune the strength of the fairness constraints. For this reason, we propose a soft-constraint version of Fair-E, which we call Fairness through Equivariance Regularization (Fair-ER). Fair-ER is defined by the following regularization objective:
\begin{equation}
    \label{eqn:fairness_eqv_obj}
    J_{\textrm{eqv}}(\phi_1, ..., \phi_i, ..., \phi_n)
    = \underset{s}{\mathbb{E}}[\underset{j \neq i}{\mathbb{E}}[1 - \cos(\pi_{\phi_i}(s) - \pi_{\phi_j}(s))\mid_{s=s_t}]]
\end{equation}

\noindent which encourages equivariance by penalizing agents proportionally to the amount their actions differ from the actions of their teammates. Using \cref{eqn:fairness_eqv_obj}, Fair-ER extends the standard RL objective from \cref{eqn:background_ddpg_obj} as follows:
\begin{equation}
    J(\phi_i) + \lambda J_{\textrm{eqv}}(\phi_1, ..., \phi_i, ..., \phi_n)
\end{equation}
\noindent where $\lambda$ is a ``fairness control parameter" weighting the strength of equivariance. Differentiating the joint objective with respect to parameters $\phi_i$ produces the Fair-ER policy gradient:
\begin{equation}
    \label{eqn:fairness_eqv_grad}
    \nabla_{\phi_i} J_{\textrm{eqv}}(\phi_i)
    = \underset{s}{\mathbb{E}}\bigg[\sum_i \frac{1}{N-1} \sum_{j \neq i} \sin(\pi_{\phi_i}(s) - \pi_{\phi_j}(s)) \nabla_{\phi_i} \pi_{\phi_i}(s) \mid_{s=s_i}\bigg]
\end{equation}

\noindent In this work, Fair-ER is applied to each agent's actor network by optimizing \cref{eqn:fairness_eqv_grad} alongside \cref{eqn:background_ddpg_grad}. Though the above derivations consider stochastic policies, we highlight that Fair-ER is also applicable to deterministic policies and is therefore useful to any multi-agent policy gradient algorithm. We provide further background and a derivation of \cref{eqn:fairness_eqv_grad} in \cref{apdx:fairness}.

\section{Results}
\label{sec:fairness_results}
Pursuit-evasion allows us to quantify the performance of emergent team behavior (in terms of both team success and fairness) under variations of ``socio-economic" parameters such as shared objectives and agent skill-level. We therefore use the pursuit-evasion game formalized in \cref{sec:fairness_preliminaries} to verify our methods. In each experiment, $n{=}3$ pursuer agents are trained in a decentralized manner (each following DDPG) for a total of 125,000 episodes, during which velocity is decreased from $\lvert \vec{v}_p\rvert = 1.2$ to $\lvert \vec{v}_p\rvert = 0.4$. The evader speed is fixed at $\lvert \vec{v}_e\rvert = 1.0$. After training, we test the resulting policies at discrete velocity steps (e.g. $\lvert \vec{v}_p\rvert {=} 1.0$, $\lvert \vec{v}_p\rvert {=} 0.9$, etc), where a decrease in $\lvert \vec{v}_p\rvert$ represents a lesser skilled pursuer. We define the sensitive attribute $z_i$ for each agent $i$ to be a unique identifier of that agent (i.e. $z_i=[0,0,1], z_i=[0,1,0]$ or $z_i=[1,0,0]$ in the $n=3$ case). Each method is evaluated in terms of both utility---through traditional measures of performance such as success rate---and fairness---through the team fairness measure proposed in \cref{sec:fairness_team_fairness}.

Our evaluation proceeds as follows: first, we study the role of mutual reward in coordination by comparing policies trained with mutual reward to those trained with individual reward. Next, we show that naive mutual reward maximization results in high utility at the expense of fairness. We then show the efficacy of our proposed solution, Fair-E, in resolving these fairness issues. Finally, we evaluate our soft-constraint method, Fair-ER, in balancing fairness and utility.

\subsection{Importance of Mutual Reward}
\begin{figure}[]
    \centering
    \makebox[\linewidth][c]{\includegraphics[width=0.75\linewidth]{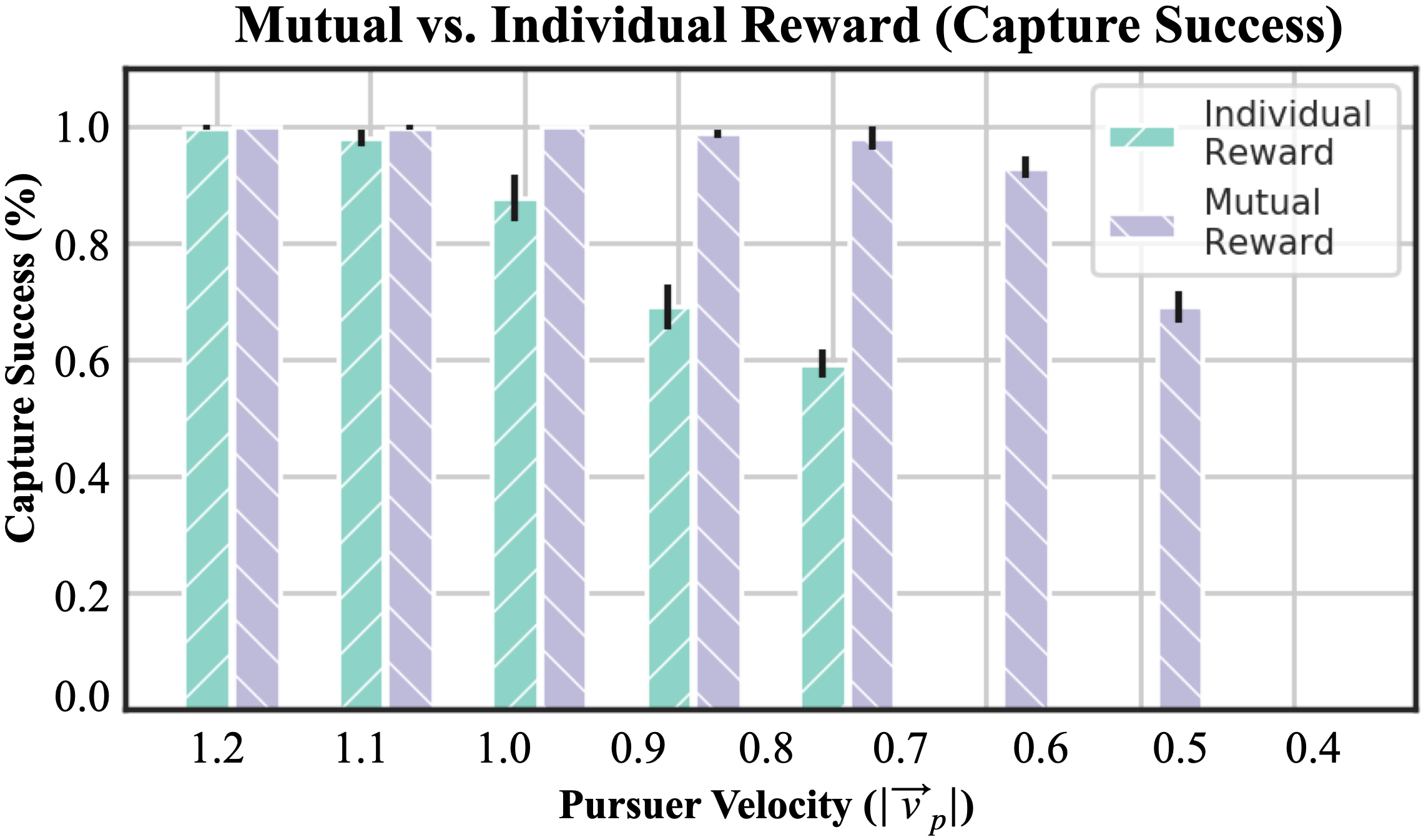}}
    \caption{Performance of policies trained with individual vs. mutual reward. As pursuer velocity decreases, the pursuit-evasion task requires more sophisticated coordination.}
    \label{fig:fairness_ind_vs_mut}
\end{figure}
We train pursuer policies with decentralized DDPG under conditions of either mutual or individual reward. In the mutual reward condition, pursuers share in the success of their teammates, each receiving the sum of the reward vector $\boldsymbol{r}$. In the individual reward condition, a pursuer is only rewarded if it captures the prey itself, which makes the pursuit-evasion task competitive. The results are shown in \cref{fig:fairness_ind_vs_mut}, where utility is the capture success rate of the multi-agent team.

We find that pursuers trained with mutual reward significantly outperform those trained with individual reward. Mutual reward pursuers maintain their performance even as speed drops to $\lvert \vec{v}_p\rvert=0.5$; which is only half of the evader's speed. Under individual reward, performance drops off quickly for $\lvert \vec{v}_p\rvert \leq 1.0$. The velocity $\lvert \vec{v}_p\rvert=1.0$ represents a crucial turning-point in pursuit-evasion---it is the point at which a straight-line chase towards the prey no longer works. These results show that, without mutual reward, the pursuers are not properly incentivized to work together and therefore do not develop a coordination strategy that is any better than a greedy individual pursuit of the evader. Thus, we confirm that mutual reward (a single, shared objective) is vital to coordination.

\subsection{Fair Outcomes With Fair-E}
\begin{figure}
    \centering
    \includegraphics[width=0.7\linewidth]{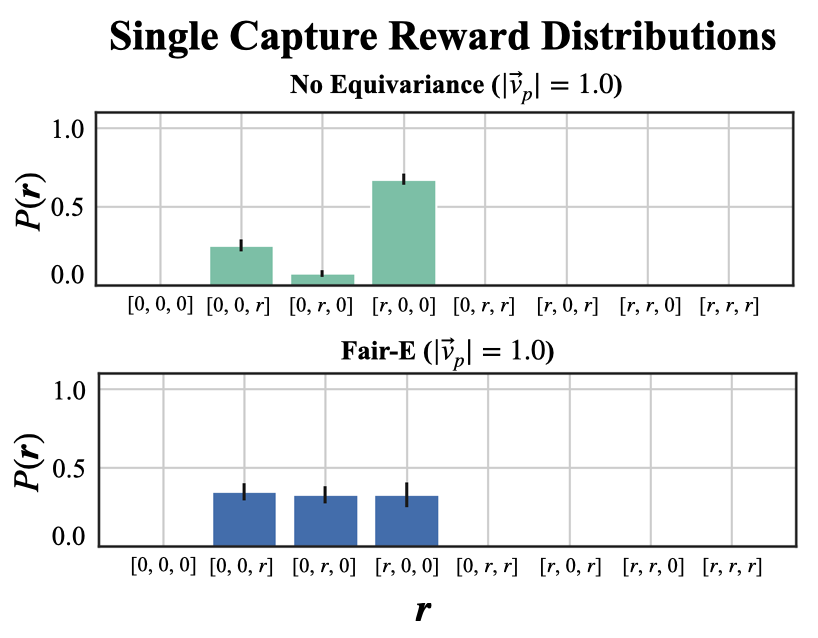}
    \caption{Quantitative comparison of performance for equivariant (i.e. Fair-E) vs. non-equivariant policies. Distribution of reward vectors for both strategies at the pursuer velocity $\lvert \vec{v}_p\rvert=1.0$. Non-equivariant policies (top), which learn strategies that push captures towards one agent, yield highly uneven reward distributions. Fair-E policies (bottom) learn to spread captures amongst teammates equally, resulting in even reward distributions.}
    \label{fig:fairness_ind_vs_sym_dist}
\end{figure}

Though mutual reward incentivizes efficient team coordination, it does not stipulate \textit{how} agents should coordinate. To study the nature of the resulting strategy, we examine the distribution of reward vectors obtained by the pursuers over 100 test-time trajectories (averaged over five random seeds each). As shown in \cref{fig:fairness_ind_vs_sym_dist} (top), in which we plot reward vector assignments for captures involving only one pursuer, the pursuer team discovers an unfair strategy---the majority of captures are accounted for by a single agent.
\begin{figure*}[t!]
    \centering
    \includegraphics[width=0.75\linewidth]{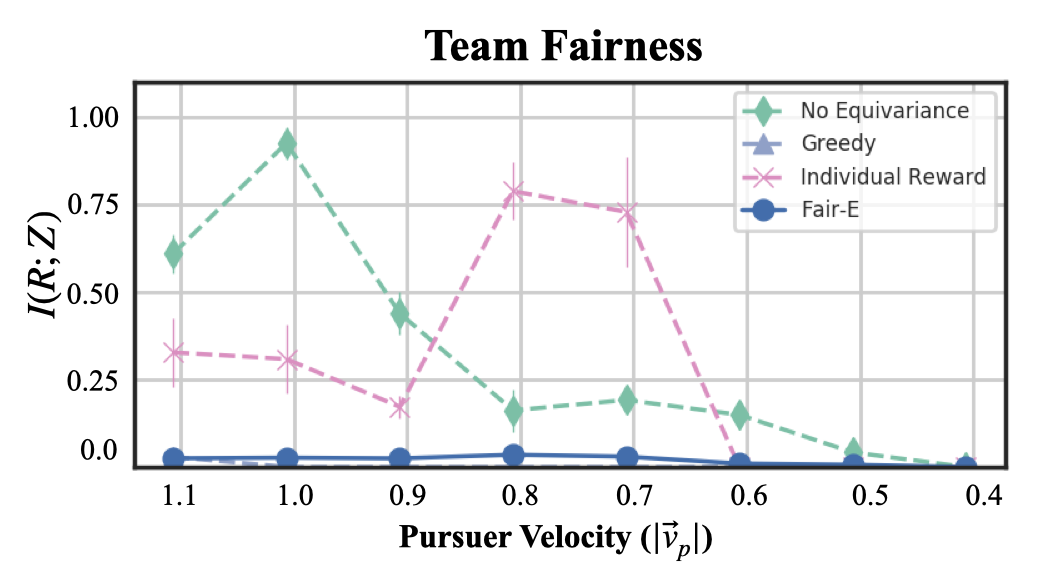}
    \caption{Team fairness scores for both strategies (lower better). Fair-E policies yield fairer outcomes than non-equivariant policies across all pursuer velocity (i.e. agent skill) levels. Note that the curve for the Greedy strategy, which is fair despite its low utility, is tucked behind the Fair-E curve.}
    \label{fig:fairness_ind_vs_sym_fairness}
\end{figure*}

The emergence of an unfair strategy reflects the difficulty of the pursuit-evasion setting. As $\lvert \vec{v}_p\rvert$ decreases, the whole pursuer team has to learn to work together to capture the evader, which is a challenging coordination task. The pursuers learn to do this effectively by assigning roles---e.g. in the $n{=}3$ case, two pursuers take supporting roles, shepherding the evader towards the third agent, who is designated the "capturer". We note that the decision of which agent becomes the capturer is an emergent phenomenon of the system. As $\lvert \vec{v}_p\rvert$ decreases further, such role assignment is not only helpful but necessary for success. Altogether, the results suggest that unconstrained mutual reward maximization prioritizes utility over fairness.

Our proposed solution, Fair-E, directly combats these fairness issues. \Cref{fig:fairness_ind_vs_sym_dist} (bottom) shows the distribution of reward vectors obtained by agents trained with Fair-E. Due to equivariance, Fair-E yields much more evenly distributed rewards. To further quantify these gains, we compare team fairness for both strategies over a variety of skill levels (i.e. $\lvert \vec{v}_p\rvert$ values). The results, shown in \cref{fig:fairness_ind_vs_sym_fairness}, confirm that Fair-E achieves much lower $I(R;Z)$ and, therefore, higher team-fairness. Note that, when $\lvert \vec{v}_p\rvert < 0.9$, $I(R;Z)$ is low for non-equivariant pursuers as well. This is an artifact of team fairness---as $\lvert \vec{v}_p\rvert$ decreases, capture success inevitably decreases as well, which is technically a fairer, albeit less desirable, outcome (all agents share equitably in failure). Nevertheless, \cref{fig:fairness_ind_vs_sym_fairness} serves as empirical evidence to backup our theoretical result from \cref{sec:fairness_equivariance} that Fair-E meets the demands of team fairness.
\begin{figure*}[t!]
    \centering
    \includegraphics[width=0.75\linewidth]{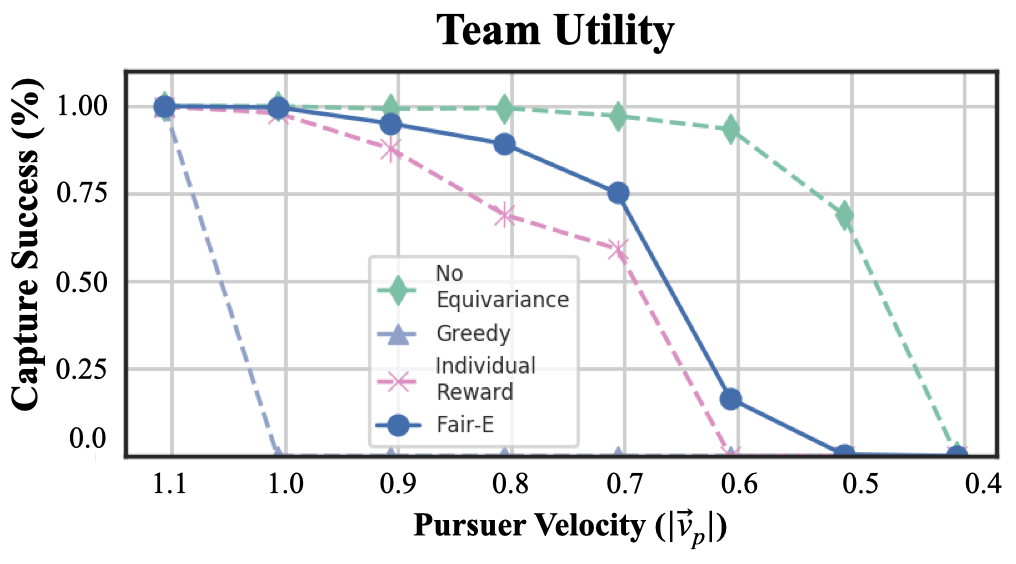}
    \caption{Team utility (i.e. capture success) achieved by both strategies (higher better). As pursuer velocity decreases, non-equivariant policies outperform Fair-E by a wide margin, indicating that a fairness-utility trade-off exists.}
    \label{fig:fairness_ind_vs_sym_utility}
\end{figure*}

Despite achieving fairer outcomes, Fair-E is subject to drops in utility as $\lvert \vec{v}_p\rvert$ decreases (see \cref{fig:fairness_ind_vs_sym_utility}). The utility curve for Fair-E drops precipitously for agent skill $\lvert \vec{v}_p\rvert < 0.9$; much faster than the drop-off for pursuers with no equivariance. This is because Fair-E directly prevents role assignment. By hard-constraining each agent's policy, Fair-E enforces $\pi_i(s_t) = \pi_j(s_t)$, whereas role assignment requires $\pi_i(s_t) \neq \pi_j(s_t)$ for $i \neq j$. We emphasize role assignment as key to this result, as parameter-sharing has been shown to be helpful in problem domains that do not require explicit role assignment \cite{baker2019emergent}. In the context of fairness, however, these results indicate that Fair-E will always elect to give up utility to preserve fairness.

For completeness, we also show results for the policies learned with individual reward (from \cref{fig:fairness_ind_vs_mut}) and a hand-crafted greedy control strategy in which each pursuer runs directly towards the evader. Note that greedy policies are equivariant---by definition, agents will select similar actions in similar states---but demonstrate no coordination. For this reason, greedy policies have high fairness, but very low utility. Utility follows the same pattern for individual reward policies. Interestingly though, individual reward policies become less fair between $\lvert \vec{v}_p\rvert=0.9$ and $\lvert \vec{v}_p\rvert=0.6$, before tapering off as performance decreases. We defer further discussion of this finding, as well as details regarding the computation of the team fairness score, $I(R;Z)$, and the hand-crafted greedy control baseline to \cref{apdx:fairness_apdx_experimental_details}.

\subsection{Modulating Fairness With Fair-ER}
\begin{figure}[t!]
    \centering
    \makebox[\linewidth][c]{\includegraphics[width=0.9\linewidth]{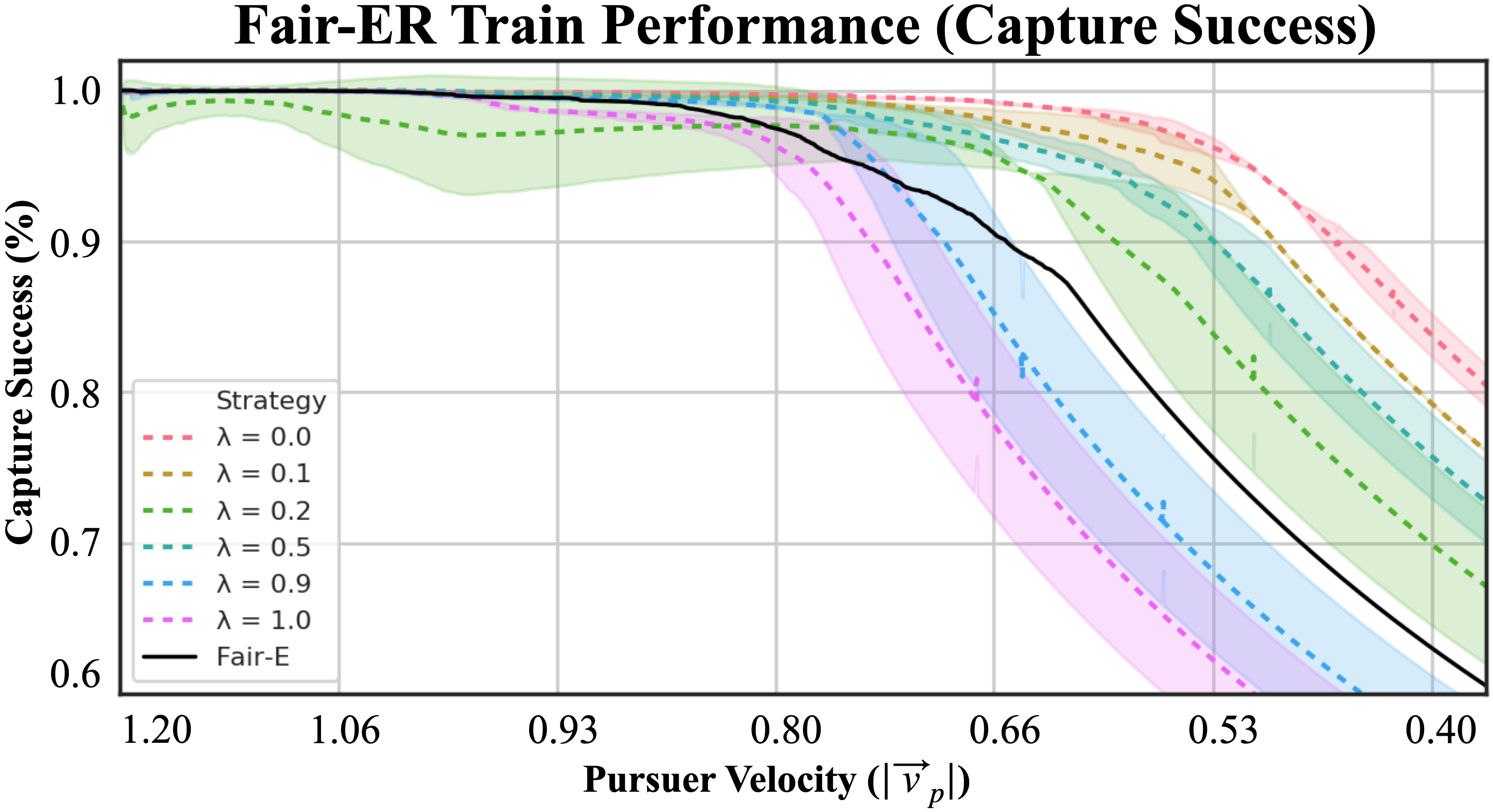}}
    \caption{Effect of the equivariance control parameter $\lambda$ on policy learning with Fair-ER. Increasing $\lambda$ yields fairer policies, but causes performance to decay more quickly in difficult environments (lower agent skill). The black line represents Fair-E's performance for comparison.}
    \label{fig:fairness_fairer_training}
\end{figure}
Unlike Fair-E, Fair-ER allows policies to balance fairness and utility dynamically. Intuitively, this is because Fair-ER incentivizes policy equivariance through the regularization objective from \cref{eqn:fairness_eqv_obj}, while still allowing agents to update their own individual policy parameters (unlike Fair-E). Therefore, the value of the fairness control weight $\lambda$ will dictate how much each agent values fairness vs. utility. 

To study the effectiveness of this method, we trained Fair-ER agents in increasingly difficult environments (by decreasing pursuer velocity $\lvert \vec{v}_p\rvert$) while modulating the fairness control parameter $\lambda$. The effect of $\lambda$ on policy training is shown in \cref{fig:fairness_fairer_training}. The results show that, for $\lambda \le 0.5$, Fair-ER is successful in bridging the performance gap between non-equivariant policies ($\lambda = 0.0)$ and Fair-E (black line). Importantly, we show that it is also possible to over-constrain the system so that it actually performs worse than Fair-E (e.g. $\lambda = 1.0$). This indicates that, though Fair-ER can mitigate the drops in performance described above, the regularization parameter $\lambda$ must be tuned appropriately.

We also performed the same test-time analysis as described for Fair-E in the previous subsection. \Cref{fig:fairness_fairness_vs_utility} shows the effect of $\lambda$ on both fairness ($I(R;Z)$) and utility (capture success). For each skill level, increasing $\lambda$ allows Fair-ER to fine-tune the balance between fair and unfair policies, achieving the highest utility possible under its given constraints. We find that, with high values of $\lambda$ (e.g. $\lambda=0.9$), Fair-ER prioritizes fairness over utility and performs in-line with (or worse than) Fair-E---achieving fair outcomes, even at the expense of utility. When $\lambda$ is in the range $\lambda=0.5$ to $\lambda=0.1$, Fair-ER withstands a drop in utility until $\lvert \vec{v}_p\rvert=0.7$ by giving up small amounts of fairness. Therefore, we find evidence that learning multi-agent coordination strategies with Fair-ER simultaneously maintains \textit{higher utility} than Fair-E while achieving \textit{higher fairness} than non-equivariant learning. Overall, tuning the fairness weight $\lambda$ allows us to directly control the strength of the fairness constraints imposed on the system, enabling Fair-ER to modulate fairness to the needs of the task.

\subsection{Fairness-Utility Trade-off}
\begin{figure}[t!]
    \centering
    \makebox[\linewidth][c]{\includegraphics[width=0.8\linewidth]{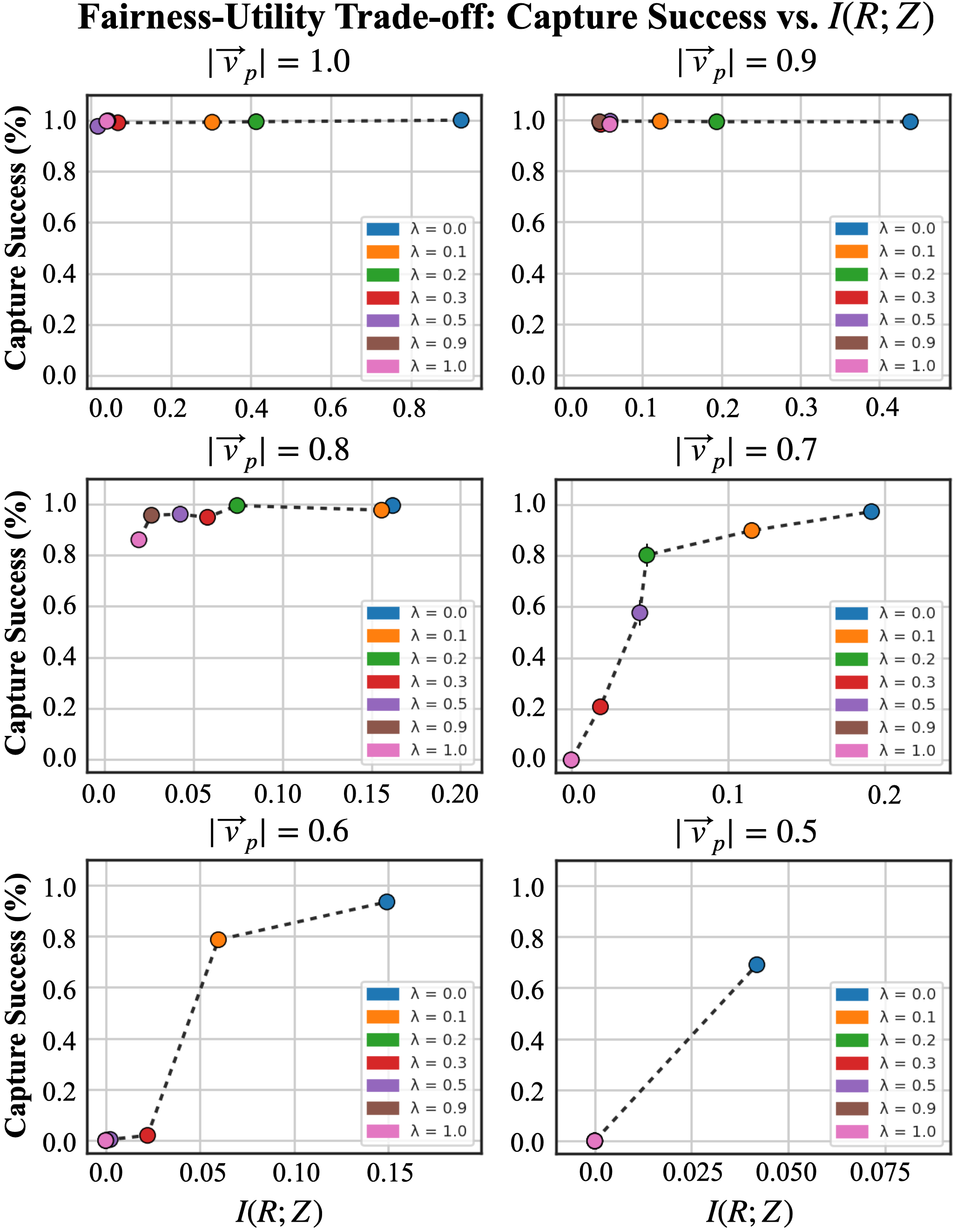}}
    \caption{Fairness vs. utility comparisons for Fair-ER trained with various values of equivariance regularization $\lambda$ and environment difficulties. Note that $\lambda=0.0$ is equivalent to no equivariance.}
    \label{fig:fairness_fairness_vs_utility}
\end{figure}
As we saw in \cref{fig:fairness_ind_vs_sym_utility}, the hard constraints that Fair-E places on each agent's policy creates an inherent commitment to achieving fair outcomes at the expense of utility (or reward). In this section, we examine the extent to which the fairness-utility trade-off exists for Fair-ER across all agent skill levels. For each agent skill level (i.e. $\lvert \vec{v}_p\rvert$ value), we computed both the fairness (through team fairness $I(R;Z)$) and utility (through capture success) scores achieved by multi-agent coordination strategies learned for $\lambda$ values in the range $\lambda \in \{0.0, 1.0\}$ over 100 test-time trajectories (averaged over five random seeds each). The results of this experiment are shown in \cref{fig:fairness_fairness_vs_utility}. 

Unlike many prior studies in both traditional multi-agent fairness settings \cite{okun2015equality, le1990equity, bertsimas2012efficiency, joe2013multiresource, bertsimas2011price} and prediction-based settings \cite{corbett2017algorithmic, zhao2019inherent}, we find that it is not always the case that fairness must be traded for utility. With Fair-ER, fairness comes for little to no cost in utility until $\lvert \vec{v}_p\rvert=0.8$. This means that, when each agent operates at a high skill level, requiring each agent in the multi-agent team to shift towards an equivariant policy (which yields fair results) does not cause coordination of the larger multi-agent team to break down. When $\lvert \vec{v}_p\rvert<0.8$, however, utility drops quickly for larger values of $\lambda$. This indicates that, when agent skill decreases (or the task becomes more complex relative the agents' current skill level), unfair strategies such as role assignment are the only effective way to maintain high levels of utility. Overall, these results serve as empirical evidence that, in the context of cooperative multi-agent tasks, fairness is inexpensive, so long as the task is easy enough (i.e. agent skill is high enough). As task difficulty increases, fairness comes at an increasingly steep cost. To the best of our knowledge, such a characterization of the fairness-utility trade-off for multi-agent settings has not been illustrated in the fairness literature.

\section{Conclusion and Future Work}
Multi-agent learning holds promise for helping AI researchers, economic theorists, and policymakers alike better evaluate real-world problems involving social structures, taxation, policy, and economic systems broadly. This work has focused on one such problem; namely, fairness in cooperative multi-agent settings. In particular, we have demonstrated that fairness issues arise naturally in cooperative, single-objective multi-agent learning problems. We have shown that our proposed method, equivariant policy optimization (Fair-E), mitigates such issues. We have also shown that soft constraints (Fair-ER) lower the cost of fairness and allow the fairness-utility trade-off to be balanced dynamically. Moreover, we have presented novel results regarding the fairness-utility trade-off for cooperative multi-agent settings; identifying a connection between agent skill and fairness. In particular, we showed that fairness comes for free when agents are highly-skilled, but becomes increasingly expensive for lesser-skilled agents.

This work represents a first step towards understanding the core factors underlying fairness and multi-agent learning in environments where team dynamics and coordination are important for task success. There are a number of exciting avenues of future work that build upon these initial ideas. First, ongoing work is investigating cooperative multi-agent fairness in more complex domains (e.g. video games, simulated economic societies). Moreover, there is room to explore indirect or backdoor causal paths between sensitive and target variables in the context of multi-agent teams, which warrant connecting additional interpretations of fairness (e.g. causal fairness) to cooperative multi-agent settings.

\part{Emergent Multi-Agent Behavior and Robustness}
\label{part:robustness}
In this part of the thesis, we turn our attention to large-scale AI systems such as AlphaZero \cite{silver2017alphazero}, and examine them through the lens of \textbf{robustness}. Specifically, we seek to address the following questions:
\begin{itemize}
    \item \textit{Why does AlphaZero's performance degrade so severely in the presence of simple adversarial examples?}
    \item \textit{Can we algorithmically improve the robustness of AlphaZero's policy and value networks?}
\end{itemize}
In \cref{chapter:robustness}, we conduct a thorough analysis of AlphaZero's failure modes and identify two issues at the root of AlphaZero's brittleness: policy-value misalignment and value function inconsistency. This analysis informs novel extensions to the AlphaZero algorithm that improve its robustness. Experimental results demonstrate the effectiveness of our method across a range of board games.

\chapter{Policy-Value Alignment and Robustness in Search-based Multi-Agent Learning}
\label{chapter:robustness}

\section{Introduction}
Learning algorithms that combine self-play reinforcement learning and search have grown in popularity in recent years due to their ability to reach superhuman levels of performance in combinatorially complex environments~\cite{silver2017alphago}. More recently, an examination of AlphaZero \cite{silver2017alphazero} revealed that it obeys neural scaling laws with respect to play strength and network capacity~\cite{neumann2022scaling}---further suggesting that hybrid algorithms, as in other domains (e.g. large language models~\cite{kaplan2020scaling}), may have significant room to scale further. As these algorithms continue to scale towards real-world complexity, it calls into question whether state-of-the-art hybrid systems are ready for deployment.

To date, researchers and practitioners have a very limited understanding of why the neural networks underpinning these hybrid algorithms make the decisions that they do---consider the infamous ``Move 37" in AlphaGo vs. Lee Sedol~\cite{move37alphago}---and what they might learn about the environment and other agents within it. Though significant progress has been made in the area of concept-based interpretability, where post-hoc analysis has shown that many human-understandable concepts can be accurately regressed from AlphaZero's representations \cite{mcgrath2022acquisition}, opaque artifacts of self-play learning still remain. For example, studies leveraging adversarial game-play have shown that AlphaZero is surprisingly susceptible to minor perturbations in opponent behavior \cite{lan2022alphazero} and can be exploited to lose games that could be easily won by amateur human players \cite{wang2022adversarial}.
\begin{figure*}[t!]
    \centering
    \includegraphics[width=0.99\textwidth]{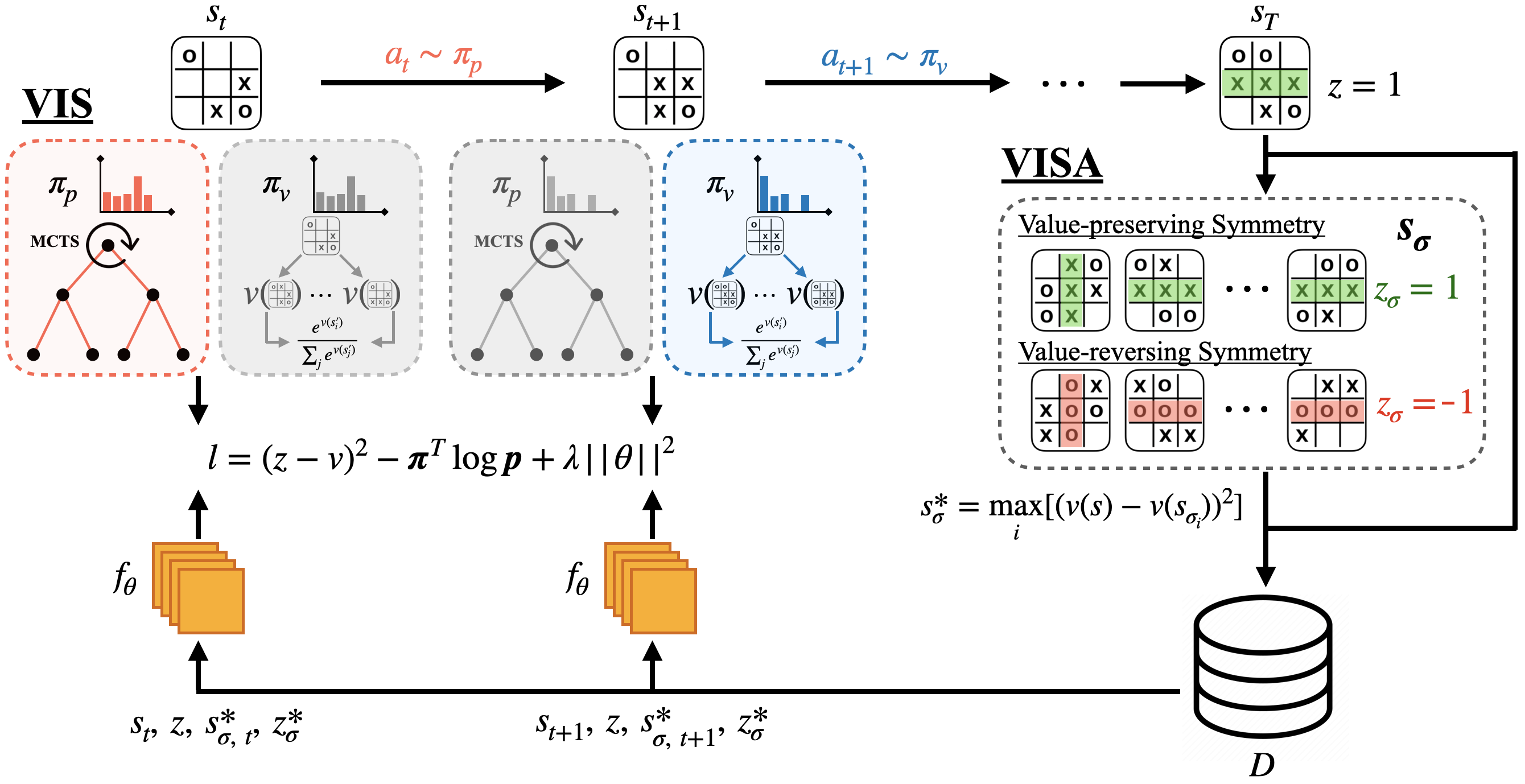}
    \caption{VISA-VIS combines two improvements to the AlphaZero algorithm: (i) \textbf{Value-Informed Selection (VIS)}, where action selection alternates between sampling actions from MCTS search probabilities ($a_t {\sim} \pi_p$) and a policy derived from AlphaZero's value function ($a_t {\sim} \pi_v$); and (ii) \textbf{Value-Informed Symmetric Augmentation (VISA)}, a form of data augmentation that targets uncertainty in AlphaZero’s value function. Before adding a state $s_t$ to the replay buffer $D$, VISA generates a set of symmetric state transformations $\boldsymbol{s_\sigma}$, evaluates each transformed state $s_{\sigma_i}$, and returns the state $s^*_\sigma$ that differs most from its value estimate for the original state $s_t$. Both $s_t$ and $s^*_\sigma$ are added to $D$.}
    \label{fig:overview}
\end{figure*}

In this work, we discover multiple phenomena that may underlie AlphaZero's documented failure modes. First, we find that, despite AlphaZero learning to play a game optimally \textit{within its search process}, its policy network and value function are \textbf{far from optimal themselves}. This reveals an interesting tension between learning and search: search is necessary to fight the combinatorial exploration problem of complex games, but also offers neural networks a crutch that leads to inadequate generalization. Next, we identify a novel emergent phenomenon, \textbf{policy-value misalignment}, in which AlphaZero's policy and value predictions contradict each other. This misalignment is a direct consequence of AlphaZero's objective, which includes separate terms for policy and value error, but no constraint on their consistency. We introduce an information-theoretic definition of policy-value misalignment and use it to quantify AlphaZero's misalignment. Lastly, we find evidence of \textbf{inconsistency within AlphaZero's value function} with respect to symmetries in the environment. As a result, AlphaZero's value predictions generalize poorly to infrequently-visited and unseen states. Crucially, these phenomena are \textbf{masked by search} and exist even when AlphaZero plays optimally within its search process. This work therefore demonstrates that, though search allows AlphaZero to achieve unprecedented scale and performance, it also hides deficiencies in its learned components.

To address these issues, we propose two modifications to AlphaZero aimed directly at policy-value alignment and value function consistency. First, we introduce Value-Informed Selection (VIS): a modified action-selection rule in which action probabilities are computed using a one-step lookahead with AlphaZero's value network. VIS then alternates stochastically between this value action selection rule and AlphaZero's standard policy (which is derived from search probabilities). During training, this forces alignment between AlphaZero's policy and value predictions. Next, we introduce Value-Informed Symmetric Augmentation (VISA): a data augmentation technique that takes into account uncertainty in AlphaZero's value network. Unlike data augmentation techniques that naively add states to a network's replay buffer, VISA considers symmetric transformations of a given state and stores only the transformed state whose value differs most from the network's value estimate for the original, un-transformed state.

VISA and VIS can improve both policy-value alignment and value robustness in AlphaZero simultaneously. We refer to this joint method as VISA-VIS. In experiments across a variety of symmetric, two-player zero-sum games, we show that VISA-VIS reduces policy-value misalignment by \textbf{up to 76\%}, reduces value generalization error by \textbf{up to 50\%}, and reduces average value prediction error \textbf{by up to 55\%}; as compared to AlphaZero.

\noindent \textbf{Contributions:}
In sum, we offer the following contributions: (i) We systematically study AlphaZero and identify multiple emergent phenomena that are masked by search. (ii) We define an information-theoretic measure of policy-value misalignment. (iii) We propose Value-Informed Selection (VIS): a modified action-selection rule for AlphaZero that forces alignment between policy and value predictions. (iv) We propose Value-Informed Symmetric Augmentation (VISA): a data augmentation technique that leverages symmetry and value function uncertainty to learn a more consistent value network. (v) We introduce a new method, VISA-VIS, that combines VIS and VISA into a single algorithm. Experimentally, we show that VISA-VIS improves the quality of AlphaZero's policy and value networks, and significantly reduces policy-value misalignment.

\section{Related Work}
The continued success of hybrid algorithms that combine search and self-play RL has drawn research interest into their robustness and interpretability.

\subsection{Adversarial Examples}
In supervised learning settings, it is well-documented that both synthetic \cite{goodfellow2014explaining} and naturally-occurring \cite{hendrycks2021natural} adversarial examples can induce sub-optimal behavior from neural network-based classifiers. Recently, similar tactics have proven effective in identifying the weaknesses of RL agents. For example, it has been shown that small perturbations of an agent's observations can cause large decreases in reward \cite{huang2017adversarial,kos2017delving} and that performance degradation can be exacerbated by leveraging knowledge of an agent's behavior \cite{lin2017tactics}.

In multi-agent settings, RL agents are particularly vulnerable to adversarial behavior from other agents in the environment \cite{gleave2019adversarial}. Of particular relevance is the work of \citet{wang2022adversarial}, which showed that a professional-level AlphaZero Go agent can be tricked into losing games by simple adversarial moves that are easily countered by amateur human players. Further studies have shown that AlphaZero is susceptible to other state- and action-space perturbation \cite{lan2022alphazero}. We posit that these surprising failure modes are a symptom of the phenomena we describe in this work---policy-value misalignment and value function inconsistency.

In multi-agent settings, RL agents are particularly vulnerable to adversarial behavior from other agents in the environment \cite{gleave2019adversarial}. Of particular relevance is the work of \citet{wang2022adversarial}, which showed that a professional-level AlphaZero Go agent can be tricked into losing games by simple adversarial moves that are easily countered by amateur human players. Further studies have shown that AlphaZero is susceptible to other state- and action-space perturbation \cite{lan2022alphazero}. We posit that these surprising failure modes are a symptom of the phenomena we describe in this work---policy-value misalignment and value function inconsistency.

\subsection{Interpretability}
RL interpretability methods are divided into those that improve the transparency of agent decision-making by either (i) introducing intrinsically-interpretable policy and value network architectures; or (ii) generating post-hoc explanations for previously-trained networks. Beyond more traditional approaches that are derived directly from supervised learning counterparts--such as saliency maps \cite{greydanus2018visualizing,selvaraju2017grad}---advances in reward-agnostic behavioral analysis \cite{omidshafiei2022beyond} and intrinsically-interpretable policy learning schemes \cite{grupen2022concept} have greatly improved the interpretability of an RL agent's decision-making (and are applicable to multi-agent settings). 

For large-scale AI systems such as AlphaZero, seminal work on post-hoc, concept-based interpretability has shown that many human-understandable concepts are encoded in AlphaZero's policy network during training \cite{mcgrath2022acquisition}. Additional studies have employed model probing and behavioral analysis to understand the concepts represented by AlphaZero for the game of Hex \cite{forde2022concepts}. These studies highlight the search-learning conundrum proposed in our work---it is possible for AlphaZero to learn complex playing strategies and sophisticated human-understandable concepts while simultaneously demonstrating inconsistencies in its policy and value predictions.

\subsection{Robustness and Generalization in RL}
The promise of RL in controlled environments has also drawn focus to their robustness and generalization in diverse real-world settings \cite{kirk2021survey}. As in supervised learning, adversarial approaches to robustness testing have been successful for exploiting learned RL behaviors, either by learning optimal destabilization policies~\cite{pinto2017robust} or sampling from populations of adversaries~\cite{vinitsky2020robust}. Following \citet{moosavi2017universal}, which showed that it is always possible to generate new adversarial examples for a given model, methods have focused on combining adversarial training with other methods to better defend against adversaries \cite{chen2019adversarial}. One exciting direction, motivated by human-agent interaction, is to construct behavioral unit tests for RL agents \cite{knott2021evaluating}.

\section{Method}
\label{sec:robustness_visa_vis}
In this section, we introduce VISA-VIS, which combines two modifications to the AlphaZero algorithm: Value-Informed Selection (VIS) and Value-Informed Symmetric Augmentation (VISA). We describe each in detail below, and provide an overview of our approach in \cref{fig:overview}.

\subsection{Value-Informed Selection}
\label{sec:robustness_vis}
Recall AlphaZero's loss function from \cref{eqn:background_az_loss}. Though this objective incentivizes accurate policy and value predictions individually, it imposes no constraint on their consistency. It is therefore possible for AlphaZero's network to output policy probabilities $\boldsymbol{p}$ that contradict its value prediction $v$---e.g. during inference, $\boldsymbol{p}$ places large probability mass on a winning action while $v$ predicts that the current state is a losing one. We find that this occurs frequently in practice. To enforce policy-value consistency, we introduce Value-Informed Selection (VIS). VIS is a novel action selection rule in which AlphaZero alternates between selecting actions with its policy (as informed by the search probabilities in \cref{eqn:background_mcts_policy}) and selecting actions with its value function.

More formally, let $\boldsymbol{s'}$ be the set of successor states from a given state $s$, where each successor $s'_a \in \boldsymbol{s'}$ is the result of taking a legal action $a$ from $s$. We then define $\pi_v$ to be the action probabilities obtained by computing a one-step value lookahead from $s$:
\begin{equation}
    \pi_v(a|s) = \frac{e^{v(s'_a)}}{\sum\limits_{s'_b \in \boldsymbol{S'}} e^{v(s'_b)}}
\end{equation}
where $v(s)$ is the value of state $s$, as predicted by AlphaZero's value network. Renaming the search policy defined in \cref{eqn:background_mcts_policy} to be $\pi_p$, VIS then defines the action selection rule:
\begin{equation}
  \pi_{\textrm{VIS}}(a|s) =
    \begin{cases}
      \pi_p(a|s) & \text{if} \; \; \eta < \epsilon\\
      \pi_v(a|s) & \text{otherwise.}\\
    \end{cases}
\end{equation}
where $\eta \sim U(0,1)$ is drawn uniformly and $\epsilon$ is a threshold determining the likelihood of policy vs. value action selection. When used in conjunction with \cref{eqn:background_az_loss}, VIS forces AlphaZero's policy and value predictions to be aligned by swapping $\boldsymbol{\pi_p}$ and $\boldsymbol{\pi_v}$ stochastically. Intuitively, VIS pushes AlphaZero to align its policy and value function because it never knows which it will have to use for any given action selection.

\subsection{Value-Informed Symmetric Augmentation}
\label{sec:robustness_visa}
Exhaustively searching a combinatorial state-space is impractical, so during training, AlphaZero relies upon MCTS to reduce the set of states it encounters. However, this also places the onus on AlphaZero's network to generalize to states that are skipped by MCTS during training. We find that AlphaZero's value predictions generalize poorly to such states.

To improve generalization in AlphaZero's value network, we propose Value-Informed Symmetric Augmentation (VISA): a form of data augmentation that explicitly targets uncertainty in AlphaZero's value function. For each state $s$ encountered during training, VISA generates a set $\boldsymbol{s_\sigma} = \{s_{\sigma_1}, ..., s_{\sigma_N}\}$ of $N$ additional states, where each $s_{\sigma_i} {\in} \boldsymbol{s_\sigma}$ is the result of applying a symmetric transformation $\sigma_i(s)$ to $s$. VISA then finds the transformed state $s^*_\sigma$ with maximum value uncertainty as:
\begin{equation*}
    s^*_\sigma = \max_i[(v(s) - v(s_{\sigma_i}))^2]
\end{equation*}
and adds both $s$ and $s^*_\sigma$ to AlphaZero's replay buffer. Note that, by adding only $s^*_\sigma$, VISA is an improvement over both random data augmentation---it targets the augmentation for which the value network is least certain---and exhaustive data augmentation---it is more memory efficient than storing all of $\boldsymbol{s_\sigma}$ in the replay buffer.

In this work, we consider three principal symmetric transformations: rotation and reflection from the dihedral group, and inversion. Intuitively, in the board game setting (as in \cref{fig:overview}), these transformations correspond to rotating the game board, mirroring the game board (vertically or horizontally), or flipping each player's pieces. Importantly, we note that board inversion also requires inverting the ground-truth game-tree value $z$ when computing \cref{eqn:background_az_loss}.

\subsection{Measuring Policy-Value Misalignment}
\label{sec:robustness_misalignment}
We also introduce an information-theoretic measure of policy-value misalignment. Given AlphaZero's standard policy $\pi_p$ and the value-informed policy $\pi_v$ from \cref{sec:robustness_vis}, it is possible to compute policy-value misalignment as the KL-divergence: $D_{\textrm{KL}}(\pi_p || \pi_v)$. Intuitively, a model with strong policy-value alignment will yield a lower divergence term, whereas policy and value networks that are inconsistent will result in higher divergence.

\section{Experiments}
In this section, we first demonstrate the presence of sub-optimal policies, policy-value misalignment, and value function inconsistency in AlphaZero. Then, we evaluate VISA-VIS and examine the extent to which it alleviates these phenomena.

\textbf{Environments:}
We consider three solved, symmetric two-player zero-sum games of increasing complexity: Tic-Tac-Toe, 4x4 Tic-Tac-Toe, and Connect Four. We emphasize the importance of studying solved games to our analysis. In order to properly measure the scope of AlphaZero's generalization error, we need a suitable ground truth with which to compare AlphaZero's predictions. Solved games fulfill that need. In each game, the outcome can take on a value of $z=1, 0,$ and $-1$ for wins, losses, and draws, respectively.

\textbf{Training Details:}
AlphaZero is trained for a fixed number of time-steps following the algorithm in \cref{sec:background_alphazero}. AlphaZero's network consists of a ResNet \cite{he2016deep} backbone with separate heads for policy and value predictions. All details are the same for VISA-VIS training runs, but following the modifications outlined in \cref{sec:robustness_visa_vis}. Additional details, including hyperparameters, are outline in \cref{apdx:robustness_training}.

\subsection{Emergent Phenomena in AlphaZero}
\begin{figure*}[t!]
    \centering
    \includegraphics[width=0.99\textwidth]{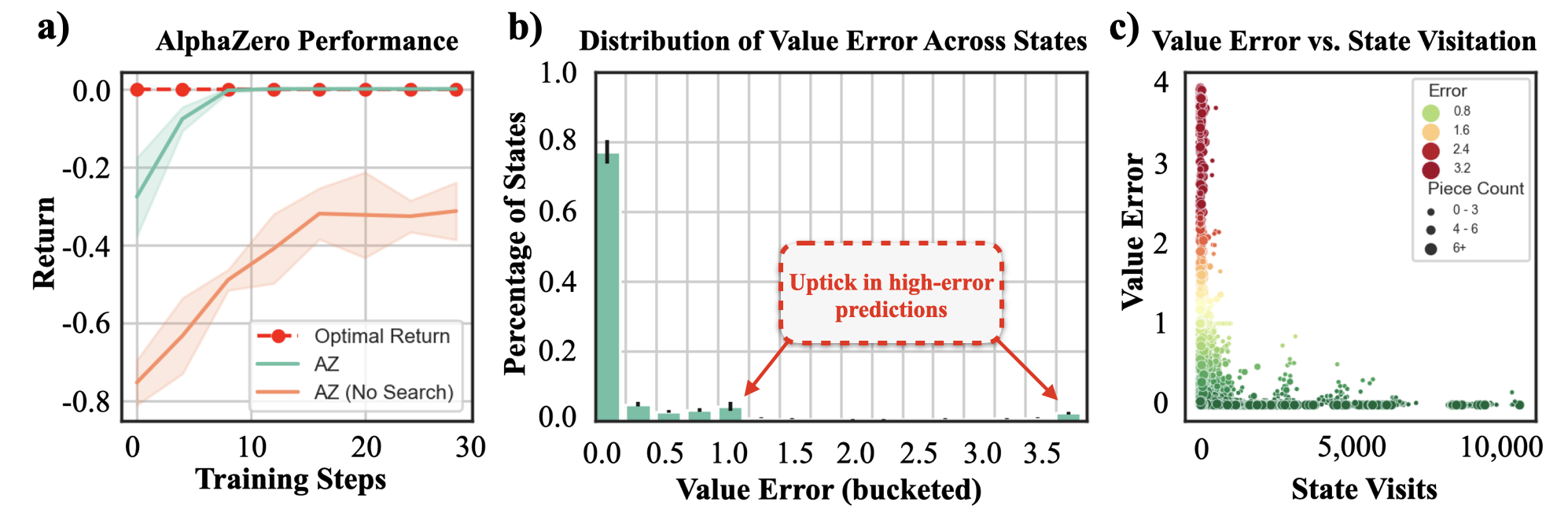}
    \caption{\textbf{a)} Though AlphaZero quickly converges to an optimal strategy, withholding search at test-time reveals that AlphaZero's policy network is a sub-optimal player. \textbf{b)} Analysis of AlphaZero's value function performance also shows that it is sub-optimal, as it makes many incorrect and high-error predictions (relative the game-tree value for each state). \textbf{c)} Computing value error as a function of state visitation reveals that AlphaZero's value function is failing to generalize to infrequently visited and unseen states.}
    \label{fig:robustness_3x3_results_quant}
\end{figure*}

\begin{figure*}[t!]
    \centering
    \includegraphics[width=0.85\textwidth]{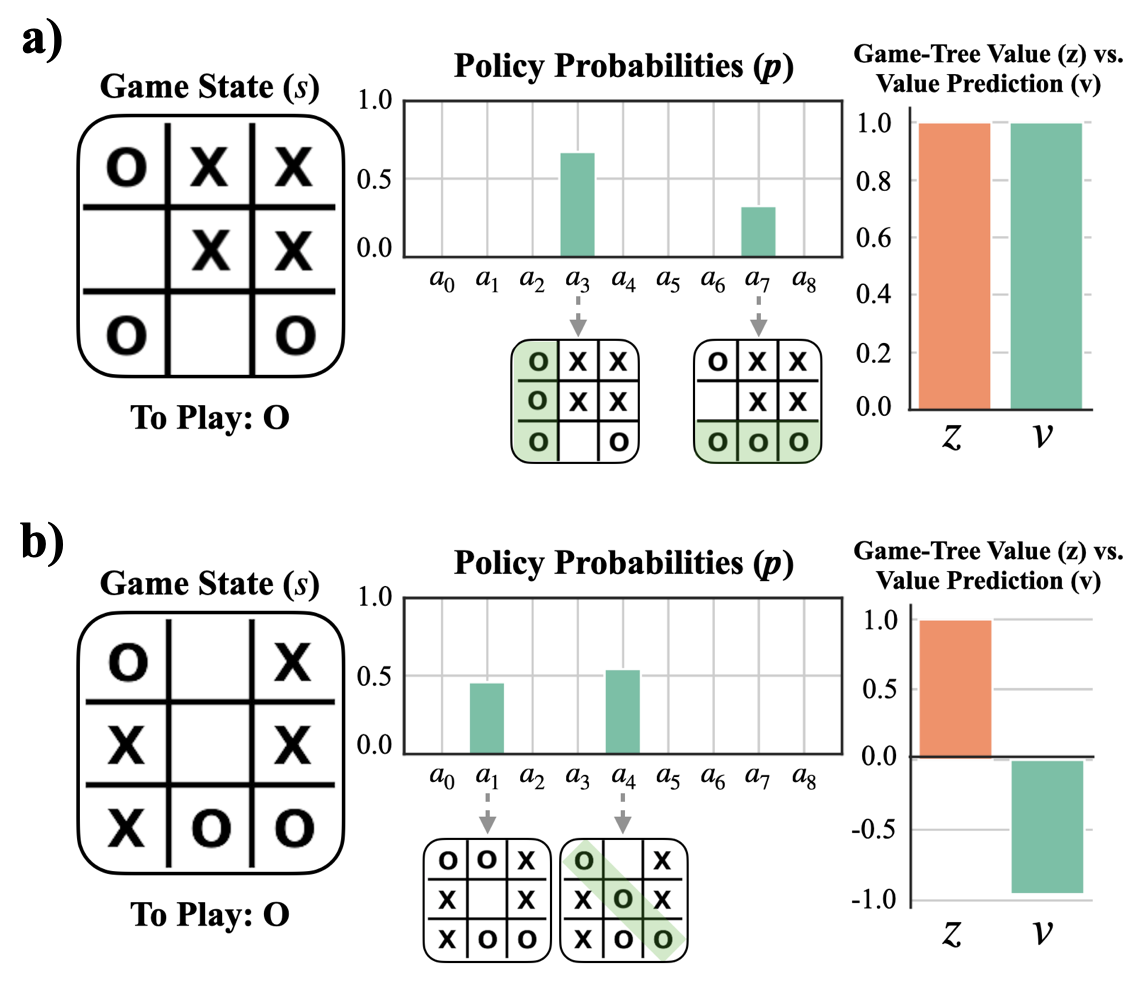}
    \caption{\textbf{a)} A qualitative example where AlphaZero's policy and value predictions are both correct. \textbf{b)} A qualitative example of policy-value misalignment---AlphaZero's policy prediction picks a winning move, but its value prediction is incorrect (predicts a loss).}
    \label{fig:robustness_3x3_results_qual}
\end{figure*}

We first present a systematic study of AlphaZero's performance in an enumerable game-tree: Tic-Tac-Toe. Due to the small state-space of Tic-Tac-Toe, it is possible to evaluate AlphaZero's predictions against every legal state, providing us with a holistic view of its behavior.

After training, the performance of AlphaZero is verified in two ways: First, we play each model against an oracle opponent for $1,000$ matches across five random seeds each and score each match according to the rules of the game. We include both games in which AlphaZero moves first, and those in which AlphaZero moves second. Next, we measure the quality of AlphaZero's policy network by re-playing each of the previous games while withholding search (AlphaZero instead samples actions directly from its policy network). Finally, we verify the accuracy of AlphaZero's value predictions by iterating through the legal state-space and computing the mean-squared error between AlphaZero's value predictions and the true game-tree value. The results of this analysis are presented in \cref{fig:robustness_3x3_results_quant}.

From the evaluation curves in \cref{fig:robustness_3x3_results_quant}a, it is clear that AlphaZero (with search) converges to a strong game-playing strategy very quickly---the expected outcome is a tie under optimal play. Once search is removed and AlphaZero is forced to play with only its policy network, however, its performance drops considerably. This result reveals a tension between search and learnability within AlphaZero---search is needed to reduce the sample complexity of exploration in complex game trees, but also prevents the policy network from learning a generalized mapping from board states to actions. In turn, \textbf{AlphaZero does not in fact learn an optimal policy}, just a policy good enough that averaging action probabilities across many MCTS simulations produces a strong search policy. Though this is enough to reach high levels of performance within subsets of the game tree that are discovered by search, it may not generalize to lesser-known states (such as those generated by \citet{wang2022adversarial}).

In \cref{fig:robustness_3x3_results_quant}b, we find that AlphaZero's value predictions are similarly sub-optimal. In fact, when the error of AlphaZero's value predictions is quantized, there is a noticeable error bump when $e > 1.0$---indicating AlphaZero is predicting a loss when the true game-tree value is a win (or vice versa) and another error bump when $e > 3.5$---indicating that AlphaZero is predicting a loss \textit{very confidently} when the true game-tree value is a win (or vice versa). When plotted as a function of state visitation during training (see \cref{fig:robustness_3x3_results_quant}c), it is clear that the error of AlphaZero's value predictions is highest for infrequently-visited or unseen states, indicating that AlphaZero has failed to generalize. Moreover, as shown in the qualitative examples from \cref{fig:robustness_3x3_results_qual}b, which includes both AlphaZero's policy predictions and its corresponding value estimates, we find that, in many instances, AlphaZero's policy predictions are correct (they choose the winning move), despite value predictions that are wildly inaccurate (they predict a loss, despite the winning opportunity). Thus, in addition to AlphaZero's sub-optimal policy, this analysis provides evidence that its value predictions are inconsistent, generalize poorly, and are not aligned with its policy predictions (i.e. AlphaZero suffers from policy-value misalignment).

\subsection{Improving Robustness and Policy-Value Alignment}
\begin{figure}[t!]
    \centering
    \makebox[\linewidth][c]{\includegraphics[width=0.9\linewidth]{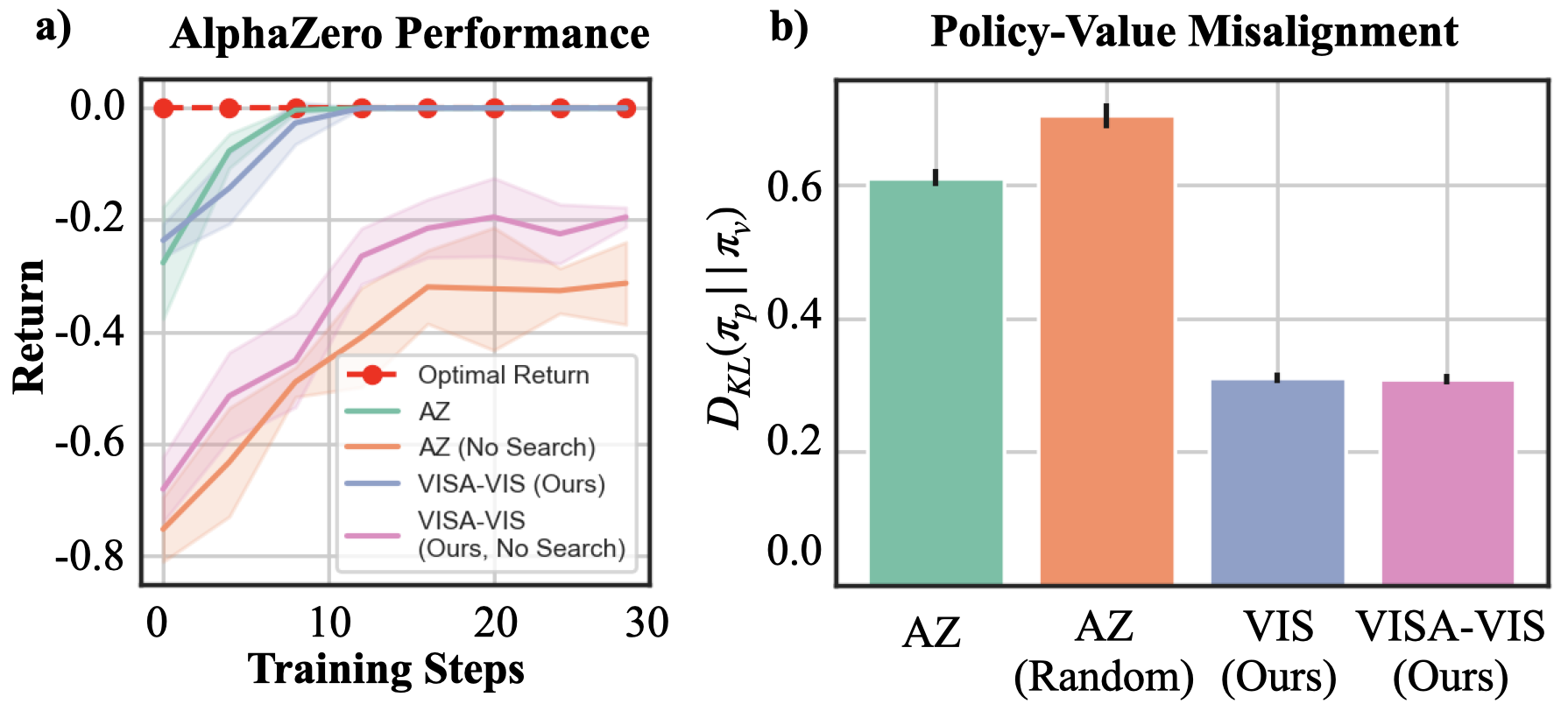}}
    \caption{\textbf{a)} Both AlphaZero and VISA-VIS reach optimal performance with search, but VISA-VIS learns a significantly stronger policy network (without search). \textbf{b)} VIS reduces policy-value misalignment by over 50\% over AlphaZero and maintains alignment when combined with VISA.}
    \label{fig:robustness_ablation_results}
\end{figure}
Next, we investigate how well our method, VISA-VIS, corrects the emergent phenomena identified in the previous section. We perform the same test-time analysis over trained VISA-VIS models and, to gauge the contributions of each component of our method and perform an ablation over both VISA and VIS separately. Aside from AlphaZero and VISA-VIS, we report results for an AlphaZero model that is trained from random initial states (each training rollout begins from a random legal, non-terminal state). Though training with random initial states does not scale to more complex games, it provides a lower-bound for value error in this setting (i.e. how well can AlphaZero do when it is allowed to visit all states many times). The results of this analysis are presented in \cref{fig:robustness_ablation_results}.

\textbf{Policy Performance}
In \cref{fig:robustness_ablation_results}, we compare AlphaZero and VISA-VIS with respect to the quality of the trained policy network. With search, both algorithms converge quickly to an optimal game-playing strategy, as expected. Without search, however, we see a difference in the quality of the policy networks. In particular, \textbf{VISA-VIS discovers a much higher-quality policy} network than vanilla AlphaZero. VISA-VIS, therefore, has reduced the burden of search on the learning algorithm---each individual network prediction is less reliant on being averaging across multiple search simulations. Though the VISA-VIS policy network does not fully reach optimal levels of performance, the goal of our analysis is to identify this search-learning tension and mitigate it partially.

\textbf{Policy-Value Misalignment}
\begin{figure}[t!]
    \centering
    \makebox[\linewidth][c]{\includegraphics[width=0.99\linewidth]{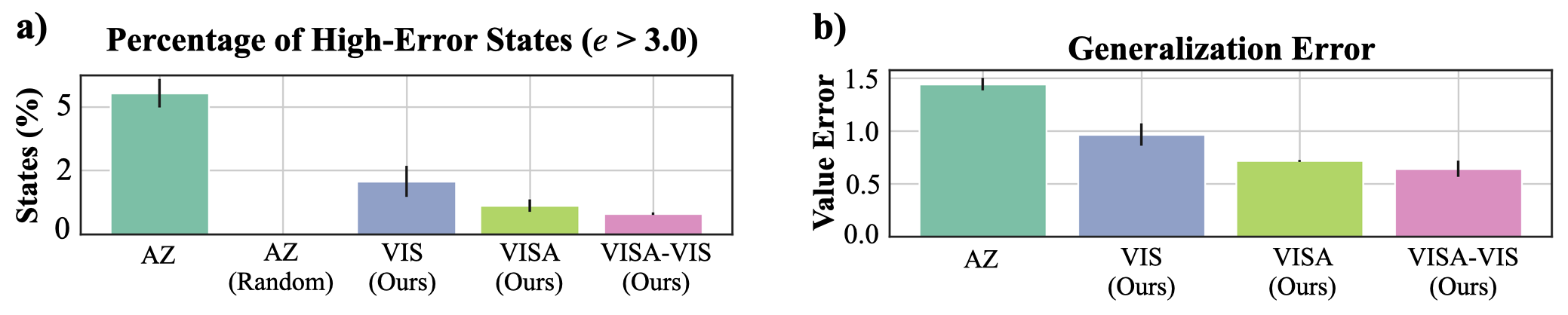}}
    \caption{VISA-VIS significantly reduces both the \textbf{a)} percentage of high error ($e > 3.0$) value predictions and \textbf{b)} average generalization error over AlphaZero's value network.}
    \label{fig:robustness_ablation_error}
\end{figure}
Next, we compute the average policy-value misalignment of each method across all game states (as defined in \cref{sec:robustness_misalignment}). \Cref{fig:robustness_ablation_results}b shows that VIS, and therefore VISA-VIS, significantly improve the consistency of the model's policy and value predictions, \textbf{reducing policy-value misalignment by 50\%} compared to AlphaZero and by over 60\% as compared to AlphaZero Random. These results also illustrate the difference between value prediction accuracy and policy-value misalignment. For example, though AlphaZero Random observes all game states directly, it still exhibits a high policy-value misalignment. This is because value estimation alone does not tell the whole story---it is still possible for AlphaZero's value function to be inconsistent with its policy, and our results show that VISA-VIS improves that inconsistency considerably.

\textbf{Value Error}
Lastly, we report value function performance in \cref{fig:robustness_ablation_error}. With respect to value error, VISA, VIS, and VISA-VIS each significantly improve the quality of AlphaZero's value predictions. In particular, VISA-VIS produces a high-value error prediction ($e > 3.0$) for only 0.6\% of states ($6.9\times$ fewer than AlphaZero), which is close to the perfect score obtained by the version of AlphaZero that observes all states (AZ Random). In terms of value generalization, \textbf{VISA-VIS reduces generalization error by over $50\%$}(from an average of 1.44 to 0.64). Notably, the performance improvements of VISA and VIS are complementary---VISA and VIS both decrease the number high-error states, but value error is decreased the most when they are combined (and likewise for generalization error).

\subsection{Scaling Up}
Here we demonstrate the scalability of our method through evaluation in significantly more complex games. The results of this analysis are presented in \cref{fig:complex_results}.
\begin{table}
  \caption{Adversarially-detected States.}
  \label{tab:adversarial_states}
  \centering
  \begin{tabular}{lll}
    \toprule
    Game     & Adversarial States     & Avg. Value Error\\
    \midrule
    4x4 Tic-Tac-Toe & \hspace{20pt}10,241 & \hspace{20pt}2.63 \\
    Connect Four & \hspace{20pt}52,768 & \hspace{20pt}2.36 \\
    \bottomrule
  \end{tabular}
\end{table}

\textbf{Adversarial State Detection}
In more complex environments, it is not possible to enumerate the state-space and evaluate each state individually. We therefore make two modifications to our evaluation strategy to support our analysis. First, we target endgame states, which allow us to target AlphaZero's most important decisions---a single move may determine the game's outcome---while still providing a ground truth outcome from the game-tree to compare against AlphaZero's predictions.

Second, inspired by prior work on adversarial examples for AlphaZero \cite{wang2022adversarial}, we define an adversarial strategy that aids in uncovering states for which AlphaZero makes inconsistent predictions. After computing search-informed action probabilities at test-time with \cref{eqn:background_mcts_policy}, we force AlphaZero to take the action $a_t = \min_a[\pi(a|s)]$ with minimum assigned probability. Because $\pi$ encodes the expected value of taking action $a$ from state $s$, this forces AlphaZero to play the worst move available. This is an important adversarial test of AlphaZero's robustness, as it nudges each trajectory towards subsets of the state-space that are most likely to have been overlooked by search during training.

In each of our environments, we played a trained AlphaZero model against itself for $100,000$ matches using the aforementioned adversarial detection strategy and collected all unique endgame states for which AlphaZero made a high-error ($e > 1.0$) value prediction. In 4x4 Tic-Tac-Toe, the adversarial detector uncovered over $10,000$ states that received high-error value estimates \textbf{in each of three separately trained AlphaZero models}. The detector similarly uncovered over $50,000$ states for Connect Four. A summary of these statistics is available in \cref{tab:adversarial_states}.

In \cref{fig:complex_results}a, which shows the results for our adversarial analysis in 4x4 Tic-Tac-Toe, we see that AlphaZero displays poor value function consistency---with an average value error of 2.63 for these states---and significant policy-value misalignment---with average KL-divergence $D_{KL}(\pi_p || \pi_v) = 1.41$. VISA-VIS corrects these issues, decreasing average value error 55\% to 1.20 and reducing policy-value misalignment by 76\% down to $D_{KL}(\pi_p || \pi_v) = 0.327$. We find a qualitatively similar result in our adversarial analysis of Connect Four (see \cref{fig:complex_results}b), where our adversarial detector identifies over 50,000 states that yield high-error predictions from AlphaZero. As before, VISA-VIS corrects the high-error predictions. More concretely, the adversarial states result in an average value error of 2.36 and an average policy-value misalignment score of 0.975 from AlphaZero. VISA-VIS compares favorably with an average error of 1.29 and an average policy-value misalignment score of 0.425---representing a 45\% and 56.5\% reduction over AlphaZero, respectively.

Altogether, the results in \cref{fig:complex_results} not only show that policy-value misalignment and value function inconsistency are general phenomena that span multiple environments, but they also provide evidence that our method, VISA-VIS, can reliably combat both phenomena simultaneously.

\begin{figure*}[t!]
    \centering
    \includegraphics[width=0.95\textwidth]{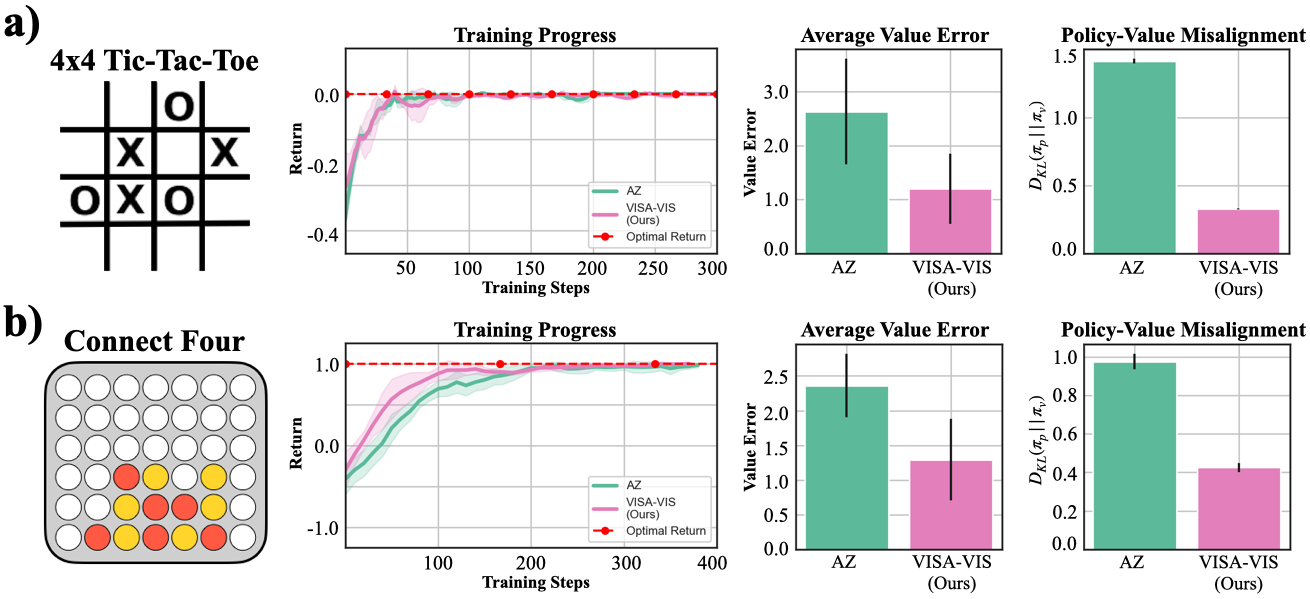}
    \caption{Scaling up game complexity. \textbf{a)} In 4x4 Tic-Tac-Toe, an adversarial analysis reveals over 10,000 states in which AlphaZero performs poorly (high value error, high policy-value misalignment). VISA-VIS corrects AlphaZero's errors in these states, reducing average value error by 55\% and policy-value misalignment by 76\%. During training, AlphaZero and VISA-VIS converge at a similar rate, indicating that VISA-VIS attains with no corresponding increase in sample complexity. \textbf{b)} In Connect Four, AlphaZero again exhibits high value function inconsistency and policy-value misalignment that is corrected by VISA-VIS (a 45\% and 56.5\% improvement, respectively). VISA-VIS also converges more quickly than AlphaZero on average, revealing an additional benefit of our method: reduced sample complexity for learning an optimal policy.}
    \label{fig:complex_results}
\end{figure*}

\textbf{Sample Complexity}
We also plot the game-playing performance of AlphaZero and VISA-VIS throughout training. As in the 3x3 case, the expected outcome of 4x4 Tic-Tac-Toe is a tie under optimal play ($z=1$). As shown in \cref{fig:complex_results}a, we find evidence that the time-to-convergence of AlphaZero and VISA-VIS are roughly the same. Overall, VISA-VIS is able to attain the aforementioned improvements in policy-value alignment and value function consistency without a loss in sample complexity. In the case of Connect Four, we see that VISA-VISA has a more significant positive impact on sample complexity. In \cref{fig:complex_results}b), it is clear that VISA-VIS converges to an optimal strategy more quickly than AlphaZero on average (the expected outcome for the first player is a win in Connect Four). This result suggests that, in this more complex game, the combined improvements of VISA-VIS reduces the sample complexity needed to learn an optimal policy.

\section{Conclusion and Future Work}
In this work, we uncovered multiple emergent phenomena that impact AlphaZero's robustness and generalization: sub-optimal policy learning, policy-value misalignment and value function inconsistency. Crucially, each of these phenomena are related to a tension between search and learning within the algorithm. For policy learning, search is necessary for fighting combinatorially-complex state spaces, but offers a crutch that prevents AlphaZero's policy network from generalizing effectively. Moreover, both policy-value misalignment and value function inconsistency are masked by search, as it is still possible for AlphaZero to play optimally despite their existence. Altogether, these phenomena may contribute to AlphaZero's susceptibility to adversarial examples; and brittleness more generally. We introduced VISA-VIS, a novel algorithm that combines Value-Informed Selection (VIS)---a novel action selection criteria for AlphaZero---and Value-Informed Symmetric Augmentation (VISA)---a novel data augmentation that targets value function uncertainty with respect to symmetry. Together, VISA-VIS leads to significant improvements for both policy-value alignment and value function robustness in AlphaZero, as demonstrated in our experiments.

\chapter{Afterword}
This dissertation has developed responsible emergent multi-agent behavior as a new discipline of AI that aims to address the challenges and opportunities of designing and deploying multi-agent systems that interact with humans or impact human experiences. In this dissertation, we have explored the importance of interpretability, fairness, and robustness in multi-agent systems, and proposed novel methods and improvements to multi-agent learning algorithms that can measure and shape emergent behavior. \cref{chapter:interpretability} introduced both low-level, behavioral techniques for analyzing coordination and implicit communication, and high-level concept-based interpretability tools for understanding multi-agent behavior through human-understandable concepts. \cref{chapter:fairness} examined algorithmic fairness through the lens of multi-agent learning, introducing both formal measures of fairness in cooperative multi-agent settings and novel algorithms for achieving provably fair outcomes for multi-agent teams. Finally, \cref{chapter:robustness} studied robustness in the context of search-based multi-agent learning systems like AlphaZero, identifying key weaknesses that underpin AlphaZero's failure modes and identifying novel algorithmic extensions to AlphaZero that improve its robustness.

\section{Future Work}
This thesis initiates the study of responsible emergent multi-agent behavior, but it is only a first step towards bridging Responsible AI and multi-agent learning. There are a number of ways that the methods proposed in this thesis can be extended to additional domains and agent incentive structures, or scaled to greater levels of complexity. The field of Responsible AI is also constantly evolving to meet the demands of new and developing algorithms and network architectures, and it will be the role of researchers interested in responsible emergent multi-agent behavior to similarly evolve the scope of this study. Here I discuss a few of the open challenges in this area:

\begin{itemize}
    \item Unsupervised concept spaces: Concept-based interpretability remains a crucial area for understanding emergent multi-agent behavior. A primary limitation of concept-based techniques, however, is that they require manual specification of human-interpretable concepts---in intrinsically-interpretable methods, these are used for training (i.e. concept bottlenecks); in post-hoc methods these are used for linear probing. Relatively unexplored in multi-agent settings is the use automatic concept extraction methods \cite{ghorbani2019towards, yeh2020completeness} to learn such a concept space directly. Learning a reliable concept space in an unsupervised manner would unlock new possibilities for concept-based interpretability in multi-agent settings.
    \item Fairness in complex systems: In this work, small multi-agent teams and assumptions of agent homogeneity go a long way in proving out useful properties of equivariant policies. Real-world multi-agent deployments, however, may vary greatly in the scale and composition of agents (i.e. differing levels of heterogeneity) and the validity of fairness criteria may vary based on this context. Consider a sports analogy---in professional soccer, players have a variety of skill-sets and their compensation is a function of both individual performance (e.g. number of goals, assists, etc) and team performance. A definition of fairness that favors equal distribution of individual outcomes may not fully capture this setting. Moreover, potentially sensitive variables, such as player health, may be important pieces of information that should be shared amongst the team. The development of novel classes of algorithms, additional formalizations, and more nuanced fairness criteria will be needed to capture more sophisticated behavior amongst larger multi-agent collectives.
    \item Scaling up robustness: With respect to the robustness of hybrid search-based learning algorithms, the goal of this work was to surface failure modes and possible root causes by combing through AlphaZero’s behavior in a fine-grained manner. For this reason, the analysis presented in this work focused on solved board games with enumerable game trees, which provide a ground truth for both action selection and value estimation. The primary feature of search-based learning algorithms, however, is their ability to scale to large game trees (e.g. Go). The adversarial state detector presented in \cref{chapter:robustness} is a first step in this direction, but identifying weak spots in AlphaZero’s policy/value network in a similarly thorough manner within complex game trees will require novel probing techniques. It remains to be seen, also, how intertwined AlphaZero's unique failure modes are with its unique strengths. For example, by solving policy-value misalignment and removing negative phenomena from emerging, are we also at risk of suppressing the positive phenomena that emerge from AlphaZero's self-play training (e.g. Move 37)? Further exploration of these dynamics is an important open challenge.
    \item Concept sharing and forward modeling: Though Concept Bottleneck Policies were primarily used as an interpretability tool in this work, endowing agents with a shared, grounded representation of their environment may also have robustness implications. For example, one can imagine using a concept bottleneck to model the future values of concepts (e.g. teammate position in $k$ time-steps), or one that, in addition to representing ego-centric concepts, estimates the concept predictions of other agents. Moreover, inline with recent advances in multi-agent communication, we can give agents the ability to communicate concepts from their bottlenecks. Enabling agents to update each other’s concept estimates via concept sharing may increase their ability to navigate conflicts and coordinate with new partners. Altogether, this functionality brings us closer to achieving a Cohen \& Levesque style framework of teamwork \cite{cohen1991teamwork} that is uniquely enabled by concept bottleneck architectures.
\end{itemize}

\section{Closing Thoughts}
\textit{``…the most fertile source of genuinely new ideas is graduate students being well advised in a university. They have the freedom to come up with genuinely new ideas, and they learn enough so that they’re not just repeating history, and we need to preserve that."}
\newline
---Geoffrey Hinton (Interviewed by Martin Ford), Architects of Intelligence, 2018.

It is hard to understate the level of excitement and activity that currently surrounds AI disciplines. Many of the algorithms and techniques discussed in this work have contributed far-reaching advances in the efficacy of AI systems and the pace of progress has escalated to the point where governments are rushing to determine how and when AI technologies should be regulated and companies in most industries are re-orienting their entire business models around AI products and services. But with this excitement comes unique pressures for researchers, especially fresh Ph.D. students learning to navigate the AI research landscape for the first time. Papers are now published at a dizzying rate---top conferences that five years ago received ~1,000 submissions per year are now receiving over 10,000 submissions---and an emphasis on scale has created a narrative that academic levels of compute and data are no longer sufficient to make tangible research progress.

However, there are still problems that cannot be solved by speed and scale alone. As researchers, we should seek out those problems. Rather than leaning-in to aspects of the current state of the art that make it successful, we should embrace the slow and search for the problems that are inefficient and messy. So for early-career researchers staring at a wave of Arxiv submissions thinking ``How will I ever make an impact?", I would offer you the same advice that Hinton offered at the end of the above interview: ``I think you need to sit and think for a few years".

\appendix

\chapter{Appendix}

\section{Implicit Signaling}
\subsection{Experimental Details}
\label{apdx:implicit_comm_apdx_details}

\subsubsection{Evader Strategy Explained}
The goal of the evader strategy in \cref{eqn:implicit_comm_evader_objective} is to run from pursuers along the maximum bisector between two pursuers. Given pursuer positions $\{q_{p_1}, ..., q_{p_n}\}$, we compute polar coordinates:
\begin{align*}
    r_i &= d(q_e, q_{p_i})
    \\
    \tilde{\theta}_i &= \textrm{atan2}(y_{p_i}, x_{p_i})
\end{align*}

\noindent for each pursuer $p_i$ relative the evader. Next, we define a potential field that will push the evader towards a bisector:
\begin{equation*}
    U(\theta_e) = \sum_i \cos(\theta_e - \tilde{\theta}_i)
\end{equation*}
\noindent Using Ptolemy's difference formula, we can expand the potential field as:
\begin{align*}
    U(\theta_e) &= \sum_i \cos(\theta_e - \tilde{\theta}_i) \\
    &= \sum_i \cos(\theta_e)\cos(\tilde{\theta}_i) + \sin(\theta_e)\sin(\tilde{\theta}_i) \\
    &= A \cos(\theta_e) + B \sin(\theta_e)
\end{align*}
\noindent when we plug-in the known $\tilde{\theta}_i$ values. The function $U(\theta_e)$ is maximized/minimized for values of $A$ and $B$ such that:
\begin{equation*}
    \nabla U(\theta_e) = -A\sin(\theta_e) + B\cos(\theta_e) = 0
\end{equation*}
\noindent which simplifies to:
\begin{equation*}
     \tan(\theta_e) = \frac{B}{A}
\end{equation*}
\noindent The evader follows the direction of the negative gradient ($-\nabla U(\theta_e)$) and pursues it at maximum speed. Modulating the cost function by $r_i$:
\begin{equation*}
    U(\theta_e) = \sum_i \bigg(\frac{1}{r_i}\bigg) \cos(\theta_e - \tilde{\theta}_i)
\end{equation*}
\noindent allows the evader to modify its bisector based on the distance to each pursuer. This helps significantly when the evader is stuck in symmetric formations.

\subsubsection{Unit Tests}
\begin{figure}[]
    \centering
    \makebox[\linewidth][c]{\includegraphics[width=0.95\linewidth]{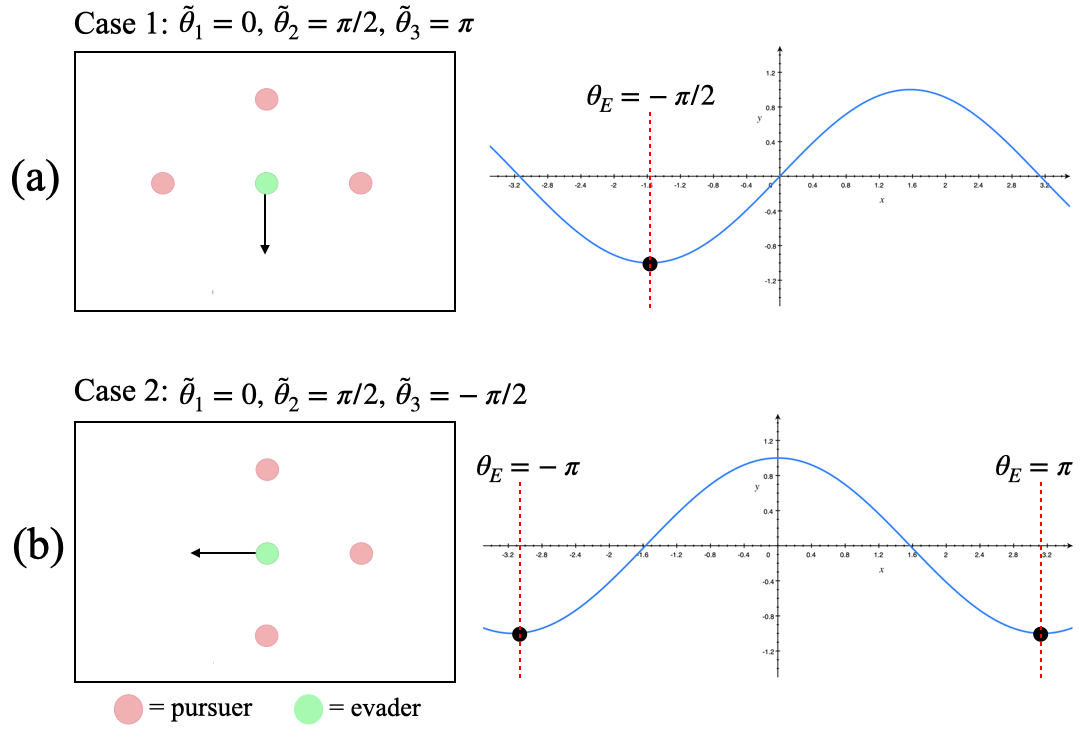}}
    \caption{Unit tests simulating one time-step of action selection as dictated by the evader's cost function. (a) Given pursuers at relative angles $\tilde{\theta}_1 = 0, \tilde{\theta}_2=\pi/2, \tilde{\theta}_3=\pi$ the evader will select the heading $\theta_e = -\pi/2$, which minimizes its cost. (b) Similarly, the evader will select the heading $\theta_e = -\pi = \pi$ when pursuers are located at relative angles $\tilde{\theta}_1 = 0, \tilde{\theta}_2=\pi/2, \tilde{\theta}_3=-\pi/2$.}
    \label{fig:implicit_comm_unit_test}
\end{figure}
We include a set of ``unit tests" that shed light on the evader's decision-making behavior. We assume $n=3$ pursuers are stationed around the evader at relative angles $\tilde{\theta}_1$, $\tilde{\theta}_2$, and $\tilde{\theta}_3$. For simplicity, we initialize the pursuers such that $\forall{i}, r_i = 1$ to negate the effects of radius modulation.
\begin{itemize}
    \item \underline{Case 1:} $\tilde{\theta}_1 = 0, \tilde{\theta}_2=\pi/2, \tilde{\theta}_3=\pi$.
    Pursuers are spaced equally around the upper-half of the unit circle. In this case, the cost minimizer occurs for $\theta_e = -\pi/2$ (see \cref{fig:implicit_comm_unit_test}a).
    \item \underline{Case 2:} $\tilde{\theta}_1 = 0, \tilde{\theta}_2=\pi/2, \tilde{\theta}_3=-\pi/2$.
    Pursuers are spaced equally around the right-half of the unit circle. In this case, the cost minimizer occurs for $\theta_e = -\pi$ and $\theta_e = \pi$(see \cref{fig:implicit_comm_unit_test}b). Either solution can be selected to move the evader towards the largest opening.
\end{itemize}

\noindent In general, the cosine function imposes structure on the evader's objective---it will oscillate between $[-1, 1]$ over a period of $\pi$, taking on a maximum value of $U(\theta_e) = 1$ when the difference between the evader's heading $\theta_e$ and the relative angle of a pursuer $\tilde{\theta}_i$ is zero and a minimum $U(\theta_e) = -1$ when $\theta_e - \tilde{\theta}_i = \pi$. Summing over all $\tilde{\theta}_i$'s incentivizes the evader to follow the heading that splits the largest bisector of the pursuers, as shown in the examples.

\subsection{Pincer Pursuit Explained}
\label{apdx:implicit_comm_details_pincer}
The Pincer strategy described by \cref{eqn:implicit_comm_pincer} is inspired by prior work on theoretical pursuit-evasion \cite{ramana2017pursuit}. Solving \cref{eqn:implicit_comm_pincer} requires optimizing over both $\boldsymbol{\tilde{\theta}_i}$ and $\boldsymbol{r_i}$. Fortunately, we can exploit the toroidal structure of the environment to construct an optimization routine that solves for $\boldsymbol{\tilde{\theta}_i}$ and $\boldsymbol{r_i}$ discretely. In particular, we can unroll the torus $k$ steps in each direction to generate $(2k + 1)^2$ replications of the current environment state. Rather than solving for optimal $\boldsymbol{\tilde{\theta}_i}$ and $\boldsymbol{r_i}$ values directly, we find the set $\boldsymbol{P}$ of pursuers that maximize \cref{eqn:implicit_comm_pincer} across all replications of the environment. We constrain the problem by limiting selections of each pursuer $p_i$ to replications of \textit{itself only}. This dramatically cuts down the number of possible sets $\boldsymbol{P}$ from $\binom{(2k+1)^2n}{n}$ to $\binom{(2k+1)^2}{1} \cdot \binom{(2k+1)^2}{1} \cdot \binom{(2k+1)^2}{1}$, where $n$ is the number of pursuers in the environment. Thus, we solve \cref{eqn:implicit_comm_pincer} via a discrete optimization over each of the $((2k + 1)^2)^3$ possible pursuer selections.
\begin{figure}[]
    \centering
    \makebox[\linewidth][c]{\includegraphics[width=0.6\linewidth]{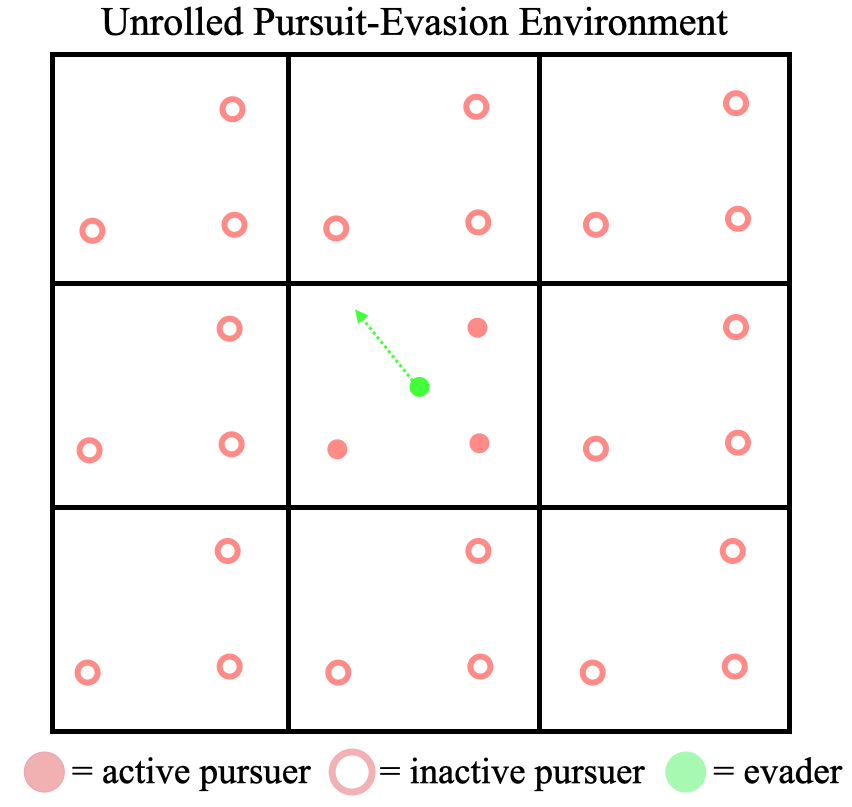}}
    \caption{The torus environment unrolled $k=1$ times in each direction. The filled in red circles denote the ``active" pursuers that are pursuing the evader at the current time-step, whereas the empty circles represent ``inactive" pursuers. We consider only a single evader, located in the center tile.}
    \label{fig:implicit_comm_replica}
\end{figure}
The resulting set $\boldsymbol{P}$ defines the set of ``active" pursuers that will pursue the evader directly at the next time-step. Due to the nature of the evader's objective function---it is attracted to bisectors and repulsed from pursuers---the maximum $\boldsymbol{P}$ tends to favor symmetric triangular formations. Though this method obviously does not scale well with $n$ and $k$, we found that we are able to find a sufficient maximizer with low values of $k$ (i.e. $k=1$ in our experiments). The replication process is shown for the $k=1$ case in \cref{fig:implicit_comm_replica}. Note that we discriminate between ``active" pursuers---i.e. those $p_i \in P$ pursuing the evader at the current time-step---from ``inactive" pursuers.

\subsection{Curriculum-driven learning}
\label{apdx:curriculum}
In this section, we highlight the benefits of our proposed curriculum-driven learning strategy strategy, showing through an ablation study that it greatly improves the quality of experience an agent obtains early in the training process. We then discuss self-play as an alternative curriculum learning paradigm. Finally, we present a formal algorithm for CD-DDPG.

\subsubsection{Velocity curriculum training phases}
The velocity curriculum introduced in \cref{eqn:implicit_comm_curriculum} defines two separate ``training phases"---one in which $\lvert \Vec{v}_P \rvert$ decays linearly over $v_{\textrm{decay}}$ epochs; and one in which $\lvert \Vec{v}_P \rvert$ remains constant, allowing the pursuers to tune their action policies at the current velocity level. We start training at a pursuer velocity ratio of $\lvert \Vec{v}_P \rvert / \lvert \Vec{v}_E \rvert = 1.2$ and anneal $\lvert \Vec{v}_P \rvert / \lvert \Vec{v}_E \rvert$ by $0.1$ over $v_{\textrm{decay}}$ epochs. At the beginning of each phase after the first, the weights of each agent's policy and action-value network are copied over from the final checkpoint of the previous training phase.

\subsubsection{Algorithm}
We present CD-DDPG in algorithmic form in \cref{alg_cd_ddpg}. Though we focus here on DDPG, we note that our strategy is applicable to any off-policy learning algorithm.
\begin{algorithm}[tb]
    \caption{Curriculum-Driven DDPG for a single agent}
    \label{alg_cd_ddpg}
    Define initial velocity $\Vec{v}_0$, final velocity $\Vec{v}_{\textrm{final}}$, decay period $v_{\textrm{decay}}$, and warm-up episode threshold W \\[2pt]
    Initialize actor $\mu_\phi(s)$ with parameters $\phi$ and target actor $\mu_{\phi'}$ with weights $\phi' \gets \phi$ \\[2pt]
    Initialize critic $Q_\omega(s, a)$ with parameters $\omega$ and target critic $Q_{\omega'}$ with weights $\omega' \gets \omega$ \\[2pt]
    Initialize replay buffer $\mathcal{D}$
    \begin{algorithmic}[1] 
    \FOR{\normalfont{i = 1} \TO \normalfont{max-epochs}}
        \STATE Initialize random process $\mathcal{N}$ for exploration \\[2pt]
        \STATE Receive initial state $s_1$ \\[2pt]
        \FOR{\normalfont{t = 1} \TO T}
            \IF {\normalfont{i} $\leq$ W}
                \STATE Sample action $a_t = \beta_0(s_t)$
            \ELSE
                \STATE Sample action $a_t = \mu_\phi(s_t) + \mathcal{N}$ 
            \ENDIF \\[2pt]
            \STATE Execute $a_t$, observe reward $r_t$ and new state $s_{t+1}$ \\[2pt]
            \STATE Store transition $(s_t, a_t, r_t, s_{t+1})$ in $\mathcal{D}$ \\[2pt]
            \STATE Sample random minibatch of transitions $\{(s_t, a_t, r_t, s_{t+1})\}_{i=1}^N$ from $\mathcal{D}$\\[2pt]
            \STATE Compute TD-target: $y_i = r_i + Q'_\omega(s_{i+1}, \mu'_\phi(s_{i+1}))$ \\[2pt]
            \STATE Update critic by minimizing the TD-error: $\mathcal{L}(\omega) = \frac{1}{N} \sum_i (y_i - Q_\omega(s_i, a_i))^2$\\[4pt]
            \STATE Update actor using the sampled policy gradient: $\nabla_\phi J(\phi) \approx \frac{1}{N} \sum_i \nabla_\phi \mu(s) \nabla_a Q_\omega(s, a)$ \\[4pt]
            \STATE Update target networks: $\phi' \gets \tau \phi + (1-\tau) \phi'$, \; $\omega' \gets \tau \omega + (1-\tau) \omega'$\\[2pt]
        \ENDFOR
        \STATE Step velocity curriculum:$\lvert \Vec{v}_P \rvert \gets \Vec{v}_{\textrm{final}} + (\Vec{v}_0 - \Vec{v}_{\textrm{final}})*\max((v_{\textrm{decay}}-\textrm{i})/v_{\textrm{decay}}, 0.0)$\\[2pt]
    \ENDFOR
    \end{algorithmic}
\end{algorithm}

\subsection{Qualitative Results}
\label{apdx:implicit_comm_apdx_qual}
\begin{figure}[t!]
    \centering
    \makebox[\linewidth][c]{\includegraphics[width=0.65\linewidth]{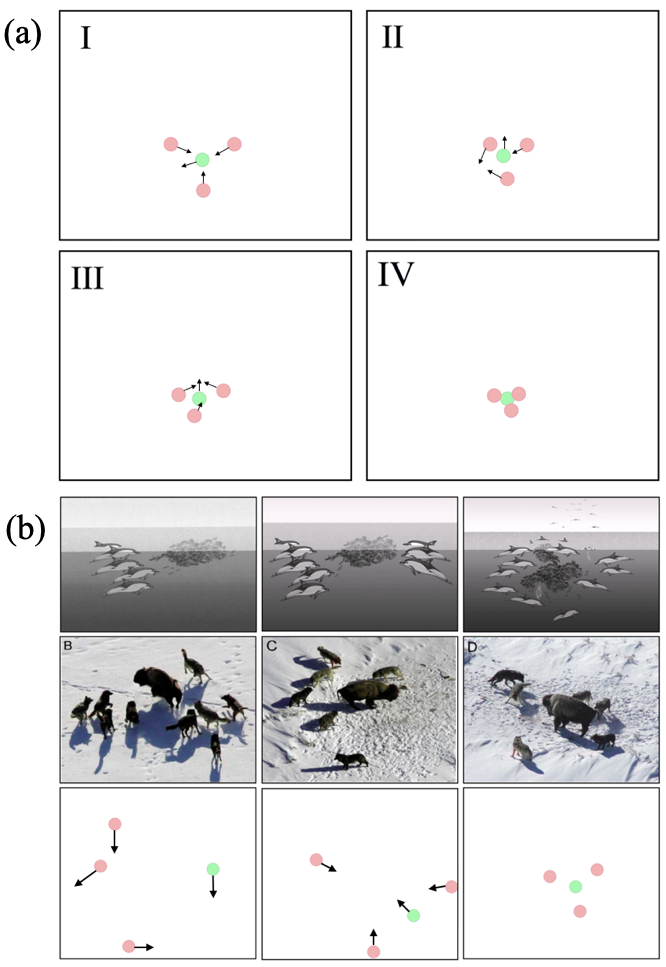}}
    \caption{Qualitative results from the pursuit-evasion experiment. (a) The pursuers coordinate to capture the evader, displaying positional shifts indicative of implicit signaling. (b) \textit{Top:} A diagram of dolphin foraging strategies documented in \protect\cite{neumann2003feeding}. \textit{Middle:} Photos of wolves coordinating while hunting, as shown in \protect\cite{muro2011wolf}. \textit{Bottom:} The learned behavior of our multi-agent system.}
    \label{fig:implicit_comm_qual_results}
\end{figure}
\subsubsection{Animalistic Coordination}
We perform post-hoc qualitative analysis of CD-DDPG trajectories. In the trajectories, the pursuers appear to adjust their position slightly in response to the movements of fellow pursuers as they close in on the evader (\cref{fig:implicit_comm_qual_results}a). Moreover, it seems that the pursuers occasionally move away from the evader---something a less coordinated strategy would not do---to maintain the integrity of the group formation. This explains the performance difference between CD-DDPG and the competing analytical strategies, as the potential-field pursuers have no basis for making small-scale adaptive movements. We interpret these results as encouraging evidence that implicit signaling has emerged amongst CD-DDPG pursuers.

\subsubsection{Phantom Coordination}
\label{apdx:implicit_comm_apdx_qual_phantom}
In \cref{sec:implicit_comm_results}, we describe ``phantom coordination" as independent action that is falsely perceived as coordination from the perspective of IC. Phantom coordination appears in the greedy pursuit strategy, where we see the IC score for greedy pursuers increase slightly as the velocity of the pursuers decreases. This is counter-intuitive because each greedy pursuer ignores the behavior of its teammates. We would expect the IC score for greedy pursuers to remain flat, mirroring the IC score of the random pursuers.

To diagnose phantom coordination in our environment, we perform qualitative analysis of greedy pursuit at low velocities. In particular, we examine $n=3$ pursuers as they chase an evader at a speed of $\lvert \Vec{v}_p \rvert / \lvert \Vec{v}_e \rvert = 0.4$. The pursuers have no chance of successfully capturing the evader, as evidenced by their capture success performance at this velocity in \cref{fig:implicit_comm_cap_success}. However, we find that the straight-line chase patterns of greedy pursuers form temporary triangular patterns around the evader. In \cref{fig:implicit_comm_phantom}, the greedy pursuers form an ad-hoc triangular formation around the evader for a duration of 60 time-steps. The leftmost pursuer lies equidistant from the evader around the periodic boundaries and iterates between moving leftward and rightward, depending on which direction creates a shorter line to the evader. The other two pursuers approach the evader from above and below, respectively. This behavior causes the evader to move in a zig-zag pattern in the center of the triangle until a large opening appears, through which the evader can escape.

This behavior leads to phantom coordination because IC is computed between two consecutive time-steps and averaged over whole trajectories. This means that, in the case of the greedy pursuers, IC scores for independent actions are averaged together with subsets of each trajectory that consist of seemingly highly coordinated behavior.
\begin{figure}[]
    \centering
    \makebox[\linewidth][c]{\includegraphics[width=0.95\linewidth]{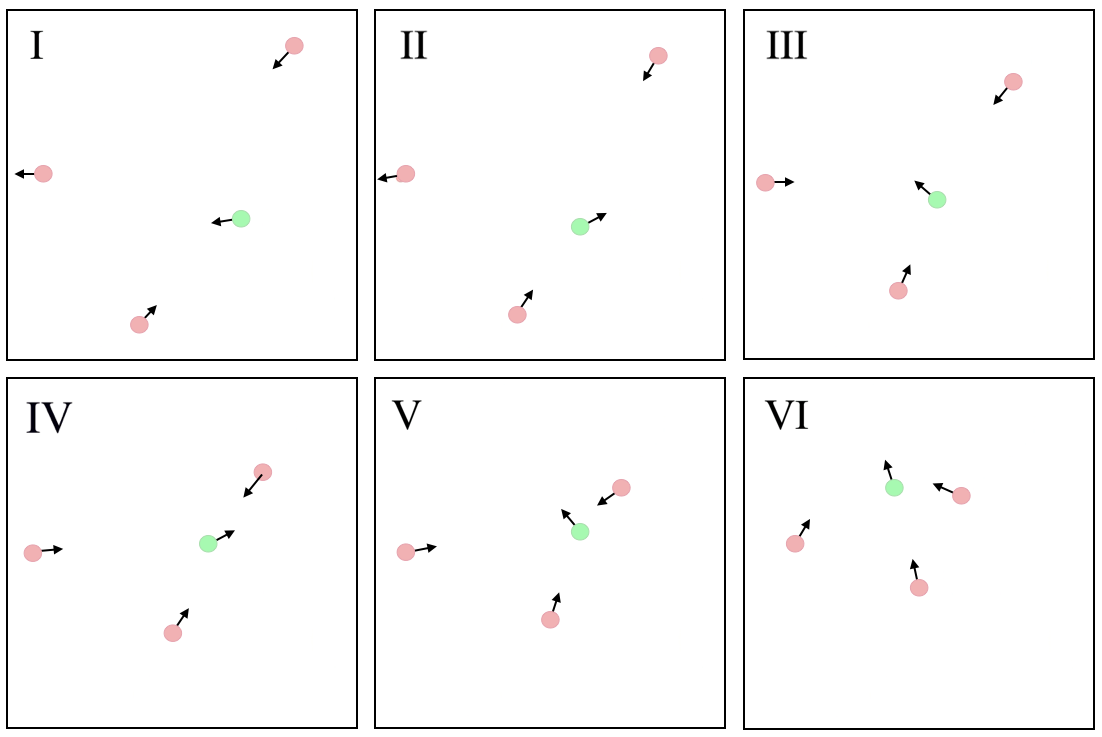}}
    \caption{Snapshots from a 60 time-step trajectory in which phantom coordination appears. Though the pursuers are following independent greedy strategies, their actions produce a triangular formation that is perceived as coordination by the IC performance measure.}
    \label{fig:implicit_comm_phantom}
\end{figure}

\subsection{Training details}
\label{apdx:implicit_comm_apdx_experimental_details}
\subsubsection{Potential-Field Hyperparameters}
The Greedy potential-field defined by \cref{eqn:implict_comm_attractive_potential} is subject to a single hyperparameter $k_{\textrm{att}}$, which defines the strength of the attractive pull of the agent towards its goal. We set $k_{\textrm{att}}=1.5$ in all of our experiments. 

\subsubsection{Policy Learning Hyperparameters}
Actors $\mu_\phi$ are trained with two hidden layers of size 128. Critics $Q_\omega$ are trained with three hidden layers of size 128. We use learning rates of $1\textrm{e}^{-4}$ and $1\textrm{e}^{-3}$ for the actor and critic, respectively, and a gradient clip of 0.5. Target networks are updated with Polyak averaging with $\tau=0.001$. We used a buffer $\mathcal{D}$ of length $500000$ and sample batches of size $512$. We used a discount factor $\gamma=0.99$. All values are the result of standard hyperparameter sweeps.

\subsubsection{Experiments}
The models for all experiments were trained for $50000$ epochs of $500$ steps. We reserve $W=1000$ epochs for warm-up, during which the behavioral policy $\beta_0$ is followed. Pursuer velocity is decayed over $v_{\textrm{decay}}=15000$ epochs. Test-time performance, as in \cref{fig:implicit_comm_cap_success}, \cref{fig:implicit_comm_ic_results}, and \cref{fig:implicit_comm_high_ic} is averaged across 100 independent trajectories with separate random seeds. For the ablation in \cref{fig:implicit_comm_ablation}, each method was trained for $7500$ epochs across 10 random seeds. All experiments leveraged an Nvidia GeForce GTX 1070 GPU with 8GB of memory.

\section{Interpretability}
\subsection{Discussion}
\label{apdx:interpretability_discussion}

\subsubsection{Limitations}
One of the limitations of the proposed method is that it requires manual specification of human-interpretable concepts to be used in the policy bottlenecks at training time. As illustrated in the experiments section of the paper, these are often intuitive characteristics of the specific environment targeted (e.g., positions and orientations of agents, states of various objects that can be interacted with, etc.), which can be then automatically extracted from the environment simulator. In future work, it would be interesting to learn such key relevant concepts automatically. Another potential limitation of the method is that the combination of traditional reward-based RL loss and concept-loss can make for a slightly complex loss landscape.
However, as illustrated by our experiments, one mechanism for controlling this is via the concept loss $\lambda$ parameter. To make the selection of this parameter easier, a baseline of the agent policies ran \emph{without} the concept bottlenecks can be used to establish reasonable bounds on their expected performance in the task. Subsequently, the parameter can be swept over (similar to our experiments) to identify a $\lambda$ that leads to accurate concept learning with minimal impact on final performance. Finally, we note that we do not provide theoretical guarantees that CBPs converge to a specific solution concept (i.e. refinement of the Nash equilibrium). In general, such guarantees are difficult to provide when combining MARL with neural networks and the addition of a concept-based loss further complicates this picture; though we do point out that the lack of such guarantees in many prior MARL works highlights the importance of the interpretability problem.

\subsubsection{Broader Impacts}
In terms of potential societal implications, modeling of interactive decision-making and concept estimations of agents in such systems might be used by adversarial actors to intervene on concepts and intentionally cause performance degradation or agents misbehaving.
On the other hand, concept bottlenecks themselves could be a means of identifying such adversarial attacks, by detecting shifts in concept estimations that were otherwise unexpected. 

\subsection{Training Details}
\label{apdx:interpretability_training}
\subsubsection{Concept Spaces}
The following details are common across all of our environments:
\begin{itemize}
    \item \emph{States}: Each environment is a grid world. Grid cells can be filled by an agent or an environment-specific object.
    \item \emph{Observations}: Agents receive partial multi-modal observations consisting of their own position and orientation in the grid, as well as a partial RGB rendering of a $5\textrm{-cell} \times 5\textrm{-cell}$ window centered at the agent.
    \item \emph{Actions}: Agents can execute one of 8 actions: no-op, move $\{\texttt{up}, \texttt{down}, \texttt{left}, \texttt{right}\}$, turn $\{\texttt{left}, \texttt{right}\}$, and interact.
\end{itemize}
\noindent The concepts supported by each of the Melting Pot environments are outlined as follows: 
\begin{itemize}
    \item \emph{Collaborative Cooking}: We assume that each environment supports the following concepts (and concept types): (i) agent position (scalar); (ii) agent orientation (scalar); (iii) whether or not an agent has a tomato, dish or soup (binary); (iv) cooking pot position (scalar) (v) the progress of the cooking pot (scalar); (vi) the number of tomatoes in the cooking pot (categorical); and (vii) the position of each tomato and dish (scalar).
    \item \emph{Clean Up}: We assume that each environment supports the following concepts (and concept types): (i) agent position (scalar); (ii) agent orientation (scalar); (iii) closest pollution position (scalar); (iv) closest apple position (scalar).
    \item \emph{Capture the Flag}: We assume that each environment supports the following concepts (and concept types): (i) agent position (scalar); (ii) agent orientation (scalar); (iii) flag position (scalar); (iv) whether or not an agent has the opponent's flag (binary) (scalar); (v) floor paint color (categorical); (vi) flag indicator tile color (categorical).   
\end{itemize}

\subsubsection{Architecture and Hyperparameters}
For both PPO and ConceptPPO, each agent's policy network consists of CNN and MLP encoders (for image and position/orientation inputs, respectively), followed by a two-layer MLP and a linear mapping that compresses the encoded inputs into concept predictions. Concept estimates are fed through a two-layer MLP, which produces the final action. ReLU activation is used throughout (except in the bottleneck layer itself). As a baseline, we use vanilla PPO (no concept loss) with the same architecture. We train 10 individual policies across each of the following values of $\lambda$: $\{0.01, 0.1, 0.25, 0.5, 0.75, 1.0. 2.0, 5.0, 10.0\}$. We conducted a wide hyperparameter sweep to train both ConceptPPO and PPO, which is summarized in \cref{tab:interpretability_hyperparams}.

\begin{table}
  \caption{Hyperparameter sweeps for training ConceptPPO and PPO. Swept values are shown in braces and highest-performing values are bolded.}
  \label{tab:interpretability_hyperparams}
  \centering
  \begin{tabular}{lll}
    \toprule
    \multicolumn{2}{c}{Hyperparameters}                   \\
    \cmidrule(r){1-2}
    Name     & Value      \\
    \midrule
    Training Steps & $25\mathrm{e}6$     \\
    Batch Size     & $\{64, 128, 256, \boldsymbol{512}, 1024\}$      \\
    Learning Rate     & $\{1\mathrm{e}{-}3, \boldsymbol{1\mathrm{e}{-}4}, 1\mathrm{e}{-}5\}$      \\
    Gradient Norm     & $\{0.1, \boldsymbol{0.5}, 1.0, 5.0, 10.\}$      \\
    PPO Unroll Length     & $\{4, 8, \boldsymbol{16}, 32\}$      \\
    PPO Clipping $\epsilon$     & $\{0.01, \boldsymbol{0.05}, 0.1, 0.2, 0.3\}$      \\
    PPO Entropy Cost    & $\{0.001, \boldsymbol{0.01}, 0.05\}$      \\
    PPO Value Cost    & $\{0.75, 0.9, \boldsymbol{1.0}\}$      \\
    \bottomrule
  \end{tabular}
\end{table}

\subsubsection{Loss Breakdown}
\begin{figure*}
    \centering
    \includegraphics[width=0.99\textwidth]{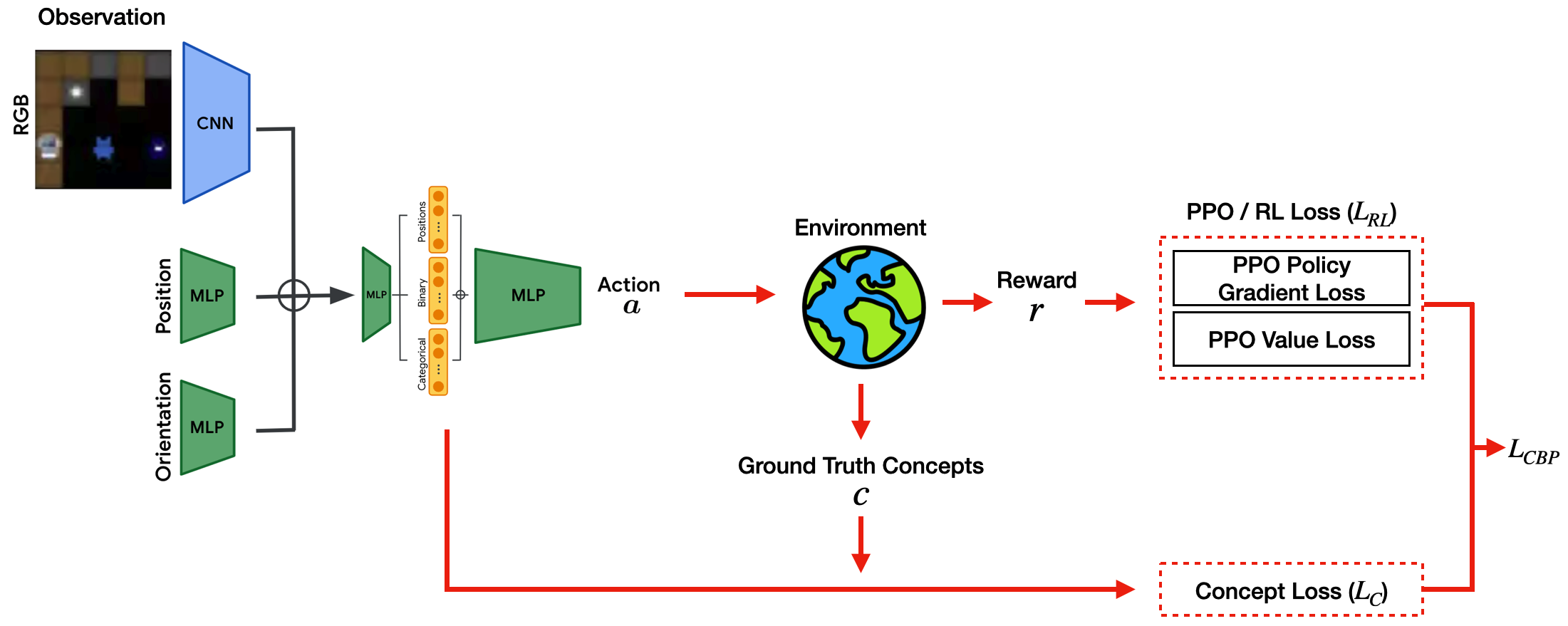}
    \caption{An architecture diagram showing how the concept bottleneck loss $L_{CBP}$ is computed in \cref{eqn:interpretability_concept_loss}
    . Crucially, PPO is applied only over $L_{RL}$, which captures environmental rewards (e.g., delivering in collaborative cooking). Separately, concept loss $L_C$ is computed in a supervised fashion using concept labels that are extracted from the environment. The concept loss is \textbf{not} incorporated into an agent's reward function and therefore does not incentivize its behavior directly.}
    \label{fig:interpretability_detailed_architecture}
\end{figure*}
In this section, we present a more detailed discussion of the concept bottleneck loss introduced in \cref{sec:interpretability_method} and, specifically, \cref{eqn:interpretability_concept_loss}. 
The objective $L_{CBP}$ is a function of two separately computed objectives: (i) $L_{RL}$, which is a reward-based loss that operates only over environmental rewards (e.g., delivering in collaborative cooking); and (ii) $L_C$, which is a supervised loss computed over concept labels that are extracted from the environment. We are leveraging PPO to optimize each agent's policy, which includes terms for both value error and generalized advantage estimation. A crucial detail of our architecture is that PPO's value and advantage estimation operates over environmental rewards only. The accuracy of concept predictions is not incorporated into an agent's reward in any way---i.e. we are not adding an auxiliary/intrinsic rewards associated with the concept predictions to the advantage function estimate. This is done intentionally so as not to bias agent behavior away from states that result in high concept error prediction.

\subsection{Environmental Demands of Coordination}
\label{apdx:interpretability_environmental_demands}
For each Concept PPO and PPO policy trained in our cooking environments, we mask out, for each agent, all of the concepts related to that agent’s teammate. We measure the average cumulative reward attained over 100 trajectories each. The results of this intervention test are shown in \cref{fig:interpretability_results_env_coordination_apdx}. The level of coordination required by the environment is apparent from the severity of performance degradation that results from intervening on each agent’s estimates of its teammate. Most notable is the contrast between the basic environment (blue) and the impassable environment (purple). The impassable environment requires strict coordination---agents can only access a subset of ingredients and must pass items to each other across the center divider. In this case, intervention performance drops to near-zero across all policies. In the basic environment, on the other hand, there are no obstacles and agents have access to their own supply of ingredients; and so both policies in which agents coordinate and policies in which agents act independently are successful. Consequently, the impact of intervention is much less severe, as the performance of policies that coordinate (and fail under intervention) is averaged in with independent policies that are unaffected by intervention. Altogether, these results indicate that our method accurately distinguishes environments that require coordination from those that do not.
\begin{figure*}
    \centering
    \includegraphics[width=0.9\textwidth]{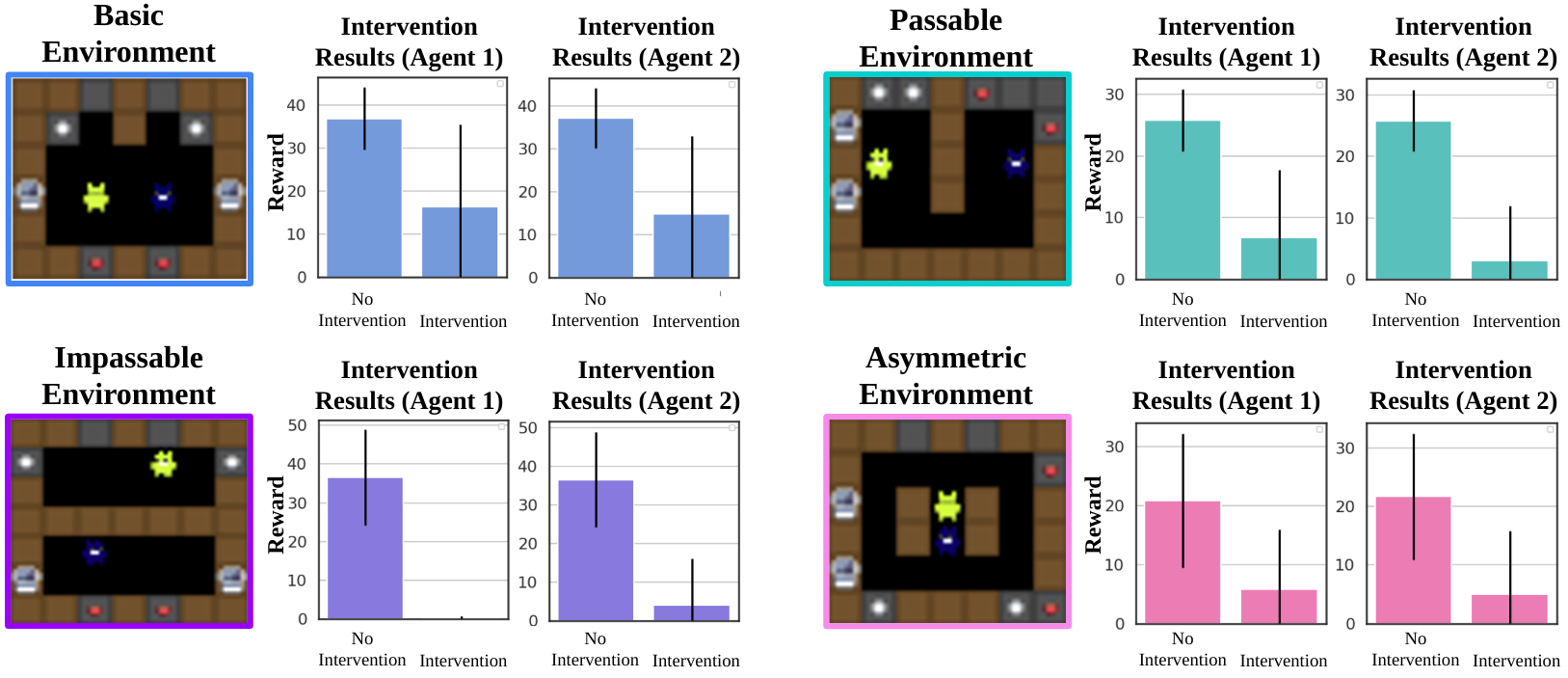}
    \caption{\textbf{Does the environment require coordination?} Averaging the impact of concept intervention over all policies trained in an environment reveals the extent to which coordination is required by that environment. In the impassable environment, agents cannot solve the task without coordinating, leading to consistent performance drops under intervention. In the basic environment, policies that coordinate (and therefore fail under intervention) are averaged with policies that act independently (and are uninterrupted by intervention), so the overall impact of intervention is less severe.}
    \label{fig:interpretability_results_env_coordination_apdx}
\end{figure*}

We also present results for this analysis in two additional cooking environments: passable and asymmetric. As in the basic and impassable environments, we find that the coordination demands of the asymmetric and passable environments can be revealed through concept intervention. Both environments involve moderate coordination---the most efficient strategy for bringing items to and from the cooking pot involves passing them over the counter from one agent to another---but this coordination is not required (like in the impassable environment). For this reason, we still see a significant drop in the performance of our agents under intervention, but not a full decrease to zero.

\subsection{Factors of Coordination Analysis (cont'd)}
\label{apdx:interpretability_coordination_analysis}
\subsubsection{In vs. Out-of-Distribution Orientation Analysis}
\label{apdx:interpretability_ood_analysis}
Here we further examine the impact of intervening on orientation. First, we compute an empirical distribution of orientations that each agent visits over 100 test-time trajectories (across five random seeds each). From those trajectories, we compute both the mean logits vector produced by each agent’s concept bottleneck for the \textit{most frequently visited orientation}, and the distribution represented by those logits. Crucially, the mean logits vector for the most frequently visited orientation can be used as an in-distribution value for concept interventions---it is an orientation estimate that each agent has likely seen before. Similarly, we can create an out-of-distribution concept intervention mask by permuting the mean logits vector such that the probability mass shifts a different cardinal orientation. Using these manufactured orientations, we perform the same intervention technique as before, iteratively replacing each agent’s orientation concept with the mean logits vector in place of each cardinal direction, and measure performance of the multi-agent team.
\begin{figure}
    \centering
    \includegraphics[width=0.99\textwidth]{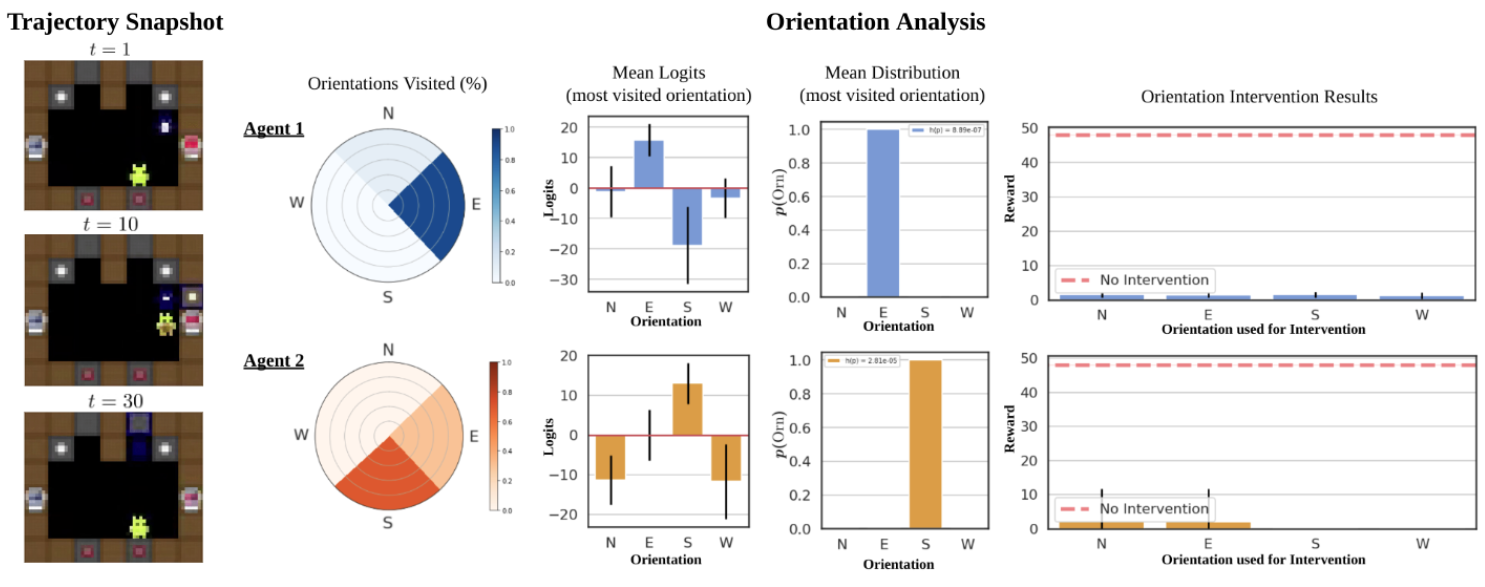}
    \caption{Overview of our method for constructing in-distribution concept masks for intervention. Using test-time trajectories, we compute an empirical distribution of the orientations experienced by each agent. Next, we compute the mean logits (and corresponding distribution) produced by each agent for for the most frequently visited orientation of its teammate. Finally, we report the results of intervening with this mean logits vector in place of each of the four cardinal orientations. The results of this intervention are consistent, regardless of whether the mask used was in- or out-of-distribution.}
    \label{fig:interpretability_results_orientation}
\end{figure}

To provide intuition for this technique, we ground it in the cooking task used for our experiments. Consider, for example, the tomato-picking agent in the trajectory snapshot of \cref{fig:interpretability_results_orientation}, who primarily faces south and east---the directions needed to pick tomatoes from the bottom counter and place them in the cooking pot on the right-hand side. The tomato-picking agent’s teammate (the waiter agent) must learn to accurately model these orientations to satisfy its concept prediction objective, and so frequently passes low-entropy distributions for south and east to its policy network. Intervening on the waiter agent’s orientation estimate with a low-entropy distribution for south and east, therefore, creates an in-distribution mask, whereas intervening with a low-entropy distribution for north or west creates an out-of-distribution mask.

The results of this intervention test are shown in \cref{fig:interpretability_results_orientation}, alongside the orientation distributions, mean logits, and the distribution represented by those logits. Interestingly, the previously observed degradation of performance as a result of intervening on teammate orientation is upheld, regardless of whether the mask value is in- or out-of-distribution. 

This provides further evidence that the agent’s reliance on orientation is a legitimate artifact of their emergent strategy and not an adversarial or OOD example.
\begin{figure}
    \centering
    \begin{subfigure}[b]{0.375\textwidth}
        \centering
        \includegraphics[width=\textwidth]{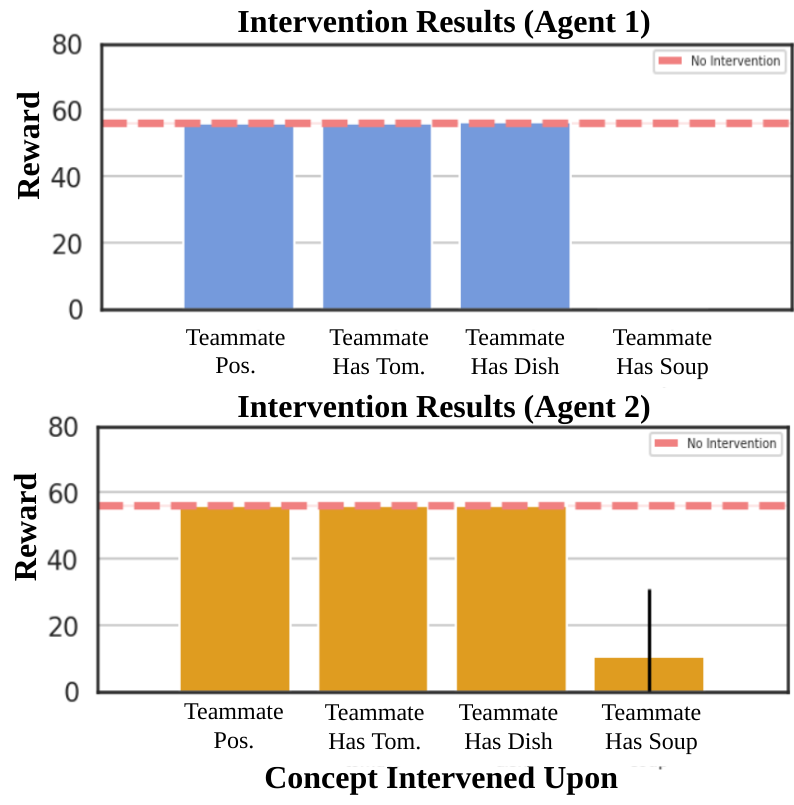}
        \caption{Intervention without orientation.}
        \label{fig:interpretability_results_no_orn}
    \end{subfigure}
    \hspace{1cm}
    \begin{subfigure}[b]{0.425\textwidth}
        \centering
        \includegraphics[width=\textwidth]{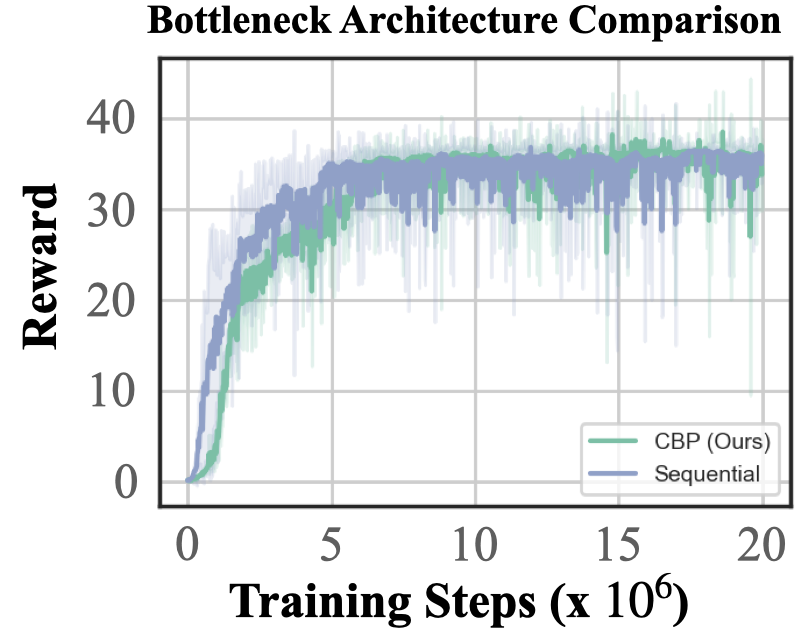}
        \caption{Bottleneck architecture test.}
        \label{fig:interpretability_joint_vs_seq}
    \end{subfigure}
    \caption{Additional bottleneck analyses. (a) Results of iterative concept intervention on Concept PPO agents trained without orientation. (b) Asymptotic performance of joint CBP architecture (ours) vs. sequential CBP architecture.}
\end{figure}

\subsubsection{Intervention Without Orientation}
\label{apdx:interpretability_int_wo_orn}
To further investigate the surprising use of orientation as the primary signal driving the emergent behavior of our learning agents, we train a set of ConceptPPO policies in our basic environment without orientation as a concept (across each $\lambda$ value as before). We then perform the same iterative intervention analysis over the concepts pertaining to each agent's teammate (as outlined in \cref{sec:interpretability_identifying_coord} (excluding orientation, of course). The results of this analysis are shown in \cref{fig:interpretability_results_no_orn}. After removing each agent's orientation concept, we find that the agents latch onto a new concept--``has soup"--as the primary concept that drives their coordination.

\subsection{Bottleneck Architecture Comparison}
Consider a concept bottleneck composed of two functions: (i) $g: x \rightarrow c$ mapping inputs to concept estimates; and (ii) $f: c \rightarrow y$ mapping concept estimates to outputs. The work of \citet{koh2020concept} compares three architectures for concept bottlenecks in supervised learning settings that lean these functions in different ways:
\begin{itemize}
    \item \textbf{Independent Bottleneck}: This architecture learns functions $\hat{f}$ and $\hat{g}$ independently, where $\hat{g}$ is trained separately to predict concept estimates $\hat{c}$ from inputs $x$ using ground truth concepts $c$ as supervision. $\hat{f}$ is trained to predict outputs $\hat{y}$ from ground truth concepts $c$, using ground truth labels $y$ as supervision.
    \item \textbf{Sequential Bottleneck}: This architecture learns $\hat{g}$ first, then uses concept predictions from $\hat{g}$ to train $\hat{f}$. Specifically, $\hat{g}$ is trained in the same manner as in the independent bottleneck. $\hat{f}$ is then trained to predict outputs $\hat{y}$ from concept estimates $\hat{c} = \hat{g}(x)$, using ground truth labels $y$ as supervision.
    \item \textbf{Joint Bottleneck}: This architecture trains $\hat{f}$ and $\hat{g}$ together. The models are trained jointly to predict $\hat{y} = \hat{f}(\hat{g}(x))$ using an objective that is a weighted sum of loss terms for concept prediction (using ground truth concepts $c$ as supervision) and target label prediction (using ground truth labels $y$ as supervision). 
\end{itemize}

Our proposed CBP architecture is most like the joint bottleneck. Here we compare the performance of our proposed architecture to a MARL equivalent of the sequential bottleneck. A simple way to obtain an approximation of the sequential bottleneck using the same architecture (shown in \cref{fig:interpretability_detailed_architecture}) and objective (defined by \cref{eqn:interpretability_concept_loss}) as CBPs is by the stopping gradient flow between the concept estimator network and policy head during backpropagation.

We train both CBPs (our proposed joint bottleneck architecture) and sequential bottlenecks in the basic cooking environment shown in \cref{fig:background_environments_cooking}. Comparing the asymptotic reward achieved by both approaches shows that joint CBPs (our architecture) and the sequential bottleneck (implemented via stop gradient) perform comparably. This gives us confidence that the joint architecture is not unfairly surpassing the performance of other bottleneck architectures by "hacking" concept estimates. We investigate concept leakage more concretely in the following subsection.

\subsection{Concept Leakage Analysis}
\begin{figure}
    \centering
    \includegraphics[width=0.75\textwidth]{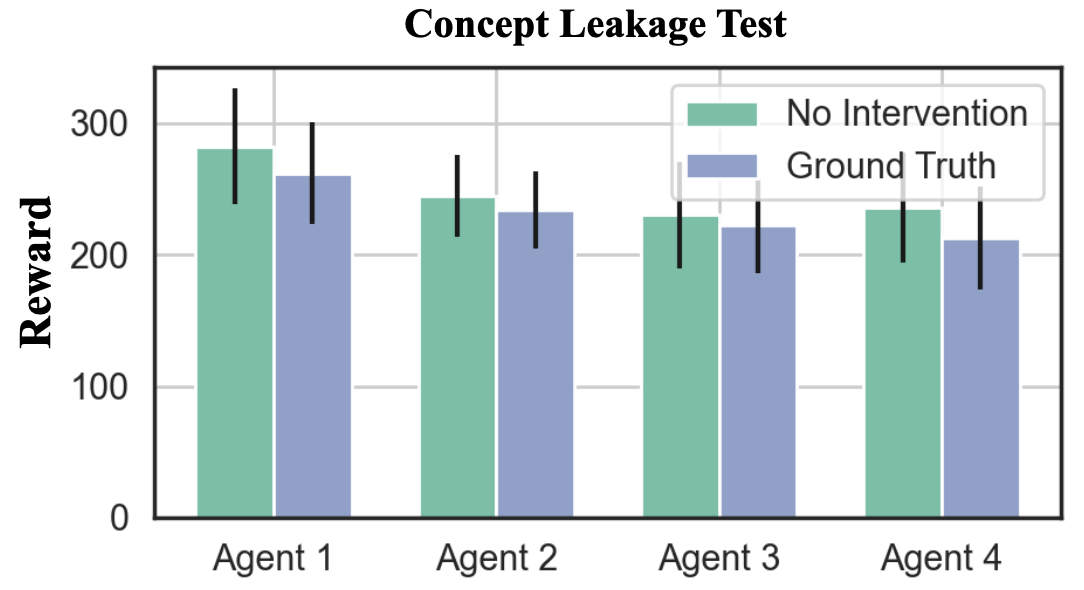}
    \caption{Concept leakage test. We perform an intervention over each agent's concept estimates with the ground truth concept labels from the environment. There is a very slight decrease in performance during intervention, which indicates that some concept leakage does exist. However, the small amount of performance degradation indicates that agents are faithfully learning to estimate concepts.}
    \label{fig:interpretability_results_ground_truth}
\end{figure}
To examine the extent to which concept leakage occurs as a result of the joint objective outlined in \cref{eqn:interpretability_concept_loss},
we perform an evaluation of CBPs while intervening over each agent's concept estimates with ground truth concept values (for all concepts). Specifically, at each time-step, we replace each agent's concept estimates with the ground truth concept values extracted from the environment. Intuitively, if concept leakage is a rampant issue in our proposed architecture and agents are learning to "hack" their concept estimates to encode additional side channel information, we should see performance degrade significantly as a result of this intervention. We perform this analysis over $100$ test-time trajectories (across $5$ random seeds each) using the trained CBP policies from the Clean Up analysis in \cref{sec:interpretability_social_dynamics}.

The results of this analysis are shown in \cref{fig:interpretability_results_ground_truth}, which compares the average reward under the aforementioned ground truth intervention analysis to the average reward obtained by agents with no interventions (using concept estimates as usual). Interestingly, there is a slight decrease in performance that occurs as a result of the ground truth intervention, indicating that some concept leakage may be present in the agents' CBPs. However, the small magnitude of performance degradation indicates that agents are reliably encoding the true concept values in their bottleneck predictions.

\subsection{Bottleneck Performance}
\begin{figure}[]
    \centering
    \includegraphics[width=0.99\columnwidth]{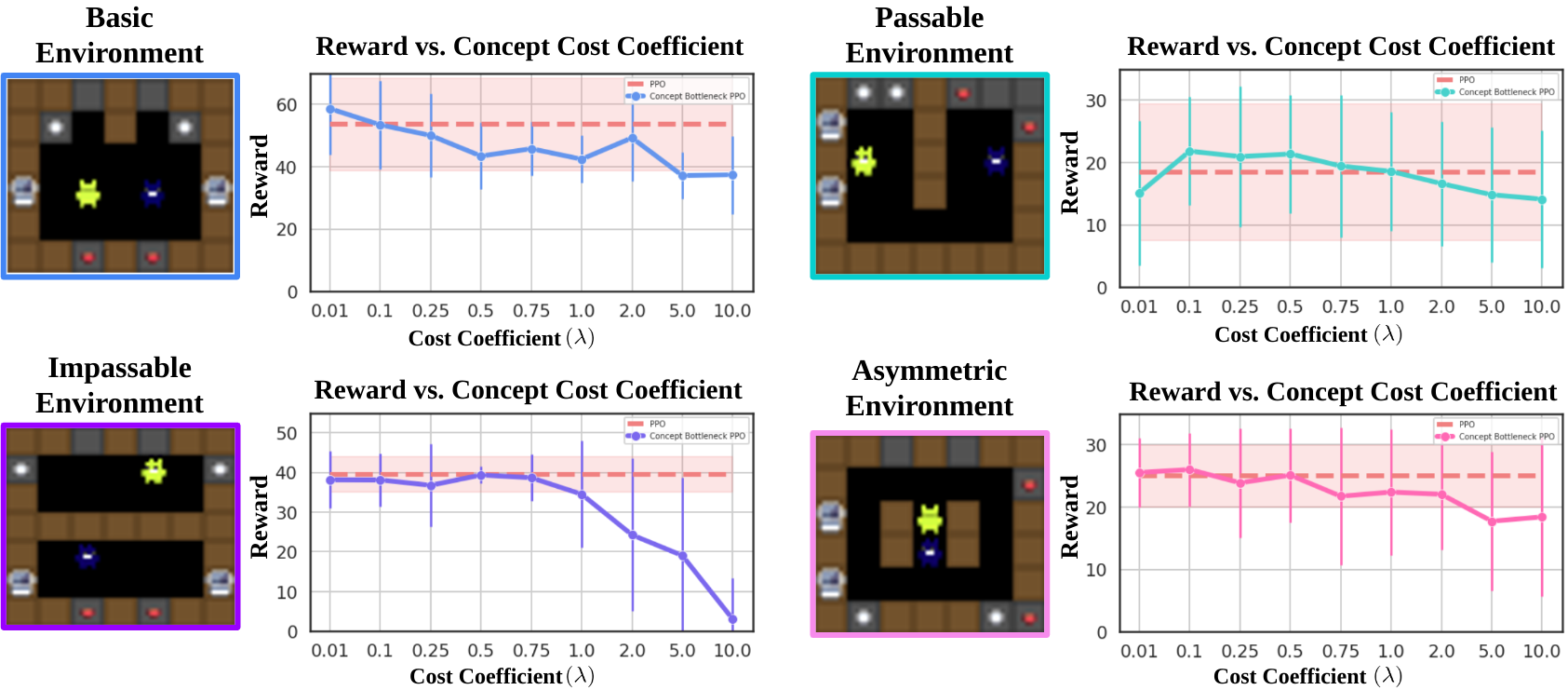}
    \caption{Asymptotic reward of ConceptPPO vs. PPO. For $\lambda {\leq} 0.5$, ConceptPPO matches the performance of non-concept-based PPO.}
    \label{fig:interpretability_reward_no_stack}
\end{figure}
To evaluate our method more generally, we compare the asymptotic performance of Concept PPO and PPO. We measure performance as the average cumulative reward obtained over 100 test-time trajectories (and five random seeds each). \Cref{fig:interpretability_reward_no_stack} provides an overview of these results. We find evidence that ConceptPPO can match the performance of PPO across each of our environments for small values of the concept cost coefficient ($\lambda \leq 0.5$).  This is an important result from the perspective of interpretability. It demonstrates that, if $\lambda$ is tuned appropriately, it is possible to train intrinsically-interpretable policy networks---where, notably, decisions are expressed in human-understandable concepts---without sacrificing in task performance. Importantly, it also demonstrates the sufficiency of the concept set. 

\Cref{fig:interpretability_reward_no_stack} also shows that over-valuing concept prediction loss causes performance to degrade. Performance falls for $\lambda > 0.75$ and, in all but the basic environment, collapses to zero for larger values ($\lambda \geq 5.0)$. This breakdown occurs because, for high $\lambda$, the total gradient from \cref{eqn:interpretability_concept_loss} ($\nabla L_{RL} + \lambda \nabla L_C$) is dominated by the gradient from the concept-based loss $L_C$. In this case, gradient descent is not able to move the policy network's parameters in the direction of $\nabla L_{RL}$, preventing policy optimization.

\subsection{State Visitation Analysis}
\label{apdx:interpretability_state_visitation}
\begin{figure*}
    \centering
    \begin{subfigure}[b]{0.45\textwidth}
        \centering
        \includegraphics[width=\textwidth]{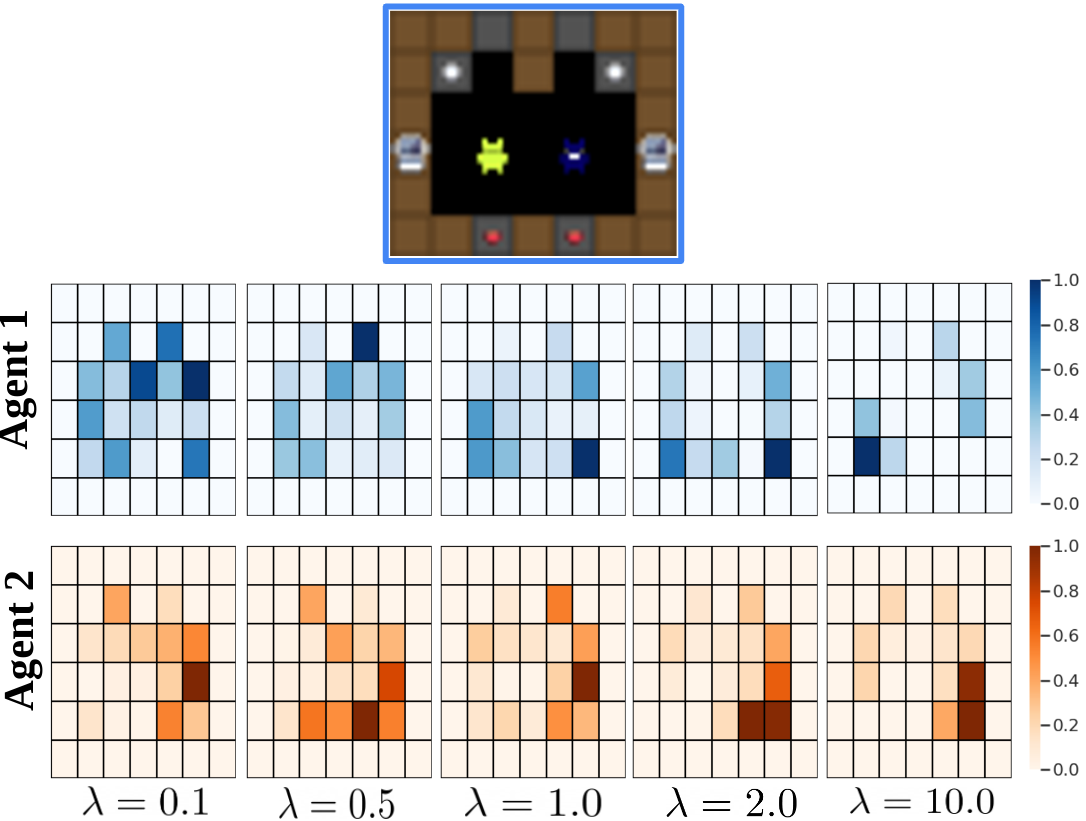}
        \caption{Basic Environment.}
        \label{fig:interpretability_state_visitation_basic}
    \end{subfigure}
    \hfill
    \begin{subfigure}[b]{0.45\textwidth}
        \centering
        \includegraphics[width=\textwidth]{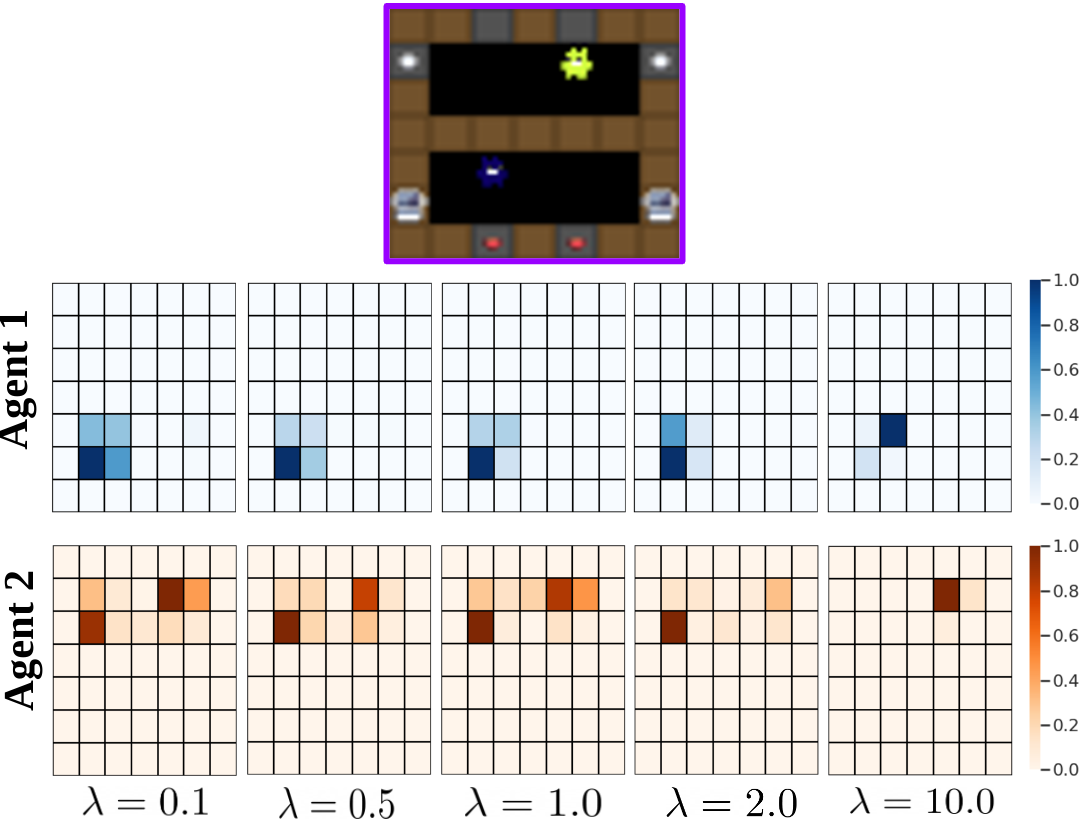}
        \caption{Impassable Environment.}
        \label{fig:interpretability_state_visitation_impassable}
    \end{subfigure}
    \caption{State visitation results. As concept cost coefficient $\lambda$ increases, behavior collapses.}
    \label{fig:interpretability_results_state_visitation}
\end{figure*}
To illustrate the behavioral changes induced by $\lambda$, we measure the distribution of states visited by each agent across a subset of the $\lambda$ values used during training. The results for each agent are plotted as a heatmap in \cref{fig:interpretability_results_state_visitation}. As $\lambda$ increases, multi-agent behavior begins to break down, as the gradient signal from environmental reward gets over-shadowed by that of the concept prediction objective. In the most extreme cases (e.g., $\lambda=10.0$ in the impassable environment) agents fail to make any progress in completing the task.

\subsection{Role Assignment Results (Cont'd)}
\paragraph{Full Temporal Analysis}
Here we supplement our analysis of role assignment in capture the flag from \cref{sec:interpretability_ctf} by presenting the complete set of reward and intervention curves over time ( \cref{fig:interpretability_ctf_int_results}). As before, we hone in on interventions over agent and flag-related concept estimates from Agent 3 and Agent 4 (from the red team). Baseline performance with no intervention is shown as the red-dashed line. As shown in \cref{fig:interpretability_ctf_int_results}, test-time reward is zero both with and without intervention in the early stages of training, as agents are first interacting with the environment. A first swing of reward occurs around checkpoint $45$, where it appears that the blue team has learned to capture the red team's flag (reward for the red agents is negative). This is followed by a quick counter-swing in which the red team begins collecting positive reward. We will hone in on the intervention analysis at checkpoint 55, as it is the first time we see the red team learn productive behavior.
\begin{figure}
    \centering
    \includegraphics[width=0.99\textwidth]{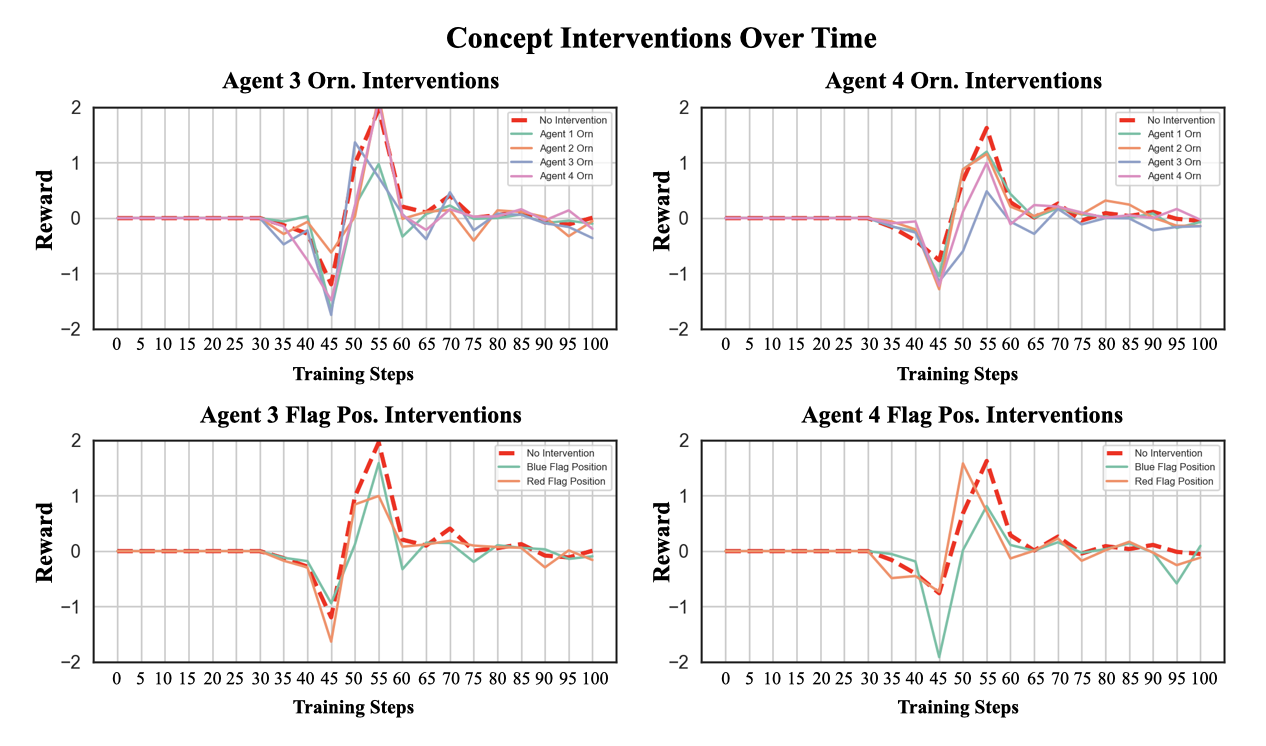}
    \caption{Concept interventions are performed intermittently over time for each agent. Results for the agents on the red team (Agent 3, Agent 4) are shown here. In the top row, we show results of an intervention over agent-related concepts (orientation), and in the bottom row we show results of an intervention over flag-related concepts (flag position). We find an interesting pattern in the intervention results at checkpoint 55. First, we find that Agent 4 is negatively impacted by all of the orientation interventions, whereas Agent 3 is only negatively impacted when intervening over the orientation of Agent 1 and Agent 3 (itself). Next, we find that Agent 4 is similarly negatively impacted by both flag interventions, whereas Agent 3 is only negatively impacted by an intervention over the red flag's position. The intervention reveals, therefore, that Agent 4 is likely an attacking agent and Agent 3 is likely a defending agent.}
    \label{fig:interpretability_ctf_int_results_temporal}
\end{figure}
At checkpoint 55, we see an interesting pattern in the intervention results for Agent 3 vs. Agent 4. Te team receives positive reward, so we know that the red team is capturing the blue team's flag, but the intervention analysis reveals more precisely how that is done. First, in the top row of \cref{fig:interpretability_ctf_int_results_temporal}, intervening over Agent 3's orientation concepts shows that the red team is negatively impacted by interventions over Agent 3's orientation estimate for Agent 1 (a blue agent) and Agent 3 (itself). Intervening over Agent 4 results in a slightly different pattern---every intervention over orientation concepts degrades the red teams performance! As we have seen previously, agents tend to model the other agents that they interact with the most, suggesting that Agent 4 is interacting with all of the other agents in the environment, whereas Agent 3 is interacting with one blue agent. We see a similar pattern in the flag-related interventions as well (bottom row of \cref{fig:interpretability_ctf_int_results}. There we find that, at checkpoint 56, Agent 3 is negatively impacted by an intervention over the red flag's position, but not the blue flag's position. Agent 4, however, is negatively impacted by intervention over both flag positions. This further suggests that Agent 3 interacts primarily

Altogether, the clues that we get from this intervention analysis suggest the following: (i) First, Agent 4 interacts with each agent in the environment (spanning both teams) and also interacts with both flags. The intervention analysis reveals, therefore, that Agent 4 has learned an attacking strategy and is achieving positive reward by capturing the blue teams flag; (ii) Next, we see that Agent 3 interacts primarily with its own flag (the red flag) and also interacts with one of the blue agents. The intervention analysis reveals, therefore, that Agent 3 has learned a defensive strategy, and is likely maintaining the red team's positive reward by fending off a blue attacker.

\subsection{Qualitative Results}
\label{apdx:interpretability_ctf_qualitative}
\begin{figure}
    \centering
    \includegraphics[width=0.85\textwidth]{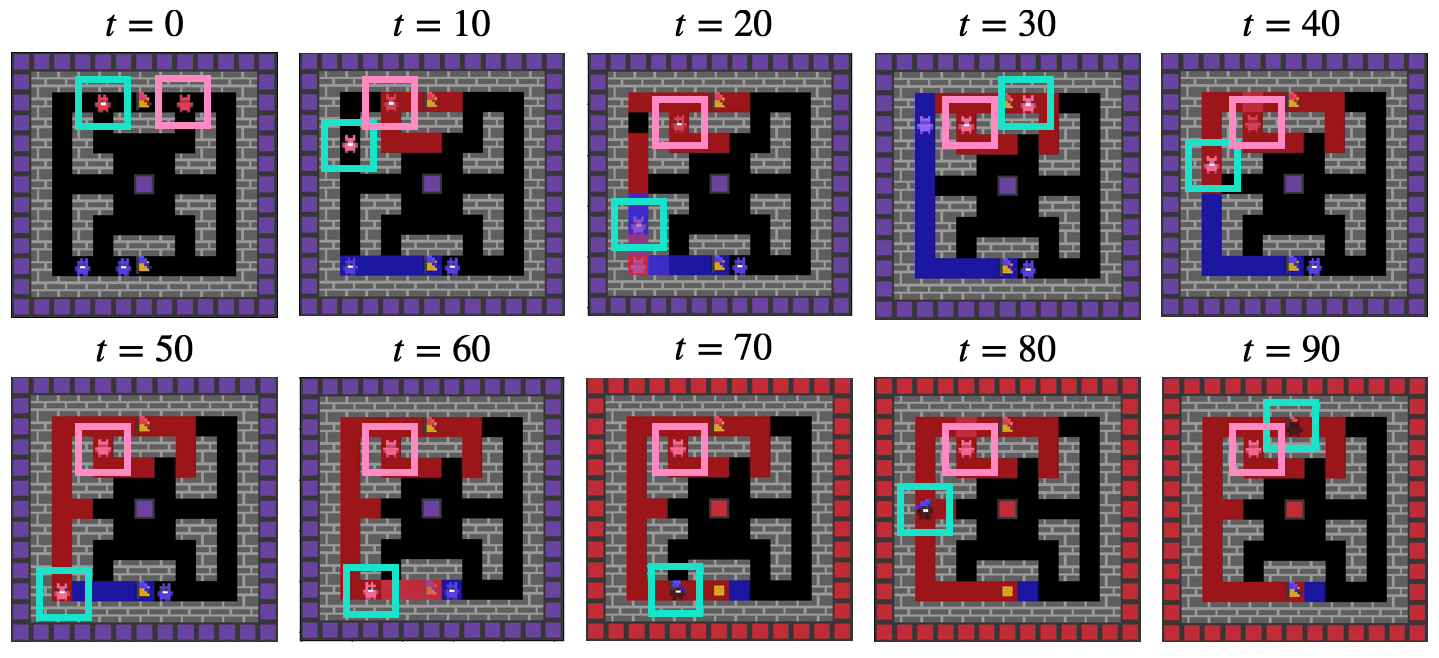}
    \caption{A Capture the Flag trajectory from trained CBP policies. The two red agents have learned emergent roles: attacker (teal square) and defender (pink square).}
    \label{fig:interpretability_ctf_trajectory}
\end{figure}
\Cref{fig:interpretability_ctf_trajectory} provides a snapshot of capture the flag behavior exhibited by the CBP agents' at episode $55$. As predicted by our intervention analysis in \cref{sec:interpretability_ctf}, one agent (R2) has learned an attacking role---traversing to the blue team's base, capturing the flag, and returning it to the red base (and passing both blue agents along the way)---and the other agent (R1) takes on a defending role---camping at the red team's base and defending it from a blue attacker. Altogether, these results confirm that intervention can be used \textit{during training} to augment reward-based analysis and investigate strategic behaviors \textit{as they emerge}.

\subsection{Lasso Neighborhood Selection}
\label{apdx:interpretability_lasso}
Lasso neighborhood selection is a simple method that poses graph learning as a Lasso regression problem~\citep{dong2019learning}. Let $\boldsymbol{X}$ be an observation matrix that is composed of some set of random variables (rows), each of which is described by a set of features (columns). Lasso neighborhood assumes that each variable can be approximated as a sparse linear combination of the observations (i.e. features) of other variables. For a random variable $x_i$, this approximation is computed as:
\begin{equation*}
    \min_{\beta_{i}} ||\boldsymbol{X}_{i} - \boldsymbol{X}_{\textbackslash i} \beta_{i} ||^2_2 + \alpha ||\beta_{i}||_1
\end{equation*}
where $\boldsymbol{X}_i$ represents the features describing the variable $x_i$ (i.e., transpose of the $i$'th row of X), $\boldsymbol{X}_{\backslash i}$ represents the features from rest of the variables (the remaining rows in $\boldsymbol{X}$, $\boldsymbol{\beta}_{i}$ is a vector of coefficients, and $\alpha$ weights the L1-regularization term (enforcing sparsity).

Importantly, coefficients $\boldsymbol{\beta}{i}$ determine which edges, if any, are connected to the node that represents $x_i$. In particular, for some additional variable $x_j$, an edge is established between $x_i$ and $x_j$ if either $\beta_{ij}$ or $\beta_{ji}$ is non-zero (or both). Intuitively, because a graph is a representation of pairwise relationships, Lasso neighborhood selection posits that learning a graph is equivalent to learning a neighborhood for each vertex---i.e., the other vertices to which it is connected---assumes that the observation at a particular vertex may be represented by observations at the neighboring vertices.

In this work, we construct an observation matrix from the outcomes of interventions and use Lasso neighborhood selection to model the relationships (i.e. similarities and dissimilarities) between interventions.

\section{Fairness}
\label{apdx:fairness}

\subsection{Deterministic Fairness Through Equivariance}
\label{apdx:fairness_apdx_deterministic_proofs}
We review the proofs from \cref{sec:fairness_equivariance} in the context of deterministic policies. We assume that agents are homogeneous in their non-sensitive variables. For $n$ agents, each with an individual policy $\mu_{\phi_i}$, let $\boldsymbol{\mu} = \{\mu_{\phi_1}, ..., \mu_{\phi_n} \}$ be the joint policy representing the team. Recall that, in the symmetric case, $\phi_1=\phi_2=...=\phi_n$ and $\omega_1=\omega_2=...=\omega_n$.

\begin{theorem}
    \label{thm:fairness_deterministic_policy_eqv_map}
    If individual policies $\mu_{\phi_i}$ are symmetric, then the joint policy $\boldsymbol{\mu} = \{\mu_{\phi_1}, ..., \mu_{\phi_n} \}$ is an equivariant map.
\end{theorem}
\begin{proof}
    Let $\sigma$ be a permutation operator that, when applied to a vector (such as a state $s_t$ or joint action $\boldsymbol{a}_t$), produces the permuted vector ($\sigma \cdot s_t {=} s^\sigma_t$ or $\sigma \cdot \boldsymbol{a}_t {=} \boldsymbol{a}^\sigma_t$, respectively). Under parameter symmetry (i.e. $\phi_1${=}$\phi_2{=}{\cdots}{=}\phi_n$), we have:
    \begin{equation*}
        \boldsymbol{\mu}(\sigma \cdot s) = \boldsymbol{\mu}(s^{\sigma}) \\
        = \boldsymbol{a}^\sigma \\
        = \sigma \cdot \boldsymbol{a} \\
        = \sigma \cdot \boldsymbol{\mu}(s)
    \end{equation*}
    \noindent where the commutative relationship $\boldsymbol{\mu}(\sigma \cdot s) = \sigma \cdot \boldsymbol{\mu}(s)$ implies that $\boldsymbol{\mu}$ is an equivariant map.
\end{proof}

\begin{theorem}
    \label{thm:fairness_deterministic_states_eqv}
    Let $p^{\boldsymbol{\mu}}(s \rightarrow s', k)$ be the probability of transitioning from $s$ to $s'$ in $k$ steps.
    Given that the joint policy $\boldsymbol{\mu}$ is an equivariant map, it follows that $p^{\boldsymbol{\mu}}(s_1 \rightarrow s_T, T) = p^{\boldsymbol{\mu}}(s_1^\sigma \rightarrow s_T^\sigma, T)$.
\end{theorem}
\begin{proof}
    It follows from agent homogeneity that permuting a state $\sigma \cdot s_t$, which in turn permutes action selection $\sigma \cdot \boldsymbol{a}_t$ (from \cref{thm:fairness_deterministic_policy_eqv_map}), also permutes the environment's transition probabilities:
    \begin{equation*}
        P(s_{t+1} \mid s_t, a_t) = P(s_{t+1}^\sigma \mid s_t^\sigma, \boldsymbol{a}_t^\sigma)
    \end{equation*}
    This is because, from the environment's perspective, a state-action pair is indistinguishable from the state-action pair generated by the same agents after swapping their positions and selected actions. Assuming that the full distribution of start-states $P_\emptyset$ is uniform, we also have $P_\emptyset(s_1) = P_\emptyset(s_1^\sigma)$. Recall the probability of a trajectory from \cref{eqn:background_traj_prob}. Given the equivariant function $\boldsymbol{\mu}$ and the two equalities above, it follows that:
    \begin{equation*}
        P_\emptyset(s_1)\prod_{t=1}^T P(s_{t+1} \mid s_t, \boldsymbol{a}_t) \mid_{\boldsymbol{a}_t = \boldsymbol{\mu}(s_t)}
        = P_\emptyset(s_1^\sigma)\prod_{t=1}^T  P(s_{t+1}^\sigma \mid s_t^\sigma , \boldsymbol{a}_t^\sigma ) \mid_{\boldsymbol{a}_t^\sigma  = \boldsymbol{\mu}(s_t^\sigma )}
    \end{equation*}
    Note that the following properties hold for $p^{\boldsymbol{\mu}}(s \rightarrow s', k)$:

    \begin{itemize}
        \item $p^{\boldsymbol{\mu}}(s \rightarrow s, 0) = 1$
        \item $p^{\boldsymbol{\mu}}(s \rightarrow s', 1) = p(s_{t+1}=s' \mid s_t=s, a_t) \mid a_t=\mu(s_t)$,
        \item $p^{\boldsymbol{\mu}}(s \rightarrow x, T) = \int_{s'} p^\mu(s \rightarrow s', T-1) p^\mu(s' \rightarrow x, 1)$
    \end{itemize}

    We can therefore represent the probability of a trajectory as a single transition from initial state $s_1$ to terminal state $s_T$ by marginalizing out the intermediate states and show that:

    \begin{align*}
        p^\mu(s_1 \rightarrow s_T, T)
        &= \int_{s_{T-1}} p^\mu(s_1 \rightarrow s_{T-1}, T-1) p^\mu(s_{T-1} \rightarrow s_T, 1) \\
        &\begin{multlined} = \int_{s_{T-2}} p^\mu(s_1 \rightarrow s_{T-2}, T-2) \big[p^\mu(s_{T-2} \rightarrow s_{T-1}, 1) p^\mu(s_{T-1} \rightarrow s_T, 1)\big] \end{multlined}\\
        &= \cdots \\
        &= \int_{s_1} \cdots \int_{s_{T-1}} P_\emptyset(s_1)\prod_{t=1}^T  P(s_{t+1} \mid s_t, a_t)\mid_{a_t = \mu(s_t)} \\
        &= \int_{s_1^\sigma} \cdots \int_{s_{T-1}^\sigma} P_\emptyset(s_1^\sigma)\prod_{t=1}^T  P(s_{t+1}^\sigma \mid s_t^\sigma, a_t^\sigma)\mid_{a_t^\sigma = \mu(s_t^\sigma)} \\
        &= \cdots \\
        &= p^\mu(s_1^\sigma \rightarrow s_T^\sigma, T)
    \end{align*}

    Thus, the probability of reaching terminal state $s_T$ from initial state $s_1$ is equivalent to the probability of reaching $s_T^\sigma$ from $s_1^\sigma$.
\end{proof}

\begin{theorem}
    \label{thm:fairness_deterministic_sym_fair}
    Equivariant deterministic policies are exactly fair.
\end{theorem}
\begin{proof}
    Exact same as \cref{thm:fairness_sym_fair}.
\end{proof}

\subsection{Gradient of Fairness Objective}
\label{apdx:fairness_apdx_derivations_sym}
We present a derivation of the equivariant objective gradient from \cref{eqn:fairness_eqv_grad}:

\begin{align*}
    \frac{\partial J_{\textrm{eqv}}}{\partial \phi_i} &\approx \frac{\partial}{\partial \phi_i} \frac{1}{M} \sum_i^M \frac{1}{N-1} \sum_{j \neq i}^{N-1}1 - \cos(\mu_{\phi_i}(s) - \mu_{\phi_j}(s))
    \\[4pt]
    &\begin{multlined} =\frac{1}{M} \sum_i^M \frac{1}{N-1} \sum_{j \neq i}^{N-1} - (-\sin(\mu_{\phi_i}(s) - \mu_{\phi_j}(s))) \frac{\partial}{\partial \phi_i} (\mu_{\phi_i}(s) - \mu_{\phi_j}(s)) \end{multlined}
    \\[4pt]
    &\begin{multlined} =\frac{1}{M} \sum_i^M \frac{1}{N-1} \sum_{j \neq i}^{N-1} \sin(\mu_{\phi_i}(s) - \mu_{\phi_j}(s)) \bigg(\frac{\partial}{\partial \phi_i} \mu_{\phi_i}(s) - \frac{\partial}{\partial \phi_i} \mu_{\phi_j}(s) \bigg)
    \end{multlined}
    \\[4pt]
    &\begin{multlined} = \frac{1}{M} \sum_i^M \frac{1}{N-1} \sum_{j \neq i}^{N-1} \sin(\mu_{\phi_i}(s) - \mu_{\phi_j}(s)) \nabla_{\phi_i} \mu_{\phi_i}(s)
    \end{multlined}
\end{align*}

\subsection{Why Not Mean Squared Error?}
\label{apdx:fairness_apdx_mse}
We could just have as easily defined $J_{\textrm{eqv}}$ as the mean-squared error (MSE) between $a_i$ and $a_j$, with the objective:

\begin{equation*}
    J_{\textrm{sym}}(\phi_1, ..., \phi_i, ..., \phi_n) = \mathbb{E}_s[ \mathbb{E}_{j \neq i}[(\mu_{\phi_i}(s) - \mu_{\phi_j}(s))^2]]
\end{equation*}
\noindent and corresponding gradient:
\begin{equation*}
    \nabla_{\phi_i} J_{\textrm{sym}}(\phi) \approx \frac{1}{M} \sum_i \frac{1}{N-1} \sum_{j \neq i} (\mu_{\phi_i}(s) - \mu_{\phi_j}(s)) \nabla_{\phi_i}\mu_{\phi_i}(s) \mid_{s=s_i}
\end{equation*}

\noindent However, recall that each $a$ is a heading angle sampled from $A$, which is a circular variable in the range $[0, 2\pi]$ or $[-\pi, \pi]$. Any difference of actions, as in MSE loss, must account for this by applying a modulo operation on top of the difference. Our objective represents angular distance as a single Fourier mode, which yields a convex optimization surface without discontinuity issues on the range of possible headings. 

\subsection{Experimental Details}
\label{apdx:fairness_apdx_experimental_details}

\paragraph{Computing Team Fairness}
We illustrate how team fairness scores $I(R;Z)$ are computed. Recall that each agent's sensitive attribute $z_i$ is an identity variable that uniquely identifies it from the group (see \cref{sec:fairness_results}). Also note that the vectorial reward vector $\boldsymbol{r}$ serves as a proxy for agent identity---e.g. $\boldsymbol{r} {=} [0,0,1]$ indicates that pursuer $p_3$ captured the evader. We can therefore compute team fairness as the mutual information obtained by the difference of entropies $h(R) - h(R_\textrm{uniform})$, where $R$ is a distribution over team rewards $\boldsymbol{r}$ and $R_\textrm{uniform}$ is a uniform reward distribution (i.e. captures are spread evenly across all pursuers). Computed this way, $I(R;Z)$ represents the extent to which knowing the outcome of the pursuit-evasion task reveals the agents identity, and vice versa. It therefore measures how fairly outcomes are distributed across cooperative teammates.

\paragraph{Greedy Control Baseline}
Let $q_i$ be the position of an agent $i$ and $U_{\textrm{att}}(q_i)$ be a quadratic function of distance between $q_i$ and a target $q_{\textrm{goal}}$:
\begin{equation}
    \label{eqn:fairness_attractive_potential}
    U_{\textrm{att}}(q_i, q_{\textrm{goal}}) = \frac{1}{2}k_{\textrm{att}} \, d(q_i, q_{\textrm{goal}})^2
\end{equation}
where $k_{\textrm{att}}$ is an attraction coefficient and $d(,)$ is a measure of distance. Taking the negative gradient $F(q_i) = -\nabla U(q_i)$ yields the following control law for agent $i$'s motion:
\begin{equation}
    F_{\textrm{att}} = -\nabla U_{\textrm{att}}(q_i, q_{\textrm{goal}}) = -k_{\textrm{att}}(q_i - q_{\textrm{goal}})
\end{equation}
\noindent In this work, the environment's action-space is defined in terms of agent headings, so only the \textit{direction} of this force impacts the agents. Setting $q_{\textrm{goal}}$ to be the position of the evader and following \cref{eqn:fairness_attractive_potential} at each time-step results in a greedy policy that runs directly towards the evader.

\paragraph{Fairness for Individual Reward Policies}
In \cref{fig:fairness_ind_vs_sym_fairness}, we find that policies learned with individual reward exhibit a temporary spike in $I(R;Z)$ from $\lvert \vec{v}_p\rvert=0.9$ to $\lvert \vec{v}_p\rvert=0.6$, indicating that their reward distributions $R$ are less fair. This occurs because, as $\lvert \vec{v}_p\rvert$ decreases, policies learned with individual reward fall into a degenerate state where only one pursuer experiences positive reward from capturing the evader. The other pursuers, never experiencing positive reward, slowly diverge from their previous greedy strategies and become incapable of capturing the evader. Their best strategy becomes hoping that the ``capturer" pursuer captures the evader, which results in a highly skewed capture distribution. Eventually this strategy fails when $\lvert \vec{v}_p\rvert$ drops low enough that all pursuers fail to capture the evader, which causes $I(R;Z)$ to fall again.

\paragraph{Policy Learning Hyperparameters}
All actors $\mu_\phi$ are trained with two hidden layers of size 128. Critics $Q_\omega$ are trained with three hidden layers of size 128. We use a learning rate of $1\textrm{e}^{-4}$ and $1\textrm{e}^{-3}$ for the actor and critic, respectively, and a gradient clip of 0.5 on both. Target networks are updated with Polyak averaging with $\tau=0.001$. We maintain a buffer $\mathcal{D}$ of length $500000$ and sample batches of size $512$. Finally, we use a discount factor $\gamma=0.99$. All values are the results of standard hyperparamter sweeps.

\subsection{Training Details}
As described in \cref{sec:fairness_results}, each pursuit-evasion experiment includes $n=3$ pursuer agents and a single evader. The pursuers each train their own policy for a total of 125,000 episodes, during which their velocity is decreased from $\lvert \vec{v}_p\rvert = 1.2$ to $\lvert \vec{v}_p\rvert = 0.4$. The evader speed is fixed at $\lvert \vec{v}_e\rvert = 1.0$. After training, we test the resulting policies at discrete velocity intervals (e.g. $\lvert \vec{v}_p\rvert = 1.2$, $\lvert \vec{v}_p\rvert = 1.1$, etc), where a decrease in pursuer velocity represents a greater ``difficulty level" for the pursuers. Test-time performance, such as is shown in \cref{fig:fairness_ind_vs_mut}, \cref{fig:fairness_ind_vs_sym_fairness}, \cref{fig:fairness_ind_vs_sym_utility}, and \cref{fig:fairness_fairness_vs_utility} is averaged across 100 independent trajectories from five different random seeds each. All experiments leveraged an Nvidia GeForce GTX 1070 GPU with 8GB of memory.

\section{Robustness}
\subsection{Training Details}
\label{apdx:robustness_training}
Here we provide additional training details, including the search parameters, training hyper parameters, and other learning details that are used when training AlphaZero and VISA-VIS.

\subsubsection{Training parameters}
Both algorithm's were trained for $500,000$ games, $1.75$M games, $1.75$M games, and $7.5$M games for Tic-Tac-Toe, 4x4 Tic-Tac-Toe, and Connect Four, respectively. The joint policy-value network was represented by a ResNet backbone with separate heads outputting policy and value predictions, as in \cref{eqn:background_az_network}. Each fully-connected layer was initialized with a width of $128$ neurons and the depth of the network backbone was scaled to account for game complexity----$d=2$ ResNet modules for Tic-Tac-Toe, $d=4$ for 4x4 Tic-Tac-Toe and Connect Four. Batch size and learning rate also varied for each game: (i) Tic-Tac-Toe used a learning rate of $1\mathrm{e}{-}3$ and batch size of $64$; (ii) 4x4 Tic-Tac-Toe used a learning rate of $1\mathrm{e}{-}4$ and batch size of $128$; and (iii) Connect Four used a learning rate of $1\mathrm{e}{-}4$ and batch size of $256$. A coefficient of $\lambda=1\mathrm{e}{-}4$ was used to weight L2-regularization in all games, as described in \cref{eqn:background_az_loss}.

\begin{table}
  \caption{Hyperparameters for AlphaZero and VISA-VIS.}
  \label{tab:hyperparams}
  \centering
  \begin{tabular}{llll}
    \toprule
    \multicolumn{4}{c}{Training Parameters and  Hyperparameters}                   \\
    \cmidrule(r){1-4}
    Parameter     & Value & Value & Value     \\
    Name       & (TTT) & (4x4 TTT) & (Connect Four)     \\
    \midrule
    Num. Games & $5\mathrm{e}5$ & $1.75\mathrm{e}6$ & $7.5\mathrm{e}6$   \\
    Batch Size     & $64$ & $128$ & $256$      \\
    Learning Rate     & $1\mathrm{e}{-}3$ & $1\mathrm{e}{-}4$ & $1\mathrm{e}{-}4$      \\
    Network Depth     & $2$ & $4$ & $4$      \\
    Network Width     & $128$ & $128$ & $128$      \\
    $\lambda$     & $1\mathrm{e}{-}4$ & $1\mathrm{e}{-}4$ & $1\mathrm{e}{-}4$      \\
    \bottomrule
  \end{tabular}
\end{table}

\subsubsection{Search parameters}
Recall from \cref{sec:background_alphazero} that there are a number of parameters that influence the search process of AlphaZero (and therefore VISA-VIS). During training, AlphaZero used $25$ simulations of MCTS search per time-step for Tic-Tac-Toe and 4x4 Tic-Tac-Toe, and $50$ simulations per time-step for Connect Four. For each simulation, the within-tree exploration parameter $c=2.0$ was used to guide exploration in the upper confidence bound specified by \cref{eqn:background_az_puct} in all games. The temperature parameter $\tau$, which guides exploration in the action selection step specified by \cref{eqn:background_mcts_policy}, was set to an initial value of $\tau=1.0$ for all games. During each training rollout, $\tau$ was dropped after a fixed number of steps---$5$, $9$, and $21$ for Tic-Tac-Toe, 4x4 Tic-Tac-Toe, and Connect Four, respectively---as in \cite{silver2017alphago}.
\begin{table}
  \caption{Search parameters for AlphaZero and VISA-VIS.}
  \label{tab:hyperparams}
  \centering
  \begin{tabular}{llll}
    \toprule
    \multicolumn{4}{c}{Search Parameters}                   \\
    \cmidrule(r){1-4}
    Parameter     & Value & Value & Value     \\
    Name       & (TTT) & (4x4 TTT) & (Connect Four)     \\
    \midrule
    MCTS Simulations     & $25$ & $25$ & $50$      \\
    $c$     & $2.0$ & $2.0$ & $2.0$      \\
    Initial $\tau$     & $1.0$ & $1.0$ & $1.0$      \\
    $\tau$ drop (steps)     & $5$ & $9$ & $21$      \\
    \bottomrule
  \end{tabular}
\end{table}

\subsubsection{Implementation}
Our implementation is based on the open-source AlphaZero provided by DeepMind's OpenSpiel project~\citep{LanctotEtAl2019OpenSpiel}. VISA-VIS extends the implementation available in OpenSpiel. The most significant implementation detail is the choice of each game's board representation. For each game, the board state was represented as a stack of $3$ $H \times W$ frames---where $H$ and $W$ are the height and width of the game board, respectively---representing the pieces of each player (noughts and crosses for both Tic-Tac-Toe variants, red and yellow pieces for Connect Four) and an extra plane to encode turns (i.e. which player's turn it is to move). Each game also encoded the game rules into the action space, which are used to specify legal moves at each turn---illegal moves are masked out and assigned a probability of zero by the value network.

\subsubsection{Resources}
All experiments were run on a cluster with 2x18 core Intel Xeon Skylake 6154 CPUs an Nvidia V100 GPU with 16GB of memory.

\subsubsection{Discussion}
\label{apdx:discussion}

\paragraph{Limitations}
The primary limitation of our method, which is true of many hybrid search-leaning methods (including AlphaZero), is that it requires a transition model that includes a specification of legal moves from each state. Though this assumption holds in game-playing environments, such as the board games studied by prior methods and in this work, many real-world environments do not provide a transition model \textit{a priori} (such a model must be learned).

\paragraph{Broader Impacts}
This work examines the robustness and reliability of AlphaZero. As large-scale AI systems like AI proliferate in real-world use cases, it is important to understand their limitations and potential failure modes. Algorithms that leverage self-play reinforcement learning are particularly susceptible to potential harms (e.g. safety, reliability, interpretability), because they do not observe any human-provided data during training. We posit that, in order to ensure that such systems are generally helpful, and to mitigate unforeseen or unintended harms during deployment, we must ensure their robustness to potential attacks (e.g. adversarial inputs) and overall reliability. The goal of this work was to take a small step in that direction. This work does not make use of sensitive data or include as part of its results a sensitive task or study.

\paragraph{Future Work}
There are a number of interesting avenues for future work in the area of concept bottlenecks for MARL. First, taking inspiration from MuZero, we can explore model-free extensions to our architecture, where the environment's transition model is learned or estimated from experience. We can also consider more complex environments. An important future work is to extend our simple thresholding coefficient $\epsilon$ that trades off between policy and value-based action selection.

\bibliography{thesis}

\end{document}